\documentclass[10pt]{article}
\usepackage{semiautomata_curriculum}
\usepackage{physics}
\allowdisplaybreaks
\usepackage{algorithm}
\usepackage[noend]{algpseudocode}
\usepackage{float}

\usepackage{etoc}

\newcommand{\dkl}{D_{\mathrm{KL}}}
\newcommand{\dtv}{\mathrm{TV}}
\newcommand{\TV}[2]{\mathrm{TV}\left(#1,#2\right)}
\newcommand{\KL}[2]{\dkl \big(#1\,\big\|\,#2\big)}

\renewcommand{\Pr}{\mathrm{Pr}}
\usepackage{subcaption}

\newcommand{\Cmax}{C_{\texttt{cov}}}

\newcommand{\piNFA}{\pi^{\texttt{NFA}}}
\newcommand{\piDFA}{\pi^{\texttt{DFA}}}

\newcommand{\abort}{\texttt{ABORT}}
\renewcommand{\root}{\texttt{root}}
\newcommand{\nburnin}{n_{\texttt{burn-in}}}
\newcommand{\Dnd}{D_{\texttt{node}}}
\newcommand{\comp}{\texttt{comp}}

\newcommand{\err}{\texttt{err}}

\renewcommand{\rank}{\mathrm{rank}}

\newcommand{\ch}{\texttt{Ch}}
\newcommand{\pt}{\texttt{Par}}
\newcommand{\sib}{\texttt{Sib}}
\newcommand{\wrho}{\widehat{\rho}}

\usepackage{cases}

\newcommand{\nuhat}{\widehat{\nu}}

\newcommand{\thetahat}{\widehat{\theta}}
\newcommand{\Vlearn}{\texttt{RLFineTune}}

\newcommand{\fhat}{\widehat{\pi}}
\newcommand{\fstar}{\pi^\star}

\newcommand{\ftilde}{\widetilde{\pi}}

\newcommand{\Xbar}{X}
\newcommand{\rhobar}{\overline{\rho}}
\newcommand{\mubar}{\overline{\mu}}
\newcommand{\sigmabar}{\overline{\sigma}}

\newcommand{\nquery}{n_\texttt{query}}

\newcommand{\ncomp}{n_\texttt{comp}}
\newcommand{\nsample}{n_\texttt{sample}}

\newcommand{\nweak}{n_\texttt{weak}}

\newcommand{\lambdahat}{\widehat{\lambda}}

\newcommand{\nd}{\texttt{node}}

\newcommand{\etoe}{\texttt{iCoT}}

\newcommand{\VCdim}{\mathrm{VC}}
\newcommand{\Ndim}{\mathrm{Ndim}}
\newcommand{\Ldim}{\mathrm{Ldim}}

\renewcommand{\calF}{\Pi}

\newcommand{\Alg}{\texttt{Alg}}
\newcommand{\AlgOL}{\texttt{Alg}_{\mathrm{on}}}
\newcommand{\Base}{\texttt{Base}}
\newcommand{\warmup}{\texttt{warmup}}
\newcommand{\TranscriptNTP}{\texttt{NTP}_\texttt{Full-CoT}}    % full-transcript NTP learner; renamed from \TranscriptERM
\newcommand{\Dstep}{D_{\mathrm{step}}}
\newcommand{\calOtilde}{\widetilde{\mathcal{O}}}

\newcommand{\rhoB}{\rho_B}

\newcommand{\bbN}{\mathbb{N}}
\newcommand{\bbF}{\mathbb{F}}

\newcommand{\piref}{\pi^{\mathrm{ref}}}
\newcommand{\pihat}{\widehat{\pi}}

\newcommand{\Cseq}{C_{\texttt{seq}}}
\newcommand{\Cblock}{C_{\texttt{seq}}^{(B)}}
\newcommand{\Cout}{C_{\texttt{ans}}^{(B)}}

\newcommand{\NTP}{\texttt{NTP}}

\newcommand{\Unif}{\operatorname{Unif}}

\newcommand{\Dprompt}[1]{D_{\mathrm{input}}^{#1}}

\newcommand{\Dout}[1]{D_{\mathrm{out}}^{#1}}

\newcommand{\maj}{\texttt{Plu}}

\newcommand{\mutilde}{\widetilde{\mu}}
\newcommand{\nutilde}{\widetilde{\nu}}
\newcommand{\subsample}
{\texttt{InvSampling}}

\renewcommand{\split}{\texttt{Split}}
\newcommand{\unif}{\operatorname{Unif}}

\newcommand{\rjmax}{\mathfrak{r}_j}

\newcommand{\ptilde}{\widetilde{p}}
\newcommand{\construct}{\texttt{ConstructEvent}}

\usetikzlibrary{arrows.meta,positioning,calc}
\definecolor{denseblue}{RGB}{55,126,184}
\definecolor{spikeorange}{RGB}{230,159,0}
\definecolor{lightblue}{RGB}{232,241,250}
\definecolor{lightorange}{RGB}{252,239,214}
\definecolor{softgray}{RGB}{95,95,95}

\newcommand{\Tmin}{T_{\mathrm{min}}}
\newcommand{\Tmax}{T_{\mathrm{max}}}

\newcommand{\WLdepth}[1]{\texttt{AutoLearn}_{#1}}
\newcommand{\WLdepthOL}{\texttt{AutoLearn}_{\mathrm{on}}}
\newcommand{\RL}{\texttt{RL}}

\usepackage{xcolor}

\newcommand{\commentsymbol}{$\blacktriangleright$}
\makeatletter
\newcommand{\LineComment}[1]{\hfill\commentsymbol{} #1}
\makeatother
\makeatletter
\newcommand{\NoComment}[1]{\hfill #1}
\makeatother

\newcommand{\RE}{\texttt{RE}}
\usepackage[most]{tcolorbox}
\tcbuselibrary{skins,breakable}

\Crefname{observation}{Observation}{Observations}
\newtcolorbox[
  auto counter,
  number within=section,
  number freestyle={\noexpand\thesection.\noexpand\arabic{\tcbcounter}~\noexpand\mytitle},
  crefname={Meta-algorithm}{Meta-algorithms}
]{metaalgorithm}[2][]{
  breakable,
  enhanced,
  colback=white,
  colframe=black,
  coltitle=black,
  fonttitle=\bfseries,
  code={\def\mytitle{#2}},
  title=Meta-algorithm \thetcbcounter,
  boxrule=0.8pt,
  arc=2pt,
  left=8pt,right=8pt,
  top=6pt,bottom=6pt,
  attach boxed title to top center={yshift=-2mm},
  boxed title style={
    colback=white,
    colframe=black,
    boxrule=0.8pt,
    sharp corners
  },
  #1
}

\newtcolorbox[
  auto counter,
  number within=section,
  number freestyle={\noexpand\thesection.\noexpand\arabic{\tcbcounter}~\noexpand\mytitle},
  crefname={Procedure}{Procedures}
]{procedure}[2][]{
  breakable,
  enhanced,
  colback=white,
  colframe=black,
  coltitle=black,
  fonttitle=\bfseries,
  code={\def\mytitle{#2}},
  title=Procedure \thetcbcounter,
  boxrule=0.8pt,
  arc=2pt,
  left=8pt,right=8pt,
  top=6pt,bottom=6pt,
  attach boxed title to top left={yshift=-3mm},
  boxed title style={
    colback=white,
    colframe=black,
    boxrule=0.8pt,
    sharp corners
  },
  #1
}

\crefname{claim}{claim}{Claims}
\newcommand{\IndHyp}{\mathrm{Hypothesis}}

\newtcolorbox[
  auto counter,
  number within=section,
  number freestyle={\noexpand\thesection.\noexpand\arabic{\tcbcounter}},
  crefname={IndHyp}{IndHyps}
]{induction}[2][]{
  breakable,
  enhanced,
  colback=white,
  colframe=black,
  coltitle=black,
  fonttitle=\bfseries,
  code={\def\mytitle{#2}},
  title=$\textbf{Hypothesis}_{\bm{#2}} (\delta)$,
  boxrule=0.8pt,
  arc=2pt,
  left=8pt,right=8pt,
  top=6pt,bottom=6pt,
  attach boxed title to top left={yshift=-2mm},
  boxed title style={
    colback=white,
    colframe=black,
    boxrule=0.8pt,
    sharp corners
  },
  #1
}

\makeatletter
\crefname{ALG@line}{Line}{Lines}
\Crefname{ALG@line}{Line}{Lines}
\newcommand{\LineLabel}[1]{%
  \begingroup
    \edef\@currentcounter{ALG@line}%
    \protected@edef\@currentlabel{\theALG@line}%
    \cref@constructprefix{ALG@line}{\cref@result}%
    \protected@edef\cref@currentlabel{[ALG@line][\theALG@line][]\theALG@line}%
    \label{#1}%
  \endgroup
}
\makeatother

\newcommand{\loose}{\looseness=-1}

\renewenvironment{proof}[1][Proof]
  {\par\noindent{\bfseries\upshape {#1.}\ }}
  {\qed\newline}

\newcommand{\mainalg}{$\WLdepth{}$\xspace}
\newcommand{\mainalgRL}{$\WLdepth{}.\RL$\xspace}
\newcommand{\rlfinetune}{$\Vlearn$\xspace}
\newcommand{\CseqT}{\smash{C_{\texttt{seq}}^{(T)}}}
\newcommand{\veps}{\varepsilon}
\newcommand{\guessandcheck}{$\texttt{GuessAndCheck} (\, \cdot \,\|\, T,B, \delta')$\xspace}

\makeatletter
\@ifundefined{abs}{}{}
\@ifundefined{brk}{}{}
\@ifundefined{crl}{\DeclarePairedDelimiter{\crl}{\{}{\}}}{}
\@ifundefined{prn}{\DeclarePairedDelimiter{\prn}{(}{)}}{}
\@ifundefined{nrm}{\DeclarePairedDelimiter{\nrm}{\|}{\|}}{}
\@ifundefined{tri}{}{}
\@ifundefined{ceil}{}{}
\@ifundefined{floor}{}{}
\makeatother

\usepackage{color-edits}
\addauthor{df}{ForestGreen}
\addauthor{ak}{BurntOrange}
\addauthor{claude}{Violet}

\usepackage{inconsolata}
\author{
Nived Rajaraman\thanks{Microsoft Research. \texttt{nrajaraman@microsoft.com}} \quad
Audrey Huang\thanks{University of Illinois Urbana-Champaign. \texttt{audreyh5@illinois.edu}} \quad
Miroslav Dud\'{i}k\thanks{Microsoft Research. \texttt{mdudik@microsoft.com}} \\[4pt]
Robert Schapire\thanks{Microsoft Research, New York, NY. \texttt{schapire@microsoft.com}} \quad
Dylan Foster\thanks{Microsoft Research. \texttt{dylanfoster@microsoft.com}} \quad
Akshay Krishnamurthy\thanks{Microsoft Research. \texttt{akshaykr@microsoft.com}}
}

\begin{document}

\title{Learning to Reason with Curriculum II:\\[2pt] Compositional Generalization}

\date{\vspace{1em}\today}
\maketitle

\begin{abstract}
\noindent Compositional generalization---the ability to solve complex problems by combining solutions to simpler sub-problems---is a fundamental capability of both natural and artificial intelligence, and a key mechanism underlying chain-of-thought reasoning. However, the theoretical underpinnings of compositional generalization remain poorly understood: when and why does decomposing a problem into parts yield more efficient learning than solving it directly? We study this question through the canonical problem of learning to simulate semiautomata (predicting the outcome of $T$ steps of sequential computation), a model that captures state tracking, regular language recognition, and modular arithmetic. We show that an \emph{autocurriculum-based} approach building on Part I of this series---recursively decomposing longer sequences into shorter sub-problems, learning to solve them, and composing the solutions---achieves dramatically better statistical complexity than direct methods. (i) For a setting inspired by supervised fine-tuning (SFT) where the learner receives interactive feedback on intermediate states of the computation, curriculum facilitates learning from only \smash{$2^{\calOtilde(\sqrt{\log T})}$} tokens of supervision; i.e., \emph{subpolynomial} in the sequence length $T$, overcoming the $\Omega(T)$ token barrier required by direct simulation. (ii)~For a setting inspired by reinforcement learning with verifiable rewards (RLVR), where the learner improves a pre-trained reference model using an outcome verifier, we show that curriculum reduces the requirement on the reference model from coverage at the full sequence length $T$ to coverage at a shorter block length $B \ll T$, an exponentially weaker condition.\loose

\end{abstract}

\begingroup
\setcounter{tocdepth}{2}
\tableofcontents
\endgroup

\section{Introduction} \label{sec:intro}

\noindent Chain-of-thought reasoning, where models expend additional computation by producing intermediate reasoning tokens prior to a final answer, has driven significant advances in the capabilities of language models. A key mechanism underlying these advances is \emph{compositional generalization}---the ability to combine solutions to simpler sub-problems in order to solve harder ones---with a growing body of evidence that it is central to how reasoning models scale to difficult tasks~\citep{lake2018generalization,saxton2019analysing,press2023measuring,dziri2023faith,zhao2024can}.

\medskip
\noindent Many works emphasize the importance of \emph{curriculum-based training}~\citep{bengio2009curriculum} for enabling compositional generalization in reasoning models~\citep{zeng2025rlve,motwani2025h1,team2025kimi}. Here, problems of increasing difficulty teach the learner increasingly complex compositions of low-level skills.
Recent work has further demonstrated that curricula can be induced by the learner itself, avoiding the need for hand-engineered notions of problem difficulty~\citep{chen2025self,shi2025efficient,yu2025dapo,zhao2025absolute}.

\medskip
\noindent While these \emph{autocurriculum} approaches are intuitively appealing, it remains poorly understood as to \emph{when}, \emph{why}, and \emph{to what extent} autocurricula can reduce the cost of learning, and what principles should guide their design. To this end, in Part~I of this series~\citep{rajaraman2026learning}, we showed that a form of autocurriculum inspired by boosting---specifically, using the model's own performance to \emph{select which problems to focus on} from a large corpus---provably improves learning efficiency for supervised fine-tuning and reinforcement learning. In this work, we ask a more ambitious question:\loose
\begin{center}
How should a learner \emph{design its own problems} to facilitate solving tasks far beyond its base capabilities, and how can compositional structure guide this process?
\end{center}

\begin{figure}[t]
\centering
\resizebox{\textwidth}{!}{\begin{tikzpicture}[x=1cm,y=1cm,>=stealth,font=\small]

% Panel boundaries (3 columns)
\draw[rounded corners,draw=gray!40] (0,0) rectangle (5.4,4.1);
\draw[rounded corners,draw=gray!40] (5.8,0) rectangle (11.2,4.1);
\draw[rounded corners,draw=gray!40] (11.6,0) rectangle (17.0,4.1);

% Panel titles
\node[font=\bfseries\small] at (2.7,3.7) {Semiautomaton Dynamics};
\node[font=\bfseries\small,align=center] at (8.5,3.56) {Semiautomata capture\\[-1pt] modern architectures};
\node[font=\bfseries\small] at (14.3,3.75) {Self-Generated Curriculum};

% ---------------- Panel A ----------------
\node (A) at (2.7,2.25) {}; % Anchor point for Panel A

\node[rounded corners,fill=green!15,align=center,minimum width=4.15cm,minimum height=0.65cm,font=\footnotesize] at (2.7,3)
{$\fstar : S \times \Sigma \to S$ (unknown)};

\draw[thick,->] ($(A)+(-1.85,-0.4)$) -- ($(A)+(2.2,-0.4)$);
\foreach \dx/\lbl in {-1.8/$s_0$,-0.9/$s_1$,0/$s_2$,0.9/$\cdots$,1.8/$s_T$} {
  \draw[fill=white] ($(A)+(\dx,-0.4)$) circle (0.08);
  \node[below] at ($(A)+(\dx,-0.48)$) {\lbl};
  \if \dx == -0.9
    \node[below,font=\footnotesize] at ($(A)+(\dx,-0.78)$) {$\fstar (s_0,w_1)$};
  \else
  \fi
}
\node[font=\footnotesize] at ($(A)+(-1.35,-0.1)$) {$w_1$};
\node[font=\footnotesize] at ($(A)+(-0.45,-0.1)$) {$w_2$};
\node[font=\footnotesize] at ($(A)+(0.45,-0.1)$) {$\cdots$};
\node[font=\footnotesize] at ($(A)+(1.35,-0.1)$) {$w_T$};

\node[align=center,font=\scriptsize] at (2.7,0.55)
{Goal: predict terminal state $\fstar_T(\bx)$\\ for $\bx=(s_0,w_{1:T}) \sim \rho$.};

% ---------------- Panel B: Connection to Modern Models ----------------
\node (B) at (8.5,1.6) {}; % Anchor point for Panel B

% SSM row
\node[rounded corners,fill=blue!12,minimum width=4.6cm,minimum height=0.5cm,align=center,font=\scriptsize] (ssm) at ($(B)+(0,1.0)$)
{Recurrent models (RNNs, SSMs):\\ $s_t = \phi ( A (w_t) s_{t-1} + B (w_t) w_t )$};

% Sliding window row
\node[rounded corners,fill=blue!12,minimum width=4.6cm,minimum height=0.5cm,align=center,font=\scriptsize] (sw) at ($(B)+(0,0)$)
{Sliding window models:\\ $s_t = (s_{t-1}[2{:}], \phi (s_{t-1}))$};

% Looped row
\node[rounded corners,fill=blue!12,minimum width=4.6cm,minimum height=0.5cm,align=center,font=\scriptsize] (looped) at ($(B)+(0,-1.0)$)
{Looped models:\\ $s_t = \phi_{\text{loop}} (s_{t-1})$};

% ---------------- Panel C: Self-Generated Curriculum ----------------
\node (C) at (14.3,1.6) {}; % Anchor point for Panel C

% Hard problem at top
\node[rounded corners,fill=red!10,minimum width=4.25cm,minimum height=0.6cm,align=center,font=\footnotesize] (hard) at ($(C)+(0,1.45)$)
{Target: length-$T$ problems};

% Arrow down
\draw[->,thick] ($(C)+(0,1.05)$) -- ($(C)+(0,0.77)$);

% Curriculum arrow: easy to hard, left to right
\node[font=\scriptsize\bfseries] at ($(C)+(0,0.58)$) {Curriculum of easier problems};

% Difficulty gradient: growing boxes left to right
\node (X) at ($(C)+(-1.95,-0.1)$) {};
\node[rounded corners,fill=green!20,minimum width=0.75cm,minimum height=0.55cm] (d1) at ($(X)+(0,0)$) {};
\node[rounded corners,fill=green!75!red!20,minimum width=0.75cm,minimum height=0.55cm] (d2) at ($(X)+(1,0)$) {};
\node[rounded corners,fill=green!50!red!20,minimum width=0.75cm,minimum height=0.55cm] (d3) at ($(X)+(2,0)$) {};
\node[rounded corners,fill=green!25!red!20,minimum width=0.75cm,minimum height=0.55cm] (d4) at ($(X)+(3,0)$) {};
\node[rounded corners,fill=red!10,minimum width=0.75cm,minimum height=0.55cm] (d5) at ($(X)+(4,0)$) {};

% Labels
\node[font=\scriptsize,align=center] at ($(X)+(0,0)$) {$t\!=\!1$};
\node[font=\scriptsize,align=center] at ($(X)+(0,-0.5)$) {(easy)};
\node[font=\scriptsize,align=center] at ($(X)+(4,0)$) {$t\!=\!T$};
\node[font=\scriptsize,align=center] at ($(X)+(4,-0.5)$) {(hard)};

% Horizontal arrow underneath
\draw[->,thick,draw=black!40] ($(C)+(-1.3,-0.6)$) -- ($(C)+(1.425,-0.6)$);

% Compose box at bottom
\node[rounded corners,fill=orange!18,minimum width=4.9cm,minimum height=0.42cm,align=center,font=\scriptsize] at ($(C)+(0,-1.13)$)
{Compose skills to solve harder tasks};

\end{tikzpicture}}
\caption{Model overview. \textit{Left:} semiautomata dynamics under an unknown transition function $\fstar$. \textit{Middle:} semiautomata capture several modern sequential reasoning architectures. \textit{Right:} self-generated curriculum: a length-$T$ instance is decomposed into shorter sub-instances of length $\tau \ll T$, short-range models are trained, and composed to solve the original problem.}
\label{fig:model-overview}
\end{figure}
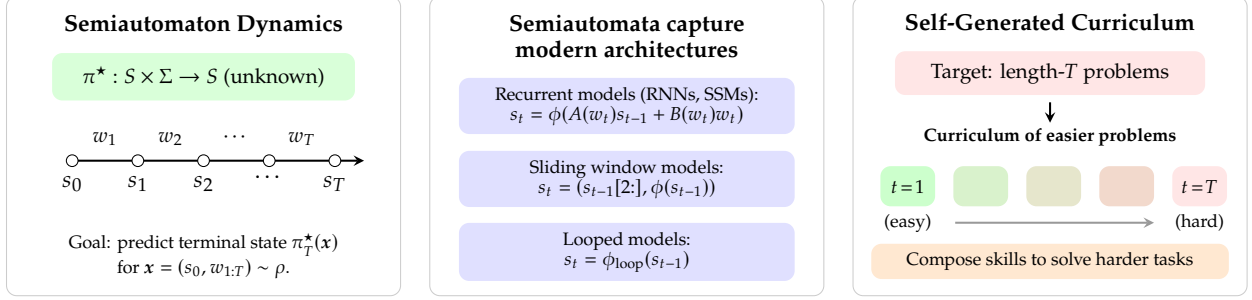

\subsection{Contributions}
    
We study these questions through the task of \emph{semiautomaton simulation}---learning to predict the outcome of $T$ steps of sequential computation---a canonical model that captures state tracking, regular language recognition, and modular arithmetic. A semiautomaton is defined by an unknown transition function $\fstar : S \times \Sigma \to S$ that updates a state via $s_t\gets\fstar(s_{t-1},w_t)$, where $\bw=\bw_{1:T}$ is an input word. The intermediate states $s_{1:T}$ constitute a chain-of-thought, and the goal is to predict the final state $s_T$ after $T$ steps. In the language modeling context, one can think of states as tokens, but they can be more general (say, a block of $k$ tokens, corresponding to a fixed-size attention window). The Markovian structure of semiautomata means that models solving shorter instances can be \emph{composed} to solve longer ones, making this a natural testbed for understanding how curriculum can enable compositional generalization.\footnote{This model can be viewed as a Markovian restriction to the chain-of-thought learning framework of \citet{joshi2025theory,rajaraman2026learning}, which allows for arbitrary dependence on previous states. 
} Our main results---building on Part I \citep{rajaraman2026learning}---show that a self-generated compositional curriculum, where the learner adaptively decomposes long problems into shorter sub-problems, yields superpolynomial reductions in supervision and computational cost relative to non-curriculum baselines, in both supervised fine-tuning (SFT) and reinforcement learning with verifiable rewards (RLVR).\loose

\paragraph{\Cref{sec:semiauto}: Fine-tuning with interactive CoT supervision (iSFT).} 
 
Our first set of results considers a setting in which the learner has interactive access to the semiautomata $\fstar$ of interest and aims to select a set of state transitions on which to fine-tune a model with next-token prediction (concretely, given any word $\bw$ and step $t$, the learner can ask for the state $s_t$, at one unit of cost). The standard approach would be to generate full reasoning traces and train on every intermediate state, but can we do better by deliberately choosing \emph{which} states to supervise? We show that the answer is yes. We give an autocurriculum-based algorithm, $\WLdepth{}$ (\Cref{alg:semiauto_highacc}), which interactively selects which intermediate states to label by constructing a curriculum of shorter sub-problems and composing them. The resulting query complexity---the number of state labels requested---is \vphantom{$t^{2^{2}_2}$} \smash{$2^{\calOtilde(\sqrt{\log T})}$}, which is \emph{subpolynomial} in the sequence length $T$. This overcomes the $\Omega(T)$ barrier inherent to direct simulation, where every intermediate state must be labeled.\loose

\paragraph{\Cref{sec:rl}: Fine-tuning a reference model (RLVR).} Our second set of results considers a setting inspired by reinforcement learning with verifiable rewards (RLVR)~\citep{lambert2024tulu}. Here, the learner has access to a pre-trained \textit{reference model} $\piref$ that can correctly solve short instances of length $B \ll T$ with constant probability, as measured by a coverage coefficient~\citep{zhu2023principled,song2024importance}, along with an \textit{outcome verifier} that checks correctness of guesses for the terminal state. The reference model's coverage at the full sequence length $T$ may be exponentially worse than at length $B$, so directly applying prior RL methods~\citep{rajaraman2026learning} at the length-$T$ scale is prohibitively expensive. We give an algorithm, $\WLdepth{}.\RL$, which uses autocurriculum to compose short-range predictions into full-length solutions, reducing the requirement on $\piref$ from coverage over full sequences of length $T$ to coverage at the shorter block scale $B$---an exponentially weaker condition. 
This is enabled by a \textit{coverage expansion} phenomenon: curriculum enables the reference model to incrementally build coverage over harder problems, starting from the short length-$B$ instances it can already solve.

\subsection{Organization}

\Cref{sec:prelim} formally introduces semiautomata and our interactive learning protocols for fine-tuning and RL, then provides a high-level overview of our main results. \cref{sec:semiauto} and \cref{sec:rl} then give a detailed presentation of our main algorithms and results for iSFT and reinforcement learning, respectively, and \cref{sec:techniques} gives an overview of the key technical challenges and analysis ideas behind the recursion used in both algorithms. We conclude with open questions and future directions in \cref{sec:discussion}. All proofs and supporting results are deferred to the appendix.

\section{Problem Setting and Overview of Results} \label{sec:prelim}

% Addressed: \dfc{merge in signposting paragraph}
% \claudeadd{
    This section formalizes the two learning settings we study---supervised fine-tuning with interactive feedback (iSFT) and reinforcement learning with verifiable rewards (RLVR)---and presents informal statements of our main results. We begin with preliminaries on semiautomata (\cref{sec:semiautomata}), then introduce the iSFT setting and results (\cref{subsec:settings,sec:isft-overview}), and finally introduce the RLVR setting and results (\cref{subsec:rl,sec:rl-overview}).

% In the iSFT setting, we show that curriculum-based composition achieves query complexity scaling \emph{subpolynomially} in $T$, overcoming the $\Omega(T)$ barrier of direct simulation; in the RLVR setting, we show that a weak reference model with coverage only at a short block scale $B \ll T$ can be bootstrapped to solve length-$T$ instances via the same compositional curriculum. Detailed algorithms and comparisons with non-curriculum baselines follow in subsequent sections.}

\subsection{Preliminaries: Semiautomata and Chain-of-Thought Learning}
\label{sec:semiautomata}

% \paragraph{Basic notation.} For nonnegative quantities $a,b$, $a \lesssim b$ (resp.\ $a \gtrsim b$) if there exists a universal constant $C>0$ such that $a \le C b$ (resp.\ $a \ge C b$) and $a \asymp b$ if both $a \lesssim b$ and $a \gtrsim b$ hold. Standard Big-Oh notation is used: $a = O(b)$ means $|a| \le C b$ for a universal constant $C$, and $a=\Omega(b)$ denotes the reverse inequality; we write $a=\Theta(b)$ when both bounds hold. We write $\calOtilde (f(n))$ to hide polylogarithmic factors in $f(n)$.

% Addressed: \dfc{(1) Lead with CoT goal + composition; (2) Compositional structure moved to after definition; (3) Part I footnote added}
We focus on learning to generate high-quality chains of thought, through direct supervision or reinforcement learning with verifier feedback. To understand how curriculum and composition can reduce the cost of learning, we adopt \emph{semiautomata} as a stylized yet rich model of chain-of-thought reasoning.\loose

\begin{definition}[Semiautomaton simulation]
A semiautomaton is defined by a transition function $\pi : S \times \Sigma \to S$ over a state space $S$ and input alphabet $\Sigma$. An \textit{instance} $\bx = (s_0,\bw_{1:t})$ is composed of an initial state $s_0 \in S$ and a word $\bw_{1:t} \in \Sigma^t$ for some $t \in \bbN$, and $t$ is the \textit{length} of the instance\footnote{For a vector $\bw$, $\bw_{1:t}$ indicates the vector obtained by slicing to length $t$, while $w_t$ (non-bold) indicates the $t^{\text{th}}$ coordinate.}. Words correspond to the input strings that we want the model to process, while the initial state $s_0$ represents the starting point of the computation. The overall instance $\bx$ can be interpreted as a \textit{prompt} in the language modeling context. The \textit{semiautomaton simulator} $\pi_t (\bx)$ returns the \textit{terminal state} reached after applying $\pi$ iteratively for $t$ steps on $\bx$.\footnote{We consider the \textit{parameter-sharing regime} where transition functions are time invariant.} Namely $\pi_t (\bx) = s_t$, where $\forall 1 \le t' \le t,\ s_{t'} \gets \pi ( s_{t'-1}, w_{t'} )$. We will use $\pi_{1:t} (\bx)$ to denote the full sequence of states $(s_1,\cdots,s_t) \in S^{t}$ generated by this process, which is referred to as the \textit{chain-of-thought} (CoT).\footnote{The notations $\pi_{t'} (\bx)$ and $\pi_{1:t'} (\bx)$ extend to $t' < t$ by truncating $\bx = (s_0,\bw_{1:t})$ to $\bx' = (s_0,\bw_{1:t'})$ and returning $\pi_{t'} (\bx')$ or $\pi_{1:t'} (\bx')$.}
%  \dfc{what is the bold convention? eg when do we use $\bw_{1:t}$ vs $w_t$ etc?} 
\end{definition}
\noindent
We use the term ``simulation'' to reflect that the objective is to predict the final state (outcome) of the semiautomaton, not specifically to learn the transition function or full CoT (except as a means to predict the outcome). In the language modeling context, the state itself can be a single token, but our setting is more general and allows for more complex state representations (e.g., blocks of tokens).\footnote{Going forward, we will often use the term ``state'' and ``token'' interchangeably when it is clear from context.}
% but it is more general and can also capture, e.g., blocks of tokens.

% We use the term \emph{state} (as opposed to \emph{token}) to refer to the 

% Furthermore, while the state itself can be a single token when connected back to the language modeling setting, it is more general and can also capture, e.g., blocks of tokens.

\loose 

% \dfc{Explain that we use ``simulation'' to reflect that we are learning to predict the outcome, not learn the transition or full CoT.}

% \claudeadd{
\medskip
\noindent The key structural property of semiautomata is that the transition function $\fstar$ is Markovian and stationary. That is, it is applied identically at every step, and the state $s_t$ is a sufficient summary of the computation history up to time $t$. This means that any contiguous chunk of a length-$T$ computation can be interpreted as a self-contained computation of shorter length $t \ll T$, and models trained to solve shorter instances can be \emph{composed} to solve longer ones. This structure makes semiautomata a natural testbed studying how curriculum, composition, and self-generated sub-problems interact. This refines the setup from Part~I~\citep{rajaraman2026learning}, which considered general, potentially non-Markovian transition functions. Note that the Markovian structure is precisely why we allow a state to encode more than a single token: at the granularity of individual tokens, the Markovian assumption would be severe (language models attend to their entire history) whereas letting $s_t$ summarize a block of tokens (such as a fixed attention window) keeps the Markovian condition milder and more realistic.

% \medskip
\subsubsection{Examples of semiautomata}
\noindent Beyond serving as a theoretical testbed, semiautomata capture a broad range of sequential computation tasks, including settings where the state space $S$ and input alphabet $\Sigma$ can be very large, necessitating function approximation. We discuss a few examples below.

\begin{enumerate}
\item \textit{Modular arithmetic and group products.} Let $(G, \cdot)$ be a finite group with $|G| = k$. The corresponding semiautomaton has state space $S = G$, input alphabet $\Sigma = G$, and transition function $\pi(s,w) = s \cdot w$. For an input word $\bw$, simulating the semiautomaton for $T$ steps computes the group product $s_0 \cdot w_1 \cdot w_2 \cdots w_T$. For the case where $G = \bbZ / k \bbZ$ under addition, this reduces to computing modular sums. As a more interesting example, when $G = S_5$ (the symmetric group on $5$ elements), iterated group multiplication can simulate any $\textsf{NC}^1$ circuit by Barrington's theorem~\citep{barrington1986bounded}. The class $\calF$ of all valid group multiplication tables on $[k]$ has Natarajan dimension at most $\Ndim (\calF) \le \calO (k^2 \log k)$. \loose

\item \textit{Linear recurrences over finite fields.} Consider the semiautomaton with state space $S = \bbF_q^d$, input alphabet $\Sigma = \bbF_q^d$, and transition function $\pi(s,w) = As + Bw$ for unknown matrices $A \in \bbF_q^{d \times d}$ and $B \in \bbF_q^{d \times d}$. Simulating for $T$ steps computes the linear recurrence $s_T = A^T s_0 + \sum_{t=1}^T A^{T-t} B w_t$. The class $\calF$ of all such linear transitions is parameterized by the pair $(A,B)$, giving $|\calF| = q^{2d^2}$, and Natarajan dimension $\Ndim (\calF) \le \calO ( d^2 \log q)$ which is exponentially smaller than the number of states, $|S| = q^d$.
\item \textit{Regular languages.} A regular language over an alphabet $\Sigma$ is the set of strings accepted by some Deterministic Finite Automaton (DFA), where input words are arbitrary strings in $\Sigma^*$. These correspond to strings that match a regular expression built from alphabet symbols, concatenation, union ($|$), and Kleene star (${}^*$). The \textit{alphabetic length} $k$---specifically, a regular expression has length $k$ if it is built from $k$ alphabet symbols, counted with multiplicity\footnote{Other related notions of length, such as the ordinary length and the reverse polish length, are upper bounded by a constant factor times the alphabetic length~\citep{lee2004enumerating}.}---is a natural complexity measure for regular languages. In particular, by Glushkov's construction~\citep{glushkov1961abstract} and the powerset constructions, a regular language with alphabetic length $k$ can be decided by a DFA with at most $2^{k+1}$ states. In \Cref{sec:regex-example}, we give a concrete example of a regular expression and describe the induced semiautomaton.
\end{enumerate}

\subsubsection{Modeling and function approximation}

To allow for generalization when the state and input alphabet are large, we assume access to a \emph{model class} $\calF \subseteq \{ S \times \Sigma \to S \}$ of candidate transition functions (equivalently, ``next-state/token predictors'') containing the unknown transition function $\fstar$. This class can represent any parameterized model family (e.g., transformers or state-space models). As a computational primitive, our algorithms make use of \emph{next-state/token prediction} (ERM on the next-state/token prediction loss), but otherwise access the model class only implicitly.\loose

\begin{definition}[Next-token prediction oracle for $\calF$] \label{def:erm}
Given any realizable dataset $D = \{ (s_i,w_i) \mapsto \fstar (s_i,w_i) \}_{i=1}^n$ of input-next-state pairs, a next-state/token prediction oracle for $\calF$, $\NTP (D)$, returns \smash{$\fhat \in \calF$} that minimizes the next-state/token prediction loss: $\frac{1}{n} \sum_{i=1}^n \mathbb{I} ( \fhat(s_i,w_i) \ne \fstar(s_i,w_i) ) = 0$.\footnote{Since semiautomata are deterministic, we work with the $0$-$1$ loss. For general (stochastic) language modeling settings, standard next-token prediction corresponds to ERM under the logarithmic loss.}
\end{definition}

\noindent Many of our results make use of the Natarajan dimension of $\calF$, a complexity measure that characterizes the distribution-free sample complexity of learning $\calF$ from i.i.d. examples (up to logarithmic factors).\loose

\begin{definition}[Natarajan dimension of $\calF$~\citep{natarajan1989learning}] \label{def:Ndim}
The Natarajan dimension of a class of models $\calF$ is the largest $d \in \bbN$ such that there exist instances $\{ (s_i,w_i) \}_{i=1}^d \subseteq S \times \Sigma$ and pairs of labels $\{ y_i,y_i' \}_{i=1}^d \subseteq S \times S$ satisfying $y_i \ne y_i'$ for all $i \in [d]$, with the following property: for every subset $I \subseteq [d]$, there exists a model $\pi_I \in \calF$ such that $\forall i \in I,\ \pi_I (s_i,w_i) = y_i$ and $\forall i \not\in I,\ \pi_I (s_i,w_i) = y_i'$.
\end{definition}

\begin{remark}[Simplifications in the semiautomaton framework] \label{remark:simplifications}
While our results are motivated by language model reasoning, our problem setting makes two fairly strong simplifications: (1) We restrict the learner to Markovian models, which rules out classes like Transformers that can attend to the full history. (2) We consider a deterministic setting, which rules out stochastic models. Regarding the first point, we note that while this may seem like a strong simplification when we interpret states $s_i$ as tokens, our guarantees allow for general, potentially large state spaces, which makes the framework quite rich nonetheless. For example, $s_i$ may represent a block of tokens of length $k$, in which case our model allows for fixed-size attention windows of size $k$. Relaxing both limitations is an important direction for future work, but we believe our framework is still a powerful setting to explore the role of composition.
\end{remark}

\subsection{\mbox{Setting I: Learning from Interactive Chain-of-Thought Supervision (iSFT)}} \label{subsec:settings}

\noindent Recall from \Cref{sec:semiautomata} that the learner's goal is to predict the terminal state $\fstar_T(\bx)$ reached by a ground-truth semiautomaton $\fstar$ on length-$T$ instances $\bx \sim \rho$. Standard supervised approaches fix in advance \emph{what} supervision to collect~\citep{joshi2025theory}: either label the full chain-of-thought $\fstar_{1:T}(\bx)$ on each instance and fit a model to its next-state pairs via next-token prediction, or learn \emph{end-to-end} from the terminal state $\fstar_T(\bx)$ alone. Our first setting considers a more flexible model---\emph{learning from interactive chain-of-thought supervision}, or \emph{iSFT}---in which the learner instead chooses \emph{which} states of the chain-of-thought to supervise, querying the state reached at any step $t$ at one unit of cost. We formalize this interactive access through the oracle below.
% We first define this oracle below and discuss stronger motivations for considering this model in \Cref{subsubsec:motivation-iSFT}.

\begin{definition}[Interactive chain-of-thought oracle] \label{def:outcome-oracle}
Let $\fstar$ denote the ground-truth semiautomaton. An \textit{interactive chain-of-thought oracle}, $\etoe (\cdot)$, is a function which takes as input an instance $\bx = (s_0,\bw_{1:t})$ of \textit{any length $t \in \bbN$} and returns the chain-of-thought state $\fstar_t (\bx)$ at step $t$.
\end{definition}

\noindent Two interpretations of this oracle will be useful. First, when applied to a length-$t$ instance, $\etoe(\bx)$ returns that instance's \emph{terminal state}: the outcome of running $\fstar$ for $t$ steps. Second, when applied to a length-$t$ \emph{prefix} of a longer length-$T$ instance, the same query returns the \emph{intermediate state} $\fstar_t(\bx)$ that the computation passes through at step $t$. Thus, by querying prefixes, the learner can recover any state along the chain-of-thought of a long instance.

\begin{problem}[Learning from Interactive Chain-of-Thought Supervision: iSFT] \label{def:pac-sft}
The learner is given a class $\calF \subseteq \{ S \times \Sigma \to S \}$ of next-state/token predictors containing the true semiautomaton $\fstar$, along with sampling access to a target distribution $\rho \in \Delta_{S \times \Sigma^T}$ of length-$T$ instances. The objective is to learn a model \smash{$\fhat : S \times \Sigma^T \to S$} such that with probability at least $1-\delta$, the \textit{terminal state prediction error} satisfies \smash{$\Pr_{\bx \sim \rho} \big[ \fhat (\bx) \ne \fstar_T (\bx) \big] \le \varepsilon$}. To do so, the learner may query the interactive chain-of-thought oracle $\etoe (\cdot)$ of \Cref{def:outcome-oracle} on instances of any length. The learner is evaluated according to the following desiderata:
\begin{enumerate}
    \item \textit{Query complexity}: The number of calls $\nquery$ made to the interactive chain-of-thought oracle $\etoe (\cdot)$.
    \item \textit{Sample complexity}: The number of instances  $\nsample$ drawn from $\rho$.
\end{enumerate}
\end{problem}

\noindent Note that in the iSFT setting, the learner can recover both standard forms of supervision (since $\etoe(\cdot)$ accepts instances of any length): $T$ prefix queries suffice to reconstruct an instance's full chain-of-thought (for next-token prediction~\citep{foster2024behavior,joshi2025theory}), and a single query suffices to obtain the end-to-end outcome. The limitation of both approaches is that they incur $\Omega(T)$ query complexity \citep{joshi2025theory}. The real leverage of the iSFT framework comes from the Markovian structure of $\fstar$ (\Cref{sec:semiautomata}), which lets the learner \emph{construct its own sub-problems}: splitting a length-$T$ instance at chosen boundaries and labeling the boundary states with $\etoe(\cdot)$ decomposes it into shorter instances, the terminal state of each becoming the start state of the next. In what follows, we will show that this form of self-generated supervision can be used to break the $\Omega(T)$ query complexity barrier faced by the approaches above.

\paragraph{Motivation for the iSFT interaction model.}

We recognize that at first glance, the iSFT setting may seem somewhat contrived or artificial. One might expect that in general, labeling an intermediate state $\fstar_t(\bx)$ should be no easier than labeling the entire intermediate CoT $\fstar_{1:t}(\bx)$, yet our framework counts the former as a single unit of cost and counts the latter as $t$ units of cost. There are two reasons why we believe the model is useful:
\begin{itemize}[leftmargin=1.5em,itemsep=1pt, topsep=3pt, parsep=5pt]
\item First, for many natural semiautomata, labeling a single state \emph{is} cheaper than producing an entire chain-of-thought. We observe that for semiautomata including modular arithmetic, Dyck languages, and linear recurrences (\Cref{sec:related-work}), the state at step $t$ can be computed in parallel in $\calO(\log t)$ time by divide-and-conquer over matrix or group products, whereas simulating $\fstar$ step by step to obtain the full length-$t$ chain-of-thought takes $\Omega(t)$ time. Thus, a single state query is a natural unit of cost here.
\item Second, the iSFT model underpins our reinforcement learning results in the sequel (\Cref{sec:rl}). There, we show that one can use a weak reference model and outcome verifier to ``simulate'' $\etoe{}$, allowing the curriculum and composition machinery we develop here to carry over directly.
\end{itemize}

\subsection{Overview of Results for iSFT Setting}
\label{sec:isft-overview}

\noindent Our main result shows that, by querying the interactive chain-of-thought oracle to assemble a self-generated curriculum of shorter instances, semiautomata can be learned in the iSFT setting with sample and query complexity scaling \emph{subpolynomially} in the sequence length $T$. Our algorithm $\WLdepth{}$ (\Cref{alg:semiauto_highacc}) thereby breaks the $\Omega(T)$ barrier faced by non-curriculum approaches.

\begin{theorem}[Main result for iSFT setting (informal; see \cref{theorem:main-generic})] \label{informaltheorem:main}
$\WLdepth{}$ (\Cref{alg:semiauto_highacc}) learns a model with \smash{$\Pr_{\bx \sim \rho} \big[ \fhat (\bx) \ne \fstar_T (\bx) \big] \le \varepsilon$} with probability at least $1-\delta$ in the iSFT setting, and is computationally efficient in terms of a next-token prediction oracle for $\calF$ (\Cref{def:erm}). The sample and query complexity used by the algorithm are upper bounded by:
\begin{equation*}
    \nsample, \nquery \le 2^{\calOtilde \big( \sqrt{\log (T)} \big)} \cdot \left[ \frac{d \log(|S|) \log^2(\frac{1}{\varepsilon}) \log(\frac{1}{\delta})}{\varepsilon} \right].
\end{equation*}
Here, $d \le \log_2 (|\calF|)$ is the Natarajan dimension of $\calF$ (\Cref{def:Ndim}).
\end{theorem}
\noindent $\WLdepth{}$ (\Cref{alg:semiauto_highacc}) builds this curriculum over the course of training by exploiting \emph{composition}: since $\fstar$ is Markovian, a model that simulates $\fstar$ accurately for $\tau$ steps (under a suitable mixture over intermediate states) can be composed with itself to simulate $\fstar$ for $T \gg \tau$ steps. The catch is that errors accumulate under composition, so the short-range accuracy must be driven high enough to tolerate them. $\WLdepth{}$ resolves this with a \emph{multiscale} variant of boosting~\citep{schapire2013boosting} that recursively interleaves two operations: $(a)$ boosting, which aggregates several short-range models to drive their error down, and $(b)$ composition, which stitches the boosted models into longer-range ones with degraded guarantees. The central difficulty, treated at length in \cref{sec:semiauto,sec:techniques}, is to turn samples from the target distribution $\rho$ into samples from the distribution over short sub-problems required by boosting, without inflating the sample complexity. The only learning primitive $\WLdepth{}$ invokes is next-token prediction (\Cref{def:erm}) on i.i.d.\ next-state data; in fact, this guarantee is a \emph{reduction}---any weak learner with low next-token-prediction error, not only ERM, can be substituted (\Cref{sec:semiauto}).

\paragraph{Comparison with non-curriculum baselines.} The subpolynomial rate in~\Cref{informaltheorem:main} stands in sharp contrast with guarantees achievable by natural non-curriculum approaches.

\begin{itemize}[leftmargin=1.5em,itemsep=1pt, topsep=3pt, parsep=5pt]
\item\textit{Learning from full chain-of-thought:} A full chain-of-thought learner labels every intermediate state of $\fstar$ on each training instance and fits a model to the resulting next-state pairs by next-token prediction. This is computationally efficient---it reduces to next-token prediction over $\calF$---but labels all $T$ states per instance, so its query complexity grows linearly in $T$. We show this is unavoidable: any such learner needs $\Omega(dT)$ queries to reach constant error, regardless of which instances are queried (\cref{prop:lower-bound-CoT}). 

\item \textit{Learning from end-to-end feedback:} 
An end-to-end learner queries only the terminal state $\fstar_T(\bx)$ on length-$T$ instances. \cite{joshi2025theory} show its statistical complexity scales as $\nsample = \nquery = \Omega(dT)$, as the generalization error can be $\Omega(T)$ times larger than that of the base class $\calF$ (\cref{prop:lower-bound-e2e}). This approach is also computationally intractable under standard cryptographic assumptions~\citep{kearns1994cryptographic}, as semiautomata can encode branching programs \citep{barrington1986bounded}.\loose 
\end{itemize}
\noindent We give the formal lower bounds and a detailed comparison in \Cref{subsec:baselines}.

\subsection{Setting II: Improving a Weak Reference Model from Verifier Feedback (RLVR)} \label{subsec:rl}

Our second setting is inspired by reinforcement learning with verifiable rewards (RLVR), and models the post-training of a pretrained model. The learner is given a \emph{reference model} $\piref$ (a weak generator of candidate solutions) and an \emph{outcome verifier} that checks whether a guessed terminal state is correct, and aims to improve the accuracy of $\piref$ for the task of simulating the final semiautomaton state output $\fstar_T$. When $\piref$ already has nontrivial \emph{coverage} of correct length-$T$ solutions, this is well understood, and standard RLVR fine-tuning succeeds \citep{foster2025good,rajaraman2026learning}. We instead study a harder regime in which $\piref$ is competent only at a much shorter block length $B \ll T$---that is, it solves length-$B$ instances with constant probability, but is exponentially unlikely to solve length-$T$ instances directly---and ask how to \emph{compose} this short-horizon competence into full-length solutions.

\medskip\noindent
We begin by formalizing the verifier and learning problem, then formally state our assumption on $\piref$.\loose

\begin{definition}[Outcome verifier] \label{def:verifier}
An outcome verifier $\calV$ takes in an instance $\bx \in S \times \Sigma^t$ of length $t\leq T$, along with a guess $s\in S$ for the terminal state $\fstar_t(\bx)$, and returns a reward $\bbI \big( s = \fstar_t (\bx) \big)$.
\end{definition}

\noindent The verifier is much weaker than the interactive chain-of-thought oracle of \Cref{def:outcome-oracle}. Instead of returning the correct state outright, it gives only binary feedback on a guess, and is uninformative when the guess is wrong (especially when $|S|$ is large). We will use two equivalent interpretations for the verifier: (1) When applied to a length-$t$ instance and a guess $s$, $\calV(\bx,s)$ tests whether $s$ is that instance's \emph{terminal state}: the outcome of running $\fstar$ for $t$ steps; (2) When applied to a length-$t$ \emph{prefix} of a longer length-$T$ instance (interpreted as a new length-$t$ instance), the same query tests whether $s$ is the \emph{intermediate state} $\fstar_t(\bx)$ that the computation passes through at step $t$. With this in hand, we state the learning problem.

\begin{problem}[RL setting] \label{def:pac-rl}
The learner is given a class $\calF \subseteq \{ S \times \Sigma \to S \}$ of next-state/token predictors containing the true semiautomaton $\fstar$, along with sampling access to a target distribution $\rho$ over length-$T$ instances. The objective is to learn a model \smash{$\fhat : S \times \Sigma^T \to S$} such that with probability at least $1-\delta$, the \textit{terminal state prediction error} satisfies \smash{$\Pr_{\bx \sim \rho} \big[ \fhat (\bx) \ne \fstar_T (\bx) \big] \le \varepsilon$}. To do so, the learner has access to a reference model $\piref : S \times \Sigma \to \Delta_S$ and an outcome verifier $\calV$ (\Cref{def:verifier}). The learner is evaluated according to the following desiderata:
\begin{enumerate}
    \item \textit{Sample complexity}: The number of instances $\nsample$ drawn from $\rho$.
    \item \textit{Query complexity}: The number of calls $\nquery$ made to the verifier $\calV$.
    \item \textit{Computational cost}: The number of state transitions $\ncomp$ generated from $\piref$ or another model in $\calF$ over the course of learning.
\end{enumerate}
\end{problem}

\begin{remark} \label{remark:cost}
The measure of computational cost we consider focuses on generation costs over the course of training, and ignores other sources of cost for the sake of simplicity. This is inspired by RLVR in practice where the dominant cost of training is autoregressive generation; model updates constitute a lower order cost~\citep{hu2024openrlhf}.
\end{remark}

\paragraph{Reference model and coverage.}

We formalize the competence of the reference model $\piref$ through its \emph{coverage}---the probability it assigns to a correct short solution---over blocks of size $B \ll T$.

\begin{definition}[Block-level coverage] \label{def:ref-model}
For a block length $B\in\bbN$ and constants \smash{$\Cblock \ge \Cout \ge 1$}, the reference model $\piref$ satisfies the following guarantees:
\begin{equation} \label{eq:cov}
    \inf_{\bx \in S \times \Sigma^B} \piref_{1:B} ( \by^\star | \bx ) \ge \frac{1}{\Cblock}, \text{ and } \inf_{\bx \in S \times \Sigma^B} \piref_B ( y^\star_B | \bx ) \ge \frac{1}{\Cout}.
\end{equation}
where $\by^\star := \fstar_{1:B} (\bx)$. We refer to $\Cout$ as the \textit{outcome coverage coefficient at length $B$}, and \smash{$\Cblock$} as the \textit{sequence-level coverage coefficient at length $B$}.
\end{definition}

\noindent The two coefficients capture different notions of block-level success: $\Cout$ is the probability that $\piref$ produces the correct terminal state of a length-$B$ instance (potentially through an incorrect CoT), and $\Cblock$ is the probability that it produces the \emph{entire} length-$B$ CoT correctly. We always have $\Cout \leq \Cblock$, but using both coefficients will allow us to state more refined guarantees
    
\medskip
\noindent Coverage assumptions of this kind are standard in the analysis of reinforcement learning and inference-time methods~\citep{xie2022role,song2024importance,huang2025best,huang2025self,foster2025good,chen2025coverage}, formalizing the idea that pretraining endows $\piref$ with a weak but nonzero probability of producing correct reasoning traces. While typically RLVR algorithms scale with coverage $\CseqT$ at the full sequence level $T$, our definition assumes competence at length $B \ll T$ only. Rolling $\piref$ out across all $T/B$ blocks of a length-$T$ instance succeeds with probability exponentially small in $T/B$, so naive RLVR at the target length is hopeless (indeed, extrapolating from short to long reasoning is a notoriously difficult form of out-of-distribution generalization~\citep{anil2022exploring,zhou2024transformers}). This raises the question: can we exploit composition to \emph{expand} $\piref$'s coverage from length $B$ to length $T$, avoiding exponential blowup in cost? \loose

\subsection{Overview of Results for RLVR Setting}
\label{sec:rl-overview}

We show that self-generated curriculum makes \emph{coverage expansion} possible. Our main algorithm, $\WLdepth{}.\RL$ (\Cref{alg:semiauto_RL}), bootstraps $\piref$ from block-level competence to high accuracy at the full length $T$, with all costs governed by $\piref$'s coverage at the short block scale $B$. $\WLdepth{}.\RL$ is a reinforcement learning counterpart to $\WLdepth{}$, following the same composition/boosting template but replacing the chain-of-thought oracle with verifier feedback.

\begin{theorem}[Main result for RLVR setting (informal; see \Cref{theorem:main-RL})] \label{informaltheorem:main-RL-epsilon}
\mainalgRL (\Cref{alg:semiauto_RL}) learns a model with \smash{$\Pr_{\bx \sim \rho} \big[ \fhat (\bx) \ne \fstar_T (\bx) \big] \le \varepsilon$} with probability at least $1-\delta$ in the RLVR setting, and is computationally efficient in terms of a next-token prediction oracle for $\calF$ (\Cref{def:erm}). The sample complexity to achieve this guarantee is bounded by $\nsample \le \calOtilde \left( \frac{d}{\varepsilon} \right)$, and the query complexity (number of calls to the outcome verifier) is bounded by:\loose
\begin{align*}
    \nquery &\le \calOtilde \left( \frac{d}{\varepsilon} \right) + \calOtilde \left( \frac{T}{B} \cdot d \Cout + d \Cblock \right).
\end{align*}
The computational cost (number of individual state generations from $\piref$ or any model in $\calF$ over the course of training) is bounded by:
\begin{align*}
    \ncomp &\le \calOtilde \left( \frac{d T}{\varepsilon} \right) + \calOtilde\left(dT \Cout + d B \Cblock\right).
\end{align*}
Above, $d$ denotes the Natarajan dimension for $\calF$, and $\calOtilde (\cdot)$ hides logarithmic factors in $|S|,\Cblock,\delta^{-1}$, $\veps^{-1}$, and $T/B$.
\end{theorem}

\noindent We interpret \Cref{informaltheorem:main-RL-epsilon} as realizing a coverage expansion phenomenon concretely. Each bound splits into: \loose
\begin{enumerate}
    \item A leading term, $\calOtilde(d/\varepsilon)$ for the sample and query complexity, and $\calOtilde(dT/\varepsilon)$ for the computational cost, capturing the irreducible cost of learning to accuracy $1-\varepsilon$.
    \item \emph{Burn-in} terms that carry all dependence on the coverage for $\piref$, yet are nearly-independent of $\varepsilon$. Crucially, the burn-in depends only on coverage at the \emph{block} scale $B$: over training, $\piref$'s effective coverage at length $T$ expands from exponentially small to a constant, after which improvement proceeds as if $\piref$ had had constant length-$T$ coverage all along. 
\end{enumerate} 
In contrast, standard RLVR fine-tuning (GRPO-style methods) applied directly at length $T$ must pay for $\piref$'s coverage at that length, $C_{\texttt{seq}}^{(T)}$.

\medskip
\noindent $\WLdepth{}.\RL$ closely follows the $\WLdepth{}$ template, decomposing each length-$T$ instance into a curriculum of length-$B$ sub-problems, which can then be solved through off-the-shelf RLVR fine-tuning; as in our iSFT results, this leverages the fact that blocks of long instances can be interpreted as short instances in the semiautomaton setting. The same obstacles arise as in iSFT: short-range models must be accurate enough to survive composition, and the right distribution over sub-problems must be sampled.

\section{iSFT: Learning from Interactive Chain-of-Thought Supervision} \label{sec:semiauto}

This section presents our algorithms and main guarantees for learning semiautomata in the iSFT setting (\Cref{def:pac-sft}). First, in \cref{subsec:sft-main-composition}, we motivate our main algorithm, $\WLdepth{}$, by highlighting the compositional structure of our semiautomaton learning setup. We then formally describe $\WLdepth{}$ and present our main guarantee in \Cref{subsec:sft-main}. Finally, in \Cref{subsec:baselines} we compare with guarantees achieved by natural non-curriculum baselines.

\subsection{Algorithm Template: Composition and Boosting} 
\label{subsec:sft-main-composition}

\noindent Our main algorithm, $\WLdepth{}$ (\cref{alg:semiauto_highacc}), exploits the compositional structure of semiautomata to achieve sublinear-in-$T$ query complexity for learning models with high accuracy on length-$T$ instances. The core idea behind composition is that a model that predicts the state at time $\tau$ with high accuracy for $\tau = \frac{T}{L}$ can be composed $L$ times to obtain a model for the terminal state on length-$T$ instances with constant accuracy. We formalize this idea through the notion of a \emph{composition operator}.

\begin{definition}[Composition operator] \label{def:composition}
Fix some $t \in \bbN$, and consider any model $\fhat : S \times \Sigma^t \to S$ for the terminal state on instances of length $t$. Then for some $L \in \bbN$, the \textit{composed model} $\fhat^{\circ L} : S \times \Sigma^{tL} \to S$ operates on instances of length $tL$. $\fhat^{\circ L}$ chunks the instance into $L$ segments of length $t$, applying $\fhat$ to predict the terminal state of each chunk, using this as the start state for the next chunk, and repeating this process until termination. Formally, for a length-$tL$ instance $\bx = (s_0, \bw_{1:tL})$, $\fhat^{\circ L}$ outputs the state $s_L$, defined recursively by the process $s_\ell \gets \fhat ( \bx^\ell )$ for $\bx^\ell = (s_{\ell-1}, \bw_{t(\ell-1)+1:t\ell})$ for $1 \le \ell \le L$.
\end{definition}

\noindent Composition allows us to extend models solving shorter instances to longer ones, and is equivalent to wrapping the short-length model within a recurrent loop. While this suggests a natural algorithm design approach (learn a model for short instances, then compose it), one expects that generically, inaccuracies in the learned model $\fhat$ will compound the more times we compose, degrading performance. This relationship between composition and accuracy is captured in the following observation.

\begin{observation}[Compositional generalization] \label{obs:lg}
Recall that $\rho$ denotes the distribution over length-$T$ instances given by nature. Let $\rhobar$ denote the uniform mixture distribution over instances of length $\tau = \frac{T}{L}$ obtained by chunking longer length-$T$ instances (defined formally in \cref{eq:rhobar}) and let $\varepsilon_0 = \frac{1}{4L}$. Let \smash{$\fhat$} be any model satisfying a \textit{high-accuracy} correctness guarantee on length-$\tau$ test instances drawn from $\rhobar$; namely, \smash{$\Pr_{\bx \sim \rhobar} \big( \fhat (\bx) \ne \fstar_\tau (\bx) \big) \le \varepsilon_0$}. Then, the $L$-fold composition model $\fhat^{\circ L}$ (defined formally in \Cref{def:composition}) satisfies $\Pr_{\bx \sim \rho} ( \fhat^{\circ L} (\bx) \ne \fstar_T (\bx) ) \le \varepsilon_0 L = 1/4$.
\end{observation}

\medskip
\noindent \Cref{obs:lg} follows by a simple union bound, and provides a mechanism by which length (i.e., hardness) can be traded off for accuracy. However, this length-versus-accuracy tradeoff is not sufficient on its own to achieve sublinear-in-$T$ query complexity, as the increased accuracy requirement (i.e., $\varepsilon_0$ in \Cref{obs:lg}) wipes out potential gains from solving instances of much shorter length. For instance, if we naively train a model $\fhat$ using next-token prediction to very high accuracy on $1$-step instances, the target error on $1$-step problems would need to be $\varepsilon_0 = \frac{1}{4T}$, which requires sample complexity scaling as $\Theta \big( d / \varepsilon_0 \big) = \Theta (dT)$.

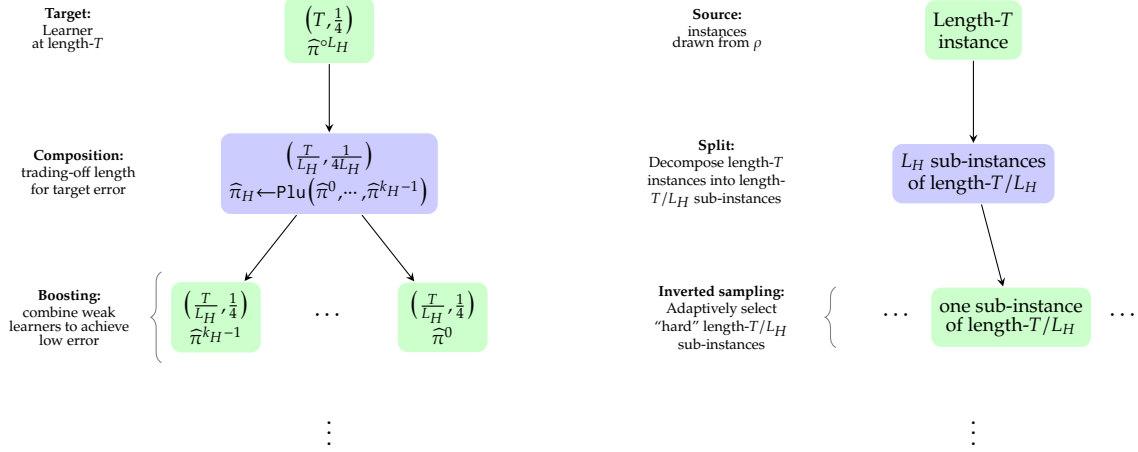
\begin{figure}[tp]
\centering

% =================================================================
% Part (a): Model-flow view
% =================================================================
\begin{subfigure}[t]{0.48\textwidth}
\centering
\begin{tikzpicture}[
  >=stealth,
  node distance=1.6cm and 2.0cm,
  box/.style={draw, rounded corners, minimum width=12mm, minimum height=6mm},
  every node/.style={font=\small}
]
\def\dx{1.5}
\def\dy{1.9}

% --- Level H: root (target model at length T) ---
\node[box,draw=white,fill=green!20] (rH) at (0,0)
  {$\substack{\big( T, \frac{1}{4}\big) \\ \fhat^{\circ L_H}}$};

% --- Level H-1/2: aggregation (plurality vote) at length T/L_H ---
\node[box,draw=white,fill=blue!20] (aH) at (0,-\dy)
  {$\substack{\big( \frac{T}{L_H}, \frac{1}{4 L_H}\big)\\ \fhat_H \gets \maj \big( \fhat^0, \cdots, \fhat^{k_H-1} \big)}$};
\draw[->] (rH) -- (aH);

% --- Level H-1: k_H weak models at length T/L_H ---
\node[box,draw=white,fill=green!20] (cHL) at (-\dx,-2*\dy)
  {$\substack{\big( \frac{T}{L_H}, \frac{1}{4}\big)\\ \fhat^{k_H-1}}$};
\node[box,draw=white,fill=green!20] (cHR) at ( \dx,-2*\dy)
  {$\substack{\big( \frac{T}{L_H}, \frac{1}{4}\big)\\ \fhat^{0}}$};
\node at (0,-2*\dy) {$\cdots$};
\draw[->] (aH) -- (cHL);
\draw[->] (aH) -- (cHR);

\node[] at (0,-2.75*\dy) {$\vdots$};

% --- Right-side level labels ---
\def\bx{1.5}
\node[left=4.75cm of rH, xshift=1.3cm, anchor=west, font=\scriptsize, align=left]
  {$\substack{\textbf{Target:}\\\text{Learner}\\\text{at length-}T}$};
\node[left=4.75cm of aH, xshift=2cm, anchor=west, font=\scriptsize, align=left]
  {$\substack{\textbf{Composition:}\\\text{trading-off length}\\\text{for target error}}$};
\draw[decorate, color=gray, decoration={brace, raise=2pt, amplitude=4pt}]
  ($(cHL)+(-0.65,-0.6)$) -- ($(cHL)+(-0.65,0.6)$);
\node[left=3.15cm of cHL, xshift=0.9cm, anchor=west, font=\scriptsize, align=left]
  {$\substack{\textbf{Boosting:}\\\text{combine weak}\\\text{learners to achieve}\\\text{low error}}$};

\end{tikzpicture}
\caption{Tree structure of models. $(t, \varepsilon)$ indicates the length of instances and the target error for the model at that node. The $k_H$ weak length-$T/L_H$ models $\fhat_H^0, \ldots, \fhat_H^{k_H-1}$ are combined by plurality vote into $\fhat_H$ and composed $L_H$-fold. The recursion continues for $H$ levels; models are trained at the bottom-most level to solve $1$-step instances, say via next-token prediction.}
\label{fig:parallel-reduction-tree:predictors}
\end{subfigure}
\hfill
% =================================================================
% Part (b): Data-flow view
% =================================================================
\begin{subfigure}[t]{0.48\textwidth}
\centering
\begin{tikzpicture}[
  >=stealth,
  node distance=1.6cm and 2.0cm,
  box/.style={draw, rounded corners, minimum width=12mm, minimum height=8mm},
  every node/.style={font=\small}
]
\def\dx{1.7}
\def\dy{1.9}

% --- Level H: source data ---
\node[box,draw=white,fill=green!20] (drH) at (0,0)
  {$\substack{\text{Length-}T \\ \text{instance}}$};

% --- Level H-1/2: post-split data ---
\node[box,draw=white,fill=blue!20] (daH) at (0,-\dy)
  {$\substack{L_H\text{ sub-instances} \\ \text{of length-}T/L_H}$};
\draw[->] (drH) -- node[right, font=\scriptsize] {} (daH);

% --- Level H-1: post-subsample datasets, one per boosting learner ---
\node[box,draw=white,fill=green!20] (dcHL) at (0.5,-2*\dy)
  {$\substack{\text{one sub-instance}\\\text{of length-}T/L_H}$};
% \node[box,draw=white,fill=green!20] (dcHR) at ( \dx,-2*\dy)
%   {$\substack{\text{Reweighted}\\\text{length-}T/L_H \\ \text{distribution} \\\text{over instances}}$};
\node at (-1,-2*\dy) {$\cdots$};
\node at (2,-2*\dy) {$\cdots$};
\draw[->] (daH) -- node[right=2pt, pos=0.65, font=\scriptsize] {} (dcHL);
% \draw[->] (daH) -- node[right=2pt, pos=0.65, font=\scriptsize] {$\subsample$} (dcHR);

\node[] at (0,-2.75*\dy) {$\vdots$};

% --- Right-side level labels ---
\node[left=4.75cm of drH, xshift=1.3cm, anchor=west, font=\scriptsize, align=left]
  {$\substack{\textbf{Source:} \\[1.5pt] \text{instances}\\[-1pt]\text{drawn from } \rho}$};
\node[left=5.75cm of daH, xshift=2.4cm, anchor=west, font=\scriptsize, align=left]
  {$\substack{\textbf{Split:}\\\text{Decompose length-}T\\\text{instances into length-}\\T/L_H\text{ sub-instances}}$};
\draw[decorate, color=gray, decoration={brace, raise=2pt, amplitude=4pt}]
  ($(dcHL)+(-2.25,-0.4)$) -- ($(dcHL)+(-2.25,0.4)$);
\node[left=4.67cm of dcHL, xshift=0.9cm, anchor=west, font=\scriptsize, align=left]
  {$\substack{\textbf{Inverted sampling:}\\\text{Adaptively select}\\\text{``hard'' length-}T/L_H\\\text{sub-instances}}$};

\end{tikzpicture}
\caption{Data flow on the same tree. Length-$T$ instances $\bx \sim \rho$ are decomposed by $\split$ into $L_H$ contiguous length-$T/L_H$ sub-instances. $\subsample$  selects at most one among these to pass on a given one of the $k_H$ boosting learners in (a). For the $j^{\text{th}}$ boosting learner, $\subsample$ selects sub-instances based on the previous $j-1$ models' performance, to adaptively focus on ``hard'' sub-instances.}
\label{fig:parallel-reduction-tree:data}
\end{subfigure}

\caption{Two views of the recursion within $\WLdepth{}$ (\Cref{alg:semiauto_main}): the \textit{(a)}~tree-like structure of induced models and \textit{(b)}~flow of data on the same tree.}
\label{fig:parallel-reduction-tree}
\end{figure}

\medskip
\noindent Building on Part I~\citep{rajaraman2026learning}, we improve this length-versus-accuracy tradeoff by combining composition with \emph{boosting}~\citep{freund1995boosting}. Boosting lifts weak (i.e., constant-accuracy) models into high-accuracy models by training them on reweighted versions of the target distribution over instances and aggregating them. When combined with composition, the algorithm takes the following form: \begin{enumerate}
    \item Train models solving short length-$\tau$ instances to high accuracy via boosting.
    \item Compose the resulting models to solve length-$T$ instances.
    \end{enumerate}
    With some care, this approach can be shown to achieve query complexity $\calOtilde(d\tau \log (1/\varepsilon_0) + d / \varepsilon_0)$ where $\varepsilon_0 = \tau/T$ when the target error under $\rho$ is a constant. Choosing $\tau$ optimally gives us a guarantee scaling as $\calOtilde(d\sqrt{T})$, which is sublinear in $T$, and improves over standard baselines. 
    
    \paragraph{From sublinear to subpolynomial query complexity.}
    To achieve the stronger \emph{subpolynomial} $2^{\calOtilde(\sqrt{\log T})}$ in our main theorem (\cref{theorem:main-generic}), our main approach, $\WLdepth{}$ (\Cref{alg:semiauto_highacc}) applies the boosting + composition template above recursively. To train models solving length-$\tau$ instances, we decompose instances into even shorter length-$\tau' \ll \tau$ ones, decompose these instances into even shorter length-$\tau''\ll \tau'$ instances, and so on, until all target instances have length-$1$. At this ``bottom level'' of the recursion, all training examples consist of next-state transitions of the form $(s,w) \mapsto \fstar (s,w)$, and weak learners can be trained via next-token prediction. Bottom-level models trained this way within $\WLdepth{}$ are then recursively \textit{aggregated} through boosting and \textit{composed} to solve longer length instances. The final length-$T$ model trained through this process recursively aggregates models across multiple scales; we refer to this overarching approach as \textit{multiscale boosting}.

\medskip
\noindent Multiscale boosting realizes a \emph{length-based curriculum}: as the algorithm proceeds to deeper levels, longer (i.e., harder) instances are progressively decomposed into shorter ones which are easier to train to solve. By composition, models which solve shorter instances are combined to realize ones which can solve harder problems. Moreover, the curriculum itself is \emph{self-generated}: at each scale, $\WLdepth{}$ adaptively selects which sub-instances to train on based on the failures of models trained at finer scales (via $\subsample$, \Cref{alg:semi-subsample}), rather than following a fixed schedule.

\subsection{Overview of $\WLdepth{}$} \label{subsec:wldepth-overview}

\noindent We now formally introduce $\WLdepth{}$ (\Cref{alg:semiauto_highacc}), which instantiates the multiscale boosting template above with $2H+1$ levels where $L$ is a branching factor and $H = \log_L(T)$. The algorithm is visualized in \cref{fig:parallel-reduction-tree:predictors}. We present the recursion in terms of half-integer levels $h = H+\frac{1}{2},H, H-\frac{1}{2}, \ldots, 0$, where the base level $h=0$ corresponds to length-$1$ instances, and larger $h$ corresponds to longer instances. $\WLdepth{}$ is parameterized by two schedules:\loose
\begin{itemize}[leftmargin=1.5em,itemsep=1pt, topsep=3pt, parsep=5pt]
    \item \textit{Decomposition schedule} $(L_h)$: determines how many sub-instances to break instances (of a certain length) into at each level of recursion; the corresponding model trained in each level is composed $L_h$-fold to induce a model solving longer instances.
    \item \textit{Branching schedule} $(k_h)$: determines how many weak models to aggregate at each level of recursion.
\end{itemize}
The half-step levels delineate two types of recursion steps. At integer levels $h$, $\WLdepth{h}$ implements both composition and boosting, setting \smash{$L_h = 2^{\Theta( \sqrt{\log(T)} )}$} and $k_h = \calOtilde (1)$. At half-integer depths $h$, $\WLdepth{h}$ implements only boosting, not composition, setting $L_h = 1$ and \smash{$k_h = \calOtilde(1)$}.

\paragraph{Splitting and inverted sampling.}
At integer depths $h \in \bbN$ where composition occurs, length-$\tau_h$ instances are broken down into length-\smash{$\tau_{h-1}$} sub-instances via subroutines $\split$ and $\subsample$ (\Cref{alg:split,alg:semi-subsample}). A key structural property is that each length-$\tau_h$ instance contributes at most one length-\smash{$\tau_{h-1}$} sub-instance, preserving independence; this is visualized in \cref{fig:parallel-reduction-tree:data}, and ensures that length-$1$ instances (the base case in the recursion), can be solved using off-the-shelf methods for classification/next-token prediction. The subroutine $\subsample$ addresses a central technical issue we have not yet discussed---\emph{distribution shift}---which complicates the process of generating examples from the appropriate distribution for boosting; we defer a dedicated discussion of the motivation and design of this component to \cref{sec:techniques}.

\paragraph{Base learner and next-token prediction.}

Crucially, all of the model training in $\WLdepth{}$ takes place at the bottom-level of the recursion, where $h=0$. By this level, instances are reduced to a collection of next-state examples $(s,w) \mapsto \fstar(s,w)$. To fit a model $\fhat$ in $\calF$, $\WLdepth{}$ makes use of a generic ``base learner'' $\Base$, representing next-token prediction or another off-the-shelf algorithm. We can interpret $\WLdepth{h}$ for $h \ge \frac{1}{2}$ as adaptively constructing a curriculum of shorter instances on which to fit models using $\Base$, concentrating on instances where the current model fails and decomposing them into sub-instances where necessary.\loose

\medskip
\noindent Our main guarantee for $\WLdepth{}$ is stated in terms of the sample complexity required to fit a next-token predictor model to constant accuracy using $\Base$, under an arbitrary but fixed distribution over instances.

\begin{definition}[Weak learning $\calF$ from i.i.d.\! next-state examples] \label{def:iid-learning-F}
Let $D = \{ (s_i,w_i) \mapsto \fstar(s_i,w_i) \}_{i=1}^n$ where \smash{$(s_i,w_i) \overset{\text{i.i.d.}}{\sim} \sigma \in \Delta_{S \times \Sigma}$} denotes a dataset of \textit{i.i.d.\! next-state examples}. A weak learner for $\calF$ is any algorithm $\Base (D \,\|\, \delta')$ parameterized by a failure probability $\delta' \in (0,1)$, which learns a model for $\fstar$ to constant accuracy from $D$ with \emph{sample complexity}  $\nweak (\delta')$. Namely, for any $\sigma$ and $\fstar\in\calF$, as long as $n \ge \nweak (\delta')$ the model $\fhat : S \times \Sigma \to S$ returned by $\Base$ satisfies w.p. at least $1-\delta'$,\loose
\begin{equation*}
    \Pr_{(s,w) \sim \sigma} \big[ \fhat(s,w) \ne \fstar(s,w) \big] \le \frac{1}{4}.
\end{equation*}
Furthermore, we say that $\Base$ admits a \textit{generic weak-learning guarantee} if $\nweak (\delta') \le \comp (\calF) \log(1/\delta')$, where $\comp (\calF)$ is some notion of complexity of the class $\calF$.
\end{definition}

\noindent \Cref{def:iid-learning-F} posits that the class $\calF$ is learnable from \textit{i.i.d.}\! next-state examples to constant accuracy. In the definition above, we demarcate the ``generic'' weak-learning guarantee to capture typical scaling behavior of the sample complexity with respect to the failure probability $\delta'$. For instance, when $\Base$ is instantiated as empirical risk minimization (ERM) under the $0$-$1$ loss, we obtain a generic weak-learning guarantee with $\comp (\calF) = d \log (|S|)$ where $d$ is the Natarajan dimension of $\calF$ (\Cref{def:Ndim})~\citep{natarajan1989learning}.

\subsection{\mbox{Main Result: Compositional Curriculum Achieves Sublinear Query Complexity}} \label{subsec:sft-main}

\noindent Our main guarantee for $\WLdepth{}$ shows that it achieves subpolynomial-in-$T$ sample and query complexity when invoked with off-the-shelf base learners such as next-token prediction.

\begin{algorithm}[t]
\caption{$\WLdepth{h} \big( \Dprompt{} \ \| \ \tau_h ,\delta \big)$}\label{alg:semiauto_main}
\begin{algorithmic}[1]
\Statex {\color{gray} \# Learning semiautomata via multiscale boosting}
\State \textbf{Input:} Class of semiautomaton transitions $\calF$ over state space $S$ and alphabet $\Sigma$,
\Statex \hspace{3em} Current depth $h \in \{ 0,\frac{1}{2},1,\frac{3}{2},\cdots\}$,
\Statex \hspace{3em} Current length of instances $\tau_h \in \bbN$,
\Statex \hspace{3em} Target failure probability $\delta$,
\Statex \hspace{3em} Dataset of instances \smash{$\Dprompt{} = \{ \bx^i = (s^i_0,\bw^i) \}_{i=1}^{n}$}.
\Statex \hspace{3em} Learner alg. $\Base(\, \cdot \,\|\, \delta')$ for $\calF$ from i.i.d. next-state data (\Cref{def:iid-learning-F}).
\Statex \LineComment{{\color{blue} $\Base(\, \cdot \,\|\, \delta')$ has sample complexity $\nweak(\delta')$ parameterized by failure probability.}}
\Statex \hspace{3em} Base failure probability $\delta_0$
\smallskip
\State \textbf{Hyperparameters:} Decomposition schedule, $( L_h : h \in \big\{ \frac{1}{2},1,\frac{3}{2},\cdots \big\} )$ where $\tau_h \pmod{L_h} \equiv 0$,
\Statex \hspace{8.33em} Branching schedule, $( k_h : h \in \big\{ \frac{1}{2},1,\frac{3}{2},\cdots \big\} )$.
\smallskip
\State \textbf{Instantiate:} $\tau_{h-\frac{1}{2}} = \tau_h/L_h$ \textbf{if} $h \ge \frac{1}{2}$.
\State \textbf{if $h=0$, Return:} $\fhat \gets \Base(\Dprompt{} \, \|\, \delta_0)$. \State Truncate \smash{$\Dprompt{}$} to the first $\nsample^\star (h,\delta)$ instances. \LineLabel{alg:semiauto_main:dbound} \LineComment{{\color{blue} $\nsample^\star$ is defined in \cref{eq:nsamplestar}.}}
\State Split $\Dprompt{}$ into $k_h$ equal parts, $\big\{ \Dprompt{j} : j \in \{ 0,\cdots,k_h-1 \} \big\}$.
\For{$j=0,1,\cdots,k_h-1$}
\State Let $\calF_j \gets \big\{ \fhat^0,\cdots, \fhat^{j-1} \big\}$ and \smash{$\Dout{j} \gets \emptyset$}. \LineComment{{\color{blue} By convention, $\calF_0 = \emptyset$}.}
\For{instances $\bx = (s_0,\bw) \in \Dprompt{j}$}
\State $D_\bx \gets \split (\bx \| L_h)$. \LineComment{{\color{blue} $\bx$ is split into $L_h$ shorter instances which}}
\State $O \gets \subsample \big(D_\bx \| \calF_{j}, k_h \big)$. \NoComment{{\color{blue} are labeled by their terminal states.}}
\State \textbf{if} $O \ne \perp$, \textbf{then} update the dataset \smash{$\Dout{j} \gets \Dout{j} \cup \{ O \}$}.
\EndFor
\State Train model $\fhat^{j} \gets \WLdepth{h-\frac{1}{2}} \big( \Dout{j} \ \| \ \tau_{h-\frac12}, \delta/2k_h \big)$.
\EndFor
\State \textbf{Return:} $\fhat^{\circ L_h}$, where $\fhat = \maj \big( \big\{ \fhat^j : j \in \{ 0,\cdots,k_h-1\} \big\} \big)$.
\end{algorithmic}
\end{algorithm}

\begin{theorem}[Main guarantee for $\WLdepth{}$] \label{theorem:main-generic}
For any $\delta' \in (0,1)$, let $\Base(\cdot \,\|\, \delta')$ be any weak learner for $\calF$ from i.i.d. next-state data, as defined in \Cref{def:iid-learning-F}. Suppose $\Base$ admits a generic weak-learning guarantee satisfying $\nweak (\delta') \le \comp (\calF) \log(1/\delta')$.

\medskip
\noindent Let $\varepsilon,\delta \in (0,1)$ and suppose $H = \sqrt{\log_2(T)} \in \bbN$. Suppose $\WLdepth{} (\cdot \,\|\, T,\varepsilon,\delta)$ (\Cref{alg:semiauto_highacc}) is invoked with a base learner $\Base(\cdot \,\|\, \delta')$ with weak-learning sample complexity $\nweak (\delta') \le \comp (\calF) \log(1/\delta')$. Then $\WLdepth{}$ learns a model $\fhat : S \times \Sigma^T \to S$ with $\Pr_{\bx \sim \rho} \big[ \fhat (\bx) \ne \fstar_T (\bx) \big] \le \varepsilon$ with probability at least $1-\delta$, using sample and query complexity upper bounded by:
\begin{equation*}
    \nsample,\nquery \le 2^{\calOtilde \big( \sqrt{\log (T)} \big)} \cdot \left[ \frac{\comp (\calF) \log^2 (1/\varepsilon) \log(1/\delta)}{\varepsilon} \right].
\end{equation*}
\end{theorem}

\noindent The main proof is provided in \Cref{sec:mainproof}. Relevant lemmas are sketched in \Cref{sec:mainproofsketch}. \Cref{informaltheorem:main} follows from this result by instantiating $\Base$ as empirical risk minimization (ERM) on the next-token prediction objective (\Cref{def:erm}), and noting that $\comp (\calF)$ can be chosen as $d \log (|S|)$, where $d$ is the Natarajan dimension. With more sophisticated choices for $\Base$, such as the recently proposed algorithm of \cite{pabbaraju2026optimal}, we can achieve improved sample complexity guarantees such as $\comp (\calF) = d_{\texttt{DS}}$, where $d_{\texttt{DS}}$ is the DS dimension~\citep{daniely2015multiclass}. Our most general guarantees for $\WLdepth{}$ (given in \Cref{theorem:main}) are directly stated in terms of $\nweak (\cdot)$ (and do not assume a generic weak learning guarantee for $\Base$).

\medskip
\noindent In \Cref{sec:semiauto-depth1-proofs} we analyze a simpler variant of $\WLdepth{}$ which considers a single round of boosting + aggregation to understand the core mechanism.

\begin{remark}[On the technical condition on length $T$]
\Cref{alg:semiauto_main} requires the condition $\sqrt{\log_2(T)} \in \bbN$. The guarantees can be extended to any value of $T \in \bbN$ by replacing $T$ by $T'$, where $T'$ is the smallest value larger than $T$ such that there exists $H \in \bbN$ satisfying $T' = (L')^H$, where \smash{$L' = 2^{\sqrt{\log_2(T')}}$},\footnote{It suffices to choose $T' = 2^{m^2}$ for some $m \in \bbN$. Such a $T'$ always exists in $[T,2T^2]$.} and changing the class $\calF$ to $\calF'$ as follows: the state space $S$ is changed to $S' = S \times [T]$ enabling the ``timestep'' to be tracked, and, each $\pi \in \calF$ is changed to $\pi' \in \calF'$ as,
\begin{equation*}
    \text{For } s' = (s,t) \in S \times [T],\quad \pi'(s',w) = \begin{cases}
    (\pi(s,w),t+1) &\text{if } t < T, \\
    (s, t) &\text{if } t = T, \\
    \end{cases}
\end{equation*}
Changing $\calF$ to $\calF'$ results in the state space growing by a factor of $T$ in size. However, this transformation itself does not affect the learnability of $\calF$, since an algorithm for learning an unknown $\fstar \in \calF$ from i.i.d. next-state data also induces a model for learning the corresponding $(\pi')^\star \in \calF'$ from i.i.d. next-state data achieving the same error guarantee, and vice versa.
\end{remark}

\begin{algorithm}[t]
\caption{$\WLdepth{} \big( \Dprompt{} \ \| \ T,\varepsilon,\delta \big)$}
\label{alg:semiauto_highacc}
\begin{algorithmic}[1]
\Statex {\color{gray} \# High-accuracy extension of \Cref{alg:semiauto_main}}
\State \textbf{Input:} Class of semiautomaton transitions $\calF$ over state space $S$ and alphabet $\Sigma$,
\Statex \hspace{3em} Target length of instances $T$,
\Statex \hspace{3em} Target error $\varepsilon$ and failure probability $\delta$.
\Statex \hspace{3em} Dataset of instances \smash{$\Dprompt{} = \{ \bx^i = (s^i_0,\bw^i) \}_{i=1}^{n} \overset{\text{i.i.d.}}{\sim} \rho$}.
\Statex \hspace{3em} Learner alg. $\Base(\, \cdot \,\|\, \delta')$ for $\calF$ from i.i.d. next-state data (\Cref{def:iid-learning-F}).
\smallskip
\State \textbf{Instantiate:} $H \gets \sqrt{\log_2(T)}$
\Statex \hspace{5.25em} $\delta_0 \gets \delta/h_\varepsilon(T)$ where $h_\varepsilon(T) = 2^{C\sqrt{\log(T)}\log\log(T)} \cdot \log^2(1/\varepsilon)$ \LineComment{{\color{blue} $C>0$ is a large constant}}
\State \textbf{Return:} $\fhat \gets \WLdepth{H+\frac{1}{2}} \big( \Dprompt{} \ \| \ T, \delta \big)$, \LineComment{{\color{blue} $\WLdepth{h}$ defined in \Cref{alg:semiauto_main}.}}
\State \hspace{3.5em} invoked with $\Base (\, \cdot \,\|\, \delta')$ with base failure probability $\delta_0$,
\State \hspace{3.5em} \textit{decomposition schedule:} \vphantom{$2^{2^2_1}$}\smash{$L_h = L \triangleq 2^{\sqrt{\log_2(T)}}$} for $h \in \bbZ$, else $L_h = 1$, and, \LineLabel{alg:semiauto_main:decomposition_schedule}
\State \hspace{3.5em} \textit{branching schedule:} \smash{$k_h = C \log (L)$ for $h \in \bbZ$, $k_{h+\frac{1}{2}} = C\log\log(T)$ for $h \in \{ 0,1,\cdots,H-1\}$}
\State \hspace{11.8em} and \smash{$k_{H + \frac{1}{2}} = C \log(1/\varepsilon) \log\log(1/\varepsilon)$} for large $C>0$. \LineLabel{alg:semiauto_main:branch_schedule}
\end{algorithmic}
\end{algorithm}

\begin{algorithm}[t]
\caption{$\subsample (D_\bx \| \calF_j, k)$}\label{alg:semi-subsample}
\begin{algorithmic}[1]
\Statex {\color{gray} \# Subsample an instance from $D_\bx$ based on how many models in $\calF_j$ mislabel the instance.}
\State \textbf{Input:} Dataset $D_\bx = \{ (\bx^i,s^i) \text{ for } \bx^i = (s^i_0,\bw^i) \}_{i=0}^{L'-1}$ where $L' \in \{ 1, L \}$,\\
\hspace{3em} Set of $j$ models $\calF_j$.
\State \textbf{Instantiate:} $\err_\star \gets \frac{1}{4}$, $\calI \gets \emptyset$.
\For{$i = 0,1,\cdots,L'-1$}
\State Draw $\eta_i \sim \unif ([0,1])$.
\If{$\eta_i \le w_j (\bx^i) / \| w_j \|_\infty$} \LineComment{{\color{blue} Defined below across \cref{eq:alpha,eq:rank}. Computing}}
\State $\calI \gets \calI \cup \{ \bx^i \}$ \NoComment{{\color{blue} $w_j (\bx^i)$ uses knowledge of $s^i = \fstar_\tau (\bx^i)$ where $\tau = |\bx^i|$.}}
\EndIf
\EndFor
\State \textbf{Return:} $\bx'$ for $\bx' \sim \unif (\calI)$ if $\calI \ne \emptyset$. Else return $\perp$.

\vspace{-0.25em}
\Statex\hrulefill
\vspace{0.25em}

\Statex For $0 \le r \le j < k$, the weight $w_j$ is defined as $w_j (\bx) = \alpha^{j,k}_{\rank_j (\bx)}$ and $\| w_j \|_\infty = \max\limits_{0 \le r \le j} \alpha^{j,k}_r$, where,
\vspace{-0.75em}
\begin{align}
    \alpha^{j,k}_r &= \beta_r^{j+1,k} - \beta_{r+1}^{j+1,k}, \text{ where, } \beta^{j,k}_r = \begin{cases}
    \bbI (r \le k/2) &\text{if } j = k \\
    \err_\star \cdot \beta_r^{j+1,k} + \left( 1 - \err_\star \right) \cdot \beta_{r+1}^{j+1,k} \quad &\text{if } j < k
\end{cases}
\label{eq:alpha}
\end{align}

\Statex $\rank_j (\bx)$ counts the number of models in $\calF_j$ which guess the terminal-state label on $\bx$ correctly: for $j \ge 0$,
\begin{equation} \label{eq:rank}
    \rank_j (\bx) \gets \big| \big\{ \pi (\bx) = \fstar_\tau (\bx) : \pi \in \calF_j \big\} \big| \in [0,j]
\end{equation}
where $\tau = |\bx|$.
\end{algorithmic}
\end{algorithm}

\begin{algorithm}[th]
\caption{$\split (\bx \| L')$ where $\bx = (s_0,\bw_{1:t})$}\label{alg:split}
\begin{algorithmic}[1]
\Statex {\color{gray} \# Split a length-$t$ instance into $L'$ instances of length $t/L'$ by querying $\etoe(\cdot)$ at boundary states.}
\State \textbf{Input:} Instance $\bx = (s_0,\bw_{1:t})$ with $t \pmod{L'} \equiv 0$, supervision oracle $\etoe(\cdot)$.
\State Let $\tau \gets t / L'$.
\For{$i = 0, \cdots, L'-1$}
\State $s_{\tau (i+1)} \gets \etoe \big( (s_{\tau i}, \bw_{\tau i+1:\tau (i+1)}) \big)$. \LineComment{{\color{blue} $L'$ calls to the $\etoe (\cdot)$ oracle in total}}
\EndFor
\State \textbf{Return:} $D_\bx = \big\{ (\bx^i,s_{\tau (i+1)}) \text{ where } \bx^i = (s_{\tau i},\, \bw_{\tau i+1:\tau (i+1)}) : i = 0,\cdots,L'-1 \big\}$
\end{algorithmic}
\end{algorithm}

\subsection{Comparison with Non-Curriculum Baselines} \label{subsec:baselines}

To interpret \cref{theorem:main-generic}, we contrast the dependence on the sequence length-$T$ achieved by $\WLdepth{}$ with that of two natural non-curriculum baselines: SFT on full sequences and end-to-end learning from final-state supervision \citep{joshi2025theory}.\footnote{We remark that these baselines are more general than $\WLdepth{}$, in the sense that they can be applied to general autoregressive models, not just semiautomata.}

\subsubsection{Vanilla SFT: Learning from Full Chains-of-Thought} \label{subsubsec:CoT}

Arguably the most natural approach to learning semiautomata is to use supervised fine-tuning (SFT) on full chains-of-thought. Concretely, in our framework, this corresponds to querying the $\etoe(\cdot)$ oracle $T$ times on each instance in the training dataset to label the full state sequence $\fstar_{1:T} (\bx^i)$. Such a labeled dataset is composed of many next-state examples of the form $(s,w) \mapsto \fstar (s,w)$, and a model $\pihat$ can be fit to this dataset via empirical risk minimization on the next-token prediction objective:
\begin{align*}
    \pihat \in \argmin_{\pi \in \calF} \sum_{i = 1}^{n} \sum_{t=1}^T \bbI \big( \pi (s^i_{t-1},w^i_t) \ne \fstar (s^i_{t-1},w^i_t) \big).
\end{align*}
where $\bx^i=(s_0^i,\bw^i_{1:T})$ and $s^i_t = \fstar_t ( \bx^i )$ for all $i \in [n]$ and $t \in [T]$, given an input dataset of $n$ instances $\{ \bx^i \}_{i=1}^n$. This is an instance of a more general class of \emph{full-CoT} learners we define as follows.

\begin{definition}[Full CoT learner] \label{def:fullCoT}
Given a dataset of instances \smash{$D = \{ \bx^i \}_{i=1}^n \overset{\text{i.i.d.}}{\sim} \rho$}, a full CoT learning algorithm $\Alg(D)$ chooses a subset of instances indexed $S \subseteq [n]$ and queries $\etoe(\cdot)$ $T$ times on each instance $\bx^i$ for $i \in S$ to observe the full state sequence $\fstar_{1:T} (\bx^i)$. The final model returned is trained on the CoT labeled dataset---for example, by using next-token prediction.
\end{definition}

\noindent As full CoT learners commit to querying full state sequences on every instance (even when potentially unnecessary), we show that any learner in this class requires linear-in-$T$ query complexity; the proof of this result is given in \Cref{proof:lower-bound-CoT}.

\begin{proposition} \label{prop:lower-bound-CoT}
Consider any full CoT learner $\Alg$ (cf. \Cref{def:fullCoT}). For any $\varepsilon,\delta \in (0,\tfrac{1}{2})$, there exists a class of next-state/token predictors $\calF$ with Natarajan dimension $d$ and a target distribution $\rho$ such that for $\Alg$ to return a model $\fhat$ such that $\Pr_{\bx \sim \rho} \big[ \fhat (\bx) \ne \fstar_T (\bx) \big] \le \frac{1}{4}$ with probability at least $\frac{1}{2}$, its sample complexity must satisfy $\nsample \ge c d$ for a universal constant $c > 0$. Furthermore, this implies that for any such approach, the query complexity must satisfy $\nquery \ge c dT$ to achieve the same guarantee.
\end{proposition}

\noindent 
In Part I \citep{rajaraman2026learning}, we showed that for the goal of achieving high accuracy ($1-\varepsilon$), it is possible to improve over vanilla SFT by adaptively choosing which instances to gather full CoTs for. This leads to query complexity $\calOtilde (d T / \varepsilon)$ to $\calOtilde \big( d / \varepsilon + dT \big)$, pushing the dependency on $T$ into a burn-in term. Ultimately though, each CoT requires making $\Omega(T)$ queries to the $\etoe(\cdot)$ oracle, so the query complexity is still linear in $T$.\loose

\subsubsection{End-to-End Feedback} \label{subsubsec:e2e}

Since the final objective in our setting is to accurately predict the final state $\fstar_T(\bx^i)$, another natural learning approach is to ignore the intermediate CoT altogether, and label each instance in the training dataset only with the final state $\fstar_T(\bx^i)$. One can then train a model which directly classifies the final state based on the instance. We formally define this class of \emph{end-to-end} learners as follows.
\begin{definition}[End-to-end learner] \label{def:e2e}
Given a dataset of instances \smash{$D = \{ \bx^i \}_{i=1}^n \overset{\text{i.i.d.}}{\sim} \rho$}, an end-to-end learning algorithm $\Alg(D)$ queries $\etoe(\cdot)$ once per instance to label the terminal state $\fstar_T (\bx^i)$. A model is then trained on this dataset. For example, by finding a model \smash{$\fhat$} consistent with all terminal labels---that is, \smash{$\fhat_T (\bx^i) = \fstar_T(\bx^i)$} for all $i \in [n]$---and predicting the terminal state on new instances using $\fhat_T$. This example corresponds to learning a model by empirical risk minimizer over the induced ``end-to-end'' class $\calF_T$ defined as $\calF_T = \{ \pi_T : \pi \in \calF \}$ \citep{joshi2025theory}.
\end{definition}

\noindent While this approach seems potentially appealing since it requires only one $\etoe(\cdot)$ query per instance, the absence of intermediate CoT grounding makes learning from such feedback more challenging statistically: \citet{joshi2025theory} show that any end-to-end learner requires a number of examples linear in $T$.
\begin{proposition}[Corollary of Theorem E.1 of \cite{joshi2025theory}] \label{prop:lower-bound-e2e}
Consider any end-to-end learning algorithm $\Alg$ (cf. \Cref{def:e2e}). For any $\varepsilon,\delta \in (0,1/2)$, there exists a class of next-token predictors $\calF$ with Natarajan dimension $d$, and a target distribution $\rho$ over instances, such that for $\Alg$ to return a model $\fhat$ such that $\Pr_{\bx \sim \rho} \big[ \fhat (\bx) \ne \fstar_T (\bx) \big] \le \varepsilon$ with probability at least $1-\delta$, $\Alg$ must incur sample complexity $\nsample \ge \Omega \big( \big( dT +\log(1/\delta) \big) \cdot \varepsilon^{-1} \big)$.
\end{proposition}

\noindent Ignoring statistical considerations, for simple models such as iterated linear classifiers, learning from end-to-end feedback is computationally hard under standard complexity theoretic assumptions, even when next-token prediction is computationally efficient \citep[Theorem 4.4]{joshi2025theory}.

\subsection{Applications of \Cref{theorem:main-generic}}
\label{sec:applications}

In this section we instantiate \Cref{theorem:main-generic} for the specific classes of semiautomata previously discussed in \Cref{sec:prelim}: Regular languages and linear recurrences over finite fields.

\paragraph{Regular languages.} As discussed earlier, the class of regular languages induced by regular expressions with alphabetic length at most $k$ induces a natural class of semiautomata. Specifically, the induced semiautomaton is always initialized at the state $s_0=0$, and input words are strings in $\Sigma^T$. The terminal state of the semiautomaton can be used to decide whether the input word is a member of the regular language (i.e., whether it is generated by the corresponding regular expression).

\medskip
\noindent In more detail, let $\RE_k (\Sigma)$ denote the set of regular expressions over alphabet $\Sigma$ with alphabetic length at most $k$. The number of syntactically distinct languages realized by $\RE_k (\Sigma)$ is at most $|\Sigma|^{\Theta(k)}$~\citep[Table 8]{lee2004enumerating}. The DFA for each regular language can be obtained by determinizing an NFA on $k+1$ states (Glushkov's construction) via the powerset construction.  Thus, $\RE_k (\Sigma)$ corresponds to a class of semiautomata $\calF_{\RE_k}$ on $|S| = 2^{\Theta(k)}$ states, and the size of the class itself is $|\calF_{\RE_k}| = |\Sigma|^{\Theta(k)}$. See \Cref{sec:regex-example} for a detailed example. \loose

\medskip
\noindent In this setting, the $\etoe (\cdot)$ oracle identifies the set of states which are reachable at time $t$. By divide-and-conquer, it can be implemented in parallel time $\poly(k, \log(T))$ (and with $T \cdot \poly(k)$ work). Furthermore, when \Cref{theorem:main-generic} is specialized to this setting, the following guarantees on sample complexity and query complexity are established for $\WLdepth{}$ (\Cref{alg:semiauto_highacc}),
\begin{align*}
    \nsample, \nquery \le 2^{\calOtilde(\sqrt{\log(T)})} \cdot \frac{k \log (|\Sigma|) \log^2 (1/\varepsilon) \log(1/\delta)}{\varepsilon},
\end{align*}
where the target failure probability is $\delta$, error rate is $\varepsilon$. Notably, this guarantee only scales linearly in the bound on the alphabetic length, $k$, even though the underlying (semi-)automata are on $2^{\Theta(k)}$ states.

\paragraph{Linear recurrences over finite fields.} Linear recurrences over finite fields correspond to semiautomata with state space $S = \bbF_q^d$, input alphabet $\Sigma = \bbF_q^d$, and transition function $\pi(s,w) = As + Bw$ for unknown matrices $A \in \bbF_q^{d \times d}$ and $B \in \bbF_q^{d \times d}$. Simulating for $T$ steps computes the linear recurrence $s_T = A^T s_0 + \sum_{t=1}^T A^{T-t} B w_t$. The class $\calF$ of all such linear transitions is parameterized by the pair $(A,B)$, giving $|\calF| = q^{2d^2}$, and Natarajan dimension $\Ndim (\calF) \le \calO ( d^2 \log q)$. Specializing \Cref{theorem:main-generic} to this setting results in the following sample complexity and query complexity upper bounds for $\WLdepth{}$:
\begin{align*}
    \nsample, \nquery \le 2^{\calOtilde(\sqrt{\log(T)})} \cdot \frac{d^2 \log (q) \log^2 (1/\varepsilon) \log(1/\delta)}{\varepsilon}.
\end{align*}
Here, the target failure probability is $\delta$, error rate is $\varepsilon$. This guarantee only scales polynomially in the dimension $d$, even though the underlying semiautomaton has $|S| = q^d$ states. In this setting, the $\etoe(\cdot)$ oracle again has a natural interpretation as before, computing the state after $t$ steps of recursion, which by divide-and-conquer can be computed in parallel in time $\poly(d,\log(T))$ with $T \cdot \poly(d)$ work.

\section{RLVR: Fine-Tuning a Weak Reference Model} \label{sec:rl}

This section adapts the compositional curriculum from \Cref{sec:semiauto} to the RLVR setting, where the learner aims to improve a weak reference model $\piref$ using an outcome verifier (\Cref{def:verifier}). We show that the compositional curriculum reduces the coverage requirement on $\piref$ from the full sequence length $T$ to the shorter block scale $B$, an exponentially weaker condition. We begin by motivating the need for a compositional curriculum through a discussion of naive RLVR approaches (\cref{subsec:Vlearn}), then describe the algorithm $\WLdepth{}.\RL$ (\Cref{sec:overview-rl}), and finally present the main guarantee (\Cref{subsec:main-rl}).

\subsection{Motivation: Naive RLVR Methods and the Need for Coverage} \label{subsec:Vlearn}
\noindent Given a pre-trained model $\piref$ and an outcome verifier $\calV$, the basic principle underlying standard RLVR fine-tuning methods (e.g., GRPO~\citep{shao2024deepseekmath}) is to draw rollouts from $\piref$, filter by outcome correctness, and train on the survivors. \rlfinetune (\Cref{alg:Vlearn}; also used in Part I \citep{rajaraman2026learning}) provides a simple, theoretically tractable instantiation of this principle. Given a distribution over problem instances, it repeatedly generates candidate CoTs from $\piref$, keeps those whose terminal state is verified as correct by $\calV$, and trains a model on the filtered dataset. Under a sequence-level coverage assumption \cref{eq:cov}, each rollout produces a fully correct CoT with probability at least $(\CseqT)^{-1}$, and altogether roughly
\begin{equation}
    \calOtilde \left( \frac{d\CseqT}{\varepsilon} \right)
\end{equation}
prompts, verifier queries, and training time suffice to find an $\veps$-optimal model. The obvious limitation of this approach is that it requires the reference model to have nontrivial coverage over instances of length $T$. If the model $\piref$  only satisfies coverage at length $B \ll T$ (i.e., it is not strong enough to solve length-$T$ instances directly), one can at best hope for a bound of the form $\CseqT \leq (\Cblock)^{T/B}$, which scales exponentially in the number of blocks.\footnote{For unstructured problems, this is the best possible bound for any algorithm, not just \rlfinetune; see \citet{foster2025good}.} 

\begin{algorithm}[t]
    \caption{$\Vlearn \big( \Dprompt{} \ \| \ B,\varepsilon, \delta \big)$} \label{alg:Vlearn}
    \begin{algorithmic}[1]
    \Statex {\color{gray} \# Sharpening a reference model $\piref$ satisfying sequence-level coverage using verifier feedback.}
    \State \textbf{Input:} Class of semiautomaton transitions $\calF$ over state space $S$ and alphabet $\Sigma$,
    \Statex \hspace{3em} Number of steps of semiautomaton simulation $B$,
    \Statex \hspace{3em} \smash{Reference model $\piref$ satisfying \cref{eq:cov} with parameters $(B,\Cblock,\Cout)$},
    \Statex \hspace{3em} Outcome verifier $\calV$,
    \Statex \hspace{3em} Target error rate $\varepsilon$ and failure probability $\delta$,
    \Statex \hspace{3em} Dataset of instances $\Dprompt{} = \{ \bx^i = (s^i_0,\bw^i) \}_{i=1}^{n}$.
    \State \textbf{Instantiate:} $m \gets \Cseq^{(B)} \log (4n \Cseq^{(B)} / \delta)$.
    \For{$\bx^i \in \Dprompt{}$}
        \State Draw CoTs, $\by^{i,j} \sim \piref_{1:B} (\, \cdot\, |\, \bx^i)$ for $j = 1, \cdots, m$.
        \State Let $D_i$ denote the set of deduplicated CoTs among $\{ \by^{i,j} \}_{j=1}^m$.
        \State Discard low-probability CoTs which lead to incorrect answers:
        \begin{equation*}
            \widetilde{D}_i \gets \big\{ \by \in D_i : \calV ( \bx^i, y_B) = 1 \text{ and } \piref_{1:B} (\, \by \,|\, \bx^i) \ge \big( \Cseq^{(B)} \big)^{-1} \big\}.
        \end{equation*}
    \EndFor
    \State \textbf{Return:} $\pihat \in \underset{\pi \in \calF}{\operatorname{argmin}} \ \calL (\pi ; \Dprompt{})$, where $\calL (\pi ; \Dprompt{}) = \frac{1}{n} \sum_{i=1}^n \mathbb{I} \big( \pi_{1:B} (\bx^i) \not\in \widetilde{D}_i \big)$. \label{alg:Vlearn:7}
    \end{algorithmic}
\end{algorithm}

\subsection{Overview of \mainalgRL}
\label{sec:overview-rl}

To overcome the coverage requirement for naive RLVR, we adopt a compositional curriculum that decomposes length-$T$ problem instances into length-$B$ instances, the scale at which $\piref$ is reliable. Our algorithm, $\WLdepth{}.\RL$ (\Cref{alg:semiauto_RL}), uses $\Vlearn$ as an inner-loop primitive to improve $\piref$ on length-$B$ instances, wrapped in an outer loop that constructs the curriculum using an adapted version of \Cref{alg:semiauto_main}. As in our previous results, we critically use that the semiautomaton is time-homogeneous: the same transition map governs every step, so any length-$B$ block of a length-$T$ instance---beginning at an intermediate state and treating it as a fresh start state---is itself a valid length-$B$ instance. Consequently, the outcome verifier $\calV$ can be applied to check attempts at solving each such block.\loose

\medskip\noindent
In more detail, $\WLdepth{}.\RL$ (\Cref{alg:semiauto_RL}) is an invocation of $\WLdepth{H}$ (\Cref{alg:semiauto_main}), but uses a simplified decomposition and branching schedule with $H=1$ (in other words, the algorithm considers instances of length $T$ at the top level and length $B$ at the bottom level, with no intermediate levels):
\begin{itemize}
\item At the top level, we set $L_1=1$ and $k_1=\Theta(\log(1/\varepsilon))$, which corresponds to boosting from constant accuracy at length-$T$ to accuracy $1-\varepsilon$ at length-$T$.
\item At level $h=\frac{1}{2}$, we set $L_{\frac{1}{2}}=T/B$ and $k_{\frac{1}{2}}=\Theta(\log(T/B))$, which corresponds to aggregating and composing short-range models on length-$B$ blocks with constant accuracy into a model solving length-$T$ instances which also achieves constant accuracy.
\item At the bottom level, we set $L_0=1$ and $k_0=1$, which corresponds to training a length-$B$ model to constant accuracy; for this, we use $\Base\equiv \Vlearn$ using feedback from the outcome verifier $\calV$.
\end{itemize} 
Similar to its counterpart in the iSFT setting, $\WLdepth{}.\RL$ can be viewed as an outer-loop which constructs a curriculum of independent length-$B$ instances, wrapped around an inner loop which improves the reference model $\piref$ on these length-$B$ instances via RL. Our analysis makes use of the following guarantee for $\Vlearn$, which asserts that it can improve $\piref$ to constant accuracy on length-$B$, but any algorithm with a similar guarantee can be used in its place. In this sense, \mainalgRL can be viewed as a reduction to RLVR fine-tuning over length-$B$ instances.
\begin{proposition}[Guarantee for $\Vlearn$; adapted from~\citealp{rajaraman2026learning}] \label{prop:Vlearn-B}
    Fix $\delta \in (0,1)$, and suppose the reference model $\piref$ satisfies sequence-level coverage with parameter $\Cblock$ at length $B$ (\Cref{def:ref-model}). There is an absolute constant $C > 0$ such that when $\Vlearn (\, \cdot \,\|\, B,\varepsilon,\delta)$ is run on a dataset of size
    \begin{equation*}
        \nsample \ge \frac{C \big( d \log \big( B |S| \Cblock \big) \log(1/\varepsilon) + \log(1/\delta) \big)}{\varepsilon}
    \end{equation*}
    drawn from some target distribution $\rhoB$ over length-$B$ instances, the resulting model $\fhat : S \times \Sigma^B \to S$ satisfies \smash{$\Pr_{\bx \sim \rhoB} \big[ \fhat (\bx) \ne \fstar_B (\bx) \big] \le \varepsilon$} with probability at least $1-\delta$. Here $d$ is the Natarajan dimension of $\calF$. Furthermore, the number of queries $\Vlearn$ makes to the outcome verifier $\calV$ is at most, \smash{$\nquery \le \calOtilde \big( \frac{d \Cblock}{\varepsilon} \big)$}, and the computational cost of $\Vlearn$ is at most, \smash{$\ncomp \le \calOtilde \big( \frac{d B \Cblock}{\varepsilon} \big)$}, where \smash{$\calOtilde(\cdot)$} hides logarithmic factors in \smash{$|S|, B, \Cblock$, $\varepsilon$ and $1/\delta$}. Moreover, the verifier is queried only to check predictions for the terminal state of length-$B$ instances. Here the computational cost is the number of states generated from $\piref$ or any other model in $\calF$ over the course of training (cf. \Cref{remark:cost}).
    \end{proposition}
\noindent 

\begin{algorithm}[t]
    \caption{$\WLdepth{}.\RL \big( \Dprompt{} \ \| \ T, \varepsilon, \delta \big)$} \label{alg:semiauto_RL}
    \begin{algorithmic}[1]
    \Statex {\color{gray} \# Learning semiautomata from verifier feedback using compositional curriculum.}
    \State \textbf{Input:} Class of semiautomaton transitions $\calF$ over state space $S$ and alphabet $\Sigma$,
    \Statex \hspace{3em} Reference model $\piref$ satisfying \cref{eq:cov} with parameters $(B,\Cblock,\Cout)$,
    \Statex \hspace{3em} Outcome verifier $\calV$,
    \Statex \hspace{3em} Number of steps of semiautomaton simulation, $T$, with $T \pmod{B} \equiv 0$,
    \Statex \hspace{3em} Target error rate $\varepsilon$ and failure probability $\delta$,
    \Statex \hspace{3em} Dataset of instances $\Dprompt{} = \{ \bx^i = (s_0^i,\bw^i) \}_{i=1}^{n}$.
    \State \textbf{Return:} $\fhat \gets \WLdepth{1} \big( \Dprompt{} \ \| \ T, \delta \big)$, with parameters:
    \Statex \hspace{5em} Feedback oracle: $\etoe(\cdot)\gets\texttt{GuessAndCheck}(\, \cdot \, \| T, B, \delta_1)$ with $\delta_1 = \delta / \nsample^{\star,\texttt{RL}}$ \Statex \LineComment{{\color{blue} $\texttt{GuessAndCheck}(\cdot)$ defined in \Cref{app:guess-and-check}.}}
    \Statex \NoComment{{\color{blue} $\nsample^{\star,\texttt{RL}}$ (cf. \Cref{theorem:main-RL}) is the sample complexity upper boundof $\WLdepth{}.\RL$.}}
    \Statex \hspace{5em} Weak learner: $\Base (\, \cdot \,\|\, \delta') \gets \Vlearn (\, \cdot\, \|\, B, \varepsilon_0, \delta' )$ where \smash{$\varepsilon_0^{-1} = C k_{1/2}^{7/2} \log\big(L_{\frac{1}{2}} \big)$}
    \Statex \hspace{5em} Base failure probability $\delta_0 = \delta/4k_{\frac12}k_1$,
    \Statex \hspace{5em} \textit{Decomposition schedule:} \smash{$\big( L_{\frac{1}{2}}, L_1 \big) = \big( \frac{T}{B}, 1 \big)$},
    \Statex \hspace{5em} \textit{Branching schedule:} $\big( k_{\frac{1}{2}} , k_1 \big) = \big( C\log(T/B), C \log (1/\varepsilon) \big)$ for a large constant $C>0$
    \Statex \LineComment{{\color{blue} $\Vlearn$ improves $\piref$ on length-$B$ instances using outcome verifier $\calV$.}}
    \end{algorithmic}
    \end{algorithm}

\paragraph{Simulating intermediate state labels via guess-and-check.}
\label{subsec:guess-and-check}
To merge models learned at the length-$B$ scale into a model at the length-$T$ scale, \mainalgRL uses the $\split (\, \cdot \,\|\, L)$ procedure from \Cref{alg:semiauto_main}, which requires intermediate state labels at the length-$B$ scale. Naively such feedback is not available to the learner in the RLVR setting, but the reference model $\piref$ and outcome verifier $\calV$ can be used to simulate this feedback using a guess-and-check approach, which we refer to as \guessandcheck. 

\medskip
\noindent \paragraph{\guessandcheck.} Given a length-$T$ instance $\bx$, a reference model which satisfies a bound on $\Cout$, outcome coverage at the length-$B$ scale, in conjunction with the outcome verifier can be used to split the instance into $T/B$ blocks of length $B$ using a guess-and-check approach. Fix a prompt $\bx$. To label the terminal state of the first block, i.e., $\fstar_B (\bx)$, we generate \smash{$\calOtilde (\Cout)$} rollouts from the reference model. With high probability, at least one rollout has a correct state at $t=B$, and we can identify it with the verifier. We then proceed to the next block, interpreting $\fstar_B(\bx)$ as a start state, and repeating the process to identify $\fstar_{2B}(\bx)$, and so on. We discuss this in more detail in \Cref{app:guess-and-check}.

\medskip
\noindent
For a given prompt, this procedure recovers the terminal states for all blocks using total query complexity $\calOtilde(T/B \cdot \Cout)$. A crucial detail, and the source for most of the technical effort behind the algorithm, is that this process does not identify correct intermediate states within each block, only the final state.

\subsection{Main Result: Coverage Expansion via Compositional Curriculum} \label{subsec:main-rl}

Our main result shows that \mainalgRL achieves high accuracy on length-$T$ instances whenever the reference model $\piref$ can solve length-$B$ instances with constant probability, a form of coverage expansion.

\begin{theorem}[Main guarantee for \mainalgRL] \label{theorem:main-RL}
Fix $\delta \in (0,1)$, block length $B$, and assume that $T\pmod B \equiv 0$ without loss of generality. Suppose the reference model $\piref$ satisfies \cref{eq:cov} with parameters $(B, \Cblock, \Cout)$, and suppose the class $\calF$ has Natarajan dimension $d$ (\Cref{def:Ndim}). Then, $\WLdepth{}.\RL$ (\Cref{alg:semiauto_RL}) returns a model $\fhat$ that satisfies \smash{$\Pr_{\bx \sim \rho} \big[ \fhat(\bx) \ne \fstar_T(\bx) \big] \le \varepsilon$} with probability at least $1-\delta$. Moreover:
\begin{enumerate}
    \item The sample complexity is bounded by, $\nsample \le \nsample^{\star,\texttt{RL}} = \calOtilde \left( \frac{d}{\varepsilon} \right)$.
    \item The total computational cost is bounded by $\ncomp \le \calOtilde \left( \frac{dT}{\varepsilon} \right) + \calOtilde \left( dT \Cout + d B \Cblock \right)$, as measured by the number of states/tokens generated.\footnote{A comment on units: the computational cost $\ncomp$ in \Cref{theorem:main-RL} counts the total number of individual \emph{state generations} produced by $\piref$ or by any model in $\calF$ over the course of training, rather than the number of length-$T$ sequences rolled out; the latter is used as the measure of computational cost in Part I~\citep{rajaraman2026learning}. This fact accounts for the $T$ dependency in the bounds on $\ncomp$: to even predict the terminal state on a single length-$T$ instance requires generating $T$ states.}
    \item The query complexity is upper bounded by $\nquery \le \calOtilde \left( \frac{d}{\varepsilon} \right) + \calOtilde \left( \frac{T}{B} \cdot d \Cout + d \Cblock \right)$.
\end{enumerate}
Above, $\calOtilde(\cdot)$ hides polylogarithmic factors in $T, |S|, \Cblock$ and $ \delta^{-1}$.
\end{theorem}

\noindent The proof of \cref{theorem:main-RL} is deferred to \Cref{subsec:main-RL-proof}. The main features of the result are that: (1) All costs only scale with the model's coverage at the length-$B$ scale, representing a form of \emph{coverage expansion} from the length-$B$ scale to the length-$T$ scale; and (2) the dependence on the model's coverage---both for statistical and computational costs---is decoupled from the target accuracy $1-\varepsilon$, only appearing in a ``burn-in'' cost. The latter is a consequence of the boosting effect of the self-generated curriculum, and has a similar flavor to our RL fine-tuning results in Part I~\citep{rajaraman2026learning}. In what follows, we interpret the sample complexity, computational cost, and query complexity in more detail.\loose

\paragraph{Sample complexity.} The number of prompts required by $\WLdepth{}.\RL$ scales as $\nsample \le \calOtilde(d/\varepsilon)$. Up to logarithmic factors, this is optimal for PAC learning the class $\calF$ to accuracy $1-\varepsilon$, and is required as soon as \smash{$\Cout > 2$}, even if $T=1$.

\paragraph{Computational cost.} The computational cost of $\WLdepth{}.\RL$---measured in terms of number of generated states---decomposes into two parts: a ``leading order'' term, which is polynomial in $\veps^{-1}$ but nearly independent of the coverage coefficients, and a ``burn-in'' term, which is polynomial in the coverage coefficients, but nearly-independent of $\veps^{-1}$:
\begin{equation} \label{eq:compcost}
    \ncomp \le \calOtilde \underbrace{\left( \frac{dT}{\varepsilon} \right)}_{\text{leading order}}+ \underbrace{\calOtilde \big( dT \Cout + d B \Cblock \big)}_{\text{burn-in cost $\nburnin$}}.
\end{equation}
The leading-order term $\calOtilde(dT/\varepsilon)$ reflects the number of state generations from $\piref$ needed to label the terminal state on each of the $\calOtilde (d/\varepsilon)$ instances in the training dataset; we expect this to be irreducible. The burn-in term depends on the pre-trained model's coverage, but critically, does so only at the length-$B$ scale. The term involving the (smaller) answer-level coverage coefficient $\Cout$ is scaled by $T$, while the term involving the (larger) block-level coverage coefficient $\Cblock$ is scaled only by the block length $B$. In more detail:\loose
\begin{itemize}[leftmargin=1.5em,itemsep=1pt, topsep=3pt, parsep=5pt]
    \item \emph{Labeling cost.}~~The first term, $\calOtilde(dT \Cout)$, can be interpreted as the computational cost required to label a deliberately chosen set of $\calOtilde (d)$ length-$T$ instances $\bx$ with ground-truth outcomes for each length-$B$ block using the guess-and-check strategy described in \Cref{subsec:guess-and-check}. In particular, each $\bx$ is decomposed into $T/B$ segments, and $\piref$ generates $\calOtilde(\Cout)$ rollouts per segment, with overall cost \smash{$\calOtilde ( d \times T/B \times \Cout \times B)$}, where the final factor of $B$ arises because $\piref$ generates $B$ states for each segment rollout.\loose
    \item \emph{Training cost.}~~The second term, $\calOtilde(d B \Cblock)$, can be interpreted as the computational cost of using the base learner $\Vlearn$ to improve $\piref$ to constant accuracy on length-$B$ instances, as in \Cref{prop:Vlearn-B}. This term can therefore be interpreted as the computational cost emerging from \emph{training} within $\WLdepth{}.\RL$, once again focusing on the cost of generation, and not model updates (cf. \Cref{remark:cost}).
\end{itemize}

\paragraph{Query complexity.} The number of queries to the outcome verifier made by $\WLdepth{}.\RL$ has a similar decomposition to the computational cost: A leading-order term scaling as \smash{$\calOtilde(d/\varepsilon)$}, and a burn-in term which inherits all dependence on the coverage coefficients $\Cblock$ and $\Cout$:
\begin{equation*}
    \nquery \le \calOtilde \left( \frac{d}{\varepsilon} \right) + \underbrace{\calOtilde \left( \frac{T}{B} \cdot d \Cout + d \Cblock \right)}_{\text{burn-in cost}},
\end{equation*}
Here: (1) The leading-order term reflects the number of queries needed to verify length-$T$ terminal state correctly on each of the $\calOtilde (d/\varepsilon)$ instances in the training dataset; (2) the burn-in term $\calOtilde\left( \frac{T}{B} \cdot d \Cout \right)$ reflects the number of queries needed to label all of the block-level terminal states for a deliberately chosen subset of $\calOtilde (d)$ instances; (3) the burn-in term $\calOtilde(d \Cblock)$ reflects the number of queries needed for $\Vlearn$ to improve $\piref$ to constant accuracy on length-$B$ instances.

\subsubsection{Comparison to Part I}
In Part I \citep{rajaraman2026learning}, we showed that in the general autoregressive setting, a simpler form of self-generated curriculum with \emph{no length curriculum} (i.e., purely selecting length-$T$ prompts from the training set adaptively)
can be used to improve a reference model $\piref$ to accuracy $1-\varepsilon$ with computational cost $\calOtilde \big( dT/\varepsilon + dT C_{\texttt{seq}}^{(T)} \big)$. Our results above recover this guarantee by setting $B=T$, but---as discussed in the previous section---can be exponentially more efficient if $\piref$ only has constant coverage for blocks of length $B\ll T$.\loose

\section{\mbox{Overview of Analysis Techniques: Recursion and Inverted Sampling}} \label{sec:techniques}

Having presented the high-level principles behind our algorithm \mainalg in \cref{sec:semiauto,sec:rl}, this section gives an in-depth overview of the multiscale boosting technique used within the algorithm. To simplify presentation, we restrict our attention to the iSFT setting (\cref{alg:semiauto_main}), and focus on the regime where the target error $\varepsilon = \frac{1}{4}$ is a constant, as this already captures the main technical challenges, and will use $d = \comp (\calF)$ to refer to the complexity of the model class (in the sense of \cref{eq:nweak-RL}). The section proceeds in three parts.\loose

\begin{itemize}[leftmargin=1.5em,itemsep=2pt,topsep=3pt,parsep=3pt]
    \item \textit{Warm-up via composition and boosting} (\Cref{sec:warmup}). As a warmup, we first show that a single layer of boosting and composition is sufficient to achieve query \emph{and} sample complexity \smash{$\calOtilde(d\sqrt{T})$}. This 
    already improves upon the linear-in-T \smash{$\calOtilde(dT)$} full-CoT baseline in \Cref{subsec:baselines}, but is nonetheless still polynomial in $T$. We then highlight a natural path to improving this guarantee further (\cref{subsec:path})---improving the sample complexity required by the \emph{rejection sampling} step within boosting.
    \item \textit{The reweighted-mixture-of-marginals problem and inverted sampling} (\Cref{sec:sampling}). 
    We isolate an abstract problem, \emph{sampling from a reweighted mixture of marginals}, which captures the key structure of the sampling problem used within our boosting approach (\Cref{sec:sampling}), and give a new algorithm, \emph{inverted sampling}, which solves it significantly more efficiently than naive rejection sampling (\cref{subsec:depth-1}).    

    \item \textit{Putting everything together} (\Cref{subsec:queryinverted,subsec:depth-H}). 
    We show that using inverted sampling within a single layer of recursion brings the sample complexity down to $\calOtilde(d)$ while maintaining query complexity scaling as $\calOtilde (d\sqrt{T})$ (\cref{subsec:queryinverted}), then show that an iterated version of this approach---recursing to depth $H = \calO(\sqrt{\log T})$---yields \smash{$2^{\calOtilde(\sqrt{\log T})} \cdot d$} sample and query complexity (\cref{subsec:depth-H}).
\end{itemize}

\noindent We begin by introducing notation used throughout the section.

\paragraph{Notation.} Recall that $\rho$ denotes the target distribution over length-$T$ instances. For $1 \le t_1 < t_2 \le T$, we let $\rho_{t_1:t_2}$ denote the law of $(s_{t_1-1},\bw_{t_1 : t_2})$, where $(s_0,\bw_{1:T}) = \bx \sim \rho$ and \smash{$s_{t_1-1} = \fstar_{t_1-1} (\bx)$}. $\rho_{t_1,t_2}$ can be viewed as a distribution over instances of length $t_2 - t_1 + 1$. For some $\tau < T$ define the uniform mixture of length-$\tau$ marginals obtained from $\rho$ as follows,
\begin{equation} \label{eq:rhobar}
    \rhobar \ = \ \frac{1}{L} \sum_{i=0}^{L-1} \rho_{\tau i + 1 : \tau (i+1)}.
\end{equation}
This is the distribution over instances by breaking down length-$T$ instances drawn from $\rho$ into segments of length $\tau$ and selecting one uniformly at random (where the start state of the segment is labeled by $\fstar$).

\subsection{A Single Level of Recursion: Achieving $\nsample=\calOtilde(d\sqrt{T})$ and $\nquery=\calOtilde(d\sqrt{T})$} \label{sec:warmup}

We first recap the idea of \emph{composition} in \mainalg. For a model $\fhat : S \times \Sigma^\tau \to S$ on length-$\tau$ instances, we let $\fhat^{\circ L}$ denote its $L$-fold composition (\Cref{def:composition}), which operates on length-$T$ instances by chunking them into $L = T/\tau$ chunks of length $\tau$. The composition framework (\Cref{obs:lg}) asserts that whenever $\fhat$ has error at most $\varepsilon_0 = \frac{1}{4L}$ under the distribution $\rhobar$, the composed model $\fhat^{\circ L}$ has error at most $\frac{1}{4}$ under the original distribution $\rho$. Thus, compositional generalization gives a mechanism by which length can be traded off for accuracy at the length-$\tau$ scale. However, this tradeoff is not sufficient on its own to achieve sublinear-in-$T$ query complexity: a single model $\fhat$ learned to error $\varepsilon_0 = \frac{1}{4L}$ under $\rhobar$ by querying full CoTs on length-$\tau$ instances has query complexity scaling as $\calO \big( d \tau / \varepsilon_0 \big) = \calO(dT)$, matching the full-CoT baseline of \Cref{subsubsec:CoT}. 

\medskip\noindent
Breaking the $\calO(dT)$ barrier requires combining composition with \emph{boosting}, building on Part I~\citep{rajaraman2026learning}. There, in the general autoregressive setting (without composition), we showed that boosting can be used to decide which instances to query \emph{full CoTs} on, adaptively focusing on harder instances as determined by outcome correctness of the current ensemble of models. Given target error $\veps>0$, boosting trains $\calO(\log(1/\varepsilon))$ weak models to constant accuracy on an adaptively chosen sequence of distributions, each of which requires only $\calOtilde(d)$ CoTs. The total query complexity required to train a length-$T$ model in our iSFT setting is $\calOtilde(dT + d/\varepsilon)$, where the former term accounts for the query complexity of training the weak learners (recall that one full CoT requires $T$ iSFT queries), and the latter term accounts for the query complexity required to evaluate the weak learners and update the training distribution. Furthermore, the sample complexity is $\calOtilde(d)$ (more general, $\calOtilde(d/\varepsilon)$ for error $\veps$).

\medskip
\noindent Combining boosting with composition gives the following template:
\begin{enumerate}
\item For a parameter $\tau\ll T$ and $L=T/\tau$, use boosting to train a model $\fhat$ to accuracy $\veps_0 = 1/4L$ on length-$\tau$ sequences drawn from $\rhobar$.
\item Return the $L$-fold composed model $\fhat^{\circ L}$, giving error $\veps=\veps_0L = \frac{1}{4}$ as desired.
\end{enumerate}
Implementing boosting in this fashion requires rejection sampling from reweighted versions of the distribution $\rhobar$---a non-trivial detail, as we will highlight below. After accounting for this, the total query complexity for this idealized template is:
\[
\nquery = \calOtilde\prn*{d\tau + d/\veps_0}
= \calOtilde\prn*{d\tau + dL}.
\]
Recalling that $L=T/\tau$ and choosing $\tau=\sqrt{T}$ to balance the two terms yields $\nquery = \calOtilde(d\sqrt{T})$ query complexity, improving over the $\calOtilde(dT)$ query complexity of the full CoT baseline. Furthermore, the sample complexity is $\nsample=\calOtilde(d/\veps_0)=\calOtilde(dL)=\calOtilde(d\sqrt{T})$.

\subsection{Improving Rejection Sampling: A Path to Sub-Polynomial Query Complexity}
\label{subsec:path}

The full version of \mainalg (\cref{alg:semiauto_main}) uses multiple levels of recursion, solving the length-$\tau$ blocks above by breaking them into smaller blocks, and so on. The main bottleneck in making this recursion fruitful is the \emph{sample complexity} in the single-level algorithm in the prequel, which scales as $\calOtilde(d/\veps_0)=\calOtilde(dL)$. To ensure that errors do not compound as we recurse, it is essential to improve this guarantee so that it scales as $\calOtilde(d)$---i.e., nearly independent of the number of blocks $L$. We will return to the full derivation in \cref{subsec:depth-H}, but for now we set our sights on improving the sample complexity to $\calOtilde(d)$.

\medskip
\noindent Showing that the sample complexity can be improved to be strictly better than $\calO(d/\varepsilon_0)$ is essential for the recursion to be fruitful. With multiple rounds of recursion, instances are decomposed into shorter ones, following the schedule $T \to T/L \to \cdots \to T/L^H$. At each depth, a model with error $\frac{1}{4}$ is desired; this is obtained by taking the $L$-fold composition of a model achieving error $\varepsilon_0 = \frac{1}{4L}$ under the appropriate mixture distribution of instances (akin to $\rhobar$). This model itself is trained by boosting, by training and aggregating an ensemble of models each of which achieves error $\frac{1}{4}$ across appropriately reweighted versions of the mixture distribution, and so on. As the depth increases, the length of instances decreases geometrically and at the deepest level, models are trained to solve instances of length $T/L^H$. As we discuss in the next paragraph, generating samples from the appropriate reweighted mixture distributions at the deepest level via rejection sampling will require drawing $d/\varepsilon_0^H$ instances at the top level (i.e., from $\rho$). Since the length of instances at the bottom level is $\calO ( T / L^H ) = \calO ( T \varepsilon_0^H )$, all in all, this gives query complexity,
\begin{equation*}
    \nquery = \calOtilde ( d T \varepsilon_0^H + d / \varepsilon_0^H )
\end{equation*}
And optimizing over $\varepsilon_0$ and $H$ still results in a guarantee that is no better than $\widetilde{\Theta} ( d \sqrt{T} )$.

\paragraph{Key bottleneck: Rejection sampling.}

To highlight why the sample complexity for the naive approach in the prequel scales as $\nsample=\calOtilde(d/\veps_0)=\calOtilde(dL)$, let us dig into the details of the boosting stage. Formally, boosting proceeds in $k=\calOtilde(\log\veps_0^{-1})=\calOtilde(\log L)$ iterations. At iteration $j$, given weak models $\fhat^0,\ldots,\fhat^{j-1}$, we train a new model $\fhat^j$ on a distribution $\rhobar_j$, which reweights $\rhobar$ to focus on instances where the current ensemble \smash{$\big( \fhat^0,\cdots,\fhat^{j-1} \big)$} errs; $\calOtilde(d)$ samples from this distribution are required to ensure that $\fhat^j$ achieves constant accuracy. After all iterations conclude, we aggregate the models through a majority vote to produce a final length-$\tau$ model
\[\fhat = \maj \big( \{ \fhat^0,\cdots,\fhat^{k-1} \} \big).\] 
As discussed above, this model has error $\veps_0=1/4L$ on $\rhobar$, leading to error $\frac{1}{4}$ for the composed model $\fhat^{\circ L}$. This is depicted pictorially in \Cref{fig:framework}.

\medskip
\noindent 
Training the $j$th model above requires generating $\calOtilde(d)$ inputs from the reweighted distribution $\rhobar_j$, which has the structure $\rhobar_j (\cdot) \propto \rhobar (\cdot) w_j (\cdot)$ for a non-negative weight function $w_j$; informally, $w_j$ downweights regions where the ensemble already has high accuracy and upweights regions with low accuracy (refer to \cref{eq:alpha} for a formal definition of $w_j$). Standard boosting---as in Part I \citep{rajaraman2026learning}---generates samples from this distribution using \emph{rejection sampling} \citep{von1963various,block2023sample}: We repeatedly draw $\bx\sim \rhobar$, then draw $\eta\sim\Unif([0,1])$ and accept if $\eta \leq \frac{w_j(\bx)}{\|w_j\|_\infty}$. Accepted samples are guaranteed to follow $\rhobar_j$, so the only question is how many attempts are required. For this, one of two good situations must occur. First, if 
\begin{equation} \label{eq:acceptance-single}
    \Pr_{\bx \sim \rhobar, \eta \overset{\text{i.i.d.}}{\sim} \Unif([0,1])} \left(\eta \le \frac{w_j(\bx)}{\| w_j \|_\infty} \right) \geq \veps_0
\end{equation}
then only $\calO(\veps_0^{-1})$ attempts are required in expectation, leading to $\calOtilde(d\veps_0^{-1})$ attempts in total to generate the full dataset for $\fhat^j$. On the other hand, if the condition \cref{eq:acceptance-single} fails, one can show that the current ensemble \smash{$\big( \fhat^0,\cdots,\fhat^{j-1} \big)$} must already be good enough ($\veps_0$-accurate under $\rhobar$), and there is no need to continue training. \loose

\medskip\noindent
While the $\calO(\veps_0^{-1})$ sample complexity for rejection sampling that we sketch above is tight for generic sampling problems \citep{block2023sample}, our compositional application has additional structure. In particular, $\rho_j$ has the form
\[
    \rhobar_j (\cdot) \propto \rhobar (\cdot) w_j (\cdot) = \frac{1}{L} \sum_{i=0}^{L-1} \rho_{\tau i + 1 : \tau (i+1)} (\cdot) w_j (\cdot),
\]
which we recall arises from mixing blocks at different positions. Writing the details explicitly, standard rejection sampling with this distribution takes the following form:
\begin{itemize}
\item Generate a length-$T$ instance $\bx\sim\rho$.
\item Sample $i\sim\crl{0,\ldots,L-1}$ uniformly at random.
\item Query the start state $\fstar_{\tau i} (\bx)$ for block $i$ and define the length-$\tau$ instance $\bx^i = (\fstar_{\tau i} (\bx),\bw_{\tau i+1:\tau(i+1)})$, which is distributed according to $\rho_{\tau i + 1 : \tau (i+1)}$. This requires a single query to the $\etoe(\cdot)$ oracle.
\item Draw $\eta\sim\Unif([0,1])$ and accept if $\eta \le w_j(\bx^i)/\nrm{w_j}_{\infty}$.
\end{itemize}
Our first key observation is as follows: Instead of only generating the start state $\bx^i$ for the uniformly sampled block $i$, we can generate the start states $\bx^1, \ldots, \bx^L$ \emph{for all $L$ blocks simultaneously}, at the cost of increasing the number of $\etoe(\cdot)$ queries to $L$. Can we use this extra information to more efficiently generate a sample from the target distribution $\rhobar_j$? 

\subsection{Detour: Rejection Sampling from Reweighted Mixtures of Marginals} \label{sec:sampling}

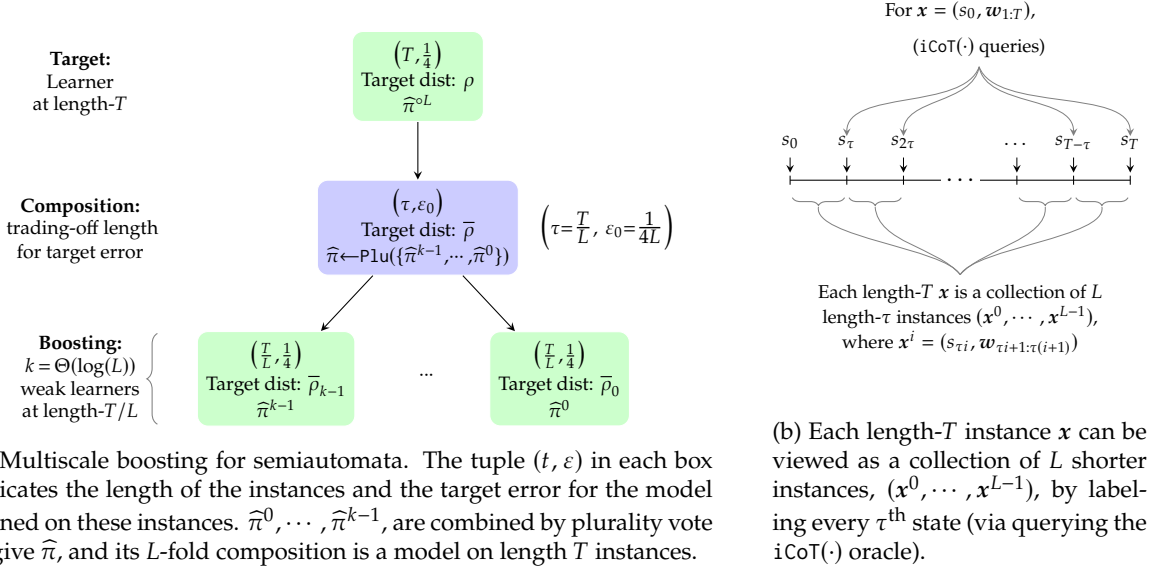
\begin{figure}
\begin{subfigure}[b]{0.6\textwidth}
\centering
\begin{tikzpicture}[
  >=stealth,
  node distance=1.6cm and 2.0cm,
  box/.style={draw, rounded corners, minimum width=12mm, minimum height=6mm},
  every node/.style={font=\small}
]
% ---------- parameters ----------
\def\k{2}          % number of middle-row boxes (εH)
\def\dx{2.5}       % horizontal spacing of middle row
\def\ymid{-1.8}
\def\ylow{-3.8}
\def\nseg{5}       % number of ε²H segments along the small line
\def\seglen{0.5}   % visual length of each segment
\def\arrowh{0.28}  % height of the small downward arrows above ticks
% --------------------------------

% Keep only the left middle box (F2) — no top box, no other middle boxes
\node[box,draw=white,fill=green!20] (e2) at (-\dx,\ymid) {$\substack{\big( T, \frac{1}{4}\big) \\ \text{Target dist: } \rho\\\fhat^{\circ L}}$};

% Left branch subtree
\node[box,draw=white,fill=blue!20] (eeL) at ($(e2)+(0,\ylow-\ymid,0)$) {$\substack{\big( \tau, \varepsilon_0 \big)\\ \text{Target dist: } \rhobar\\ \fhat \gets \maj ( \{ \fhat^{k-1},\cdots,\fhat^0 \} )}$};
\draw[->] (e2) -- (eeL);
\node (descrip) at ($(eeL)+(2.5,0)$) {$\substack{\Big(\tau = \tfrac{T}{L} ,\ \varepsilon_0 = \tfrac{1}{4L}}\Big)$};

\pgfmathsetmacro{\offset}{0.75*\dx}
\node[box,draw=white,fill=green!20] (eeL2) at ($(eeL)+(-\offset,\ylow-\ymid,0)$) {$\substack{ \big( \frac{T}{L}, \frac{1}{4} \big) \\[1pt] \text{Target dist: } \rhobar_{k-1}\\\fhat^{k-1}}$};
\node[box,draw=white,fill=green!20] (eeR2) at ($(eeL)+(\offset,\ylow-\ymid,0)$) {$\substack{\big( \frac{T}{L}, \frac{1}{4} \big)\\[1pt] \text{Target dist: } \rhobar_0\\\fhat^0}$};
% \node[below=0cm of eeR2, align=center] (text1) {$\substack{\fhat^0}$};
% \node[below=0cm of eeL2, align=center] (text2) {$\substack{}$};
\node[left=0.61cm of eeR2, align=center] (text3) {$\substack{\cdots}$};
\draw[->] (eeL) -- (eeL2);
\draw[->] (eeL) -- (eeR2);

% --- Small horizontal line centered under left bottom box ---
% \coordinate (baseL) at ($(eeL.south)+(5.5,0.25)$);
% \coordinate (lineStart) at ($(baseL)+(-0.5*\nseg*\seglen,0.25)$);
% \coordinate (lineEnd)   at ($(baseL)+( 0.5*\nseg*\seglen,0.25)$);

% \draw (lineStart) -- (lineEnd) node[right=3pt, font=\small] {$T$};

% \node[above right=0cm and 2.2cm of eeL, align=center] (text) {$\Big($ {\scriptsize\parbox{4.3cm}{$\rhobar$: decompose $\bx \sim \rho$ into $L$ shorter instances, sampling one at random}} $\Big)$};

% \foreach \j in {0,...,\nseg} {
%   \coordinate (tj) at ($(lineStart)+(\j*\seglen,0)$);
%   \draw (tj)++(0,-0.06) -- ++(0,0.12);
%   \draw[->, >=stealth, line width=0.3pt] 
%     ($(tj)+(0,\arrowh)$) -- ($(tj)+(0,0.10)$);
% }

% \node[below, font=\tiny] at ($(lineStart)+(0.5*\seglen,0)$) {$\tau$};

% Level labels (no Level 0)
% \node[below=2.75cm of eeL, anchor=south, font=\small] {$\vdots$};

\draw[decorate, color=gray, decoration={brace, mirror, raise=2pt, amplitude=4pt}] ($(eeR2)+(-5.25,0.6)$) -- ($(eeR2)+(-5.25,-0.6)$);
\node[left=4.5cm of eeR2, anchor=east, font=\small] (BOT) {$\substack{\textbf{Boosting:}\\k \,=\, \Theta ( \log (L))\\\text{weak learners}\\[1.5pt]\text{at length-}T/L}$};

\node[above=0.7cm of BOT, anchor=south, font=\small] (MID) {$\substack{\textbf{Composition:}\\\text{ trading-off length}\\\text{for target error}}$};

\node[above=0.85cm of MID, anchor=south, font=\small] (TOP) {$\substack{\textbf{Target:}\\\text{Learner}\\[1.5pt]\text{at length-}T}$};
\end{tikzpicture}
\caption{Multiscale boosting for semiautomata. The tuple $(t, \varepsilon)$ in each box indicates the length of the instances and the target error for the model trained on these instances. $\fhat^0,\cdots,\fhat^{k-1}$, are combined by plurality vote to give $\fhat$, and its $L$-fold composition is a model on length $T$ instances.}
\label{fig:framework}
\end{subfigure}\qquad
\begin{subfigure}[b]{0.3\textwidth}
\centering
\begin{tikzpicture}[
  >=stealth,
  node distance=1.6cm and 2.0cm,
  box/.style={draw, rounded corners, minimum width=12mm, minimum height=6mm},
  every node/.style={font=\small}
]
% ---------- parameters ----------
\def\k{2}          % number of middle-row boxes (εH)
\def\dx{2.5}       % horizontal spacing of middle row
\def\ymid{-1.8}
\def\ylow{-3.8}
\def\nseg{7}       % number of ε²H segments along the small line
\def\seglen{0.75}   % visual length of each segment
\def\arrowh{0.28}  % height of the small downward arrows above ticks
\def\linepad{0.9}  % extra horizontal padding to elongate the timeline
\def\braceyoffset{0.16} % vertical offset of braces below the timeline
\def\gap{0.03} % horizontal gap for braces
\pgfmathtruncatemacro{\penseg}{\nseg-3}
\pgfmathtruncatemacro{\prevseg}{\nseg-2}
\pgfmathtruncatemacro{\lastseg}{\nseg-1}
\pgfmathsetmacro{\halfspan}{0.5*(\nseg-1)*\seglen}
\def\arrowlabelA{\smash{$s_0$}}
\def\arrowlabelB{\smash{$s_{\tau}$}}
\def\arrowlabelC{\smash{$s_{2\tau}$}}
\def\arrowlabelD{\smash{$\dots$}}
\def\arrowlabelE{\smash{$s_{T-\tau}$}}
\def\arrowlabelF{\smash{$s_T$}}
% --------------------------------

% Keep only the left middle box (F2) — no top box, no other middle boxes
% \node[box] (e2) at (-\dx,\ymid) {$\substack{\big( T, \frac{1}{4}\big) \\ \text{Target dist: } \rho}$};

% Left branch subtree
% \node[box] (eeL) at ($(e2)+(0,\ylow-\ymid,0)$) {$\substack{\big( 4 \varepsilon_0 T, \varepsilon_0 \big)\\ \text{Target dist: } \rhobar}$};
% \draw[->] (e2) -- (eeL);

% % \pgfmathsetmacro{\offset}{0.75*\dx}
% % \node[box] (eeL2) at ($(eeL)+(-\offset,\ylow-\ymid,0)$) {$\substack{ \big( 4 \varepsilon_0 T, \frac{1}{4} \big) \\[1pt] \text{Target dist: } \rhobar_{k-1} \\[1pt] \text{(weak learner)}\\[1pt] \text{Block } \# k}$};
% % \node[box] (eeR2) at ($(eeL)+(\offset,\ylow-\ymid,0)$) {$\substack{\big( 4 \varepsilon_0 T, \frac{1}{4} \big)\\[1pt] \text{Target dist: } \rhobar_0\\[1pt] \text{(weak learner)}\\[1pt] \text{Block } \# 1}$};
% % \draw[->] (eeL) -- (eeL2);
% % \draw[->] (eeL) -- (eeR2);

% --- Small horizontal line centered under left bottom box ---
\coordinate (baseL) at ($(0,0)$);
\coordinate (tickStart) at ($(baseL)+(-\halfspan,0.25)$);
\coordinate (tickEnd)   at ($(baseL)+(\halfspan,0.25)$);
\coordinate (lineStart) at ($(tickStart)+(0,0)$);
\coordinate (lineEnd)   at ($(tickEnd)+(0,0)$);

\draw (lineStart) -- (lineEnd) node[right=3pt, font=\small] {};

% \node[below right=0.3cm and -2.35cm of baseL, align=center] (text) {$\Big($ {\scriptsize\parbox{4.3cm}{$\rhobar$: decompose $\bx \sim \rho$ into $L$ shorter instances, sampling one at random}} $\Big)$};

\foreach \j/\label in {0/\arrowlabelA,1/\arrowlabelB,2/\arrowlabelC,\penseg/\arrowlabelD,\prevseg/\arrowlabelE,\lastseg/\arrowlabelF} {
  \coordinate (t\j) at ($(tickStart)+(\j*\seglen,0)$);
  \draw (t\j)++(0,-0.06) -- ++(0,0.12);
  \draw[->, >=stealth, line width=0.3pt] 
    ($(t\j)+(0,0.3)$) -- ($(t\j)+(0,0.10)$);
  \node[above=1pt, font=\scriptsize] (lab\j) at ($(t\j)+(0,0.3)$) {\label};
}

\node[above=2.25cm of baseL, align=center, font=\scriptsize] (etoe2) {For $\bx = (s_0,\bw_{1:T})$,};

\node[above right=1.75cm and -0.75cm of baseL, align=center, font=\scriptsize] (etoe) {($\etoe(\cdot)$ queries)};

\node[fill=white, align=center, inner sep=2pt, font=\small] (midellipsis) at ($(baseL)+(0,0.25)$) {$\ldots$};

\draw[->, line width=0.3pt, color=gray] (etoe.south) to[out=-110, in=90, looseness=0.8] (lab1.north);
\draw[->, line width=0.3pt, color=gray] (etoe.south) to[out=-100, in=90, looseness=0.95] (lab2.north);
\draw[->, line width=0.3pt, color=gray] (etoe.south) to[out=-80, in=90, looseness=0.9] (lab\prevseg.north);
\draw[->, line width=0.3pt, color=gray] (etoe.south) to[out=-70, in=90, looseness=0.75] (lab\lastseg.north);

\draw[decorate, color=gray, decoration={brace, mirror, raise=2pt, amplitude=3pt}] ($(t0)+(\gap,-\braceyoffset)$) -- ($(t1)+(-\gap,-\braceyoffset)$);
\coordinate (brace01mid) at ($($(t0)!0.5!(t1)$)+(0,-0.33)$);
\node[inner sep=0pt] (brace01cusp) at (brace01mid) {};

\draw[decorate, color=gray, decoration={brace, mirror, raise=2pt, amplitude=3pt}] ($(t1)+(\gap,-\braceyoffset)$) -- ($(t2)+(-\gap,-\braceyoffset)$);
\coordinate (brace12mid) at ($($(t1)!0.5!(t2)$)+(0,-0.33)$);
\node[inner sep=0pt] (brace12cusp) at (brace12mid) {};

\draw[decorate, color=gray, decoration={brace, mirror, raise=2pt, amplitude=3pt}] ($(t\penseg)+(\gap,-\braceyoffset)$) -- ($(t\prevseg)+(-\gap,-\braceyoffset)$);
\coordinate (bracePenultMid) at ($($(t\penseg)!0.5!(t\prevseg)$)+(0,-0.33)$);
\node[inner sep=0pt] (bracePenultCusp) at (bracePenultMid) {};

\draw[decorate, color=gray, decoration={brace, mirror, raise=2pt, amplitude=3pt}] ($(t\prevseg)+(\gap,-\braceyoffset)$) -- ($(t\lastseg)+(-\gap,-\braceyoffset)$);
\coordinate (braceLastMid) at ($($(t\prevseg)!0.5!(t\lastseg)$)+(0,-0.33)$);
\node[inner sep=0pt] (braceLastCusp) at (braceLastMid) {};

\node[below=1cm of baseL, align=center, font=\scriptsize] (sampleone) {Each length-$T$ $\bx$ is a collection of $L$\\length-$\tau$ instances $( \bx^0,\cdots,\bx^{L-1} )$,\\where $\bx^i = (s_{\tau i},\bw_{\tau i + 1 : \tau(i+1)})$};

\draw[-, line width=0.3pt, color=gray] (brace01cusp.south) to[out=-100, in=120, looseness=0.5] (sampleone.north);
\draw[-, line width=0.3pt, color=gray] (brace12cusp.south) to[out=-85, in=130, looseness=0.85] (sampleone.north);
\draw[-, line width=0.3pt, color=gray] (bracePenultCusp.south) to[out=-85, in=50, looseness=0.75] (sampleone.north);
\draw[-, line width=0.3pt, color=gray] (braceLastCusp.south) to[out=-80, in=60, looseness=0.5] (sampleone.north);

% \node[above=0cm, font=\tiny] at ($(tickStart)+(0.5*\seglen,0)$) {$4 \varepsilon_0 T$};

\end{tikzpicture}
\vspace{0.25em}
\caption{Each length-$T$ instance $\bx$ can be viewed as a collection of $L$ shorter instances, $(\bx^0,\cdots,\bx^{L-1})$, by labeling every $\tau^{\text{th}}$ state (via querying the $\etoe(\cdot)$ oracle).}
\label{fig:sampling-basic}
\end{subfigure}
\caption{\textit{Left:} Multiscale boosting. \textit{Right:} Decomposing longer instances into shorter ones.}
\label{fig:1}
\end{figure}

To highlight our goal as crisply as possible, we define an abstract setting, \emph{sampling from reweighted mixtures of marginals}, which captures the essential structure of the problem above, but may be of interest in its own right. In this setting, we assume sampling access to a joint distribution $\mu$ over potentially correlated random variables $(\bx^0,\cdots,\bx^{L-1})$, and aim to sample from a reweighted mixture of their marginals according to a weight function $w(\bx)$. In the boosting context, these random variables correspond to the length-$\tau$ segments of a length-$T$ input instance, with $\mu$ playing the role of $\rho$.
\loose

\begin{definition}[Sampling from reweighted mixtures of marginals] \label{def:srmm} Let $\bX$ denote a probability space, and fix $L \in \bbN$. We are given sampling access to a joint distribution $\mu$ over $\bX^L$, and query access to a weight function $w : \bX \to [0,1]$. Define $\nu (\cdot) \propto \mubar(\cdot) w(\cdot) \in \Delta_{\bX}$, where $\mubar (\cdot) = \frac{1}{L} \sum_{i=0}^{L-1} \mu(X_i = \cdot\,)$ is the uniform mixture of marginals of $\mu$. Our objective is to generate samples (approximately) from $\nu$, making as few queries to $w(\cdot)$ and drawing as few joint samples from $\mu$ as possible.
\end{definition}

\noindent For this setting, standard rejection sampling takes the following form:
\begin{enumerate}
    \item[$(a)$] Draw $X = (X_0, \cdots,X_{L-1}) \sim \mu$. Sample an index $I \sim \Unif (\{ 0,\cdots,L-1 \})$ and select $X_I$. Since $I$ was selected uniformly at random, the random variable $X_I$ is sampled according to $\mubar$.
    \item[$(b)$] Use rejection sampling to accept/reject $X_I$. Namely, draw $\eta \sim \Unif([0,1])$ and accept $X_I$ if $\eta \le w (X_I) / \| w \|_\infty$ (where the RHS is an upper bound on the density ratio $\nu/\mubar$). Repeat from $(a)$ until a sample is accepted and declare success if a sample is accepted.
\end{enumerate}
As discussed in the prequel, a standard sufficient condition for this process to accept in $\calOtilde(\veps_0^{-1})$ attempts is that \loose
\begin{equation} \label{eq:accept-uniform}
    \Pr_{X \sim \mu, i\sim\Unif(\crl{0,\ldots,L-1}), \eta_i \sim \Unif([0,1])} \left( \eta_i \le \frac{w(X_i)}{\| w \|_\infty} \right) \geq \veps_0.
\end{equation}
To see why we might hope for better, consider the case where $X_0,\ldots,X_{L-1}$ are independent and identically distributed. If each $X_i$ has acceptance probabilty $\veps_0$ marginally, then \cref{eq:accept-uniform} holds with parameter $\veps_0$, and vanilla rejection sampling will indeed use $\veps_0^{-1}$ draws from $\mu$. However, this is wasteful---we are effectively throwing away $L-1$ perfectly good samples---and we can improve our probability of success by a factor of $L$ by simply testing and accepting/rejecting all $X_0,\ldots,X_{L-1}$.

\medskip\noindent
Motivated by this observation, a more optimistic sampling procedure should aim for the following condition:
\begin{equation} \label{eq:always-accepted}
    p_{\mu,w} = \Pr_{X \sim \mu, \eta_i \overset{\text{i.i.d.}}{\sim} \Unif([0,1])} \left( \exists i \in \{0,\ldots,L-1\} : \eta_i \le \frac{w(X_i)}{\| w \|_\infty} \right) \geq \veps.
\end{equation}
This condition asserts that \emph{counterfactually, at least one of samples $X_i$ would have been accepted had we chosen the index $i$ to begin with}. Note that we have $\veps\approx \veps_0\cdot L$ in the case where $X_i$ are independent, suggesting that the quantity in \cref{eq:always-accepted} captures the easiness of this setting.

\paragraph{A barrier to further improvement: Hardness of sampling.}
Taking advantage of the condition in \cref{eq:always-accepted} is more challenging when $X_0,\ldots,X_{L-1}$ are correlated. For example, the following lemma shows that vanilla rejection sampling incurs $\Omega(L)$ sample complexity even when \cref{eq:always-accepted} holds with $\veps=1$.
\begin{lemma}[Informal; see \Cref{lemma:worst-case-sampling}] \label{lemma:rejection-sampling-informal}
Even when $\veps=1$ in \cref{eq:always-accepted}, the expected number of samples from $\mu$ required for rejection sampling to succeed scales with $L$.
\end{lemma}
\noindent More generally, we show now that \textit{any approach} to generate a sample from a distribution $\nuhat$ such that \smash{$\dtv (\nuhat, \nu) < \frac{1}{2}$} requires drawing at least \smash{$\Omega (\mathrm{poly}(L))$} samples from $\mu$ in the worst case. The full version of this result (\Cref{lemma:worst-case-sampling}) holds even when $w$ and $\mu$ satisfy \cref{eq:always-accepted} with $\varepsilon=1$; we sketch a simplified version below.

\begin{example}[An example showcasing the hardness of sampling from $\nu$]
    \label{ex:samplinghardness}
Consider the setting where $\bX = \{ 0,1,2 \}$, and the weight function $w : \bX \to [0,1]$ is defined as $w(z) = \bbI (z \in \{ 1,2\})$. Now, consider the distribution $\mu \in \Delta_{\bX^L}$ which is uniform on a set of $L+1$ strings, namely, $\{ (1,1,\cdots,1), (2,0,\cdots,0),(0,2,\cdots,0),\cdots,(0,0,\cdots,2) \}$. The distribution $\mu$ with probability $\approx \frac{1}{L}$ generates the all-$1$ sequence, and with remaining probability sets a uniformly random position as $2$ and the remaining positions as $0$. A simple calculation by Bayes rule shows that the distribution $\mubar$ satisfies $\mubar(1) = \mubar(2) \approx \frac{1}{L}$ and $\mubar(0) \approx 1 - \frac{1}{L}$. Consequently, $\nu \propto \mubar (\cdot) w(\cdot)$ is the uniform distribution over $\{ 1,2 \}$. \loose

\medskip\noindent For this example, \cref{eq:always-accepted} holds with $\veps=1$, since every string contains either $1$ or $2$ at some position. However, note the structure of the distribution $\mu$: the symbol $1$ is only ever seen in an instance drawn from $\mu$ with probability $\approx \frac{1}{L}$, even though this symbol is highly represented under $\nu$. Thus, any natural algorithm---including rejection sampling---to sample from $\nu$ (or a nearby distribution with, say, $\TV{ \nuhat }{\nu} \leq \frac{1}{2}$) requires drawing $\Omega(L)$ samples from $\mu$. 
\end{example}

\noindent Thus, in the worst case, sampling (even very approximately) from $\rhobar_j$ cannot be carried out with fewer than \smash{$\mathrm{poly}(L)$} samples from $\rho$. This makes it unlikely to improve the sample and query complexity by developing better methods for sampling from the reweighted mixture of marginals (\Cref{def:srmm}) under standard notions of distribution approximation (say, TV-distance or KL divergence). 

\subsection{Inverted Sampling: More Efficient Sampling from Reweighted Mixtures} 
\label{subsec:depth-1} 

Recall that standard rejection sampling---when applied to the problem of sampling from reweighted mixtures of marginals---samples a shorter instance from each longer instance uniformly at random, and checks whether it passes a filter. The remaining $L-1$ sub-instances contained within are thrown away. This is inherently wasteful since one of the remaining (unselected) instances may have passed the filter, but the index of the uniformly random instance selected by rejection sampling will only ``hit'' this index with small probability $1/L$. A natural algorithm which does not suffer from this issue reverses the order of operands: first accept/reject each of the $L$ shorter instances by applying rejection sampling independently, then select one of the accepted instances uniformly at random. We refer to this approach as \textit{inverted sampling}, implemented as follows: \loose
\begin{enumerate}
    \item[$(a)$] Draw $X = (X_0,\cdots,X_{L-1}) \sim \mu$. Let $\calI$ denote the set of indices,
    \begin{equation*}
        \left\{ i \in \{0,\ldots,L-1\} : \eta_i \le w(X_i) / \| w \|_\infty \right\}
    \end{equation*}
    Here $\eta_i \sim \Unif([0,1])$ independently across $i$. If $\calI = \emptyset$, reject and repeat from $(a)$.
    \item[$(b)$] Let $I$ be an index drawn uniformly at random from $\calI$. Return $X_I$.
\end{enumerate}
The benefit of inverted sampling is that it is guaranteed to accept as long as at least one of the $X_i$'s passes the filter; this is much less pessimistic than rejection sampling, which may fail to select this $X_i$ to filter in the first place. However, this approach is still subject to the strong sampling lower bound discussed in \Cref{ex:samplinghardness}, and so the increased acceptance rate must come at a large cost in the TV distance to $\nu$.\loose

\medskip
\noindent 
In spite of the TV distance lower bound, we show that inverted sampling can be used to sample from a distribution which is close to $\nu$ in the sense of a weaker notion of approximation, which is nonetheless sufficient for our application within \mainalg.

\begin{lemma}[Informal; see \Cref{lemma:apxsample,lemma:apxsample-specialcase}] \label{lemma:apxsample-informal}
Fix a constant $c \in (0,1)$. Let $\nuhat$ denote the distribution over instances accepted by inverted sampling. There exists an event $\calK \subseteq \bX^L \times [0,1]^L$ which is a measurable function of $(X,\bm{\eta}) \sim \mu \times \Unif([0,1])^{\otimes L}$ where $\bm{\eta} = (\eta_i)_{i=0}^{L-1}$, such that $\Pr ( (X,\bm{\eta}) \in \calK) \ge 1-c$, and the following density ratio bound is satisfied,
\begin{equation} \label{eq:apxsample}
    \left\| \frac{\nu_{\calK} (\cdot)}{\nuhat (\cdot)} \right\|_\infty \le \calO( c^{-1}\cdot\log(L))
\end{equation}
Here, $\nu_\calK (\cdot)\propto \frac{1}{L} \sum_{i=0}^{L-1} \Pr (X_i = \cdot \,|\, (X,\bm{\eta}) \in \calK) \, w (\cdot)$ is the reweighted mixture distribution of marginals when the underlying sample and random coins, $(X,\bm{\eta})$, are conditioned on the event $\calK$. Finally, the marginal probability that inverted sampling accepts a sample $X \sim \mu$ is $p_{\mu,w}$, defined in \cref{eq:always-accepted}.
\end{lemma}

\noindent That is, by shifting our target to the distribution $\nu_\calK$---where $\calK$ is an appropriate high-probability event defined formally in \cref{eq:calK}---we can sample from a distribution $\nuhat$ that approximates the target in the density ratio sense up to only a $\log(L)$ factor (as opposed to polynomial dependence on $L$), and achieve an acceptance probability that matches the skyline in \cref{eq:always-accepted}. This result sidesteps the hardness of sampling (\Cref{lemma:worst-case-sampling}) because the distribution approximated by inverted sampling, $\nu_\calK$, can be very different from $\nu$. Nonetheless, because $\calK$ has high probability, this guarantee is sufficient for our application within \mainalg.

\paragraph{How does \Cref{lemma:apxsample-informal} sidestep the hard sampling instance of \Cref{ex:samplinghardness}?} 

In order to better understand this result, we put it into context of the hard instance in \Cref{ex:samplinghardness}. For the construction in the example, if we define $\calK$ as $\{ X \ne (1,\cdots,1)\}$, we see that $\Pr (\calK) \ge 1 - \frac{1}{L}$ is a high probability event. This conditioning on $\calK$ has a dramatic effect on the resulting distribution $\nu_\calK$: since the all-$1$ sequence is the only sequence which contains the symbol $1$, conditioning on $\calK$ results in $\nu_\calK$ becoming the delta distribution on the symbol $2$. This is a very different target from $\nu = \Unif (\{1,2\})$, and sampling from $\nu_\calK$ no longer presents the same pathologies of sampling from $\nu$.

\subsection{Inverted Sampling Achieves $\nsample=\calOtilde(d)$ and $\nquery=\calOtilde(d\sqrt{T})$}\label{subsec:queryinverted}

We now return to the context of \mainalg, where we use inverted sampling to implement boosting at the length-$\tau$ scale. Let $\sigma$ denote the joint distribution over the $L$ length-$\tau$ instances obtained by labeling the boundary states on $\bx \sim \rho$. In each iteration $j$ of the boosting process, we apply inverted sampling  guarantees to draw instances from a distribution that \emph{covers} $\lambda_j$, which is the distribution proportional to $\overline{\sigma_{\calK_j}} (\cdot) \, w_j(\cdot)$ where $\overline{\sigma_{\calK_j}} = \frac{1}{L} \sum_{i=0}^{L-1} \Pr (X_i = \cdot \mid (X,\bm{\eta}) \in \calK_j)$ for $X \sim \sigma$, where $\calK_j$ denotes the high probability event defined in \cref{lemma:apxsample-informal} for an appropriate choice of $c$. The event $\calK_j$ in the boosting context has an intuitive description: Associating $X_i=\bx^i$ for $i=0,\cdots,L-1$, an instance $X = (X_0,\cdots,X_{L-1})$ is \textit{likely} to belong to $\calK_j$ (with a slight abuse of notation, we define this to mean $\Pr ((X,\bm{\eta}) \in \calK_j \mid X)$ is large) if, effectively, on some non-empty subset of $X_i$'s, the first $j-1$ models all tend to make mistakes together. On such instances, querying the label on any of the erroneous $X_i$'s is maximally informative for the learner, as it corrects the mistakes of many of the existing models. Furthermore, if it appears challenging to sample an $X$ such that any of the $X_i$'s are erroneous, i.e., $p_{\mu,w}$ is small, this is an indication that the weak learners are already good, since they agree on the correct label on all $X_i$'s.

\medskip
\noindent It remains to argue the conditioning on the event $\calK_j$ does not significantly influence the overall accuracy of the aggregated weak learners. For this, if $\calK_j$ is constructed choosing \smash{$c \le \frac{1}{8k}$} in \Cref{lemma:apxsample-informal}, we are guaranteed that the probability of \smash{$\calK_j^c$} under $\rho$ is at most \smash{$c \le \frac{1}{8k}$}. Viewing \smash{$\calK_j^c$} as a ``failure region'' and accounting for the probability mass which falls in these regions separately, we are guaranteed that \smash{$\cup_{j=0}^{k-1} \calK_j^c$} has probability at most $ck \le \frac{1}{8}$, which is within the total error budget of $\frac{1}{4}$.

\medskip
\noindent The analysis sketched here requires some subtle additional conditioning arguments to make everything go through, and we defer the discussion of these details to \Cref{subsec:depth1-final-proof-sketch}. Overall, by using inverted sampling to appropriately define the target weak learning distributions, we obtain an algorithm with the following query and sample complexity guarantees:
\begin{equation*}
    \nquery = \calOtilde \left( d \tau \cdot \polylog(L) + d/\veps_0 \right), \quad \text{and}, \quad \nsample = \calOtilde \big( d \cdot \polylog(L) \big)
\end{equation*}
Where $\veps_0 = \calO(1/L)$. The $d/\veps_0$ term in the query complexity comes from the cost of using inverted sampling to generate $\calO(d)$ samples for the full-CoT learners at the length-$\tau$ scale, while the $d \tau \cdot \polylog(L)$ term comes from labeling all $\tau$ intermediate states on the instances that are selected by inverted sampling (to train full-CoT learners via next-token prediction).
Optimizing over $L$, we get an algorithm with query complexity scaling as $\calOtilde (d\sqrt{T})$, but with an \textit{improved sample complexity of $\calOtilde (d)$} compared to the rejection sampling based approach described in \Cref{sec:warmup}. This improvement shows that it is possible to achieve sublinear-in-$T$ query complexity without paying for a comparable blowup in sample complexity.

\medskip\noindent
\emph{For more detail, we refer the reader to \cref{sec:semiauto-depth1-proofs}, which presents full pseudocode for a simplified version of \mainalg that uses a single level of recursion, along with a complete proof that it acheives $\nsample=\calOtilde(d)$ and $\nquery=\calOtilde(d\sqrt{T})$.}

\subsection{\texorpdfstring{$2^{\calOtilde(\sqrt{\log (T)})}$ Query and Sample Complexity via Deeper Recursion}{Subpolynomial Query and Sample Complexity via Deeper Recursion}}
\label{subsec:depth-H}

\noindent The depth-$1$ algorithm described above uses a full-CoT learner (i.e., vanilla next-token prediction) to weakly learn $\fstar_\tau$ for each distribution over length-$\tau$ instances, $\wrho_j$. Each invocation of this weak learning algorithm requires $\Omega(\tau)$ queries. To achieve the main guarantee for the full version of \mainalg in \cref{informaltheorem:main}, we use a deeper recursion in which the weak learner at the length-$\tau$ scale is itself an instantiation of the depth-$1$ algorithm in the prequel (for $T=\tau$). We iterate this recursion over multiple levels, decomposing instances into smaller and smaller sub-instances, with the length decreasing by a factor of $L$ at each level. We refer the reader back to \Cref{fig:framework} for the overall structure of the algorithm, noting that the learner sets $\varepsilon_0 = \calO \big( \frac{1}{L} \big)$ at each depth.\loose

\medskip
\noindent 
In the approach above, the depth of the recursion $H$ and the decomposition factor $L$ are free parameters, but need to satisfy $T = L^H$ to ensure that the deepest level of recursion (i.e., depth $0$) corresponds to solving length-$1$ problems (which is achieved through fitting a next-token predictor). The analysis of inverted sampling we worked out in \Cref{lemma:apxsample-informal} plays a key role in  analyzing the sample complexity and query complexity of the resulting algorithm. In particular, when we extend the same analysis to deeper levels of recursion, at any depth $h \le H$ (corresponding to solving instances of length $T/L^{H-h}$), generating samples from the appropriate weak learning distribution at this depth requires drawing $(\log(L))^{\calO (H-h)}$ examples from $\rho$. In particular, this implies that the total sample complexity of the algorithm scales as $\nsample = d \cdot (\log(L))^{\calO(H)}$, noting that we require $\calOtilde(d)$ samples at the deepest level to train a model via next-token prediction to achieve constant accuracy. On the other hand, the overall query complexity of the algorithm scales as $\nquery \le L H \nsample$: at each depth $h$ of the algorithm, inverted sampling subselects a single length-$T/L^{H-h+1}$ instance from a length $T/L^{H-h}$-length instance. This means that on each input in the dataset, inverted sampling makes $LH$ $\etoe$ queries, leading to the following bound on the overall query complexity:
\begin{equation*}
    \nsample = d \cdot (\log(L))^{\calO(H)} \text{ and } \nquery = LH \cdot d \cdot (\log(L))^{\calO(H)}
\end{equation*}
While the sample complexity alone is minimized with $H=1$ and $L=T$, the query complexity is large in this regime, scaling as $\calO(dT)$. However, by choosing $L$ and $H$ differently, both quantities can be made to scale subpolynomially in $T$. In particular, with \smash{$L = 2^{\sqrt{\log T}}$} (which corresponds to \smash{$H = \sqrt{\log T}$} to ensure $L^H = T$), we balance the $L$ and $(\log (L))^{\calO(H)}$ terms in the query complexity, which both scale as \vphantom{$2^{2^2_1}$}\smash{$2^{\calOtilde(\sqrt{\log(T)})}$}. With this choice, both the sample and query complexity scale as \vphantom{$2^{2^2_1}$}\smash{$2^{\calOtilde(\sqrt{\log T})} \cdot d$}, yielding \Cref{informaltheorem:main}.

\section{Discussion} \label{sec:discussion}

\noindent This work asks how \emph{curriculum} and \emph{composition} let a learner solve problems far beyond the reach of its base capabilities---a mechanism widely credited for the success of modern reasoning models, yet one with little theoretical foundation. Adopting semiautomaton simulation as a testbed, our results show that composition combined with curriculum can achieve superpolynomial reductions in supervision and computational cost relative to non-curriculum baselines, in both supervised fine-tuning (SFT) and reinforcement learning with verifiable rewards (RLVR). We close by discussing the limitations of our results, followed by several concrete technical questions they leave open.

\subsection{Simplifications in the Problems Formulation}

While our work is motivated by language model reasoning, our problem setting makes several simplifications; relaxing these is an important direction for future work.

\paragraph{Non-Markovian models.}
Our results use the Markovian structure of semiautomata simulation task to decompose the problem into shorter sub-problems. However, our results also require that the \emph{model we train} is Markovian, which rules out classes like Transformers that can attend to the full history. It is an interesting question to understand how to extend our results to settings in which the underlying transition function is still Markovian, but the model we train may not be.

\paragraph{Stochastic transition functions.}

Our main results apply for semiautomata transition functions, which are inherently deterministic. However, in many natural settings, including language model reasoning, it is natural to consider stochastic transition functions $\pi : S \times \Sigma \to \Delta_S$. This encompasses Markov chains, and other (controlled) stochastic processes, for which there is a rich history of work studying learnability~\citep{baum1966statistical,han2023optimal,gaitonde2025bypassing}. A natural objective is to return a terminal-state distribution $\fhat : (S \times \Sigma^T) \to \Delta_S$ which minimizes some divergence to the ground-truth terminal state distribution $\fstar_T$.\footnote{we will use $\fstar_T (\cdot|\bx)$ to denote the marginal distribution of the terminal state generated by $\fstar$ on the instance $\bx$.} When the learning objective is the (forward) KL divergence, $\bbE_{\bx \sim \rho} \big[ \KL{\fstar_T (\cdot|\bx)}{\fhat(\cdot|\bx)} \big]$, the idea of composition is still functional, as suggested by the chain rule for KL divergences.

\begin{proposition} \label{prop:chain-rule}
For any $\tau : T \pmod{\tau} \equiv 0$ and $L = \frac{T}{\tau}$,
\begin{align*}
    \bbE_{\bx \sim \rho} \big[ \KL{\fstar_T (\cdot|\bx)}{\fhat^{\circ L} (\cdot|\bx)} \big] \le L \cdot \bbE_{\bx \sim \rhobar} \big[ \KL{\fstar_\tau (\cdot|\bx)}{\fhat (\cdot|\bx)} \big]
\end{align*}
\end{proposition}

\noindent It is worth pointing out that this type of decomposition does not necessarily translate to other divergences (TV or Hellinger distances) which do not satisfy the chain rule. Furthermore, estimating whether $\pi_t (\cdot|\bx) \approx \fstar_t (\cdot|\bx)$ for some model $\pi$ (to decide whether an instance $\bx$ is ``solved'' or not), also requires observing the probability $\fstar_t (y|\bx)$ for $y \sim \fstar_t (\cdot|\bx)$, which may require considering other observation models which go beyond just observing states~\citep{mahajan2023learning}.

\subsection{Technical Questions}

Our results leave open several interesting technical questions, including tightening our sample complexity and computational guarantees, as well as further relaxing the coverage assumptions in our RLVR results. Below, we highlight a particularly interesting question: Improving the $T$ dependence in our iSFT results.

\subsubsection{Optimality of $T$ Dependency in \Cref{theorem:main-generic}}

Our main result in the SFT setting establishes query and sample complexity scaling \smash{$2^{\calOtilde(\sqrt{\log T})}$} for learning semiautomata. Understanding whether there are statistical or computational barriers to improving this further to achieve \emph{polylogarithmic} dependency on $T$ appears to be a fascinating and deep question. $\log(T)$ query complexity is a natural endpoint for divide-and-conquer approaches like \mainalg, but our current results fall short due to additional complications around error accumulation. Below we show a special case where $\polylog(T)$ is indeed achievable under stronger assumptions on the class $\calF$ than i.i.d. learnability: namely that it admits a finite mistake bound. We define the relevant notation below.

\begin{definition}[Mistake-bounded online learning for $\calF$ {\citep{littlestone1988learning}}] \label{def:mistake-bound}
Consider the following online learning game between an algorithm $\AlgOL$ and an adversary. The adversary fixes a transition function $\fstar \in \calF$, and at each step $i = 1, \ldots, N$: (1) the adversary adaptively selects a state-letter pair $(s^i,w^i) \in S \times \Sigma$ (possibly depending on all prior predictions of $\AlgOL$), (2) the algorithm predicts $\widehat{s}^i_+ \in S$ based on all prior observations, and (3) the true next-state $\fstar (s^i,w^i)$ is revealed. The worst-case number of mistakes of $\AlgOL$ is defined as,
\begin{equation*}
    M_N (\AlgOL) = \sup_{\fstar \in \calF} \; \sup \; \sum\nolimits_{i=1}^N \bbI \big( \widehat{s}^i_+ \ne \fstar (s^i,w^i) \big)
\end{equation*}
where the inner supremum is over all adaptive adversary strategies. $\calF$ admits a \textit{finite mistake bound} if there exists an online learning algorithm $\AlgOL$ such that $\sup_{N \ge 1} M_N (\AlgOL) \triangleq M^\star (\AlgOL) < \infty$.
\end{definition}

\begin{remark}
A finite mistake bound for $\calF$ is a stronger assumption than i.i.d. learnability. When $S = \{0,1\}$ and $\Sigma$ is arbitrary, i.i.d. learnability is equivalent to the finiteness of the VC dimension of $\calF$. A finite mistake bound is equivalent to the finiteness of the Littlestone dimension, $\Ldim (\calF)$~\citep{littlestone1988learning}, which is no smaller than the VC dimension. For simple classes such as thresholds, $\VCdim(\calF)=1$, while $\Ldim(\calF)$ can scale with the size of the domain of the functions.
\end{remark}

\noindent Assuming that $\calF$ admits a finite mistake bound, we can improve the sample and query complexity of learning semiautomata to achieve logarithmic dependence on $T$.

\begin{theorem} \label{theorem:online}
Let $\AlgOL$ denote an online learning algorithm with mistake bound $M^\star (\AlgOL) < \infty$. Consider any $\varepsilon \in [0,1]$. There exists an algorithm (\Cref{alg:semiauto_online}) which uses $\AlgOL$ as a subroutine; the randomized model $\fhat$ returned by this algorithm satisfies,
\begin{equation*}
    \bbE_{\bx \sim \rho, y_T \sim \fhat (\cdot|\bx)} \big[ \bbI (y_T \ne \fstar_T(\bx)) \big] \le \varepsilon,
\end{equation*}
as long as \smash{$\nsample \ge \Theta \big( M^\star (\AlgOL) \cdot \varepsilon^{-1} \big)$}. The query complexity is \smash{$\nquery = \Theta \big( \nsample \cdot \log(T) \big)$}.

\medskip
\noindent If $\AlgOL$ is instantiated as the standard optimal algorithm~\citep{littlestone1988learning}, we achieve sample complexity $\Theta \big( d \log(|S||\Sigma|) \cdot \varepsilon^{-1} \big)$, where $d$ denotes the Littlestone dimension of $\calF$, and query complexity at most a $\calO(\log(T))$ factor larger.
\end{theorem}
\noindent The proof of this result is given in \Cref{sec:online}.

\newpage

\bibliographystyle{plainnat}
\bibliography{refs}

\newpage
\phantomsection
\addcontentsline{toc}{section}{Appendices}

\appendix

\makeatletter
\let\mainaddcontentsline\addcontentsline
\renewcommand{\addcontentsline}[3]{
    \def\appendixtype{#2}
    \def\appendixsection{section}
    \def\appendixsubsection{subsection}
    \ifx\appendixtype\appendixsection
        \mainaddcontentsline{atoc}{#2}{#3}
    \else
        \ifx\appendixtype\appendixsubsection
            \mainaddcontentsline{atoc}{#2}{#3}
        \else
            \mainaddcontentsline{#1}{#2}{#3}
        \fi
    \fi
}
\makeatother

\section*{Appendix Contents}
\begingroup
\setcounter{tocdepth}{2}
\makeatletter
\@starttoc{atoc}
\makeatother
\endgroup

\subsection*{Appendix Organization}

In \Cref{sec:related-work} presents additional related work which studies the problem of learning automata; in \Cref{sec:regex-example} we discuss a concrete example illustrating the semiautomaton induced by a regular expression. \Cref{sec:semiauto-depth1-proofs} gives a self-contained treatment of $\WLdepth{}.\warmup$, which is a simplified, depth-$\frac12$ version of $\WLdepth{}$. \Cref{sec:mainproofsketch} lifts this analysis to the full recursive variant of $\WLdepth{}$ and proves \Cref{theorem:main-generic}, while \Cref{subsec:main-RL-proof} analyzes $\WLdepth{}.\RL$ and proves \Cref{theorem:main-RL}. The remaining appendices collect the supporting technical material: \Cref{sec:proof-lemmas} proves the main lemmas used in the analysis of $\WLdepth{}.\warmup$, and \Cref{sec:proof-auxiliary} gives auxiliary results, including the non-curriculum lower bounds, the potential and density-ratio calculations, the sampling lower bound, and the discussion-section results on online learning and stochastic model classes.

\section{Additional Related Work} \label{sec:related-work}

Semiautomata are closely related to \emph{automata}, which have a rich history in theoretical computer science~\citep{hopcroft2001introduction}. In learning theory specifically, automata have been the subject of a rich line of work, originating with Angluin's seminal $L^\star$ algorithm~\citep{angluin1987learning,rivest1989inference,freund1993efficient,kearns1994cryptographic,ergun1995learning}. The two models differ only in that automata define languages over words $\bw \in \Sigma^*$ based on state evolution and an acceptance function, whereas semiautomata only define state evolution. This distinction is salient when considering the task of PAC learning: for automata, the learner only accesses the acceptance decisions of the underlying system, while we consider learning semiautomata from feedback that reveals the state evolution more directly. 
    
\medskip\noindent
Most of the literature on learning automata has focused on settings where the state is completely unobserved, and the learner must infer about it through accept/reject decisions. Most closely related to the setting considered in this paper are the works of~\cite{angluin2009learning} and~\cite{angluin2015learning} who consider learning settings where terminal states are either fully or partially observable. The recent work of \cite{giapitzakis2025statistical} studies the statistical query complexity of PAC learning semiautomata from random examples, and establish strong SQ hardness for this task.

\subsection{An example of a semiautomaton induced by a regular expression} \label{sec:regex-example}

In this section we expand upon the regular language example in \cref{sec:applications}, by giving a detailed example of a regular expression with alphabetic length $k$ and induced semiautomaton. Consider the regular expression $\texttt{RegEx} \equiv (a|b)^*ab$ over $\Sigma=\{a,b\}$. This corresponds to the set of all strings which terminate in $ab$. The alphabetic length of $\texttt{RegEx}$ is $4$, with symbol occurrences $a_1,b_2,a_3,b_4$, where $a_1$ and $b_2$ correspond to occurrences in the subexpression $(a|b)^*$, while $a_3$ and $b_4$ correspond to occurrences in the final $ab$ term. Each regular expression can be compiled into a non-deterministic finite automaton, as we discuss next.

\begin{definition}[Non-deterministic finite automaton]
Fix an input vocabulary $\Sigma$ and a finite state space $Q$. A nondeterministic finite automaton (NFA) is defined by a transition function $f : Q \times \Sigma \to 2^Q$, which defines a set of states reachable from the current state for each input symbol, and a subset of states $Q_{\texttt{acc}} \subseteq Q$ known as the accepting states. Given an initial state $q_0 \in Q$ and an input sequence $\bw_{1:T} \in \Sigma^T$, the NFA realizes the sequences, $(\bm{q}_t)_{t=0}^T$, where $\bm{q}_0 = \{ q_0 \}$ and $\bm{q}_t = \cup_{q \in \bm{q}_{t-1}} f(q,w_t)$ captures all states reachable from $q_0$ under the input sequence $\bw_{1:t}$. The NFA accepts the input string $\bm{w}_{1:T}$ if $Q_{\texttt{acc}} \cap \bm{q}_T \ne \emptyset$.
\end{definition}

\medskip
\noindent \cite{glushkov1961abstract} shows how to write down an automaton corresponding to a regular language by associating one state for each symbol occurrence, as well as an initial state $0$. Namely, the state space of the Glushkov NFA is $Q=\{ 0,\cdots,4 \}$. The NFA is initialized at the state $0$; when the first letter of an input word is processed, the subsequent state can be either $a_1 \equiv 1$, $b_2 \equiv 2$ or $a_3 \equiv 3$. The (set-valued) transition function of the corresponding NFA, $\piNFA$, is described by the following map: \loose
\begin{align*}
    \piNFA (0,a) = \piNFA(1,a) = \piNFA (2,a) &= \{ 1,3 \} \\
    \piNFA (0,b) = \piNFA (1,b) = \piNFA (2,b) &= \{ 2 \} \\
    \piNFA (3,b) &= \{ 4 \} \\
    \text{For all other } (s,w) \in Q \times \Sigma,\ \piNFA (s,w) &= \emptyset
\end{align*}
An input word is accepted by $\piNFA$ if there exists a computational path which terminates in the state $4$. $\piNFA$ can be translated into a DFA $\piDFA$ using powerset determinization, which tracks the set of all possible states reachable by some computational path in the NFA. Thus, the states in $\piDFA$ correspond to subsets $S\subseteq Q$. For example, on the input word $\bw=(b,a,a,b)$, the subset-state trajectory is:
\begin{equation*}
    \{ 0 \} \xrightarrow{b} \{ 2 \} \xrightarrow{a} \{ 1,3 \} \xrightarrow{a} \{ 1,3 \} \xrightarrow{b} \{ 2,4 \}.
\end{equation*}
$\piDFA$ accepts the subset state as long as the final returned set contains any of the accepting state in $\piNFA$, i.e., $4$. Thus $(b,a,a,b)$ is accepted because its terminal subset is $\{2,4 \}$, which contains the accepting position $4$. This is correct as $\texttt{RegEx}$ recognizes the set of strings which end in $ab$.

\section{\texorpdfstring{Analyzing a Simpler Variant of \Cref{alg:semiauto_main}: $\WLdepth{}.\warmup$}{Analyzing a Variant of \Cref{alg:semiauto_main}: $\WLdepth{}.\warmup$}} \label{sec:semiauto-depth1-proofs}

This section presents a complete description and analysis for the simplified version of \mainalg sketched in \cref{sec:techniques}, which uses only a single level of recursion. The algorithm, which we denote by $\WLdepth{}.\warmup$, decomposes length-$T$ instances into length-$\tau$ ones, where $\tau = T/L$ and $L>1$ is a parameter to be specified. Then, a sequence of models $\fhat^0,\cdots,\fhat^{k-1}$ are trained to solve these length-$\tau$ instances, which are then aggregated using majority voting and composed $L$-fold to get a model which solves length-$T$ instances. Each $\fhat^j$ is trained by querying full CoTs on length-$\tau$ instances and carrying out next-token prediction (\Cref{def:erm}), which we refer to as $\TranscriptNTP$ (\Cref{alg:WLdepth0}). A pictorial depiction is provided in \Cref{fig:depth1}, and a formal description of $\WLdepth{}.\warmup$ is provided within \Cref{alg:semiauto_warmup}.

\medskip
\noindent The purpose of analyzing this variant of $\WLdepth{}$ is to expose and understand the mechanism by which inverted sampling (\Cref{alg:semi-subsample}) converts longer length-$T$ instances (i.e., the dataset $\Dprompt{}$) into shorter ones of length $\tau = T/L$ (i.e., the datasets \smash{$\Dout{j}$} for $j=0,\cdots,k-1$). Our main guarantee, established below, shows that $\WLdepth{}.\warmup$ achieves sample complexity scaling as $\calOtilde (d)$ and query complexity $\calOtilde (d\sqrt{T})$ in order to achieve constant error on length-$T$ instances.

\begin{theorem}[Guarantee for $\WLdepth{}.\warmup$] \label{theorem:depth1-final}
$\WLdepth{}.\warmup$ (\Cref{alg:semiauto_warmup}) with $L = \sqrt{T}$ draws a total of,
\begin{equation*}
    \nsample = \calO \big( d \log(|S|) \log (1/\delta) \cdot \log^9 (T) \big)
\end{equation*}
instances from $\rho$, and makes \smash{$\nquery = \calO(\nsample \sqrt{T})$} calls to the $\etoe(\cdot)$ oracle to return a model \smash{$\fhat$} such that with probability $1-\delta$,
\begin{equation*}
    \Pr_{\bx \sim \rho} \big[ \fhat (\bx) \ne \fstar_T (\bx) \big] \le \frac{1}{4}
\end{equation*}
\end{theorem}
\begin{proof}
The proof sketch of this result and all necessary lemmas are introduced in \Cref{subsec:depth1-final-proof-sketch}. The final proof invoking these lemmas is presented in \Cref{subsec:depth1-final-proof}.
\end{proof}

\begin{algorithm}[t]
\caption{$\WLdepth{}.\warmup (\Dprompt{} \| T,\delta)$: Learning semiautomata via curriculum}\label{alg:semiauto_warmup}
\begin{algorithmic}[1]
\State \textbf{Input:} Class of semiautomaton transitions $\calF$ over state space $S$ and alphabet $\Sigma$,
\Statex \hspace{3em} Number of steps of semiautomaton simulation, $T$,
\Statex \hspace{3em} Target failure probability $\delta$,
\Statex \hspace{3em} Dataset of instances $\Dprompt{} = \{ \bx^i = (s^i_0,\bw^i) \}_{i=1}^{n}$.
% \Statex \hspace{3em} Weak learning algorithm $\Base(\cdot)$ for $\calF$ from i.i.d. next-state data (\Cref{def:iid-learning-F})
\State \textbf{Hyperparameters:} $L \in \bbN$ such that $T \pmod{L} \equiv 0$.
\State \textbf{Instantiate:} $\tau = T/L$,
\Statex \hspace{5.3em} $k = C\log(L)$ for a large constant $C>0$,
%\LineComment{{\color{blue}In $\WLdepth{}.\warmup$, this line is replaced by}}
% \Statex \NoComment{{\color{blue}\smash{$\fhat \gets \texttt{TranscriptERM} (\Dprompt{} \| \tau_0, \delta)$} (\Cref{alg:WLdepth0})}}
\State Split $\Dprompt{}$ into $k$ equal parts, $\big\{ \Dprompt{j} : j \in \{ 0,\cdots,k-1 \} \big\}$.
\For{$j=0,1,\cdots,k-1$,}
\State $\Dout{j} \gets \emptyset$.
\State Let $\calF_j \gets \big\{ \fhat^0,\cdots, \fhat^{j-1} \big\}$. \LineComment{{\color{blue} By convention, $\calF_0 = \emptyset$}}
\For{instances $\bx = (s_0,\bw) \in \Dprompt{j}$}
\State $D_\bx = \split (\bx \| L)$. \LineComment{{\color{blue} Defined in \Cref{alg:split}}}
\State $O \gets \subsample \big(D_\bx \| \calF_j, k \big)$. \LineComment{{\color{blue}Defined in \Cref{alg:semi-subsample}}}
\State If $O \ne \perp$, update $\Dout{j} \gets \Dout{j} \cup \{ O \}$
\EndFor
\State Train model $\fhat^j \gets \TranscriptNTP (\Dout{j} \| \tau, \delta/2k)$
\EndFor
\State \textbf{Return:} $\fhat = \ftilde^{\circ L}$, where $\ftilde = \maj \big( \big\{ \fhat^j : j \in \{ 0,\cdots,k-1\} \big\} \big)$. \LineComment{{\color{blue} $\maj (\cdot)$ returns the plurality of}}
\Statex \NoComment{{\color{blue}the models in its argument}}
\end{algorithmic}
\end{algorithm}

% =============================================================================
% $\WLdepth{}.\warmup$ as an instantiation of $\WLdepth{h}$
% =============================================================================
\begin{algorithm}[t]
\caption{$\WLdepth{}.\warmup \big( \Dprompt{} \, \big\|\, T, \delta \big)$}\label{alg:semiauto_warmup_inst}
\begin{algorithmic}[1]
\Statex {\color{gray} \# Learning semiautomata via curriculum, $\WLdepth{\frac{1}{2}}$ instantiated with base learner $\TranscriptNTP$}
\State \textbf{Input:} Class of semiautomaton transitions $\calF$ over state space $S$ and alphabet $\Sigma$,
\Statex \hspace{3em} Number of steps of semiautomaton simulation, $T$, with $T \pmod{L} \equiv 0$,
\Statex \hspace{3em} Target failure probability $\delta$,
\Statex \hspace{3em} Dataset of instances \smash{$\Dprompt{} = \{ \bx^i = (s_0^i,\bw^i) \}_{i=1}^{n}$}.
\State \textbf{Hyperparameters:} $L \in \bbN$ such that $T \pmod{L} \equiv 0$.
\State \textbf{Instantiate:} $\Base ( \, \cdot \,\|\, \delta' ) \triangleq$ $\TranscriptNTP ( \, \cdot \,\|\, T/L, \delta')$ \LineComment{{\color{blue} $\TranscriptNTP$ defined in \Cref{alg:WLdepth0}}}
\State Decomposition schedule: $L_{\frac{1}{2}} = L$.
\State Branching schedule: $k_{\frac{1}{2}} = C \log (L)$ for a large constant $C>0$.
\State \textbf{Return:} $\fhat \gets \WLdepth{\frac{1}{2}} \big( \Dprompt{} \,\big\|\, T, \delta \big)$ with weak learner $\Base ( \, \cdot \,\|\, \delta')$,
\Statex \hspace{3.5em} decomposition schedule $L_{\frac{1}{2}}$, and branching schedule $k_{\frac{1}{2}}$.
\end{algorithmic}
\end{algorithm}

\begin{algorithm}[t]
\caption{$\TranscriptNTP (D \| \tau, \delta)$: Full-CoT next-token prediction~\citep{joshi2025theory}}\label{alg:WLdepth0}
\begin{algorithmic}[1]
\Statex {\color{gray} \# Learn a model solving length-$\tau$ instances by labeling all intermediate states and running the $\NTP$ oracle.}
\State \textbf{Input:} Dataset of length-$\tau$ instances $D = \{ \bx^i = (s^i_0, \bw^i_{1:\tau}) \}_{i=1}^{n}$,
\Statex \hspace{3em} Supervision oracle $\etoe(\cdot)$, NTP oracle $\NTP(\cdot)$ for $\calF$,
\Statex \hspace{3em} Failure probability $\delta$.
\State \textbf{Instantiate:} $\Dstep \gets \emptyset$.
\For{$i = 1, \cdots, n$}
\For{$t = 0, \cdots, \tau - 1$}
\State Query intermediate state: $s^i_{t+1} \gets \etoe \big( (s^i_0, \bw^i_{1:t+1}) \big)$.
\State $\Dstep \gets \Dstep \cup \big\{ (s^i_t, w^i_{t+1}) \mapsto s^i_{t+1} \big\}$.
\EndFor
\EndFor
\State \textbf{Return:} $\fhat \gets \NTP (\Dstep)$.
\end{algorithmic}
\end{algorithm}

\begin{figure}[t]
\centering
\begin{tikzpicture}[
  >=stealth,
  node distance=1.6cm and 2.0cm,
  box/.style={draw, rounded corners, minimum width=12mm, minimum height=6mm},
  every node/.style={font=\small}
]
% ---------- parameters ----------
\def\k{2}          % number of middle-row boxes (εH)
\def\dx{2.5}       % horizontal spacing of middle row
\def\ymid{-1.8}
\def\ylow{-3.8}
\def\nseg{5}       % number of ε²H segments along the small line
\def\seglen{0.5}   % visual length of each segment
\def\arrowh{0.28}  % height of the small downward arrows above ticks
% --------------------------------

% Keep only the left middle box (F2) — no top box, no other middle boxes
\node[box,draw=white,fill=green!20] (e2) at (-\dx,\ymid) {$\substack{\big( T, \frac{1}{4}\big) \\ \text{Target dist: } \rho\\\fhat^{\circ L}}$};

% Left branch subtree
\node[box,draw=white,fill=blue!20] (eeL) at ($(e2)+(0,\ylow-\ymid,0)$) {$\substack{\big( \frac{T}{L}, \frac{1}{4L} \big)\\ \text{Target dist: } \rhobar\\ \fhat \gets \maj ( \{ \fhat^0,\cdots,\fhat^{k-1} \} )}$};
\draw[->] (e2) -- (eeL);

\pgfmathsetmacro{\offset}{0.75*\dx}
\node[box,draw=white,fill=green!20] (eeL2) at ($(eeL)+(-\offset,\ylow-\ymid,0)$) {$\substack{ \big( \frac{T}{L}, \frac{1}{4} \big) \\[1pt] \text{Target dist: } \rhobar_{k-1}}$};
\node[box,draw=white,fill=green!20] (eeR2) at ($(eeL)+(\offset,\ylow-\ymid,0)$) {$\substack{\big( \frac{T}{L}, \frac{1}{4} \big)\\[1pt] \text{Target dist: } \rhobar_0}$};
\node[below=0.2cm of eeR2, align=center] (text1) {$\substack{\fhat^0}$};
\node[below=0.2cm of eeL2, align=center] (text2) {$\substack{\fhat^{k-1}}$};
\node[left=0.61cm of eeR2, align=center] (text3) {$\substack{\cdots}$};
\draw[->] (eeL) -- (eeL2);
\draw[->] (eeL) -- (eeR2);

\draw[decorate, color=gray, decoration={brace, mirror, raise=2pt, amplitude=4pt}] ($(eeR2)+(-5.25,0.6)$) -- ($(eeR2)+(-5.25,-0.6)$);
\draw[decorate, color=gray, decoration={brace, mirror, raise=2pt, amplitude=3pt}] ($(eeR2)+(-5.25,-0.85)$) -- ($(eeR2)+(-5.25,-1.25)$);
\node[left=4.5cm of eeR2, anchor=east, font=\small] (BOT) {$\substack{\textbf{Boosting:}\\k \,=\, \Theta ( \log (L))\\\text{weak models}\\[1.5pt]\text{at length-}T/L}$};
\node[below left=0.55cm and 4.5cm of eeR2, anchor=east, font=\small] (BOT2) {$\substack{\text{Trained by \Cref{alg:WLdepth0}}\\\text{(querying full CoTs)}}$};

\node[above=0.7cm of BOT, anchor=south, font=\small] (MID) {$\substack{\textbf{Composition:}\\\text{ trading-off length}\\\text{for target error}}$};

\node[above=0.85cm of MID, anchor=south, font=\small] (TOP) {$\substack{\textbf{Target:}\\\text{Model solving}\\[1.5pt]\text{length-}T\text{ instances}}$};

\def\k{2}          % number of middle-row boxes (εH)
\def\dx{2.5}       % horizontal spacing of middle row
\def\ymid{-1.8}
\def\ylow{-3.8}
\def\nseg{7}       % number of ε²H segments along the small line
\def\seglen{0.75}   % visual length of each segment
\def\arrowh{0.28}  % height of the small downward arrows above ticks
\def\linepad{0.9}  % extra horizontal padding to elongate the timeline
\def\braceyoffset{0.16} % vertical offset of braces below the timeline
\def\gap{0.03} % horizontal gap for braces
\pgfmathtruncatemacro{\penseg}{\nseg-3}
\pgfmathtruncatemacro{\prevseg}{\nseg-2}
\pgfmathtruncatemacro{\lastseg}{\nseg-1}
\pgfmathsetmacro{\halfspan}{0.5*(\nseg-1)*\seglen}
\def\arrowlabelA{\smash{$s_0$}}
\def\arrowlabelB{\smash{$s_{\tau}$}}
\def\arrowlabelC{\smash{$s_{2\tau}$}}
\def\arrowlabelD{\smash{$\dots$}}
\def\arrowlabelE{\smash{$s_{T-\tau}$}}
\def\arrowlabelF{\smash{$s_T$}}
\def\arrowlabelAA{\smash{$s_0$}}
\def\arrowlabelBB{}
\def\arrowlabelCC{}
\def\arrowlabelDD{}
\def\arrowlabelEE{}
\def\arrowlabelFF{}
\coordinate (baseL) at ($(4,-4.1)$);
\coordinate (tickStart) at ($(baseL)+(-\halfspan,0.25)$);
\coordinate (tickEnd)   at ($(baseL)+(\halfspan,0.25)$);
\coordinate (lineStart) at ($(tickStart)+(0,0)$);
\coordinate (lineEnd)   at ($(tickEnd)+(0,0)$);

\draw (lineStart) -- (lineEnd) node[right=3pt, font=\small] {};
\node[fill=white, align=center, inner sep=2pt, font=\small] (midellipsis) at ($(baseL)+(0,0.25)$) {$\ldots$};
\node[below=0.1cm of midellipsis, align=center, font=\small] (text3) {$\substack{\text{Label states at intervals of }\tau\\ (L \text{ calls to } \etoe(\cdot))}$};

% \node[below right=0.3cm and -2.35cm of baseL, align=center] (text) {$\Big($ {\scriptsize\parbox{4.3cm}{$\rhobar$: decompose $\bx \sim \rho$ into $L$ shorter instances, sampling one at random}} $\Big)$};

\foreach \j/\label in {0/\arrowlabelA,1/\arrowlabelB,2/\arrowlabelC,\penseg/\arrowlabelD,\prevseg/\arrowlabelE,\lastseg/\arrowlabelF} {
  \coordinate (t\j) at ($(tickStart)+(\j*\seglen,0)$);
  \draw (t\j)++(0,-0.06) -- ++(0,0.12);
  \draw[->, >=stealth, line width=0.3pt] 
    ($(t\j)+(0,0.3)$) -- ($(t\j)+(0,0.10)$);
  \node[above=1pt, font=\scriptsize] (lab\j) at ($(t\j)+(0,0.3)$) {\label};
}

\coordinate (baseL2) at ($(4,-2.1)$);
\coordinate (tickStart2) at ($(baseL2)+(-\halfspan,0.25)$);
\coordinate (tickEnd2)   at ($(baseL2)+(\halfspan,0.25)$);
\coordinate (lineStart2) at ($(tickStart2)+(0,0)$);
\coordinate (lineEnd2)   at ($(tickEnd2)+(0,0)$);

\draw (lineStart2) -- (lineEnd2) node[right=3pt, font=\small] {};
\node[fill=white, align=center, inner sep=2pt, font=\small] (midellipsis2) at ($(baseL2)+(0,0.25)$) {$\ldots$};
\node[below=0.1cm of midellipsis2, align=center, font=\scriptsize] (text4) {$\bx = (s_0,\bw_{1:T}) \sim \rho$};

\foreach \j/\label in {0/\arrowlabelAA,1/\arrowlabelBB,2/\arrowlabelCC,\penseg/\arrowlabelDD,\prevseg/\arrowlabelEE,\lastseg/\arrowlabelFF} {
  \coordinate (t\j) at ($(tickStart2)+(\j*\seglen,0)$);
  \pgfmathparse{(\j==0)||(\j==6)?1:0}
  \ifdim\pgfmathresult pt>0pt
    \draw (t\j)++(0,-0.06) -- ++(0,0.12);
  \fi
  \pgfmathparse{(\j==0)?1:0}
  \ifdim\pgfmathresult pt>0pt
    \draw[->, >=stealth, line width=0.3pt] 
      ($(t\j)+(0,0.3)$) -- ($(t\j)+(0,0.10)$);
    \node[above=1pt, font=\scriptsize] (lab\j) at ($(t\j)+(0,0.3)$) {\label};
  \fi
}

\coordinate (baseL3) at ($(4,-6.1)$);
\coordinate (tickStart3) at ($(baseL3)+(-\halfspan,0.25)$);
\coordinate (tickEnd3)   at ($(baseL3)+(\halfspan,0.25)$);
\coordinate (lineStart3) at ($(tickStart3)+(0,0)$);
\coordinate (lineEnd3)   at ($(tickEnd3)+(0,0)$);

\draw (lineStart3) -- (lineEnd3) node[right=3pt, font=\small] {};
\node[fill=white, align=center, inner sep=2pt, font=\small] (midellipsis3) at ($(baseL3)+(0,0.25)$) {$\ldots$};
\node[below=0.1cm of midellipsis3, align=center, font=\small] (text5) {$\substack{\text{Select one of the $L$ instances to}\\[1.5pt]\text{train on using } \subsample (\cdot)\\\text{to get an example from } \smash{\rhobar_j}}$};

% \node[below right=0.3cm and -2.35cm of baseL, align=center] (text) {$\Big($ {\scriptsize\parbox{4.3cm}{$\rhobar$: decompose $\bx \sim \rho$ into $L$ shorter instances, sampling one at random}} $\Big)$};

\foreach \j/\label in {0/\arrowlabelA,1/\arrowlabelB,2/\arrowlabelC,\penseg/\arrowlabelD,\prevseg/\arrowlabelE,\lastseg/\arrowlabelF} {
  \coordinate (t\j) at ($(tickStart3)+(\j*\seglen,0)$);
  \draw (t\j)++(0,-0.06) -- ++(0,0.12);
  \draw[->, >=stealth, line width=0.3pt] ($(t\j)+(0,0.3)$) -- ($(t\j)+(0,0.10)$);
  \node[above=1pt, font=\scriptsize] (lab\j) at ($(t\j)+(0,0.3)$) {\label};
}
\end{tikzpicture}
\caption{A depiction of $\WLdepth{}.\warmup$. The notation $(T,\varepsilon)$ indicates the length of instances being solved and the desired error tolerance, respectively.}
\label{fig:depth1}
\end{figure}
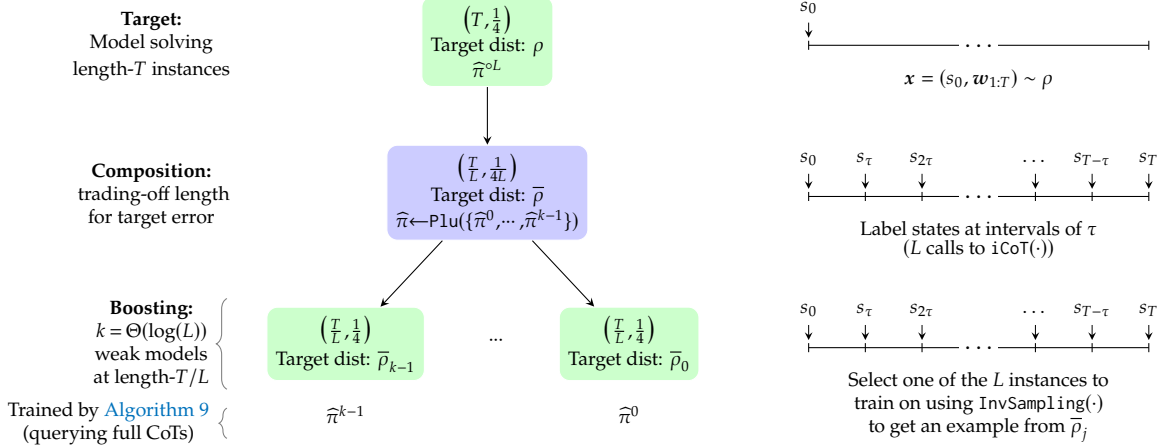

\subsection{Proof Sketch of \Cref{theorem:depth1-final}} \label{subsec:depth1-final-proof-sketch}

In this section, we sketch the proof of \Cref{theorem:depth1-final}, first introducing some relevant notation. We will let $\fhat$ denote the model returned by \Cref{alg:semiauto_warmup}.

\paragraph{Notation.} Recall the definition of $\rho_{t_1:t_2}$ which is the distribution of $(s_{t_1-1},\bw_{t_1:t_2})$ for $(s_0,\bw_{1:T}) = \bx \sim \rho$ and where $s_{t_1-1} = \fstar_{t_1-1} (\bx)$. Let $\sigma$ denote the joint distribution of $\underline{X} = (X_0,\cdots,X_{L-1})$, where each $X_i = (s_t,\bw_{t+1:t+\tau})$, for $t = \tau i$ and $\bx = (s_0,\bw_{1:T}) \sim \rho$. Note that $\underline{X} \sim \sigma$ is a measurable function of the initial state $s_0$ and the word $\bw_{1:T}$, and can be constructed from the instance $\bx = (s_0,\bw_{1:T}) \sim \rho$ by querying $\etoe(\cdot)$ at $L$ timesteps ($t=\tau,2\tau,\cdots,T$). Let $\bX = S \times \Sigma^{\tau}$ denote the space of length-$\tau$ instances. Note that the support of $\sigma$ is $\bX^L$. \textit{Note that the distributions $\sigma$ and $\rho$, while closely related, are different from one another.} We also define the filtration $\calH_{j-1}$ which captures all the randomness up until the point where the dataset \smash{$\Dout{j}$} is collected in iteration $j$ (i.e., the randomness in the construction of the models \smash{$\fhat^0,\cdots,\fhat^{j-1}$} and the randomness in the instances $\bx \sim \rho$ which contribute to the datasets \smash{$\Dout{0},\cdots,\Dout{j-1}$}). In particular, $\calH_{k-1}$ denotes all the random variables involved in the execution of $\WLdepth{}.\warmup$.

\medskip
\noindent With this, we let $\sigmabar$ denote the distribution $\frac{1}{L} \sum_{i=0}^{L-1} \sigma (X_i = \cdot\,)$, which is the uniform mixture over marginals of $\sigma$, which is identical to $\rhobar$. Thus, our earlier discussion in \Cref{obs:lg} demonstrating the central role played by $\rhobar$ in enabling compositional generalization applies to $\sigmabar$ as well. All-in-all, the purpose of introducing additional notation in the form of $\sigma$ is to make it easier to write down the distributions over length-$\tau$ instances on which the weak learners $\fhat^j$ are trained.

\medskip
\noindent  First, we define a sequence of distributions, $\{ \rhobar_j \}_{j = 0}^{k-1}$ supported on $\bX$ obtained by applying a tilt to $\sigmabar$ (or equivalently $\rhobar$),
\begin{align} \label{eq:lambdajstar}
    \rhobar_j (\cdot) &\propto \Big( \frac{1}{L} \sum_{i=0}^{L-1} \sigma ( X_i = \cdot \,) \Big) w_j (\cdot) = \sigmabar (\cdot) \ w_j (\cdot)
\end{align}
where $w_j(\cdot)$ is the rank-based weight function defined in \cref{eq:alpha}. $\rhobar_j$ is the idealized ``good distribution'' under which the weak model in iteration $j$ should be trained as prescribed by boosting-by-filtering~\citep{freund1995boosting}. Sampling from \smash{$\rhobar_j$} is prohibitive, a point which we discuss more in \Cref{subsec:weak-learner-proof}. To this end, we define an alternate sequence of distributions which are easier to sample from, denoted $\{ \lambda_j \}_{j = 0}^{k-1}$, 
\begin{align}
        \lambda_j (\cdot) &\propto \overline{\sigma_{\calK_j}} (\cdot) \ w_j (\cdot) \text{, where, } \overline{\sigma_{\calK_j}} (\cdot) = \Big( \frac{1}{L} \sum_{i=0}^{L-1} \Pr \big( X_i = \cdot \mid (\underline{X},\bm{\eta}_{0:j}) \in \calK_j \text{ and } \calH_{j-1} \big) \Big)  \label{eq:lambdaj}
\end{align}
where the probability is calculated under $\underline{X} = (X_i)_{i=0}^{L-1} \sim \sigma$ and $\bm{\eta}_j = (\eta_{j,i})_{i=0}^{L-1} \sim \Unif([0,1])^{\otimes L}$, is a set of i.i.d. uniform random variables, with \smash{$\bm{\eta}_{0:j} = (\bm{\eta}_{j'})_{j'=0}^j$} and $\bm{\eta} = \bm{\eta}_{0:k-1}$. This definition hinges on the sequence of ``good'' events \smash{$(\calK_j)_{j=0}^{k-1}$}, such that $\calK_0 \subseteq \bX^L \times [0,1]^{(j+1)L}$, which we defer the definition of to later in \cref{eq:Kj-def}. Finally, define the sequence of random variables $\{ R_j \}_{j=0}^{k-1}$ which are measurable functions of $\underline{X} \sim \sigma$,
\begin{equation} \label{eq:Rj}
    R_j = \sum_{i=0}^{L-1} w_j (X_i), \text{ where, } w_j (X_i) = \alpha^{j,k}_{\rank_j (X_i)} \text{ and } \rank_j (X_i) = | \{ \fhat (X_i) = \fstar_\tau (X_i) : \fhat \in \calF_j \} |,
\end{equation}
where $\alpha^{j,k}_r$ is defined in \cref{eq:alpha}. $R_j$ will be referred to as the \textit{regression} in iteration $j$. The name ``regression'' is used to denote the fact that $R_j$ is a measure of how much the accuracy of the plurality of the models $\fhat^0,\cdots,\fhat^{j-1}$ can possibly worsen by, when the new model $\fhat^j$ is added to the ensemble.

\paragraph{Additional Notation.}  In order to state our results succinctly, we introduce some additional notation to make it easier to write down expressions involving events on different spaces. In particular for events $\calA$ defined on $\bX^L \times [0,1]^{k_A L}$ and $\calB$ defined on $\bX^L \times [0,1]^{k_B L}$ with $k_A \ge k_B$, we can define intersections by considering the smallest subset of $(\underline{X},\bm{\eta})$ with respect to which both events are measurable. Namely,
\begin{equation} \label{eq:intersection}
    \calA \cap \calB \iff \big\{ (\underline{X}, \bm{\eta}_{0:k_A-1}) : (\underline{X},\bm{\eta}_{0:k_A-1}) \in \calA \text{ and } (\underline{X},\bm{\eta}_{0:k_B-1}) \in \calB \big\}.
\end{equation}
Similarly, for a third event $\calC \subseteq \bX^L \times [0,1]^{k_C L}$ with $k_C \ge k_A$, containment statements of the form,
\begin{equation} \label{eq:containment}
    \calC \subseteq \calA \cap \calB \iff \calC  \subseteq \big\{ (\underline{X}, \bm{\eta}_{0:k_C-1}) : (\underline{X},\bm{\eta}_{0:k_A-1}) \in \calA \text{ and } (\underline{X},\bm{\eta}_{0:k_B-1}) \in \calB \big\}.
\end{equation}
are defined by defining both sides with respect to the minimal subset of $(\underline{X},\bm{\eta})$ with respect to which all events are measurable. A similar logic can be used to parse statements of the form $(\underline{X},\bm{\eta}) \in \calK_j$ by viewing the LHS as the singleton event in $\bX^L \times [0,1]^{kL}$ and defining membership via containment: $\{ (\underline{X},\bm{\eta}) \} \subseteq \calK_j$.

\medskip
\noindent We first decompose the test error of the model \smash{$\fhat$} returned by $\WLdepth{}.\warmup$ by a potential function argument, to track how the test error of the aggregated ensemble improves as more models are added.

\begin{lemma}[Test-error decomposition] \label{lemma:depth1-breakdown}
Let $\fhat$ denote the model returned by $\WLdepth{}.\warmup$. Then, for a sequence of events $\calK_{k-1} \subseteq \cdots \subseteq \calK_0 \subseteq \bX^L \times [0,1]^{kL}$, the prediction error of $\fhat$ under $\rho$ can be upper bounded as,
\begin{align*}
    \Pr_{\bx \sim \rho} \big[ \fhat (\bx) \ne \fstar_T (\bx) \big] &\le L \cdot \beta^{0,k}_0 + \sum_{j = 0}^{k-1} (\err_j - \err_\star) \cdot \bbE [ R_j \cdot \bbI ( (\underline{X},\bm{\eta}) \in \calK_j) \mid \calH_{k-1} ] + \Pr ( (\underline{X},\bm{\eta}) \not\in \calK_{k-1} \mid \calH_{k-1} )
\end{align*}
Recall that, $\beta^{j,k}_r$ is the weight function defined across \cref{eq:alpha,eq:rank}, $\err_\star = \frac{1}{4}$, while $\err_j$ is the test error of $\fhat^j$ under the distribution $\lambda_j \propto \overline{\sigma_{\calK_j}} (\cdot) w_j (\cdot)$ where $w_j$ is defined in \cref{eq:Rj},
\begin{equation*}
    \err_j = \Pr_{\Xbar \sim \lambda_j} \big[ \fhat^j_\tau (\Xbar) \ne \fstar_\tau (\Xbar) \big].
\end{equation*}
\end{lemma}
\begin{proof}
The proof of this result is presented in \Cref{subsec:depth1-breakdown-proof}.
\end{proof}

\noindent Note that \Cref{lemma:depth1-breakdown} allows the events $\calK_j$ to be defined on a subset of $\bX^L \times [0,1]^{kL}$ and the containment $\calK_j \subseteq \calK_{j-1}$ is defined via \cref{eq:containment}. Overall this result decomposes the test error of $\fhat$ into $3$ terms; the remainder of the proof of \Cref{theorem:depth1-final} is to show that, with high probability, each of these terms is upper bounded by a constant.

\medskip
\noindent The proof of \Cref{lemma:depth1-breakdown} introduces a potential function $\Phi_j$ which tracks the performance of the plurality of the models $\fhat^0,\cdots,\fhat^{j}$ trained until step $j+1$. The term \smash{$L \cdot \beta^{0,k}_0$} bounds the initial value of this potential, $\Phi_0$. The quantity $\beta^{0,k}_0$ decays exponentially with $k$, and so, as long as $k = \Omega (\log(L))$, the first term is bounded by a sufficiently small absolute constant, say $\frac{1}{16}$, which is within the target error guarantee of \Cref{theorem:depth1-final}. This bound on $k$ should be viewed as a requirement on the number of weak models to be trained for the plurality of the models to achieve error $\calO (1/L)$ under $\rho$.

\smallskip
\noindent The remaining terms in \Cref{lemma:depth1-breakdown} are easiest to interpret when $\calK_{k-1} = \cdots = \calK_0 = \bX^L \times [0,1]^{kL}$, where they reduce to $\sum_{j=0}^{k-1} (\err_j - \err_\star) \cdot \bbE [ R_j \mid \calH_{k-1} ]$. This expression brings out the role of the regression $R_j$ and a clean tradeoff to establish: in iterations where generating training examples from $\lambda_j$ is easy, $\err_j$ is likely to be smaller than the threshold $\err_\star = \frac{1}{4}$, resulting in $(\err_j - \err_\star) \cdot \bbE [ R_j \mid \calH_{k-1} ] \le 0$. In the iterations where sampling from $\lambda_j$ is expensive, we need to argue that $R_j$ is also likely to be small. The interpretation of such a result is that, in the iterations where sampling from $\lambda_j$ is expensive, there is sufficient leeway in the plurality of the existing models that adding an arbitrary \smash{$\fhat^j$} to the ensemble does not hurt the error significantly.

\medskip
\noindent Toward establishing such a guarantee, we first define an ``abort'' event which captures whether the dataset collected in iteration $j$ is sufficiently large. This is a proxy for whether sampling from $\lambda_j$ is easy, and will be a sufficient condition to capture whether $\err_j \le \err_\star$. We first introduce some notation in order to define this event. Let,
\begin{equation} \label{eq:pj}
    p_j = \Pr_{\bx \sim \rho} \big( \subsample (D_\bx \ \| \ \calF_j, k) \ne \perp \big| \ \calH_{j-1} \big), \text{ where, } D_\bx = \split (\bx \| L)
\end{equation}
denote the probability that an instance $\bx \sim \rho$ contributes a shorter instance $O$ to the dataset $\Dout{j}$ through split-and-inverted-sampling (cf. \Cref{alg:semi-subsample,alg:split}). It is a short calculation to derive the formula,
\begin{align} \label{eq:pj-formula}
    p_j = \Pr ( |\calI_j| \ge 1 \mid \calH_{j-1}) = \bbE \left[ 1 - \prod_{i=0}^{L-1} \bigg( 1 - \frac{\alpha^{j,k}_{\rank_j (X_i)}}{\alpha^{j,k}_{\max}} \bigg) \ \middle| \ \calH_{j-1} \right]
\end{align}
where $\calI_j = \Big\{ 0 \le i \le L-1 : \eta_{j,i} \le \frac{w_j (X_i)}{\| w_j \|_\infty} \Big\}$ denotes the random set of accepted indices in a run of $\subsample$ (\Cref{alg:semi-subsample}) when applied to $D_{\bx}$ for $\bx \sim \rho$ in the $j^{\text{th}}$ iteration; \smash{$\underline{X} = (X_i)_{i=0}^{L-1}$} is obtained by decomposing the length-$T$ instance $\bx$ into $L$ length-$\tau$ instances.\footnote{i.e., define $\underline{X} = (X_j)_{j=0}^{L-1}$ with $X_j = (s_{j\tau}, \bw_{j\tau+1:(j+1)\tau})$, is constructed from $\bx=(s_0,\bw_{1:T})$.} For completeness, we prove \cref{eq:pj-formula} later via the chain of equations in \cref{eq:010291900} to \cref{eq:010291901}. We also define a slight modification of $p_j$, which captures the same acceptance probability for split-and-inverted-sampling, but conditioned on the event $\calK_{j-1}$,
\begin{equation} \label{eq:pjtilde}
    \ptilde_j = \Pr \left( |\calI_j| \ge 1 \ \middle| \ \calK_{j-1}, \calH_{j-1} \right), \text{ and define } p_j^\star = \frac{1}{1024k^{5/2}}.
\end{equation}
While $p_j$ above captures the probability that a sample $\bx \sim \rho$ is accepted by split-and-inverted-sampling, we define the abort event in terms of $\ptilde_j$. In particular, this implies that there may be iterations $j$ where the dataset \smash{$\Dout{j}$} is large (i.e., $p_j$ is large), but the iteration still aborts because the effective dataset collected (as determined by $\ptilde_j$) is too small. Define,
\begin{equation} \label{eq:abort}
    \abort[j] = \big\{ \ptilde_j \le p_j^\star \big\}
\end{equation}

\subsubsection{Analyzing terms in the test-error decomposition of \Cref{lemma:depth1-breakdown}}
Our main argument to analyze the main summation $\sum_{j=0}^{k-1} (\err_j - \err_\star) \cdot \bbE [R_j \cdot \bbI ( (\underline{X},\bm{\eta}) \in \calK_j) \mid \calH_{k-1} ]$ in \Cref{lemma:depth1-breakdown} is shown next. We will split our analysis of the main summation into two parts, one across iterations that abort, and the other across iterations that don't abort. The high-level idea is to show that in any iteration $j$ which does not abort, $\err_j \le \err_\star$ with high probability, which means that the corresponding term in the summation is non-positive; in any iteration $j$ which aborts, the regression $R_j$ is small with moderate probability, and using the event $\calK$ to avoid the potentially long tails of this random variable. Throughout the sketch, we will emphasize the conditions required to define the event $\calK$, before ultimately defining it.
 
\paragraph{Aborted iterations.}  In any iteration $j$ which aborts, we show in \Cref{lemma:truthful-abort-semi-1} that this implies that $\ptilde_j$ must be small with high probability. This is further used to imply, in \Cref{lemma:abort=>lowreg}, that the regression $R_j$ is typically small: it will only ever exceed some threshold $\rjmax$ with moderately low probability $q_j \lesssim \frac{1}{k}$. Suppose the event $\calK_j$ satisfies the following condition in any iteration $j$,
\begin{equation*} \label{eq:C1} \tag{C1}
    (\underline{X},\bm{\eta}) \in \calK_j \implies \big\{ \abort[j] \text{ and } R_{j} > \rjmax \big\}^c.
\end{equation*}
Then, in any iteration $j$ which aborts, we can bound the corresponding term in the summation in \Cref{lemma:depth1-breakdown} by, 
\begin{align*}
    (\err_j - \err_\star) \cdot \bbE [ R_j \cdot \bbI ( (\underline{X},\bm{\eta}) \in \calK_j) \mid \calH_{k-1} ] &\le \bbE [ R_j \cdot \bbI ( R_{j} \le \rjmax ) \mid \calH_{k-1} ] \\
    &\le \bbE [ R_j \cdot \bbI (R_j \le \rjmax) \mid \calH_{k-1} ] \\
    &\le \rjmax.
\end{align*}
The $\rjmax$'s are shown to satisfy the upper bound, $\sum_{j=0}^{k-1} \rjmax \le \frac{1}{8}$ in \Cref{lemma:CoT-alphamax-sum}. This means that the contribution to the summation across all aborted iterations is bounded by a constant. The role of the failure probability on the regression upper bound, $q_j$, will appear later in showing that the (yet to be defined) sequence of events $\calK_0,\cdots,\calK_{k-1}$ satisfies with high probability,
\begin{equation*} \label{eq:C2}
    \Pr( (\underline{X},\bm{\eta}) \not\in \calK_{j-1} \mid \calH_{j-1} ) \le \frac{1}{16} \tag{C2}
\end{equation*}
With $j=k$, this is the final term in the error decomposition in \Cref{lemma:depth1-breakdown} and is established in \Cref{lemma:notUk}. The same lower bound on the probability of each $\calK_j$ will also be useful in showing guarantees for the weak models, $\fhat^j$, trained in non-aborted iterations. This is discussed as a part of the next section.

\paragraph{Non-aborted iterations.} If iteration $j$ does not abort, we show in \Cref{lemma:weak-learner} that with high probability, $\err_j \le \err_\star$. By extension, this implies that the corresponding term in the summation in \Cref{lemma:depth1-breakdown}  $(\err_j - \err_\star) \cdot \bbE [ R_j \cdot \bbI ( (\underline{X},\bm{\eta}) \in \calK_j ) \mid \calH_{k-1} ] \le 0$. Recall \smash{$\err_j = \Pr_{\Xbar \sim \lambda_j} \big[ \fhat^j_\tau (\Xbar) \ne \fstar_\tau (\Xbar) \big]$} itself is the test error of \smash{$\fhat^j$} under the distribution of instances \smash{$\lambda_j \propto \overline{\sigma_{\calK_j}} (\cdot) w_j(\cdot)$}, a distribution which depends on the event $\calK_j$. Letting \smash{$\lambdahat_j$} denote the distribution over instances which examples in the i.i.d. dataset \smash{$\Dout{j}$} follow, \Cref{lemma:weak-learner} is proved by showing that the following condition,
\begin{equation} \label{eq:C3}
    \left\| \frac{\lambda_j}{\lambdahat_j} \right\|_\infty \le \calO \left( \frac{\log(L)}{\ptilde_j} \right) \tag{C3}
\end{equation}
is satisfied for all iterations $j$ which do not abort, as long as $\calK_j$ is defined appropriately in terms of $\calK_{j-1}$ and $\calU_j$ (\Cref{lemma:density-ratio-1}). From the definition of an aborted iteration, we know that \smash{$\ptilde_j$} is at least \smash{$p_j^\star$}, which implies a density ratio upper bound scaling as \smash{$k^{\calO (1)} \log(L)$}. Note finally, that examples in the distribution \smash{$\Dout{j}$} come from \smash{$\lambdahat_j$}, and the acceptance probability of split-and-inverted-sampling to accept $\bx \sim \rho$ to generate a sample from this distribution is $p_j$; noticing that,
\begin{equation*}
    p_j = \Pr \big( |\calI_j| \ge 1 \mid \calH_{j-1} \big) \ge \Pr \big( |\calI_j| \ge 1 \mid \calK_{j-1}, \calH_{j-1} \big) \cdot \Pr \big( \calK_{j-1} \mid \calH_{j-1} \big) \gtrsim \ptilde_j
\end{equation*}
where the last inequality uses \eqref{eq:C2}.

\medskip
\noindent Finally, $\WLdepth{}.\warmup$ uses next-token prediction to learn the model $\fhat^j$ on this dataset. Using prior guarantees for next-token prediction~\citep{joshi2025theory} and a change-of-measure argument using \eqref{eq:C2}, we have learning guarantees for $\fhat^j$ under $\lambda_j$.

\medskip
\noindent The existence of an event $\calK_j$ establishing \eqref{eq:C1}, \eqref{eq:C2} and \eqref{eq:C3} is established in \Cref{lemma:density-ratio-1,lemma:notUk}. The above analysis across aborted and non-aborted iterations combines to simplify the summation in \Cref{lemma:depth1-breakdown} to the following inequality: with high probability,
\begin{align} 
    \Pr_{\bx \sim \rho} \big[ \fhat (\bx) \ne \fstar_T (\bx) \big] &\le L \cdot \beta^{0,k}_0 + \sum_{j=0}^{k-1} \rjmax  + \Pr ( (\underline{X},\bm{\eta}) \not\in \calK_{k-1} \mid \calH_{k-1} ) \nonumber \\
    &\le \frac{1}{16} + \frac{1}{8} + \frac{1}{16} = \frac{1}{4} \label{eq:testerrorub}
\end{align}
The final inequality follows from the bound on $\sum_{j=0}^{k-1} \rjmax \le \frac{1}{8}$, and choosing $k$ to be sufficiently large (as $\Theta (\log(L))$), which gives \smash{$L \cdot \beta_0^{0,k} \le \frac{1}{16}$}. This uses a bound which shows that \smash{$\beta^{0,k}_0$} decays exponentially fast in $k$ (\Cref{lemma:beta00}).

\subsection{Proving \Cref{theorem:main-generic}: relevant lemmata}
\noindent Having established the broader program, we introduce the lemmas mentioned above. We will focus on the non-aborted iterations first, and then discuss the aborted iterations next.

\subsubsection{Non-aborted iterations $\big(\, \ptilde_j \ge p^\star_j \big)$.}

In iterations $j$ which do not abort, we show that $\fhat^j$ achieves low test error under $\lambda_j$ with high probability, establishing a weak learning guarantee. Since the proof relies heavily on the $\calK_j$ events, we first define these events before proving that they satisfy the structural conditions laid out in \eqref{eq:C1} to \eqref{eq:C3}.

\paragraph{Definition of \smash{$(\calK_j)_{j=0}^{k-1}$}.} For each $0 \le j \le k-1$, the event $\calK_j \subseteq \bX^L \times [0,1]^{(j+1)L}$, we construct, is a function of $\underline{X} \sim \sigma$ and $\bm{\eta}_{0:j}$; recall \smash{$\bm{\eta} = (\bm{\eta}_{j'} : 0 \le j' \le k-1)$}, where \smash{$\bm{\eta}_{j'} = (\eta_{j',i})_{i=0}^{L-1} \sim \Unif([0,1])^{\otimes L}$} is a sequence of i.i.d. random variables. The $\calK_j$ events will be constructed recursively: for all $0 \le j \le k-1$,
\begin{equation} \label{eq:Kj-def}
    \calK_j \gets \construct ( \calU_j \cap \calK_{j-1} \, \| \, c,\sigma, \calH_{j-1} ) \text{ where } c = \frac{1}{32k},
\end{equation}
with $\calK_{-1} = \bX^L$. The definition of $\construct$ is deferred to later (\Cref{lemma:density-ratio-1}), and we define the \smash{$(\calU_j)_{j=0}^{k-1}$} events next. Each $\calU_j \subseteq \bX^L$ is a \textit{random} event which is a measurable function of $\calH_{j-1}$, and is defined as,
\begin{equation} \label{eq:Uj-def}
    \calU_j = \big\{ \ptilde_j \le p^\star_j \text{ and } R_j > \rjmax \big\}^c.
\end{equation}
Here, the regression, $R_j$ was defined earlier in \cref{eq:Rj} and the thresholds $\rjmax>0$ are defined later in \cref{eq:rj}, and the conditional acceptance probability $\ptilde_j$ is defined in \cref{eq:pjtilde}.

\medskip
\noindent In order to complete the definition of the $\calK_j$'s, we still need to define the $\construct$ algorithm which defines $\calK_j$ in terms of $\calU_j$ and $\calK_{j-1}$. In the subsequent lemma, we show the existence of such an algorithm which enables the resulting $\calK_j$'s to satisfy the structural constraints in \eqref{eq:C1} to \eqref{eq:C3}.

\begin{lemma}[$\calK_j$'s satisfy \eqref{eq:C1} and \eqref{eq:C3}] \label{lemma:density-ratio-1}
Fix some $0 \le j \le k-1$, and recall that \smash{$\lambdahat_j$} denotes the distribution over instances the i.i.d. dataset \smash{$\Dout{j}$} is sampled from. Consider any filtration \smash{$(\calH_j)_{j=0}^{k-1}$} and sequence of events, \smash{$(\calU_j)_{j=0}^{k-1}$} such that $\calU_j \subseteq \bX^L$ and is a measurable function of $\calH_{j-1}$. Fix some constant $c \in (0,1)$. There exists an algorithm $\construct ( \,\cdot\, \| \, c, \sigma, \calH_{k-1})$ which iteratively constructs a sequence of events $\calK_j \in \bX^L \times [0,1]^{(j+1)L}$ as:
\begin{equation*}
    \calK_j \gets \construct ( \calU_j \cap \calK_{j-1} \, \| \, c, \sigma, \calH_{j-1} )
\end{equation*}
where $\calK_{-1} = \bX^L$. The $\calK_j$ events satisfy the following conditions:
\begin{equation} \label{eq:12311111}
    \calK_j \subseteq \calU_j \cap \calK_{j-1} \quad \text{and} \quad \Pr ( (\underline{X},\bm{\eta}) \not\in \calK_j \mid (\underline{X},\bm{\eta}) \in \calU_j \cap \calK_{j-1} , \calH_{j-1} ) \le c
\end{equation}
Finally, the distribution $\lambda_j (\cdot) \propto \overline{\sigma_{\calK_j}} (\cdot) w_j(\cdot)$ over $\bX$ introduced in \cref{eq:lambdaj}, which depends on $\calK_j$ satisfies the following condition,
\begin{equation} \label{eq:11231}
    \left\| \frac{\lambda_j}{\lambdahat_j} \right\|_\infty \lesssim \frac{ c^{-1} \log (L)}{\Pr \left( \calU_j \cap \calK_{j-1}  \, \middle| \, |\calI_j| \ge 1 , \calH_{j-1} \right)}
\end{equation}
Here, $\calI_j$, defined in \cref{eq:pj-formula}, is the set of accepted indices in a run of split-and-inverted-sampling (\Cref{alg:split,alg:semi-subsample}) for $\bx \sim \rho$ in the $j^{\text{th}}$ iteration.
\end{lemma}
\begin{proof}
The proof of this lemma is discussed in \Cref{subsec:density-ratio-1-proof}. $\construct(\cdot)$ is defined in \cref{eq:calK-new}.
\end{proof}

\paragraph{Interpretation of \Cref{lemma:density-ratio-1}.} In order to interpret this result, consider the special case where $k=1$, $\calU_0 = \bX^L$ and $c$ is small (we remove the $j$ subscript on $\calK_j$, $w_j$ for succinctness). In this setting, $\construct$ creates an event $\calK \subseteq \bX^L \times [0,1]^L$ such that, $(a)$ $\Pr ( \calK ) \ge 1-c$ is a high probability event, and furthermore, $(b)$ \smash{$\big\| \frac{\nu}{\nuhat} \big\|_\infty \lesssim c^{-1} \log(L)$} where $\nu$ is the distribution $\propto \sigma_{\calK} (\cdot) w (\cdot)$ and $\nuhat$ is the distribution over inputs realized by split-and-inverted-sampling for $\bx \sim \rho$. We will focus on this special case and discuss it further in \Cref{subsec:density-ratio-1-proof}, showing that this guarantee is an exponential improvement over the best guarantee achievable if the event $\calK$ was forced to equal $\bX^L \times [0,1]^L$.

\medskip
\noindent While \Cref{lemma:density-ratio-1} gives us a bound on the density ratio between $\lambda_j$ and $\lambdahat_j$, the term in the denominator depends on the quantity $\Pr \big( \calU_j \cap \calK_{j-1} \mid |\calI_j| \ge 1 ,\calH_{j-1} \big)$ which is a-priori different from the bound we hinted at earlier in \eqref{eq:C3}, involving $\ptilde_j$. Our main argument will be to show that this term cannot be too small in an iteration which does not abort. Observe that when $\ptilde_j > p^\star_j$, the event $\calU_j$ is satisfied, and therefore,
\begin{align}
    \Pr ( \calU_j \cap \calK_{j-1} \mid |\calI_j| \ge 1, \calH_{j-1}) &= \Pr ( \calK_{j-1} \mid |\calI_j| \ge 1, \calH_{j-1})  \label{eq:drb-den-bd-1}\\
    &\ge \Pr ( |\calI_j| \ge 1 \text{ and } \calK_{j-1} \mid \calH_{j-1}) \nonumber\\
    &\ge \Pr ( |\calI_j| \ge 1 \mid \calK_{j-1} , \calH_{j-1}) \cdot \Pr (\calK_{j-1} \mid \calH_{j-1}) \ge \frac{15}{16} \cdot \ptilde_j \label{eq:drb-den-bd-2}
\end{align}
where we use \eqref{eq:C2} (which we establish next in \Cref{lemma:notUk}) to lower bound $\Pr (\calK_{j-1} \mid \calH_{j-1})$. By definition of the iteration $j$ not aborting, we know that \smash{$\ptilde_j \ge p^\star_j$} is lower bounded. And furthermore, by the definition of \smash{$p_j^\star$} as well as the manner in which we designed the sequence of events $\calK_j$ in \cref{eq:Kj-def} choosing \smash{$c=\frac{1}{32k}$}, we have the density ratio bound,
\begin{equation} \label{eq:lambdalambdahat-dbound}
    \left\| \frac{\lambda_j}{\lambdahat_j} \right\|_\infty \lesssim k^{7/2} \log(L)
\end{equation}
Next we establish that the $\calK_j$ events satisfy the condition \eqref{eq:C2}, showing that it covers a moderate amount of mass.

\begin{lemma}[$\calK_j$'s satisfy \eqref{eq:C2}] \label{lemma:notUk}
For all $0 \le j \le k-1$, $\Pr ( \calK_j \mid \calH_j ) = \Pr ( \calK_j \mid \calH_{j-1} ) \ge \frac{15}{16}$.
\end{lemma}
\begin{proof}
The proof of this result is discussed in \Cref{subsec:notUk-proof}.
\end{proof}

\medskip
\noindent \Cref{lemma:notUk} is proved by an iterative argument which peels off the contribution of the previous iterations $j' \le j$ to $\calK_j$ one step at a time. Note that these events depend on the behavior of the regression $R_j$ across aborted iterations, and so the tools required to prove this result are discussed later in \Cref{sec:aborted}.

\medskip
\noindent While \Cref{lemma:notUk} is used to arrive at the density ratio bound in \cref{eq:lambdalambdahat-dbound}, this result also gives us a handle on $\Pr (\calK_{k-1} \mid \calH_{k-1})$, which appears in the breakdown of the test-loss of $\fhat$ in \Cref{lemma:depth1-breakdown}.

\medskip
\noindent Finally, note that while \cref{eq:lambdalambdahat-dbound} gives us a bound on the density ratio between $\lambda_j$ and $\lambdahat_j$ in iterations which don't abort, we still need to argue that the learner can generate sufficiently many samples from the latter distribution (i.e., a sufficiently large dataset $\Dout{j}$) to learn from in these iterations. This will follow from the fact that a lower bound on $\ptilde_j$ in the non-aborted iteration, \textit{also certifies a lower bound on $p_j$ in these iterations.} Namely,
\begin{equation} \label{eq:pj-lb}
    \ptilde_j = \Pr ( |\calI_j| \ge 1 \mid \calK_{j-1}, \calH_{j-1}) \le \frac{\Pr ( |\calI_j| \ge 1 \text{ and } \calK_{j-1} \mid \calH_{j-1})}{\Pr (\calK_{j-1} \mid \calH_{j-1})} \le \frac{p_j}{(15/16)}
\end{equation}
Where the last inequality is by \Cref{lemma:notUk} and the definition of $p_j$. All in all, via a concentration argument, this implies that in iterations which don't abort, the dataset \smash{$\Dout{j}$} is unlikely to be small, and is at least $|\Dprompt{}| / k^{\calO(1)}$ in size. This is shown formally in the following lemma.

\begin{lemma}[Large $\ptilde_j \implies \nd_j$ instantiated with a large dataset] \label{lemma:truthful-abort-semi-1}
Suppose for a large constant $C_1 > 0$,
\begin{equation} \label{eq:dp-bound}
    |\Dprompt{}| \ge C_1 k^7 \log(L) \cdot ( d \log(|S|) + \log(k/\delta)).
\end{equation}
If $\ptilde_j > p_j^\star$ in iteration $j$, then $\Pr \big( |\Dout{j}| < C_2 k^{7/2} \log(L) \cdot ( d \log(|S|) + \log(k/\delta)) \mid \calH_{j-1} \big) \le \frac{\delta}{2k}$.
\end{lemma}
\begin{proof}
The proof of this lemma is discussed in \Cref{subsec:truthful-abort-semi-proof}.
\end{proof}

\medskip
\noindent Finally, we invoke the density ratio bound in \cref{eq:lambdalambdahat-dbound}, along with the sufficiently large size of the dataset \smash{$\Dout{j}$} in any non-aborted iteration $j$ (\Cref{lemma:truthful-abort-semi-1}) to establish weak learning guarantees for \smash{$\fhat^j$}. Recall from $\WLdepth{}.\warmup$, that $\fhat^j$ is trained via next-token prediction on \smash{$\Dout{j}$}, which is a dataset of length-$\tau$ instances drawn from \smash{$\lambdahat_j$}. The prior work of \citep{joshi2025theory} established learning guarantees for next-token prediction, and via a change-of-measure argument to transfer learning guarantees under $\lambdahat_j$ to $\lambda_j$ (facilitated by the density ratio bound in \Cref{lemma:density-ratio-1}), this implies learning guarantees for $\fhat^j$ under the distribution $\lambda_j$. This is shown in the following lemma.

\begin{lemma}[Weak learning guarantee for $\fhat^j$] \label{lemma:weak-learner}
Consider any iteration $j$ where $\ptilde_j > p_j^\star$. With probability at least \smash{$1-\frac{\delta}{k}$}, the model \smash{$\fhat^j$} trained in iteration $j$ satisfies,
\begin{align*}
    \Pr_{\Xbar \sim \lambda_j} \big[ \fhat^j (\Xbar) \ne \fstar_\tau (\Xbar) \big] = \err_j \le \err_\star = \frac{1}{4}
\end{align*}
\end{lemma}
\begin{proof}
The proof of this result is discussed in \Cref{subsec:weak-learner-proof}.
\end{proof}

\noindent With this, we complete the analysis in the non-aborted iterations. Next we move on to the aborted iterations.

\subsubsection{Aborted iterations $\big( \, \ptilde_j \le p^\star_j \big)$} \label{sec:aborted}

In the aborted iterations, $\ptilde_j$ may be too small, so we are no longer guaranteed a small bound on the density ratio between $\lambda_j$ and $\lambdahat_j$ through \Cref{lemma:density-ratio-1}. In such iterations, our argument will be to show that the regression $R_j$, which controls the test-error of $\fhat$ (cf. \Cref{lemma:depth1-breakdown}) is small with moderate probability. In particular, from its definition in \cref{eq:pjtilde}, and the structure of $\calI_j$ as being a sum of independent indicators,
\begin{align}
    \ptilde_j = \Pr \big( |\calI_j| \ge 1 \mid \calH_{j-1}, \calK_{j-1} \big) &= \bbE \left[ 1 - \prod_{i=0}^{L-1} \bigg( 1 - \frac{\alpha^{j,k}_{\rank_j (X_i)}}{\alpha^{j,k}_{\max}} \bigg) \ \middle| \ \calK_{j-1}, \calH_{j-1} \right]
\end{align}
On the other hand, we have \smash{$R_j = \sum_{i=0}^{L-1} \alpha^{j,k}_{\rank_j (X_i)}$}. When $\ptilde_j$ is small, we can show a moderate probability upper bound on $R_j$ via an application of Markov's inequality. This implies that in iterations which abort, the regression is likely to be small. In order to state the bound on the regression, we define the threshold,
\begin{equation} \label{eq:rj}
    \rjmax \triangleq \frac{\alpha^{j,k}_{\max}}{16 \sqrt{k}}
\end{equation}

\begin{lemma}[Aborted iterations have low regression] \label{lemma:abort=>lowreg}
If $\ptilde_j \le p_j^\star$, then, $\Pr \big( R_j > \rjmax \mid \calK_{j-1}, \calH_{j-1} \big) \le \frac{1}{32k^2}$.
\end{lemma}
\begin{proof}
The proof of this result is given in \Cref{lemma:abort=>lowreg-proof}.
\end{proof}

\noindent While the above result shows that $R_j$ is small with moderate probability, the bound on the error of $\fhat$ in \Cref{lemma:depth1-breakdown} relies on the \textit{expected} regression, which is a-priori not guaranteed to be small. This is handled by the fact that we are guaranteed $(\underline{X},\bm{\eta}) \in \calK_j \implies \{ \ptilde_j \le p^\star_j \text{ and } R_j > \rjmax \}^c$. This means that the regression in \Cref{lemma:depth1-breakdown} in aborted iterations is gated by the $\rjmax$ threshold, which is small, and we avoid the tails of $R_j$. We introduce a short lemma proving that the cumulative regression threshold, $\sum_{j=0}^{k-1} \rjmax$ is indeed small, necessary to complete this argument.

\begin{lemma}[Bound on cumulative regression] \label{lemma:CoT-alphamax-sum}
$\sum_{j=0}^{k-1} \rjmax \le \frac{1}{8}$.
\end{lemma}
\begin{proof}
The proof of this lemma is similar to that of \cite[Lemma 3.9]{freund1995boosting}. Noting that for $\err_\star = \frac{1}{4}$, the mode of the binomial PMF gives for $0 \le j \le k-2$, \smash{$\alpha_{\max}^{j,k} \le \frac{1}{\sqrt{2\pi \cdot \frac{3}{16}}} \cdot \frac{1}{\sqrt{k-j-1}} \le \frac{0.95}{\sqrt{k-j-1}}$}, while for $\alpha_{\max}^{k-1,k} \le 1$. As a consequence,
\begin{equation*}
    \sum_{j=0}^{k-1} \rjmax \le \frac{1}{16\sqrt{k}} \left( \sum_{j=0}^{k-2} \frac{0.95}{\sqrt{k-j-1}} + 1 \right) \le \frac{1}{8}.
\end{equation*}
\end{proof}

\medskip
\noindent Having established all the necessary tools, we are ready to furnish a proof of \Cref{theorem:depth1-final} via bounding the test-error decomposition for $\fhat$ in \Cref{lemma:depth1-breakdown}.

\subsection[Proof of Theorem~\ref{theorem:depth1-final}]{Proof of \Cref{theorem:depth1-final}} \label{subsec:depth1-final-proof}

We prove accuracy, sample complexity and query complexity guarantees for $\WLdepth{}.\warmup$ below.

\paragraph{Accuracy bound.} Let $\calJ_{\abort} \triangleq \big\{ 0 \le j \le k-1 : \ptilde_j \le p^\star_j \big\}$ denote the set of aborted iterations, which is $\calH_{k-2}$ measurable. By the test-error decomposition in \Cref{lemma:depth1-breakdown},
\begin{align*}
    &\Pr_{\bx \sim \rho} \big[ \fhat (\bx) \ne \fstar_T (\bx) \big] \\&\le L \cdot \beta^{0,k}_0 + \sum_{j = 0}^{k-1} (\err_j - \err_\star) \cdot \bbE [ R_j \cdot \bbI ( (\underline{X},\bm{\eta}) \in \calK_j) \mid \calH_{k-1} ] + \Pr ( (\underline{X},\bm{\eta}) \not\in \calK_{k-1} \mid \calH_{k-1} )\\
    &\overset{(a)}{\le} \frac{1}{16} + \sum_{j = 0}^{k-1} (\err_j - \err_\star) \cdot \bbE [ R_j \cdot \bbI ( (\underline{X},\bm{\eta}) \in \calK_j) \mid \calH_{k-1} ] + \frac{1}{16} \\
    &\le \frac{1}{8} + \sum_{j \in \calJ_{\abort} } \bbE [ R_j \cdot \bbI ( (\underline{X},\bm{\eta}) \in \calK_j) \mid \calH_{k-1} ] + \sum_{j \not\in \calJ_{\abort} } (\err_j - \err_\star) \cdot \bbE [ R_j \cdot \bbI ( (\underline{X},\bm{\eta}) \in \calK_j) \mid \calH_{k-1} ] \\
    &\le \frac{1}{8} + \sum_{j \in \calJ_{\abort} } \bbE [ R_j \cdot \bbI \big( \{ \ptilde_j \le p_j^\star \text{ and } R_j > \rjmax \}^c \big) \mid \calH_{k-1} ] + L \sum_{j \not\in \calJ_{\abort} } (\err_j - \err_\star)_+ \\
    &\le \frac{1}{8} + \sum_{j \in \calJ_{\abort} } \rjmax + L \sum_{j \not\in \calJ_{\abort} } (\err_j - \err_\star)_+ \\
    &\overset{(b)}{\le} \frac{1}{4} + L \sum_{j \not\in \calJ_{\abort} } (\err_j - \err_\star)_+,
\end{align*}
where in $(a)$ we use the choice of $k = \Theta (\log(L))$ and the exponential decay of $\beta^{0,k}_0$ as a function of $k$ (cf. \Cref{lemma:beta00}), the upper bound on $\Pr ( (\underline{X},\bm{\eta}) \not\in \calK_{k-1} \mid \calH_{k-1})$ in \Cref{lemma:notUk}, and the definition of $\calK_j$. And in $(b)$ we invoke \Cref{lemma:CoT-alphamax-sum} below, showing $\sum_{j=0}^{k-1} \rjmax \le \frac{1}{8}$. In any iteration $j$ where $\abort[j]$ does not occur, by \Cref{lemma:weak-learner}, with probability $1 - \frac{\delta}{k}$, $\err_j \le \err_\star$. This implies, with probability at least $1-\delta$,
\begin{align*}
    \Pr_{\bx \sim \rho} \big[ \fhat (\bx) \ne \fstar_T (\bx) \big] &\le \frac{1}{4}.
\end{align*}

\paragraph{Bound on sample complexity of $\WLdepth{}.\warmup$.} The bound on the size of the instance dataset required to establish the above guarantee for the depth-$1$ construction is demonstrated in \Cref{lemma:truthful-abort-semi-1} and is,
\begin{equation*}
    |\Dprompt{}| \ge \log^8 (T) \cdot \big[ d \log(|S|) + \log(T/\delta) \big],
\end{equation*}
where we set $L = \sqrt{T}$ and $k = \log(L)$. This gives the desired bound on the size of the dataset of instances in \Cref{theorem:depth1-final}.

\paragraph{Bound on query complexity of $\WLdepth{}.\warmup$.} Observe that on every instance in the dataset $\Dprompt{}$, $\WLdepth{}.\warmup$ queries $\etoe (\cdot)$ exactly $T/\tau = L$ times within the $\split(\, \cdot \,\|\, L)$ subroutine. On the other hand, to train the weak model $\fhat^j$, $\WLdepth{}.\warmup$ makes a call to $\TranscriptNTP \big( \Dout{j} \,\|\, \tau, \delta/2k \big)$ which queries $\etoe (\cdot)$ $\tau$ times per instance in $\Dout{j}$. Summation across the invocations of this subroutine for $j=0,\cdots,k-1$, the overall query complexity is upper bounded by,
\begin{equation*}
    \sum_{j=0}^{k-1} \tau |\Dout{j}| + L |\Dprompt{}| \le (\tau + L) |\Dprompt{}| = 2 \sqrt{T} |\Dprompt{}|.
\end{equation*}
where in the last equation we plug in the choice $L = \tau = \sqrt{T}$. This gives the desired bound on the query complexity in \Cref{theorem:depth1-final}.

\begin{lemma}[Exponential decay of $\beta^{0,k}_0$] \label{lemma:beta00}
There exists an absolute constant $c > 0$ such that, $\beta^{0,k}_0 \le e^{-ck}$.
\end{lemma}
\begin{proof}
Consider a set of $k$ biased coins, $Z_1,\cdots,Z_k$, with probability of heads equal to $1-\err_\star = \frac{3}{4}$. It is a short calculation to see that \smash{$\beta^{0,k}_0$} equals the probability that at most $\frac{k}{2}$ of them come up heads. Since the expected number of heads is $\frac{3k}{4}$, by standard concentration arguments (say, Hoeffding's inequality), the statement of the lemma is established.
\end{proof}

\section{Analysis of $\WLdepth{}$ (\Cref{alg:semiauto_highacc}): Proof of \Cref{theorem:main-generic}} \label{sec:mainproofsketch}

This section presents the full analysis of \mainalg; for first-time readers, we recommend first reading the analysis of the simplified depth-1 version of \mainalg in \cref{sec:semiauto-depth1-proofs} before proceeding to this section. We prove the following generalized version of \Cref{theorem:main-generic}, which states the guarantee in terms of the sample complexity of the weak learner for next-state/token data, $\Base(\cdot)$, without assuming it satisfies a ``generic'' rate: $\nweak (\delta') \le \comp (\calF) \cdot \log (1/\delta')$.

\begin{theorem}[\Cref{theorem:main-generic} for general weak learners] \label{theorem:main} For any $\delta' \in (0,1)$, let $\Base(\, \cdot \,\|\, \delta')$ be any algorithm which learns an unknown $\fstar \in \calF$ to constant error from i.i.d. next-state data (\Cref{def:iid-learning-F}). Let $n_{\texttt{weak}} (\delta')$ denote its sample complexity to achieve error $\frac{1}{4}$ with probability at least $1-\delta'$.

\medskip
\noindent Let $\varepsilon,\delta \in (0,1)$ and suppose $H = \sqrt{\log_2(T)} \in \bbN$. Suppose $\WLdepth{} (\,\cdot\, \|\, T,\varepsilon,\delta)$ (\Cref{alg:semiauto_highacc}) is invoked using the weak learner $\Base$ on a dataset of \smash{$\nsample$} i.i.d. instances drawn from $\rho$, and queries the $\etoe (\cdot)$ oracle $\nquery$ times. The sample and query complexity required by the algorithm for the resulting model $\fhat : S \times \Sigma^T \to S$ to achieve $\Pr_{\bx \sim \rho} \big[ \fhat (\bx) \ne \fstar_T (\bx) \big] \le \varepsilon$ with probability at least $1-\delta$ are upper bounded by,
\begin{equation*}
    \nsample,\nquery \le h_\varepsilon (T) \cdot \left[ \frac{\nweak (\delta/h_\varepsilon (T)) \vee \log(h_\varepsilon (T)/\delta)}{\varepsilon} \right],
\end{equation*}
where $h_\varepsilon(T) = 2^{C \sqrt{\log(T)} \cdot \log \log (T)} \cdot \log^2(1/\varepsilon)$ for an absolute constant $C>0$. Furthermore, \smash{$n_{\texttt{base}} \le h_\varepsilon(T)$}, where \smash{$n_{\texttt{base}}$} is the number of invocations of \smash{$\Base$}.
\end{theorem}
\begin{proof}
The main proof is discussed in \Cref{sec:mainproof}.
\end{proof}

\noindent \Cref{theorem:main-generic} itself follows from this result by setting $\nweak (\delta') = \comp(\calF) \log(1/\delta')$ and simplifying.

\medskip
\noindent $\WLdepth{}$ is an invocation of \smash{$\WLdepth{H+\frac{1}{2}}$} with an appropriately chosen decomposition and branching schedule. \smash{$\WLdepth{H+\frac{1}{2}}$} itself follows a tree structure where each instance of \smash{$\WLdepth{h} (\, \cdot \, \| \, t, \delta )$} invokes \smash{$\WLdepth{h-\frac{1}{2}} (\, \cdot \, \| \, t, \delta )$} $k_h$ times. This is depicted in \Cref{fig:flowchart}. Prior to this we introduce some notation.

\begin{figure}[t]
\centering
\begin{tikzpicture}[
  >=stealth,
  node distance=1.6cm and 2.0cm,
  box/.style={draw, rounded corners, minimum width=12mm, minimum height=6mm},
  every node/.style={font=\small}
]
% ---------- parameters ----------
\def\k{3}          % number of middle-row boxes (εH)
\def\dx{2.5}       % horizontal spacing of middle row
\def\ymid{-1.75}
\def\ylow{-3.75}
\def\nseg{5}       % number of ε²H segments along the small line
\def\seglen{0.5}   % visual length of each segment
\def\arrowh{0.28}  % height of the small downward arrows above ticks
% --------------------------------

\node[box,draw=white,fill=green!20] (H1/2) at (2.75,4) {$\substack{\left( T,\varepsilon \right)\\ \text{Target dist: } \rho \\ \text{Return: } \maj ( \{ \fhat^{k-1},\cdots,\fhat^0\})}$};
\draw[decorate, color=gray, decoration={brace, mirror, raise=2pt, amplitude=4pt}] ($(H1/2)+(-1.75,0.6)$) -- ($(H1/2)+(-1.75,-0.6)$);
\node[left=0.6cm of H1/2, anchor=east, font=\small] (texttt) {$\substack{\textbf{Boosting by}\\\textbf{majority vote}}$};

\draw[draw=white,fill=orange!20,rounded corners] (4.2,-2.6) rectangle ++(2,5.45);
\node[box,draw=white,fill=green!20] (H0-2) at (5.25,2) {\vphantom{$\substack{\left( T,\varepsilon \right)\\ \text{Target dist: } \rho \\ \text{Return: }\fhat^{\circ L}}$} \hphantom{asefaera}};

% Top node
\draw[draw=white,fill=orange!20,rounded corners] (-4.13,-2.6) rectangle ++(8.15,5.45);
\node[box,draw=white,fill=green!20] (H0) at (0,2) {$\substack{\left( T,\frac{1}{4} \right)\\ \text{Target dist: } v \gets \lambdahat_{\nd} \\ \text{Return: }\pihat^{(\nd)} \gets (\ftilde)^{\circ L}}$};
\node[box, draw=white,fill=blue!20] (H) at (0,0) {$\substack{( \tau_{H-1},\varepsilon_0 )\\ \text{Target dist: } \overline{v} \\ \ftilde \gets \maj ( \{ \fhat^{(\nd_{k-1})}, \cdots,\fhat^{(\nd_0)} \}) }$};
\node[right=0.33cm of H, anchor=west, font=\scriptsize] (H0label) {\big($\varepsilon_0 = \frac{1}{4L}$\big)};
\draw[->] (H1/2) -- (H0);
\draw[->] (H1/2) -- (H0-2);

\node[box,fill=red!20,draw=orange!20] (labelnd) at (3.25,2.35) {$\substack{\nd}$};

\draw[decorate, color=gray, decoration={brace, mirror, raise=2pt, amplitude=4pt}] ($(H0)+(-1.4,0.6)$) -- ($(H0)+(-1.4,-0.6)$);
\node[left=0.2cm of H0, anchor=east, font=\small] (texttt) {$\substack{\textbf{Compose model}}$};

\node[align=center, font=\footnotesize] at ($(H0)!0.5!(H) + (-3,0)$) {$\substack{\text{Decompose into}\\\text{shorter problems}}$};
\draw[{Circle[length=2pt]}-{Circle[length=1pt]}] ($(H0)!0.5!(H)+(0.033,0)$) -- ($(H0)!0.5!(H)+(-1.9,0)$);
\draw[decorate, color=gray, decoration={brace, mirror, raise=2pt, amplitude=4pt}] ($(H)+(-1.9,0.55)$) -- ($(H)+(-1.9,-0.55)$);
\node[left=0.25cm of H, anchor=east, font=\small] (texttt) {$\substack{\textbf{Boosting by}\\[1.5pt]\textbf{majority vote}}$};

\draw[draw=white,fill=gray!20,rounded corners] (-7.15,-7.75) rectangle ++(9.25,4.2);

% \draw[decorate, color=gray, decoration={brace, mirror, raise=2pt, amplitude=4pt}] ($(H)+(-4.5,-2.5)$) -- ($(H)+(-4.5,-4)$);

\draw[->] (H0) -- (H);
% Middle row: k boxes labeled εH
\foreach \i [evaluate=\i as \x using (-\i+0.5*(\k+1))*1.03*\dx] in {1,3} {%
  \node[box, draw=white,fill=green!20] (e\i) at (\x,\ymid) {$\substack{\left(\tau_{H-1},\frac{1}{4}\right)\\\text{Target dist: }
    \ifnum\i=3
      \lambda_{\nd_{k-1}}\\\fhat^{(\nd_{k-1})} % <-- your special label for i=3
    \else
      \lambda_{\nd_0}\\\fhat^{(\nd_0)}     % <-- default label
    \fi}$};
  \draw[->] (H) -- (e\i);
}
\node[align=center] at ($(e1)!0.5!(e3)$) {$\substack{\cdots}$};

% First middle box splits into two ε²H boxes (with wider separation)
\pgfmathsetmacro{\offset}{1*\dx}
\pgfmathsetmacro{\offsettwo}{0.9*\dx}
\node[box,fill=red!20,draw=gray!20] (labelnd) at (0.92,-4.03) {$\substack{\nd' = \; \nd_{k-1}}$};
\node[box, dotted] (eeL) at ($(e3)+(0,-2.75,0)$) {$\substack{\big( \tau_{H-1}, \gamma_0 \big)  \\ \text{Target dist: } \lambdahat_{\nd'}\\[1pt] \fhat^{(\nd')} \gets \maj (\{ \fhat^{(\nd'_{k-1})},\cdots, \fhat^{(\nd'_0)} \})}$};
\node[below right=-0.4cm and 0.1cm of eeL, anchor=west, font=\tiny] (eeLlabel) {$\left(\gamma_0 = \frac{1}{\polylog(L)}\right)$};

\draw[decorate, color=gray, decoration={brace, mirror, raise=2pt, amplitude=4pt}] ($(eeL)+(-2.3,0.69)$) -- ($(eeL)+(-2.3,-0.69)$);
\node[left=0.33cm of eeL, anchor=east, font=\footnotesize] (texttt) {$\substack{\textbf{Boosting by}\\[1pt]\textbf{majority vote}}$};

\draw[->] (e3) -- (eeL);
\node[align=center, font=\footnotesize] at ($(eeL)!0.55!(e3) + (-3,0)$) {$\substack{\text{Change of measure}\\\text{(only in the analysis)}}$};
\draw[{Circle[length=2pt]}-{Circle[length=1pt]}] ($(eeL)!0.55!(e3)+(0.033,0)$) -- ($(eeL)!0.55!(e3)+(-1.7,0)$);

\node[box,draw=white,fill=green!20] (eeL2) at ($(eeL)+(-\offsettwo,-2.25,0)$) {$\substack{\left(\tau_{H-1}, \frac{1}{4} \right)\\ \text{Target dist: } \lambdahat_{\nd'_{k-1}}\\[1pt] \fhat^{(\nd'_{k-1})} }$};
\node[box,draw=white,fill=green!20] (eeR2) at ($(eeL)+(\offsettwo,-2.25,0)$) {$\substack{\left( \tau_{H-1}, \frac{1}{4} \right)\\ \text{Target dist: } \lambdahat_{\nd'_{0}}\\[1pt] \fhat^{(\nd'_{0})}}$};
\draw[->] (eeL) -- (eeL2);
\draw[->] (eeL) -- (eeR2);
\node[align=center] at ($(eeL2)!0.5!(eeR2)$) {$\substack{\cdots}$};

\node (bbL) at ($(e1)+(0,-2.12,0)$) {$\vdots$};
\draw[->] (e1) -- (bbL);

% --- Small horizontal line centered under left bottom box ---
% Baseline y for the small line (lowered slightly)
% Your existing code
\coordinate (baseL) at ($(bbL.south)+(1,-2)$);
% Compute start and end so it's centered under eeL
\coordinate (lineStart) at ($(baseL)+(-0.5*\nseg*\seglen,0.25)$);
\coordinate (lineEnd)   at ($(baseL)+( 0.5*\nseg*\seglen,0.25)$);

% Main H-length line
\draw[decorate, decoration={brace, raise=3pt}]
  ($(lineStart)+(0,0.25)$) -- ($(lineEnd)+(0,0.25)$);
\node[above, font=\tiny] at ($(lineStart)!0.5!(lineEnd) + (0,0.45)$) {$\tau_H = T$};
\draw (lineStart) -- (lineEnd) node[below left=2pt and -5pt, font=\tiny] {$L$ pieces};

% === New code for underbrace + zoomed-in line ===

% Last segment of the main line
\coordinate (lastSegStart) at ($(lineEnd)+(-\seglen,0)$);
\coordinate (lastSegMid)   at ($(lineStart)!0.5!(lineEnd)$);

% Underbrace under the last segment
\draw[decorate, decoration={brace, raise=4pt}]
  ($(lastSegMid)+(0.25,0)$) -- ($(lastSegMid)-(0.25,0)$);

% Arrow down from the last segment (zoom arrow)
\coordinate (zoomCenter) at ($(lastSegMid)+(0,-0.75)$);
\coordinate (zoomStart) at ($(lastSegMid)+(0,-0.23)$);
\draw[-] (zoomStart) -- (zoomCenter);
\draw[decorate, decoration={brace, mirror, raise=-3pt}]
  ($(zoomCenter)+(1,0)$) -- ($(zoomCenter)-(1,0)$);

% Parameters for the zoomed-in line
\def\nsubseg{4}      % number of little segments in the zoomed-in view

% Zoomed-in line (again centered below)
\coordinate (subStart) at ($(zoomCenter)+(-0.5*\nsubseg*\seglen,-0.5)$);
\coordinate (subEnd)   at ($(zoomCenter)+( 0.5*\nsubseg*\seglen,-0.5)$);

\draw (subStart) -- (subEnd);
\draw (subStart) -- (subEnd) node[below left=2pt and -5pt, font=\tiny] {$L$ pieces};

\foreach \j in {0,...,\nsubseg} {
  \coordinate (tsj) at ($(subStart)+(\j*\seglen,0)$);
  % tick
  \draw (tsj)++(0,-0.06) -- ++(0,0.12);
  % smaller downward arrow (reduced arrowhead size)
  \draw[->, >=stealth, line width=0.3pt] 
    ($(tsj)+(0,\arrowh)$) -- ($(tsj)+(0,0.10)$);
}
\node[below, font=\tiny] at ($(subStart)+(0.5*\seglen,0)$) {$\tau_{H-2}$};

% Subdivision ticks and labels; downward arrows above each tick
\foreach \j in {0,...,\nseg} {
  \coordinate (tj) at ($(lineStart)+(\j*\seglen,0)$);
  % tick
  \draw (tj)++(0,-0.06) -- ++(0,0.12);
  % smaller downward arrow (reduced arrowhead size)
  \draw[->, >=stealth, line width=0.3pt] 
    ($(tj)+(0,\arrowh)$) -- ($(tj)+(0,0.10)$);
}

% Only label the 1st and 2nd sub-segments with ε²H
\node[below, font=\tiny] at ($(lineStart)+(0.5*\seglen,0)$) {$\tau_{H-1}$};
\node[left=8.35cm of H1/2, anchor=east, font=\small, text width=2.25cm, align=center] (L-2) {Level $H+\frac{1}{2}$ ($\WLdepth{H+\frac{1}{2}}$)};
\node[left=5.95cm of H0, anchor=east, font=\small, text width=2.25cm, align=center] (L-1) {Level $H$ ($\WLdepth{H}$)};
\node[left=6cm of H, anchor=east, font=\small] (L0) {};
\node[left=3.45cm of e3, anchor=east, font=\small, text width=2.25cm, align=center] (L1) {Level $H-\frac{1}{2}$ ($\WLdepth{H-\frac{1}{2}}$)};
\node[left=2.9cm of eeL, anchor=east, font=\small] (L2) {};
\node[left=1.25cm of eeL2, anchor=east, font=\small, text width=2.25cm, align=center] (L3) {Level $H-1$ ($\WLdepth{H-1}$)};
\node[below=3.25cm of eeL, anchor=south, font=\small] (L3) {$\vdots$};

\end{tikzpicture}
\caption{The recursive call of invocations within structure of $\WLdepth{}$ when the target error is $\varepsilon$: each green box corresponds to a call to $\WLdepth{h}(\cdot)$ for some $h \in \{0,\frac{1}{2},1,\ldots,H+\frac{1}{2}\}$. From each integer level $h$ to the next one, instances decrease by a factor $L_h = L$ in length.}
\label{fig:flowchart}
\end{figure}
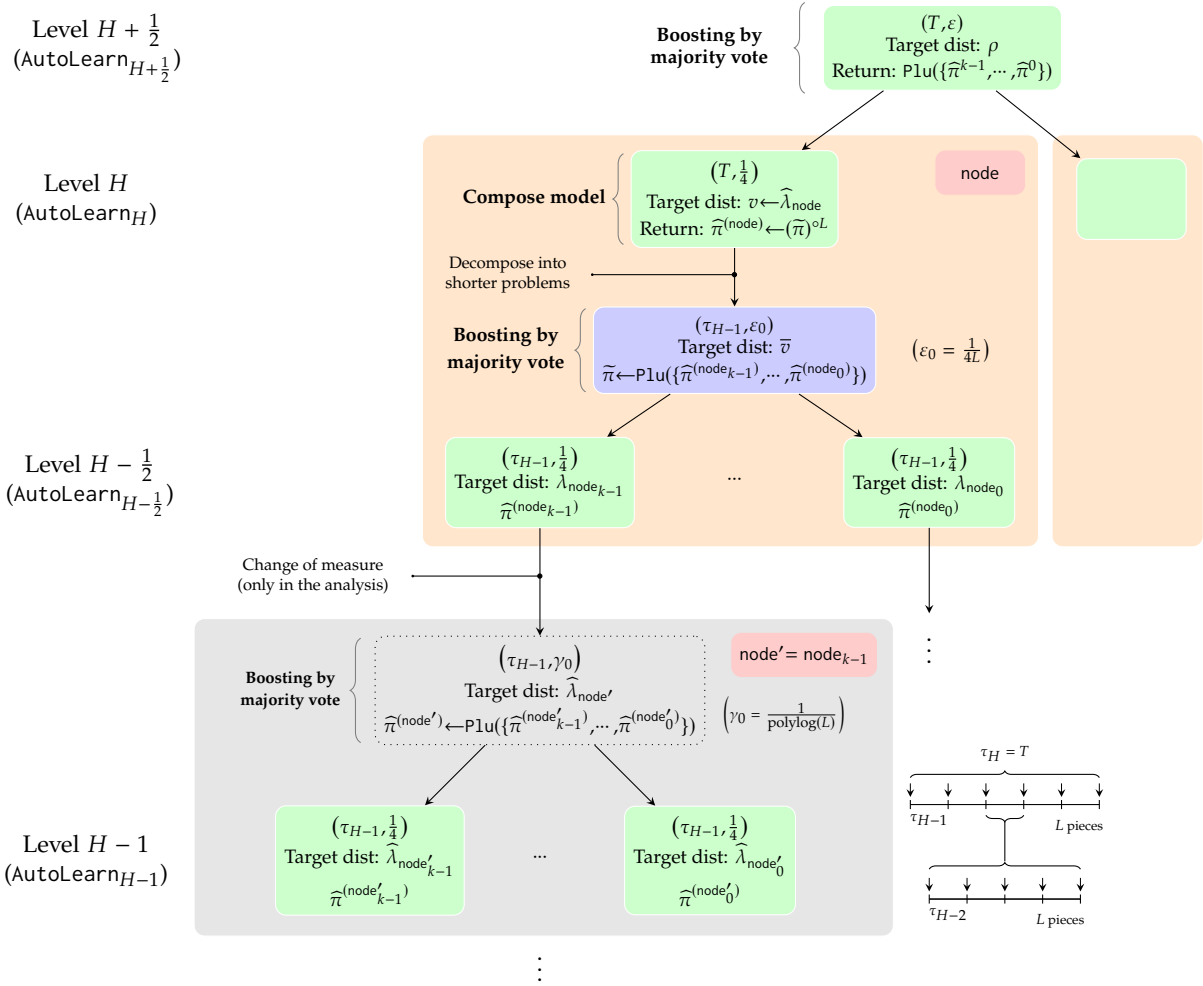

\paragraph{Notation.} For $h \in \bbZ$, let $\bX_h = S \times \Sigma^{\tau_h}$ denote the space of $\tau_h$-length instances. Let $\bX_{h+\frac{1}{2}} = \bX_h$ for all such $h$.

\paragraph{Tree diagram corresponding to the recursive structure of \Cref{alg:semiauto_main}.} The recursive curriculum proceeds with $\WLdepth{h}$ calling $k_h$ instances of $\WLdepth{h-\frac{1}{2}}$. The diagram in \Cref{fig:flowchart} arranges these invocations into a tree: the root corresponds to a call of $\WLdepth{H+\frac{1}{2}}$ and going down the tree recursively invokes $\WLdepth{H},\WLdepth{H-\frac{1}{2}}$ and so on. The number of nodes at any level $h$ is $k_h \times k_{h+\frac{1}{2}} \times \cdots \times k_{H+\frac{1}{2}}$. We let ``$\nd$'' denote a node in the tree, let $\pt (\nd)$ denote the parent of $\nd$, $\ch (\nd)$ denote the ordered set of children of $\nd$\footnote{Note that the children of a node are not interchangeable} and $\sib (\nd) = \ch (\pt (\nd))$ denote its siblings. Models are \textit{trained} only at nodes at level $0$ of the tree, but are hierarchically combined to define models induced at intermediate nodes of the tree. We also refer to a node $\nd$ being \textit{resolved} to indicate that the model corresponding to $\nd$ has been learned. This is only possible if all level-$0$ nodes in the induced subtree with $\nd$ as root have been trained. Furthermore, the construction can be thought of as resolving nodes in the tree in a depth-first fashion: at any level $h$, the instantiations of $\WLdepth{h}$ can only be resolved when the ones ``prior'' to it (in the depth-first ordering) at the same level have been resolved. We use the notation $\calT$ to collect the set of all nodes in the tree, and $\calT_h$ to denote the set of all nodes at a given level $h \in \{ \frac{1}{2},\cdots,H+\frac{1}{2} \}$. $\root$ denotes the singular node in $\calT_{H+\frac12}$ (i.e., the root node of $\calT$).

\medskip
\noindent Each node is associated with a collection of objects we define next. For any $h \in \{ 0,\frac{1}{2},\cdots,H+\frac{1}{2} \}$ and $\nd \in \calT_h$, let \smash{$\fhat^{(\nd)}$} denote a model associated with it, defined recursively as,
\begin{equation} \label{eq:fcombine}
    \fhat^{(\nd)} = \fhat^{\circ L_h}, \text{ where, } \fhat \gets \maj \big( \big\{ \fhat^{(\nd')} : \nd' \in \ch (\nd) \big\} \big).
\end{equation}
where $\fhat^{(\nd)}$ for $\nd \in \calT_0$ are induced by invocations of $\Base$. Next, as we carried out in \cref{lemma:weak-learner}, define a sequence of distributions for each node. For $\nd$ at level $h$, let $j$ denote its index within its ordered set of siblings, $\sib (\nd)$ and let the first $j$ siblings of $\nd$ be denoted by the set $\sib_{< j} (\nd)$. Then, for $X \in \bX_h$,
\begin{equation} \label{eq:wnode}
    w_{\nd} (X) = \alpha^{j,k}_{\rank (X \| \nd)} \text{, where, } \rank (X \| \nd ) = \big| \big\{ \fhat^{(\nd')} (X) = \fstar_{\tau_h} (X) : \nd' \in \sib_{< j} (\nd) \big\} \big| \in [0,j]
\end{equation}
and $\alpha^{j,k}_r$ is defined in \cref{eq:alpha}. The rank of a node captures how many of its earlier siblings collect any reward on $X$. Finally, for any $\nd \in \calT$ we use the notation $\Dnd$ to indicate the dataset of instances which the corresponding invocation of $\WLdepth{h}$ processes. First, we prove a short lemma which argues that for any $\nd$, $\Dnd$ is composed of independent instances. 

\begin{lemma}[Independence of per-node datasets] \label{lemma:node-independence}
For any $\nd \in \calT$ the instances in the corresponding dataset, $\Dnd$, are independent.
\end{lemma}
\begin{proof}
This is proved recursively. When $\nd = \root$, this is true since $\Dnd = \Dprompt{}$. For its subsequent children nodes, $\WLdepth{H}$ constructs the dataset \smash{$\Dnd \gets \Dout{j}$} by using inverted sampling (\Cref{alg:semi-subsample}) to process instances in $\Dprompt{}$ to generate shorter instances. This sampling routine takes in a single instance $\bx$ and returns either $\perp$ or a single shorter instance computed from $\bx$. The instance $\bx$ is never reused elsewhere in the recursion, implying that independence is preserved across the dataset of shorter instances computed. The same argument inductively shows that $D_{\nd}$ is composed of independent instances so long as $D_{\pt (\nd)}$ is composed of independent instances.
\end{proof}

\noindent With this, for $\nd$ at level $h$ or $h + \frac{1}{2}$, we let $\lambdahat_{\nd}$ denote the distribution over instances of length $\tau_h$ from which the elements of $\Dnd$ are sampled. Define $\lambdahat_{\nd}$ as the target distribution over instances of length $\tau_h$. At level $h=H+\frac12$, we have that, $\lambdahat_{\root} = \rho$. With this notation introduced, we are ready to state two inductive assumptions on the correctness of invocations of $\WLdepth{h}$ at level $h$ (for integer and non-integer values of $h$), which together show that with high probability, the depth $h+1$ version succeeds conditioned on this. Define the following sequences recursively,
\begin{align}
    \text{If } h \in \bbN,\ \nsample^\star (h+\tfrac{1}{2},2k_{h+\frac12} \delta) &= \nsample^\star(h,\delta) \cdot 2 k_{h+\frac12} e^{c k_{h+\frac12}}, \nonumber\\
    \text{And, } \nsample^\star (h+1,2k_{h+1} \delta) &= \nsample^\star(h+\tfrac12,\delta) \cdot C_3 k_{h+1}^{7/2} \log(L_{h+1})\label{eq:nsamplestar}\\
    \text{And, } \forall h \ge 0,\ \nquery^\star (h+\tfrac{1}{2},2k_h \delta) &= L_h \cdot \nsample^\star (h+\tfrac{1}{2},\delta ) + k_h \cdot \nquery^\star (h,\delta), \label{eq:nquerystar}
\end{align}
where $\nsample^\star (0,\delta) = \nquery^\star (0,\delta) = \big[ \nweak (\delta) \vee \log(1/\delta) \big]$, and $C_3 > 0$ is a large absolute constant, while in the definition of $\nsample^\star$, $c$ is chosen to be the same constant as in the exponent of \smash{$\beta^{0,k}_0$} in \cref{lemma:beta00}. With the choices of decomposition schedule $( L_h )_{h \ge 0}$ and branching schedule $( k_h )_{h \ge 0}$ as described in \Cref{alg:semiauto_highacc}, we get that, for $h \le H$,
\begin{align*}
    \nsample^\star (h,\delta),\ \nquery^\star (h,\delta) &\le (\log(T))^{\calO(h)} \cdot \big[ \nweak (\delta / (\log(T))^{\calO(h)} ) \vee \log( 1/\delta ) \big]
\end{align*}
and for $h = H+\frac{1}{2}$,
\begin{align*}
    \nsample^\star (H+\tfrac{1}{2},\delta),\ \nquery^\star (H+\tfrac{1}{2},\delta) &\le \frac{h_\varepsilon(T)}{\varepsilon} \cdot \big[ \nweak \big(\delta / h_\varepsilon(T) \big) \vee \log(1/\delta) \big].
\end{align*}
where $h_\varepsilon(T) = 2^{C \sqrt{\log(T)} \log\log(T)} \cdot \log^2 (1/\varepsilon)$. The proof of \Cref{theorem:main} relies on establishing inductive guarantees on the sample and query complexity needed at each level of the algorithm. This is separated into guarantees for integer levels, and for half-integer levels in the two hypotheses below.

\begin{induction}[label=induction:h1/2]{h+\frac{1}{2}}
Consider any \smash{$\nd \in \calT_{h+\frac{1}{2}}$} and semiautomata simulation instances of length $\tau_{h}$ over the target distribution over instances, $\lambdahat_{\nd}$. When invoked on a dataset of independent instances drawn from $\lambdahat_{\nd}$ of size $\nsample (h+\tfrac{1}{2},\delta) \ge \nsample^\star (h+\tfrac{1}{2},\delta )$,
$\WLdepth{h+\frac{1}{2}}$ returns a deterministic model $\fhat$ such that with probability $1-\delta$,
\begin{equation*}
    \Pr_{X \sim \lambdahat_{\nd}} \big( \fhat (X) \ne \fstar_{\tau_h} (X) \big) \le 3 \sqrt{k_{h + \frac{1}{2}}} \cdot e^{- c k_{h + \frac{1}{2}}}.
\end{equation*}
where $c>0$ is the constant in the exponent of $\beta^{0,k}_0$ in \cref{lemma:beta00}. In the process, the number of queries made to $\etoe(\cdot)$ is upper bounded by $\nquery^\star (h+\tfrac{1}{2},\delta)$.
\end{induction}

\begin{induction}[label=induction:h1]{h+1}
Consider any $\nd \in \calT_{h+1}$ and semiautomata simulation instances of length $\tau_{h+1}$ over the target distribution over instances, \smash{$\lambdahat_{\nd}$}. When invoked on a dataset of independent instances drawn from $\lambdahat_{\nd}$ of size at least, $\nsample (h+1, \delta) \ge \nsample^\star (h+1, \delta)$, then the invocation $\WLdepth{h+1}$ corresponding to $\nd$ returns a deterministic model $\fhat$ such that with probability $1-\delta$,
\begin{equation*}
    \Pr_{X \sim \lambdahat_{\nd}} \big( \fhat (X) \ne \fstar_{\tau_{h+1}} (X) \big) \le \frac{1}{4}.
\end{equation*}
The number of queries made to $\etoe(\cdot)$ is at most $\nquery^\star (h+1,\delta)$.
\end{induction}

\begin{lemma}[Aggregation steps] \label{lemma:induction-half}
For any $h \in \bbZ$, $\IndHyp_{h} (\delta) \implies \IndHyp_{h+\frac{1}{2}} (2k_{h+\frac{1}{2}} \delta)$.
\end{lemma}
\begin{proof}
This lemma is restated in more detail and proved in \Cref{lemma:easy-half}.
\end{proof}

\begin{lemma}[Aggregation+composition steps] \label{lemma:induction}
Consider any value of $h \in \{ 0,1,\cdots,H-1 \}$ and suppose $k_{h+1} = \log(L) = c_2 \sqrt{\log(T)}$ and $k_{h+\frac{1}{2}} = c_3 \log\log(T)$ for appropriate constants $c_2,c_3>0$. Then, we have the recursion, $\IndHyp_{h+\frac{1}{2}} (\delta) \implies \IndHyp_{h+1} (2k_{h+1} \delta)$.
\end{lemma}
\begin{proof}
The base case (sample and query complexity of $\IndHyp_{0} (\delta')$ at $h=0$) are directly implied by the definition of the sample and query complexity of $\Base (\, \cdot \,\|\, \delta')$ in terms of $\nweak (\delta')$. This lemma is restated in more detail and proved in \Cref{lemma:easy}.
\end{proof}

\noindent Of the two claims, the first one is easier to show, since it is a direct consequence of boosting-by-filtering~\citep{freund1995boosting}. In particular, since there is no composition step involved in going from level $h$ to level $h+\frac{1}{2}$, the distributional shift is only due to the boosting step, and the analysis of this is standard. Next, we sketch how to bootstrap $\IndHyp_h$ into $\IndHyp_{h+\frac{1}{2}}$, implying the statement of \Cref{lemma:induction-half}.

\subsection{Analysis of half-integer steps: Proof of \Cref{lemma:induction-half}}

In this section we analyze \Cref{lemma:induction-half}. We fix a level $h \in \bbZ$ and a particular $\nd \in \calT_{h+\frac{1}{2}}$. We let $\fhat^0,\cdots,\fhat^{k-1}$ denote the $k$ models returned by $\WLdepth{h}$, corresponding to nodes in $\ch (\nd)$ (and which are invoked on distributions $\lambdahat_{\nd'}$ for $\nd' \in \ch (\nd)$). Let us index the nodes in $\ch (\nd)$ as $( \nd_0,\cdots,\nd_{k-1})$. To prove \Cref{lemma:induction-half}, we translate the analysis of $\WLdepth{}.\warmup$ (cf. \Cref{theorem:depth1-final}) for the case of $L=1$ and setting $\bX \gets \bX_h$. Since $L=1$, some parts of the analysis greatly simplify.

\medskip
\noindent First we begin by introducing some notation before jumping into the proof of the result. Let $(\calH_{\nd_j})_{j=0}^{k-1}$ denote the filtration where $\calH_{\nd_{j-1}}$ captures all randomness up until the point where the dataset $D_{\nd_j}$ is collected: this includes the randomness in the construction of the models \smash{$\fhat^{(\nd_0)},\cdots,\fhat^{(\nd_{j-1})}$} corresponding to nodes $\nd_0,\cdots,\nd_{j-1} \in \ch(\nd)$, as well as the randomness in the instances from $\Dnd$ which contribute to the datasets $D_{\nd_0},\cdots,D_{\nd_{j-1}}$. Define the following distribution over $\bX_h$,
\begin{equation} \label{eq:lambdahatnodejstar-half}
    \lambda_{\nd_j} (\cdot) \propto \lambdahat_{\nd} (\cdot) w_{\nd_j} (\cdot),
\end{equation}
where $w_{\nd_j} (\bx) = \alpha_{r}^{j,k}$ for $r = \rank ( \bx \|\, \nd_j)$.
Recall that $\rank (\,\cdot \,\|\, \nd_j)$ (\cref{eq:wnode}) captures the number of models among \smash{$\big\{ \fhat^0,\cdots,\fhat^{j-1} \big\}$} which correctly predict the terminal state of the instance $X \in \bX_h$. First, we write down a test-error decomposition, similar to the one in \cref{lemma:depth1-breakdown}, where we break down the accuracy of the model returned by an instantiation of $\WLdepth{h+\frac{1}{2}}$ (i.e., corresponding to some $\nd$) into the accuracy of the models returned by the instances of $\WLdepth{h}$ it spawns (i.e., corresponding to $\ch (\nd)$).

\medskip
\noindent For the purpose of this section, we will denote $k \gets k_{h+\frac{1}{2}}$ and $L \gets L_{h+\frac{1}{2}} = 1$ to simplify notation.

\begin{lemma}[Test-error decomposition] \label{lemma:depthH-breakdown-half} Let $\fhat^{(\nd)}$ denote the model returned by the instantiation of $\WLdepth{h+\frac{1}{2}}$ corresponding to $\nd \in \calT_{h+\frac{1}{2}}$. Let \smash{$\fhat^{(\nd_j)}$} denote the model returned by $\WLdepth{h}$ corresponding to each \smash{$\nd_j \in \ch(\nd)$}. Then,
\begin{equation*}
    \Pr_{X \sim \lambdahat_{\nd}} \big[ \fhat^{(\nd)} (X) \ne \fstar_{\tau_h} (X) \big] \le \beta^{0,k}_0 + \sum_{j=0}^{k-1}  (\err_j - \err_\star) \cdot \bbE_{X \sim \lambdahat_{\nd}} \big[ \alpha^{j,k}_{\rank (X \| \nd_j)} \mid \calH_{\nd_{k-1}} \big]
\end{equation*}
Here, $\beta^{j,k}_r$ is defined in \cref{eq:alpha}, and $\err_j$ is the test error of $\fhat^{(\nd_j)}$ under the distribution $\lambda_{\nd_j}$,
\begin{equation*}
    \err_j = \Pr_{X \sim \lambda_{\nd_j}} \big[ \fhat^{(\nd_j)} (X) \ne \fstar_{\tau_h} (X) \big]
\end{equation*}
\end{lemma}
\begin{proof}
Since there is no composition step involved, this result is immediate from the proof of \Cref{lemma:depth1-breakdown} by the choice $L=1$, $\calK = \bX^L$, and $\rho \gets \lambdahat_{\nd}$, $\sigma \gets \lambdahat_{\nd}$ and $\lambda_j \gets \lambda_{\nd_j}$.
\end{proof}

\noindent Since there is no composition involved in going from step $h$ to $h+\frac{1}{2}$, the learner has a sampling oracle for $\lambdahat_{\nd}$, which can be used to generate samples from $\lambda_{\nd_j}$ by rejection sampling (which coincides with inverted sampling when $L=1$). In order to analyze the error decomposition in \Cref{lemma:depthH-breakdown-half} first we define events which track whether the instances of $\WLdepth{h}$ corresponding to nodes in $\ch (\nd)$ are invoked on sufficiently large datasets to satisfy the inductive guarantee we assume for nodes at level $h$. To this end, define for $0 \le j \le k-1$,
\begin{equation*}
    \abort_h [\nd_j] = \big\{  p_{\nd_j} \le p^\star_j \big\}, \text{ where, } p^\star_j = e^{-ck}
\end{equation*}
where $c > 0$ is the constant in \cref{lemma:beta00}. With this, the next lemma simply uses the definition of $\IndHyp_h$ to argue that if any node $\nd_j \in \ch (\nd)$ does not abort, we can get a weak learning guarantee for $\nd_j$.

\begin{lemma}[Weak learner guarantee for $\WLdepth{h}$] 
\label{lemma:weak-learner-H-half}
Assume the inductive hypothesis $\IndHyp_h (\delta)$. Conditioned on the event that $\abort_h [\nd_j]$ is false, with probability at least $1-\delta$, the model \smash{$\fhat^{(\nd_j)}$} returned by $\nd_j$ satisfies,
\begin{align*}
    \Pr_{X \sim \lambda_{\nd_j}} \big[ \fhat^{(\nd_j)} (X) \ne \fstar_{\tau_h} (X) \big] = \err_j \le \err_\star \triangleq \frac{1}{4}
\end{align*}
\end{lemma}
\begin{proof}
This result is a direct consequence of the definition of $\IndHyp_h (\delta)$.
\end{proof}

\noindent Next we argue about the probability that any of the nodes in $\ch (\nd)$ abort. Define,
\begin{align*}
    p_{\nd_j} = \Pr_{\bx \sim \lambdahat_{\nd}} \big( \subsample ( \{ \bx \} \ \| \ \calF_j, k ) \ne \perp \big), \text{ where, } \calF_j = \big\{ \fhat^{(\nd_0)},\cdots,\fhat^{(\nd_{j-1})} \big\}.
\end{align*}
This is the probability with which an instance $\bx \sim \lambdahat_{\nd}$ is accepted by split-and-inverted-sampling (\Cref{alg:split,alg:semi-subsample}). Next we show that if $p_{\nd_j}$ is large, then the instantiation of $\WLdepth{h}$ corresponding to $\nd_j$ is likely to be instantiated with a large dataset, implying that $\nd_j$ indeed returns a weak learner with high probability.

\begin{lemma}[Truthful aborts at half-levels] \label{lemma:truthful-abort-semi-H-half}
Suppose $\WLdepth{h+\frac{1}{2}} ( \,\cdot\, \|\, \tau_{h+\frac{1}{2}},\delta)$ (cf. \Cref{alg:semiauto_main}) is invoked on a dataset of at least \smash{$\nsample^\star (h+\frac{1}{2},\delta) = 2k e^{c k} \cdot \nsample^\star (h,\delta)$} instances where $c>0$ is the constant in the exponent of \smash{$\beta^{0,k}_0$} in \cref{lemma:beta00}. Then, for any index $j$ such that $p_{\nd_j} \ge e^{-c k}$, $\Pr \big( |D_{\nd_j}| < \nsample^\star (h, \delta) \mid \calH_{\nd_{j-1}} \big) \le \delta$.
\end{lemma}
\begin{proof}
The proof of this result is identical to that of \Cref{lemma:truthful-abort-semi-1}, where we use the Chernoff bound to control the deviations of $n'$ coins each having probability of heads at least $p_{\nd_j}$.
\end{proof}

\noindent The complement case is when $p_{\nd_j}$ is small and $\nd_j$ is no longer likely to succeed. Similar to in the analysis of $\WLdepth{}.\warmup$ where we established an explicit formula for $p_j$, $p_{\nd_j}$ satisfies the equation,
\begin{equation*} \label{eq:pj-formula-1}
    p_{\nd_j} = \bbE_{X \sim \lambdahat_{\nd}} \left[ \frac{\alpha^{j,k}_{\rank (X \| \nd_j)}}{\alpha^{j,k}_{\max}} \ \middle| \ \calH_{\nd_{j-1}} \right]
\end{equation*}
Which corresponds to setting $L=1$ in \cref{eq:pj-formula}. In particular, combining with \Cref{lemma:truthful-abort-semi-H-half}, the test error decomposition in \Cref{lemma:depthH-breakdown-half}, the upper bound on $\beta_0^{0,k}$ from \Cref{lemma:CoT-alphamax-sum} and simplifying, we arrive at the following result which implies \Cref{lemma:induction-half}.

\begin{lemma}[Analysis of aggregation step] \label{lemma:easy-half}
Let $\fhat^{(\nd)}$ denote the model returned by the instantiation of \smash{$\WLdepth{h+\frac{1}{2}}$} corresponding to some \smash{$\nd \in \calT_{h+\frac{1}{2}}$}. Then, assuming the hypothesis $\IndHyp_h (\delta)$, we have that with probability at least $1 - 2 k \delta$,
\begin{equation*}
    \Pr_{X \sim \lambdahat_{\nd}} \big[ \fhat^{(\nd)} (X) \ne \fstar_{\tau_h} (X) \big] \le 3 \sqrt{k} \cdot e^{-c k}
\end{equation*}
Further, under $\IndHyp_h (\delta)$, this guarantee is achieved if $\WLdepth{h+\frac{1}{2}}$ draws $\nsample^\star (h+ \frac{1}{2}, 2k\delta)$ samples from \smash{$\lambdahat_{\nd}$}. The number of queries made to $\etoe(\cdot)$ is at most \smash{$\nquery^\star (h+ \frac{1}{2}, 2k\delta)$}.
\end{lemma}
\begin{proof}
We first prove the bound on the accuracy, and subsequently show the recursion on the sample and query complexity.

\paragraph{Bound on accuracy.} Conditioned on the event that $\abort_h [\nd_j]$ is false for some $j$, we have that $\err_j \le \err_\star$. By a union bound, across all iterations where $\abort_h [\nd_j]$ is true, by \Cref{lemma:depthH-breakdown-half}, and using the definition of $p_{\nd_j}$, w.p. at least $1 - 2k\delta$, 
\begin{align*}
    \Pr_{X \sim \lambdahat_{\nd}} \big[ \fhat^{(\nd)} (X) \ne \fstar_{\tau_h} (X) \big] &\le \beta_0^{0,k} + \sum_{j=0}^{k-1} e^{-k} \cdot \alpha^{j,k}_{\max} \\
    &\overset{(i)}{\le} e^{-c k} + 2 \sqrt{k} \cdot e^{-c k} \\
    &\le 3 \sqrt{k} \cdot e^{-c k}
\end{align*}
In $(i)$, we use \Cref{lemma:CoT-alphamax-sum,lemma:beta00} (the bound on $\sum_{j=0}^{k-1} \alpha^{j,k}_{\max}$ is implied by the bound on $\sum_{j=0}^{k-1} \rjmax$). The bound on the number of samples drawn required from $\lambdahat_\nd$ is derived in \Cref{lemma:truthful-abort-semi-H-half}.

\paragraph{Bound on query and sample complexity of $\WLdepth{h+\frac{1}{2}}$.}

$\WLdepth{h+\frac{1}{2}}$ must be invoked on a dataset of size \smash{$\nsample^\star \big(h + \frac{1}{2}, 2k\delta\big) = \nsample^\star (h, \delta) \cdot 2 k e^{ck}$} to be able to invoke \Cref{lemma:truthful-abort-semi-H-half}. On the other hand, the query complexity of $\WLdepth{h+\frac{1}{2}}$ is upper bounded by,
\begin{align} \label{eq:110000}
    \nquery^\star \big(h + \tfrac{1}{2}, 2k\delta\big) &= \nsample^\star (h+\tfrac{1}{2}, 2k\delta) + k \cdot \nquery^\star (h,\delta).
\end{align}
The $\nsample^\star (h+\tfrac{1}{2},2k\delta)$ term accounts for labeling the terminal state in each invocation of $\subsample$ - note that $\WLdepth{h+\frac{1}{2}}$ processes a dataset of at most $\nsample^\star (h+\tfrac{1}{2},2k\delta)$ instances (cf.~\Cref{alg:semiauto_main:dbound} of \Cref{alg:semiauto_main}). The $k \cdot \nquery^\star (h,\delta)$ term arises as the cost of training the $k$ children invocations of $\WLdepth{h}$.
\end{proof}

\subsection{Proof of \Cref{lemma:induction}}

This is the more challenging of the two induction steps (\Cref{lemma:induction,lemma:induction-half}), since we have to deal with the composition step. Fortunately, the analysis is quite similar to the one we carried out for $\WLdepth{}.\warmup$. To draw the parallel most easily, in this section we will denote the branching factor as $k \gets k_{h+1}$, and the decomposition factor as $L \gets L_{h+1}$.

\medskip
\noindent We define a sequence of events $\calK_{\nd_0},\cdots\calK_{\nd_{k-1}} \subseteq \bX_h^L \times [0,1]^{kL}$ which play the same role as the $\calK_0,\cdots,\calK_{k-1}$ events in \Cref{lemma:depth1-breakdown}; here $\calK_{\nd_j} \subseteq \bX_h^L \times [0,1]^{(j+1)L}$. We index the nodes in $\ch (\nd)$ as $( \nd_0,\cdots,\nd_{k-1})$. To prove \Cref{lemma:induction}, we translate the analysis of $\WLdepth{}.\warmup$ (cf. \Cref{theorem:depth1-final}) by setting $\bX \gets \bX_h$, $\rho \gets \lambdahat_{\nd}$. In correspondence with this choice, we denote $\underline{\lambdahat_{\nd}}$ as the distribution obtained from $\lambdahat_{\nd}$ by the same operation which takes $\rho \mapsto \sigma$, by converting it into a distribution over $(S \times \Sigma^{\tau_h})^{L}$ from a distribution over $S \times \Sigma^{\tau_{h+1}}$ by labeling every $L^{\text{th}}$ state via $f^\star$. On the other hand, the equivalent of $\lambda_j$ in \Cref{lemma:density-ratio-1} is the following distribution over $\bX_h$,
\begin{equation} \label{eq:lambdahatnodejstar}
    \lambda_{\nd_j} (\cdot) \propto \left( \frac{1}{L} \sum_{i=0}^{L-1} \Pr \big( X_i = \cdot \mid (\underline{X},\bm{\eta}) \in \calK_{\nd_j}, \calH_{\nd_{j-1}} \big) \right) w_{\nd_j} (\cdot),
\end{equation}
where the probability is computed over $X = (X_0,\cdots,X_{L-1}) \sim \underline{\lambdahat_{\nd}}$ and $\bm{\eta} \sim \Unif([0,1])^{\otimes kL}$, and recall $\calH_{\nd_{j-1}}$ captures all randomness until the dataset $D_{\nd_j}$ is collected. The definition of the events $\calK_{\nd_j}$ is deferred to below \cref{eq:Und-def}, while the weight function \smash{$w_{\nd_j} (X) = \alpha_{r}^{j,k}$} for $r = \rank (X \| \nd_j)$ is as defined in \cref{eq:wnode}. For \smash{$\underline{X} = (X_0,\cdots,X_{L-1}) \sim \underline{\lambdahat_{\nd}}$}, we analogously define the regression of $\nd_j$ as the random variable,
\begin{equation*}
    R_{\nd_j} = \sum_{i=0}^{L-1} \alpha^{j,k}_{\rank (X_i \| \nd_j)}.
\end{equation*}
We first present a decomposition of the test error analogous to \Cref{lemma:depth1-breakdown} in the analysis of $\WLdepth{}.\warmup$. Prior to doing so, we mention that expressions involving events on different spaces are interpreted as described in \cref{eq:intersection,eq:containment}.

\begin{lemma}[Test-error decomposition] \label{lemma:depthH-breakdown} Let $\fhat^{(\nd)}$ denote the model returned by the instantiation of $\WLdepth{h+1}$ corresponding to $\nd \in \calT_{h+1}$. Let $\fhat^{(\nd_j)}$ denote the model returned by $\WLdepth{h+\frac{1}{2}}$ corresponding to $\nd_j \in \ch(\nd)$. For any sequence of events \smash{$\calK_{\nd_{k-1}} \subseteq \cdots \subseteq \calK_{\nd_0} \subseteq \bX_h^L \times [0,1]^{kL}$}, 
\begin{align*}
    &\Pr_{X \sim \lambdahat_{\nd}} \big[ \fhat^{(\nd)} (X) \ne \fstar_{\tau_{h+1}} (X) \big] \\
    &\le L \cdot \beta^{0,k}_0 + \sum_{j=0}^{k-1}  (\err_j - \err_\star) \cdot \bbE \big[ R_{\nd_j} \cdot \bbI ( (\underline{X},\bm{\eta}) \in \calK_{\nd_j}) \mid \calH_{\nd_{k-1}} \big] + \Pr \big( (\underline{X},\bm{\eta}) \not\in \calK_{\nd_{k-1}} \mid \calH_{\nd_{k-1}} \big)
\end{align*}
where, $\underline{X} \sim \underline{\lambdahat_{\nd}}$, $\bm{\eta} \sim \Unif([0,1])^{\otimes kL}$, the weight function $\beta^{j,k}_r$ is defined in \cref{eq:alpha,eq:rank} while $\err_j$ is the test error of \smash{$\fhat^{(\nd_j)}$ under the distribution $\lambda_{\nd_j}$, $\err_j = \Pr_{X \sim \lambda_{\nd_j}} \big[ \fhat^{(\nd_j)} (X) \ne \fstar_{\tau_h} (X) \big]$}.
\end{lemma}
\begin{proof}
The proof of this result follows from the statement of \Cref{lemma:depth1-breakdown} by setting $\rho \gets \lambdahat_{\nd}$, which corresponds to setting $\sigma \gets \underline{\lambdahat_{\nd}}$ and $\lambda_j \gets \lambda_{\nd_j}$, and $\calK_j \gets \calK_{\nd_j}$ and $\calH_j \gets \calH_{\nd_j}$.
\end{proof}

\noindent The prediction error of $\fhat^{(\nd_j)}$ is measured under $\lambda_{\nd_j}$ in the definition of $\err_j$, but the corresponding invocation of $\WLdepth{h+\frac{1}{2}}$ trains on data drawn from a different distribution \smash{$\lambdahat_{\nd_j}$}. Our next result shows that these two distributions satisfy a density ratio bound when the event $\calK_{\nd_j}$ is chosen appropriately.

\begin{lemma}[Density ratio bound at depth $h+\frac12$] \label{lemma:density-ratio-H}
Fix some $0 \le j \le k-1$, and recall that \smash{$\lambdahat_{\nd_j}$} denote the distribution over instances the i.i.d. dataset \smash{$D_{\nd_j}$} is sampled from. Consider the filtration \smash{$(\calH_{\nd_j})_{j=0}^{k-1}$} and sequence of events, \smash{$(\calU_{\nd_j})_{j=0}^{k-1}$} such that $\calU_{\nd_j} \subseteq \bX_h^L$ and is a measurable function of $\calH_{j-1}$. Fix some constant $c \in (0,1)$. Construct a sequence of events $\calK_{\nd_j} \in \bX_h^L \times [0,1]^{(j+1)L}$ via:
\begin{equation*}
    \calK_{\nd_j} \gets \construct ( \calU_{\nd_j} \cap \calK_{{\nd_{j-1}}} \, \| \, c, \underline{\lambdahat_{\nd}}, \calH_{{\nd_{j-1}}} )
\end{equation*}
where $\calK_{\nd_{-1}} = \bX_h^L$ and $\construct(\cdot)$ is defined in \Cref{lemma:density-ratio-1}. Then, for each $0 \le j \le k-1$,
\begin{equation*}
    \calK_{\nd_j} \subseteq \calU_{\nd_j} \cap \calK_{{\nd_{j-1}}} \quad \text{and} \quad \Pr ( \calK_{\nd_j} \mid \calU_{\nd_j} \cap \calK_{\nd_{j-1}} , \calH_{\nd_{j-1}} ) \ge 1-c
\end{equation*}
And the following density ratio bound is satisfied,
\begin{equation*}
    \left\| \frac{\lambda_{\nd_j}}{\lambdahat_{\nd_j}} \right\|_\infty \lesssim \frac{ c^{-1} \log (L)}{\Pr \big( \calU_{\nd_j} \cap \calK_{{\nd_{j-1}}}  \mid |\calI_{\nd_j}| \ge 1 , \calH_{{\nd_{j-1}}} \big)}
\end{equation*}
Here, $\calI_{\nd_j}$ is the set of accepted indices in a run of split-and-inverted-sampling (\Cref{alg:split,alg:semi-subsample}) for $\bx \sim \lambdahat_{\nd}$ (i.e., \smash{$\underline{X} \sim \underline{\lambdahat_{\nd}}$}).

\end{lemma}
\begin{proof}
The statement and proof of this result mirrors that of \Cref{lemma:density-ratio-1}. The relationship between $\underline{\lambda_{\nd}}$, $\lambda_{\nd_j}$, $\lambdahat_{\nd_j}$ is identical to the relationship between $\sigma$, \smash{$\lambda_j$} and \smash{$\lambdahat_j$} in the earlier lemma.
\end{proof}

\noindent The density ratio bound from this lemma allows the transfer of learning guarantees under $\lambdahat_{\nd_j}$ to $\lambda_{\nd_j}$ via a change-of-measure argument in iterations where $\Pr \big( \calU_{\nd_j} \cap \calK_{{\nd_{j-1}}}  \mid |\calI_{\nd_j}| \ge 1 , \calH_{{\nd_{j-1}}} \big)$ is sufficiently large. As in the analysis of $\WLdepth{}.\warmup$, we define marginal and conditional acceptance probability of split-and-inverted-sampling (similar to $\ptilde_j$ in \cref{eq:pjtilde}), and a corresponding abort event to track which child invocations of $\WLdepth{h}$ are invoked on sufficiently large effective datasets. For $0 \le j \le k-1$, define,
\begin{equation*}
    \ptilde_{\nd_j} = \Pr ( |\calI_{\nd_j}| \ge 1 \mid \calK_{\nd_{j-1}}, \calH_{\nd_{j-1}}) = \bbE \left[ 1 - \prod_{i=0}^{L-1} \Bigg( 1 - \frac{\alpha^{j,k}_{\rank (X_i \| \nd_j)}}{\alpha^{j,k}_{\max}} \Bigg) \, \middle| \, \calK_{\nd_{j-1}}, \calH_{\nd_{j-1}} \right]
\end{equation*}
Note that $\ptilde_{\nd_j}$ is closely related to $p_{\nd_j} = \Pr ( |\calI_{\nd_j}| \ge 1 \mid \calH_{\nd_{j-1}})$, which is the probability that split-and-inverted-sampling accepts \smash{$\bx \sim \lambdahat_{\nd}$}. With this in place, we define,
\begin{equation*}
    \abort_{h+\frac{1}{2}} [\nd_j] = \big\{ \ptilde_{\nd_j} \le p^\star_j \big\}
\end{equation*}
where $p^\star_j$ is as defined earlier in \cref{eq:pjtilde}. With this in place, we explicitly instantiate the sequence of $\calK_{\nd_j}$ events. First, we define,
\begin{equation} \label{eq:Und-def}
    \calU_{\nd_j} = \big\{ \ptilde_{\nd_j} \le p^\star_j \text{ and } R_{\nd_j} \ge \mathfrak{r}_j \big\}^c,
\end{equation}
where the threshold $\mathfrak{r}_j$ is as defined earlier in \cref{eq:rj}. $\calU_{\nd_j}$ captures the event that the regression is controlled across all aborted iterations. This definition is the extension of the version we defined earlier in the analysis of $\WLdepth{}.\warmup$ in \cref{eq:Uj-def}. With this definition of $\calU_{\nd_j}$, $\calK_{\nd_j}$ is defined recursively as,
\begin{equation*}
    \calK_{\nd_j} \gets \construct ( \calU_{\nd_j} \cap \calK_{{\nd_{j-1}}} \, \| \, c, \underline{\lambdahat_{\nd}}, \calH_{{\nd_{j-1}}} ) \text{ with } c = \frac{1}{32k},
\end{equation*}
with $\calK_{\nd_{-1}} = \bX_h^L$.

\medskip
\noindent In the following lemma, we control the total probability mass of $\calK_{\nd_j}$. 

\begin{lemma}[Mass on good events] \label{lemma:notUk-H}
For all $0 \le j \le k-1$, $\Pr ( \calK_{\nd_j} \mid \calH_{\nd_{j-1}}) \ge \frac{15}{16}$.
\end{lemma}
\begin{proof}
The proof of this result is identical to that of \Cref{lemma:notUk}.
\end{proof}

\noindent With this in place, in any iteration $j$ which does not abort, i.e., $\ptilde_{\nd_j} \ge p^\star_j$, following the analysis we described earlier in \cref{eq:drb-den-bd-1} to \cref{eq:drb-den-bd-2}, allows lower bounding \smash{$\Pr \big( \calU_{\nd_j} \cap \calK_{\nd_{j-1}} \mid |\calI_{\nd_j}| \ge 1, \calH_{\nd_{j-1}} \big)$} by $\Omega(\ptilde_{\nd_j})$. This implies that in any such iteration we have the density ratio bound,
\begin{equation} \label{eq:lambdalambdahat-dbound-H}
    \left\| \frac{\lambda_{\nd_j}}{\lambdahat_{\nd_j}} \right\|_\infty \lesssim k^{7/2} \log(L)
\end{equation}
Furthermore, in any non-aborted iteration $j$, we also show that the size of the dataset $D_{\nd_j}$ passed on to the child node $\nd_j$ to train a model via $\WLdepth{h}$.

\begin{lemma}[Large $p_{\nd_j} \implies \nd_j$ instantiated with a large dataset] \label{lemma:truthful-abort-semi-H}
Fix any $0 \le j \le k-1$ such that $\ptilde_{\nd_j} \ge p^\star_j$. If $\WLdepth{h+1}$ is invoked on a dataset with at least \smash{$C_2 k^{7/2} \log(L) \cdot \nsample^\star (h+\frac{1}{2}, \delta)$} instances for sufficiently large $C_2>0$, then \smash{$\Pr ( |D_{\nd_j}| < \nsample^\star ( h + \frac{1}{2}, \delta) \mid \calH_{\nd_j} ) \le \delta$}.
\end{lemma}
\begin{proof}
The proof of this result follows the same approach as that of \Cref{lemma:truthful-abort-semi-1}, where we use a Chernoff bound to control the deviations of $|D_{\nd_j}|$ which is an independent subsample of a $1/k$ fraction of $\Dnd$ with probability $p_{\nd_j}$, and noticing that $p_{\nd_j} \ge \Omega (p_j^\star)$ in non-aborted iterations (cf. \cref{eq:pj-lb}). Using the definition of $p^\star_j$ completes the proof.
\end{proof}

\medskip
\noindent Next, we use the density ratio bound in \cref{eq:lambdalambdahat-dbound-H}, along with the sufficiently large size of the dataset \smash{$D_{\nd_j}$} in any non-aborted iteration $j$ (as implied by \Cref{lemma:truthful-abort-semi-H}) to establish weak learning guarantees for \smash{$\fhat^{(\nd_j)}$}. By the induction hypothesis at the previous depth (\smash{$\IndHyp_{h+\frac{1}{2}}$}), and a change of measure argument, we next show that the model $\fhat^{(\nd_j)}$  satisfies a weak learning guarantee under the distribution $\lambda_{\nd_j}$.

\begin{lemma}[Weak learning guarantee at depth $h+\frac{1}{2}$] \label{lemma:weak-learner-H}
Consider any iteration $j$ such that $\ptilde_{\nd_j} \ge p^\star_j$. Assume \smash{$\IndHyp_{h+\frac{1}{2}} (\delta)$} and furthermore that \smash{$k_{h+1},\log(L_{h+1}) \le \log^{\calO(1)}(T)$}. Suppose \smash{$k_{h+\frac{1}{2}} = c_1 \log\log(T)$} for a sufficiently large constant $c_1$, then with probability at least $1-2\delta$, the model \smash{$\fhat^{(\nd_j)}$} returned by $\nd_j$ satisfies,
\begin{align*}
    \Pr_{X \sim \lambda_{\nd_j}} \big[ \fhat^{(\nd_j)} (X) \ne \fstar_{\tau_h} (X) \big] = \err_j \le \err_\star = \frac{1}{4}
\end{align*}
\end{lemma}
\begin{proof}
The proof of this result uses the density ratio bound we derived in \Cref{eq:lambdalambdahat-dbound}, the accuracy guarantee implied by $\IndHyp_{h+\frac{1}{2}} (\delta)$, and a change-of-measure argument. In particular, if $\abort_{h+\frac{1}{2}} [\nd_j]$ does not occur, then with probability at least $1-\delta$, the instance of $\WLdepth{h+\frac{1}{2}}$ corresponding to $\nd_j$ is invoked on a dataset of size at least \smash{$\nsample^\star (h+\tfrac{1}{2}, \delta)$} (cf. \Cref{lemma:truthful-abort-semi-H}). By $\IndHyp_{h+\frac{1}{2}} (\delta)$, with probability at least $1-\delta$, the model \smash{$\fhat^{(\nd_j)}$} returned by this node satisfies,
\begin{equation*}
    \Pr_{X \sim \lambdahat_{\nd_j}} \big[ \fhat^{(\nd_j)} (X) \ne \fstar_{\tau_h} (X) \big] \le 3 \sqrt{k_{h+\frac{1}{2}}} \cdot e^{-c k_{h+\frac{1}{2}}}.
\end{equation*}
By the density ratio bound in \cref{eq:lambdalambdahat-dbound-H}, with probability at least $1-2\delta$,
\begin{equation*}
    \Pr_{X \sim \lambda_{\nd_j}} \big[ \fhat^{(\nd_j)} (X) \ne \fstar_{\tau_h} (X) \big] \le \calO \big( k_{h+1}^{7/2} \log(L_{h+1}) \big) \cdot \Pr_{X \sim \lambdahat_{\nd_j}} \big[ \fhat^{(\nd_j)} (X) \ne \fstar_{\tau_h} (X) \big] \le \frac{1}{4}.
\end{equation*}
for $k_{h+1} = \log(L) = \log^{\calO(1)} (T)$, choosing $k_{h+\frac{1}{2}} = \Omega (\log \log (T))$ with a sufficiently large implicit constant gives the final inequality.
\end{proof}

\noindent Next we analyze the aborted iterations, i.e., $\ptilde_{\nd_j} \le p_j^\star$. In these iterations, we analyze the regression $R_{\nd_j}$ and show that if $p_{\nd_j}$ is small, then $R_{\nd_j}$ is also likely to be small.

\begin{lemma}[Aborted iterations have low regression] \label{lemma:abort=>lowreg-H}
Define $\mathfrak{r}_j$ as in \cref{eq:rj}. Supposing that $\ptilde_{\nd_j} \le p_j^\star$, then, \smash{$\Pr_{\underline{X} \sim \underline{\lambdahat_{\nd}}} \big( R_{\nd_j} \ge \mathfrak{r}_j \mid \calK_{\nd_{j-1}}, \calH_{\nd_{j-1}} \big) \le \frac{1}{32k}$}.
\end{lemma}
\begin{proof}
The proof of this structural result is identical to the analysis of \Cref{lemma:abort=>lowreg}, connecting $\ptilde_{\nd_j}$ and $R_{\nd_j}$ via Markov's inequality.
\end{proof}

\noindent Finally, we combine these results to prove \Cref{lemma:induction}. We restate this result in slightly more detail below.

\begin{lemma}[Analysis of composition step] \label{lemma:easy}
Let $\fhat^{(\nd)}$ denote the model returned by the instantiation of \smash{$\WLdepth{h+1}$} corresponding to some \smash{$\nd \in \calT_{h+1}$}. Then, assuming \smash{$L_{h+1} = L = 2^{\sqrt{\log_2(T)}}$}, \smash{$k_{h+1} = k = c_1 \sqrt{\log(T)}$} and $k_{h+\frac{1}{2}} = c_2 \log\log(T)$ for appropriate $c_1,c_2>0$, as well as $\IndHyp_{h+\frac{1}{2}} (\delta)$, with probability at least $1 - 4 \delta k_{h+1}$,
\begin{equation*}
    \Pr_{X \sim \lambdahat_{\nd}} \big[ \fhat^{(\nd)} (X) \ne \fstar_{\tau_{h+1}} (X) \big] \le \frac{1}{4}
\end{equation*}
Assuming $\IndHyp_{h+\frac{1}{2}} (\delta)$, this guarantee is achieved as long as $\WLdepth{h+1}$ draws,
\begin{equation*}
    \nsample^\star (h + 1, 2k\delta) = \nsample^\star (h + \tfrac{1}{2}, \delta) \cdot C_2 k_{h+1}^{7/2} \log(L_{h+1})
\end{equation*}
samples from $\lambdahat_{\nd}$, for some large constant $C_2 > 0$. In the process, the number of queries made to $\etoe(\cdot)$ is upper bounded by,
\begin{equation*}
    \nquery^\star (h + 1, 2k\delta) = L_{h+1} \cdot \nsample^\star (h+1,2k\delta) + k_{h+1} \cdot \nquery^\star (h+\tfrac{1}{2},\delta).
\end{equation*}
\end{lemma}

\begin{proof}
We first prove the bound on the accuracy, and subsequently show the recursion on the sample and query complexity.

\paragraph{Bound on accuracy.} Let $\calJ_{\abort} [\nd] \triangleq \big\{ 0 \le j \le k-1 : \ptilde_{\nd_j} \le p^\star_j \big\}$ denote the set of aborted iterations, which is $\calH_{\nd_{k-2}}$ measurable. By the test-error decomposition in \Cref{lemma:depthH-breakdown}, the choice of $k = \Omega (\log(L))$ (along with the upper bound on $\beta^{0,k}_0$ in \Cref{lemma:beta00}) and the lower bound we prove in \Cref{lemma:notUk-H} showing $\Pr ( \calK_{\nd_{k-1}} \mid \calH_{\nd_{k-1}}) \ge \frac{15}{16}$,
\begin{align}
    &\Pr_{X \sim \lambdahat_{\nd}} \big[ \fhat^{(\nd)} (X) \ne \fstar_{\tau_{h+1}} (X) \big] \nonumber\\
    &\le \frac{1}{16} + \sum_{j=0}^{k-1}  (\err_j - \err_\star) \cdot \bbE \big[ R_{\nd_j} \cdot \bbI ( (\underline{X},\bm{\eta}) \in \calK_{\nd_j}) \mid \calH_{\nd_{k-1}} \big] + \frac{1}{16} \label{eq:4300651}\\
    &= \frac{1}{8} + \sum_{j \in \calJ_{\abort} } \bbE [ R_{\nd_j} \cdot \bbI \big( \{ \ptilde_{\nd_j} \le p_{\nd_j}^\star \text{ and } R_{\nd_j} > \rjmax \}^c \big) \mid \calH_{\nd_{k-1}} ] + L \sum_{j \not\in \calJ_{\abort} } (\err_j - \err_\star)_+ \nonumber\\
    &= \frac{1}{8} + \sum_{j \in \calJ_{\abort} } \rjmax + L \sum_{j \not\in \calJ_{\abort} } (\err_j - \err_\star)_+ \nonumber\\
    &\le \frac{1}{4} + L \sum_{j \not\in \calJ_{\abort} } (\err_j - \err_\star)_+ \label{eq:430065}
\end{align}
where the last inequality uses \Cref{lemma:CoT-alphamax-sum}, which shows that $\sum_{j=0}^{k-1} \mathfrak{r}_j \le \frac{1}{8}$. By \Cref{lemma:weak-learner-H}, with probability $1 - 2\delta$, in any iteration $j \not\in \calJ_{\abort} [\nd]$, $\err_j \le \err_\star$. A union bound completes the bound on accuracy.

\paragraph{Query and sample complexity of $\WLdepth{h+1}$.} The sample complexity of $\WLdepth{h+1}$ is $\nsample \big(h + 1,2k\delta\big) = \nsample^\star (h+\tfrac{1}{2}, \delta) \cdot C_2 k^{7/2} \log(L)$ as required by \Cref{lemma:truthful-abort-semi-H}. The query complexity of $\WLdepth{h+1}$ is then bounded as,
\begin{align} \label{eq:11000000}
    \nquery^\star \big(h + 1, 2k\delta\big) &= L \cdot \nsample^\star (h+1,2k\delta) + k \cdot \nquery^\star (h+\tfrac{1}{2},\delta).
\end{align}
where the $L \cdot \nsample^\star (h+1,2k\delta)$ term accounts for the query cost of invoking $\split(D_\bx \| L)$ and the $k \cdot \nquery^\star (h+\tfrac{1}{2},\delta)$ term for training the $k$ children via $\WLdepth{h+\frac{1}{2}}$.
\end{proof}

\noindent With this, we are ready to furnish a proof of the final recursive guarantee.

\subsection[Final Recursive Guarantee: Proof of \Cref{theorem:main}]{Final Recursive Guarantee: Proof of \Cref{theorem:main}} \label{sec:mainproof}

\paragraph{Correctness guarantee.} By \Cref{lemma:induction-half,lemma:induction}, we know that $\IndHyp_{H+\frac{1}{2}} (\delta_{H+\frac{1}{2}})$ is true where $\delta_{H+\frac{1}{2}} = \delta \cdot \prod_{h' \in \{ 0,\frac{1}{2},\cdots,H+\frac{1}{2} \}} 2 k_{h'}$. With $k_{H+\frac{1}{2}}$ chosen as $c^{-1} \log(1/\varepsilon) \log\log(1/\varepsilon)$ and redefining $\delta \gets \delta / \prod_{h' \in \{ 0,\frac{1}{2},\cdots,H+\frac{1}{2} \}} 2 k_{h'}$, we get that with probability at least $1-\delta$,
\begin{equation*}
    \Pr_{\bx \sim \lambdahat_{\root}} \left[ \fhat (\bx) \ne \fstar_T (\bx) \right] \le \varepsilon.
\end{equation*}
Noting that $\lambdahat_{\root}$ is the distribution over $\Dnd$ with $\nd = \root$, which is simply the initial distribution over instances, $\rho$, this proves the correctness guarantee for $\WLdepth{}$. Note that $\prod_{h' \in \{ 0,\frac{1}{2},\cdots,H+\frac{1}{2} \}} 2 k_{h'} \le h_\varepsilon(T) = (\log(T))^{\calO(H)} \cdot \log^2 (1/\varepsilon)$. This guarantee requires $\Base$ to be instantiated with $\delta_0 = \delta/h_\varepsilon(T)$.

\paragraph{Sample and query complexity.}

By unrolling the induction guarantees in \Cref{lemma:induction,lemma:induction-half}, and using the definition of $\IndHyp (H+\frac{1}{2},\delta)$, the sample complexity requirement of $\WLdepth{}$ resolves to, $\nsample (H + \frac{1}{2},\delta) \ge \nsample^\star (H + \frac{1}{2},\delta)$ where, plugging in the choices of $(k_h)_{h \ge 0}$ and $(L_h)_{h \ge 0}$,
\begin{equation*}
    \nsample^\star (H + \tfrac{1}{2},\delta) \le \frac{h_\varepsilon (T)}{\varepsilon} \cdot \big[ \nweak \big( \delta/h_\varepsilon (T) \big) \vee \log (1/\delta) \big]
\end{equation*}
where $h_\varepsilon (T) = (\log(T))^{\calO(H)} \cdot \log^2(1/\varepsilon)$ upper bounds $\prod_{h' \in \{ 0,\frac{1}{2},\cdots,H+\frac{1}{2}\}} 2k_{h'}$. The query complexity upper bound also resolves to the same quantity,
\begin{align*}
    \nquery^\star (H + \tfrac{1}{2},\delta) \le \frac{h_\varepsilon (T)}{\varepsilon} \cdot \big[ \nweak \big( \delta/h_\varepsilon (T) \big) \vee \log (1/\delta) \big]
\end{align*}
This completes the proof of \Cref{theorem:main}.

\section{Analysis of $\WLdepth{}.\RL$ (\Cref{alg:semiauto_RL}): Proof of \Cref{theorem:main-RL}} \label{sec:semiautomata_RL_proofs}

$\WLdepth{}.\RL$ is an invocation of $\WLdepth{1}$, but where the base learner is instantiated as $\Base (\, \cdot \,\|\, \delta') \equiv \Vlearn \big( \cdot \big\|\, B,\varepsilon_0, \delta' \big)$ (cf \Cref{alg:Vlearn}) where \smash{$\varepsilon_0^{-1} \asymp k_{\frac12}^{7/2} \log \big( L_{\frac12} \big)$}, with decomposition schedule $L_{\frac{1}{2}} = L \triangleq T/B$ and $L_1 = 1$, and branching schedule $k_{\frac{1}{2}} = C \log(T/B)$ and $k_1 = C \log(1/\varepsilon)$ for a sufficiently large absolute constant $C>0$.

\medskip
\noindent We inherit the tree notation from the previous section to arrange the recursive calls of $\WLdepth{h}$. Here, the tree $\calT$ has levels $\{0,\frac{1}{2},1\}$, with $\calT_1=\{\root\}$. For $h\in\{\frac{1}{2},1\}$ and each $\nd\in\calT_h$, let $\ch(\nd)=(\nd_0,\ldots,\nd_{k_h-1})$ denote the ordered collection of $k_h$ children of $\nd$, each lying at level $h-\frac{1}{2}$, and set
\begin{equation*}
    \calT_{h-\frac{1}{2}} = \bigcup_{\nd\in\calT_h}\ch(\nd).
\end{equation*}
For $1 \le j\le k_h$, write $\ch_{<j}(\nd)=\{\nd_0,\ldots,\nd_{j-1}\}$ for the first $j$ children of $\nd$. In particular, $\calT_{\frac{1}{2}}=\ch(\root)$ and $|\calT_{\frac{1}{2}}|=k_1$, while $|\calT_0|=k_1 k_{\frac{1}{2}}$. As before, $\Dnd$ denotes the dataset processed by the invocation of $\WLdepth{h}$ corresponding to $\nd$, and \smash{$\lambdahat_{\nd}$} denotes the distribution from which the elements of $\Dnd$ are sampled; in particular, $\lambdahat_{\root}=\rho$. If $\nd\in\calT_0$, then $\fhat^{(\nd)}$ is the model returned by $\Base$. If $\nd\in\calT_h$ for $h\in\{1,\frac{1}{2}\}$, define
\begin{equation*}
    \fhat^{(\nd)} = \big(\ftilde^{(\nd)}\big)^{\circ L_h}, \text{ where, } \ftilde^{(\nd)} = \maj\big(\{\fhat^{(\nd')}:\nd'\in\ch(\nd)\}\big)
\end{equation*}
Thus the root has $k_1$ level-$\frac{1}{2}$ children which it aggregates with $L_1=1$, while each level-$\frac{1}{2}$ node returns the $L$-fold composition of $k_{\frac{1}{2}}$ models which solve length-$B$ instances.

\subsection{Guess-and-Check: Simulating $\etoe(\cdot)$ via $\piref$ and $\calV$} \label{app:guess-and-check}

The $\etoe(\cdot)$ oracle is used in \Cref{alg:semiauto_main} in two places: first, to compute the block-boundary states needed by $\split (\, \cdot \,\|\, L)$, and second, to compute the ranks used by $\subsample (\cdot)$. In \Cref{alg:semiauto_RL}, the $\etoe(\cdot)$ oracle is not available to the learner. Instead, the learner has access to a reference model $\piref$ and an outcome verifier $\calV$. However, both uses of $\etoe(\cdot)$ can be simulated using the reference model and the outcome verifier via a guess-and-check procedure.

\begin{observation}[Simulating $\etoe(\cdot)$ via guess-and-check] \label{obs:etoe-simulation}
Fix any length-$B$ instance $\bx \in S \times \Sigma^{B}$. By rolling out $\piref$ on $\bx$ and querying $\calV$ to verify each candidate terminal state, the correct terminal state $\fstar_B (\bx)$ can be identified with probability at least $1-\delta'$ using at most $\calO(\Cout \log(1/\delta'))$ calls to $\piref$ and $\calV$.
\end{observation}

\noindent The proof is immediate from the outcome coverage assumption: each independent rollout produces the correct terminal state with probability at least $(\Cout)^{-1}$. Therefore, for a length-$T$ instance, the boundary states required by $\split (\bx \| L)$ can be identified by applying \Cref{obs:etoe-simulation} sequentially to the $L=T/B$ consecutive length-$B$ blocks, where the blocks are processed sequentially, utilizing the terminal state of each instance as the starting state for the next block. This costs $\calO(L\Cout\log(L/\delta'))$ verifier queries and $\calO(T\Cout\log(L/\delta'))$ reference-model state generations per instance. We will refer to this approach as \guessandcheck.

\medskip
\noindent Similarly, to compute ranks in $\subsample$, for a child $\nd_j$ of a level-$\frac{1}{2}$ node $\nd$, the learner evaluates $\fhat^{(\nd')}(\bx)$ for $\nd'\in\ch_{<j}(\nd)$ and queries $\calV (\bx, \fhat^{(\nd')}(\bx))$ to test whether this terminal prediction is correct. This requires at most $k_{\frac{1}{2}}$ verifier queries per length-$B$ sub-instance.

\subsection{Analysis of the $h=\frac{1}{2}$ Level: Boosting with Composition} \label{subsec:main-RL-half-proof}

We first analyze a node \smash{$\nd\in\calT_{\frac{1}{2}}$}, which is one of the \smash{$k_1$} invocations of $\WLdepth{\frac{1}{2}}$ within $\WLdepth{}.\RL$. The analysis closely follows the proof of $\WLdepth{}$ in \Cref{lemma:induction} with $\tau=B$, and with $\TranscriptNTP$ replaced by $\Vlearn$. Calls to $\etoe(\cdot)$ are simulated as described in \Cref{obs:etoe-simulation}. In this section, we will let $k$ denote $k_{\frac{1}{2}}$ for succinctness.

\medskip
\noindent We define an abstract sequence of events $\calK_{\nd_j}$ for $0 \le j \le k-1$, which will be useful in defining the distributions under which we instantiate weak learning guarantees for $\fhat^0,\cdots,\fhat^{k-1}$. We translate the analysis of $\WLdepth{}.\warmup$ to this setting, with the translation $\bX \to S\times\Sigma^B$ and $\rho \to \lambdahat_{\nd}$. In correspondence with this choice, let $\underline{\lambdahat_{\nd}}$ denote the joint distribution of $\underline{X} = (X_0,\ldots,X_{L-1})$, where $X_i = ( s_{Bi}, \bw_{Bi+1:B(i+1)})$ and $s_{Bi}=\fstar_{Bi}(\bx)$ for \smash{$\bx=(s_0,\bw_{1:T})\sim\lambdahat_{\nd}$}. Similar to how we defined it earlier in \cref{eq:lambdahatnodejstar} in the analysis of $\WLdepth{}$, $\lambda_{\nd_j}$ is defined as the following distribution over $S\times\Sigma^B$,
\begin{equation} \label{eq:lambda-RL}
    \lambda_{\nd_j}(\cdot) \propto \left( \frac{1}{L}\sum_{i=0}^{L-1} \Pr \big( X_i=\cdot \mid ( \underline{X},\bm{\eta}) \in \calK_{\nd_j}, \calH_{\nd_{j-1}} \big) \right) w_{\nd_j}(\cdot),
\end{equation}
where the probability is computed under $\underline{X} \sim \underline{\lambdahat_{\nd}}$ and $\bm{\eta} \sim \Unif ([0,1])^{\otimes kL}$. The $\calK_{\nd_j}$ events are defined later, below \cref{eq:Und-RL}, while the weight function \smash{$w_{\nd_j}(X)$ equals $\alpha^{j,k}_{r}$} with $r=\rank(X\|\nd_j) = \big| \big\{ \pihat^{(\nd')} (X) = \fstar_{B} (X) : \nd' \in \sib_{< j} (\nd) \big\} \big|$. For $\underline{X}=(X_0,\ldots,X_{L-1})\sim\underline{\lambdahat_{\nd}}$, we define the regression of $\nd_j$ as,
\begin{equation*}
    R_{\nd_j} = \sum_{i=0}^{L-1} \alpha^{j,k}_{\rank(X_i \| \nd_j)}.
\end{equation*}
As a consequence of the test error breakdown in \Cref{lemma:depthH-breakdown}, we have the following decomposition: for any sequence of events $\calK_{\nd_{k-1}} \subseteq \cdots \subseteq \calK_{\nd_0} \subseteq \bX_h^L \times [0,1]^{kL}$, 
\begin{align}
    &\Pr_{X \sim \lambdahat_{\nd}} \big[ \fhat^{(\nd)} (X) \ne \fstar_T (X) \big] \nonumber\\
    &\le L \cdot \beta^{0,k}_0 + \sum_{j=0}^{k-1}  (\err_j - \err_\star) \cdot \bbE \big[ R_{\nd_j} \cdot \bbI ( (\underline{X},\bm{\eta}) \in \calK_{\nd_j}) \mid \calH_{\nd_{k-1}} \big] + \Pr \big( (\underline{X},\bm{\eta}) \not\in \calK_{\nd_{k-1}} \mid \calH_{\nd_{k-1}} \big) \label{eq:breakdown-RL}
\end{align}
where probabilities and expectations are over  $\underline{X} \sim \underline{\lambdahat_{\nd}}$, $\bm{\eta} \sim \Unif([0,1])^{\otimes kL}$, the weight function $\beta^{j,k}_r$ is defined in \cref{eq:alpha,eq:rank}, while $\err_\star = \frac{1}{4}$ and $\err_j$ is the test error of the model \smash{$\fhat^{(\nd_j)}$} returned by $\Base$ corresponding to $\nd_j \in \ch (\nd)$ under the distribution $\lambda_{\nd_j}$. That is,
\begin{equation} \label{eq:errj-RL}
    \err_j = \Pr_{X \sim \lambda_{\nd_j}} \big[ \fhat^{(\nd_j)} (X) \ne \fstar_{B} (X) \big].
\end{equation}
% where probabilities and expectations are over $\underline{X} \sim \underline{\lambdahat_{\nd}}$ and $\bm{\eta} = (\eta_{i,j} : 0 \le i \le L-1, 0 \le j \le k-1) \overset{\text{i.i.d.}}{\sim} \Unif([0,1])$, while $\err_\star=\frac{1}{4}$ and $\err_j$ is the error of the model $\fhat^{(\nd_j)}$ under the distribution over inputs $\lambda_{\nd_j}$,
% \begin{equation*}
%     \err_j = \Pr_{X \sim \lambda_{\nd_j}} \big[ \fhat^{(\nd_j)}(X) \ne \fstar_B(X) \big].
% \end{equation*}
Note that the model $\fhat^{(\nd_j)}$ is trained under i.i.d. examples drawn from the distribution \smash{$\lambdahat_{\nd_j}$}, but the prediction error $\err_j$ above is calculated under the distribution $\lambda_{\nd_j}$. In the sequel, we will show the $\calK_{\nd_j}$ events can be chosen in such a way that the two distributions satisfy a bound on the density ratio. To this end, define,
\begin{equation} \label{eq:Und-RL}
    \calU_{\nd_j} = \big\{ \ptilde_{\nd_j} \le p_j^\star \text{ and } R_{\nd_j} \ge \mathfrak{r}_j \big\}^c
\end{equation}
where $\mathfrak{r}_j$ is defined in \cref{eq:rj}. This event captures the regression across all aborted iterations not growing to be too large, and extends the definitions we introduced earlier (e.g., \cref{eq:Und-def}) in the analysis of $\WLdepth{}.\warmup$ and $\WLdepth{}$. The event $\calK_{\nd_j}$ is obtained recursively via,
\begin{equation*}
    \calK_{\nd_j} \gets \construct \big( \calU_{\nd_j} \cap \calK_{\nd_{j-1}} \mid c,\underline{\lambdahat_{\nd}},\calH_{\nd_{j-1}} \big) \text{ where } c= \frac{1}{32k}
\end{equation*}
where $\construct(\cdot)$ was defined earlier in \cref{lemma:density-ratio-H}. With this definition, the $\calK_{\nd_j}$ events inherit the property we established earlier in \Cref{lemma:notUk-H}: for all $0 \le j \le k-1$,
\begin{equation} \label{eq:notUk-RL}
    \Pr ( \calK_{\nd_j} \mid \calH_{\nd_{j-1}}) \ge \frac{15}{16}.
\end{equation}
By following the same argument we used to derive \cref{eq:lambdalambdahat-dbound-H}, in any iteration $j$ which does not abort, i.e., $\ptilde_{\nd_j} \ge p^\star_j$, where $p^\star_j$ was defined in \cref{eq:pjtilde}, we have,
\begin{equation} \label{eq:density-ratio-RL}
    \left\| \frac{\lambda_{\nd_j}}{\lambdahat_{\nd_j}} \right\|_\infty \lesssim k^{7/2} \log(L).
\end{equation}
Define $\nweak^{\mathrm{RL}} (\delta)$ as any upper bound on the sample complexity of $\Vlearn$ (\Cref{prop:Vlearn-B}) at the target error $\varepsilon_0 = \frac{1}{C_1 k^{7/2} \log(L)}$, say,
\begin{equation} \label{eq:nweak-RL}
    \nweak^{\mathrm{RL}} (\delta) = \calO \Big( k^{7/2} \log(L) \cdot \big( d \log(B |S| \Cblock) + \log(1/\delta) \cdot \log(k)\log\log(L) \big) \Big).
\end{equation}
where $C_1$ is sufficiently large.

\begin{lemma}[Large $\ptilde_{\nd_j} \implies \nd_j$ instantiated with a large dataset] \label{lemma:abort=>lowreg-RL}
Suppose, $|\Dnd| \ge \nsample^\star(\frac{1}{2},\delta) = C_1 k^{7/2} \log(L) \cdot \nweak^{\mathrm{RL}}(\delta/2k)$ and consider any iteration $j$ which does not abort, i.e., $\ptilde_{\nd_j} \ge p^\star_j$. Then, with probability $1-\frac{\delta}{2k}$, $|D_{\nd_j}| \ge \nweak^{\mathrm{RL}}(\delta/2k)$.
\end{lemma}
\begin{proof}
The proof is identical to that of \Cref{lemma:truthful-abort-semi-H}.
\end{proof}

\noindent Next, we establish weak learning guarantees for the model \smash{$\fhat^j$} in iterations $j$ which did not abort. Recall that $\WLdepth{}.\RL$ uses $\Vlearn$ (\Cref{alg:Vlearn}), which is trained to solve length-$\tau$ instances in the datasets $D_{\nd_j}$ for $j=0,\cdots,k-1$. Our prior work~\citep{rajaraman2026learning} established guarantees for $\Vlearn$, and via a change-of-measure argument to transfer learning guarantees under $\lambdahat_{\nd_j}$ to $\lambda_{\nd_j}$ (facilitated by the density ratio bound in \cref{eq:density-ratio-RL}), this implies learning guarantees for $\fhat^{(\nd_j)}$ under the distribution $\lambda_{\nd_j}$, which is precisely $\err_j$.

\begin{lemma}[Weak learning guarantee for $\Vlearn$] \label{lemma:weak-learner-RL}
Fix a child $\nd_j\in\ch(\nd)$. With probability at least $1-\frac{\delta}{k}$, in any iteration $j$ which does not abort, $\err_j \le \err_\star = \frac{1}{4}$, where $\err_j$ is defined in \cref{eq:errj-RL} and $\err_\star = \frac{1}{4}$.
\end{lemma}
\begin{proof}
By \Cref{eq:density-ratio-RL}, \smash{$\|\lambda_{\nd_j}/\lambdahat_{\nd_j}\|_\infty \lesssim k^{7/2} \log (L)$}. In order for $\fhat^{(\nd_j)}$ to achieve error $\frac{1}{4}$ under $\lambda_{\nd_j}$, it suffices for it to achieve prediction error at most \smash{$\varepsilon_0 = (C_1 k^{7/2} \log (L))^{-1}$} under \smash{$\lambdahat_{\nd_j}$} as long as $C_1>0$ is sufficiently large. This is guaranteed by \Cref{prop:Vlearn-B} and noting that by \Cref{lemma:abort=>lowreg-RL}, with probability at least $1-\frac{\delta}{2k}$, $|D_{\nd_j}| \ge \nweak^{\mathrm{RL}} (\delta/2k)$ meets the required sample threshold for $\Vlearn$ to achieve error $\varepsilon_0$ with probability at least $1-\frac{\delta}{2k}$.
\end{proof}

\noindent With these guarantees in place, we can finally derive a guarantee for $\pihat^{(\nd)}$.

\begin{lemma}[Constant-accuracy guarantee for $h=\frac{1}{2}$ learners] \label{prop:RL-half-constant}
Suppose $\nd\in\calT_{\frac{1}{2}}$ is invoked on a dataset $\Dnd$ consisting of at least \smash{$\nsample^\star(\tfrac{1}{2},\delta)$} independent instances drawn from \smash{$\lambdahat_{\nd}$}, where \smash{$\nsample^\star(\tfrac{1}{2},\delta)$} is defined in \Cref{lemma:abort=>lowreg-RL}. The associated model \smash{$\fhat^{(\nd)}:S\times\Sigma^T\to S$} satisfies, with probability at least $1-\delta$, 
\begin{equation*}
    \Pr_{\bx\sim\lambdahat_{\nd}} \big[ \fhat^{(\nd)}(\bx)\ne \fstar_T(\bx) \big]\le\frac{1}{4}.
\end{equation*}
Furthermore, the query complexity in achieving this guarantee is upper bounded by,
\begin{equation*}
    \calOtilde\left( \nsample^\star(\tfrac{1}{2},\delta) \left( \frac{T \Cout}{B} + \Cblock\right) \right),
\end{equation*}
and computational cost $\calOtilde\left( \nsample^\star(\tfrac{1}{2},\delta) \cdot \big( T\Cout+B\Cblock \big) \right)$.
\end{lemma}
\begin{proof}
Let $\calJ_{\abort} [\nd] = \{ 0 \le j \le k-1 : \ptilde_{\nd_j} \le p_j^\star \}$ denote the set of aborted iterations. Combining \cref{eq:breakdown-RL} and \Cref{lemma:weak-learner-RL}, we arrive at the following sequence of inequalities,
\begin{align*}
    &\Pr_{\bx \sim \lambdahat_{\nd}} \big[ \fhat^{(\nd)}(\bx) \ne \fstar_T(\bx) \big] \\
    &\le L \cdot \beta^{0,k}_0
    + \sum_{j=0}^{k-1} (\err_j - \err_\star) \cdot \bbE \big[ R_{\nd_j} \cdot \bbI ((\underline{X},\bm{\eta}) \in \calK_{\nd_j}) \mid \calH_{\nd_{k-1}} \big] + \Pr ( (\underline{X},\bm{\eta}) \not\in \calK_{\nd_{k-1}} \mid \calH_{\nd_{k-1}} ) \\
    &\overset{(a)}{\le} \frac{1}{8}
    + \sum_{j=0}^{k-1} (\err_j - \err_\star) \cdot \bbE \big[ R_{\nd_j} \cdot \bbI ((\underline{X},\bm{\eta}) \in \calK_{\nd_j}) \mid \calH_{\nd_{k-1}} \big] \\
    &\overset{(b)}{\le} \frac{1}{8}
    + \sum_{j \in \calJ_{\abort}} \rjmax + L \sum_{j \not\in \calJ_{\abort}} (\err_j - \err_\star)_+
\end{align*}
where $(a)$ uses \cref{eq:notUk-RL} and the choice of $k=\Theta(\log L)$, which bounds $L \beta^{0,k}_0 \le 1/16$ (cf. \Cref{lemma:beta00}), and $(b)$ uses the same analysis we carried out in \cref{eq:4300651} to \cref{eq:430065}. By \Cref{lemma:CoT-alphamax-sum} and \Cref{lemma:weak-learner-RL}, with probability at least $1-\delta$,
\begin{equation*}
    \Pr_{\bx \sim \lambdahat_{\nd}} \big[ \fhat^{(\nd)}(\bx) \ne \fstar_T(\bx) \big] \le \frac{1}{8} + \frac{1}{8} = \frac{1}{4}.
\end{equation*}

\paragraph{Sample and query complexity, and computational cost.} The sample complexity at the $h=\frac{1}{2}$ depth is \smash{$\nsample^\star(\tfrac{1}{2},\delta)$} from \cref{prop:RL-half-constant}. For the bounds on the query complexity and computational cost, simulating block boundaries on \smash{$\nsample^\star(\tfrac{1}{2},\delta)$} instances contributes a total of \smash{$\calOtilde( \nsample^\star(\tfrac{1}{2},\delta) \cdot (T/B)\Cout)$} verifier queries and \smash{$\calOtilde( \nsample^\star(\tfrac{1}{2},\delta) \cdot T\Cout)$} reference-model state generations. The calls to $\Vlearn$ use $\calOtilde( \nsample^\star(\tfrac{1}{2},\delta) )$ accepted length-$B$ training instances in total, which contributes \smash{$\calOtilde( \nsample^\star(\tfrac{1}{2},\delta) \cdot \Cblock)$} verifier queries and requires \smash{$\calOtilde( \nsample^\star(\tfrac{1}{2},\delta) \cdot \Cblock)$} length-$B$ rollouts from $\piref$. Combining both bounds results in the bound on query complexity and computational cost.
\end{proof}

\subsection{Analysis of the $h=1$ Level: Boosting to Error $\varepsilon$} \label{subsec:main-RL-proof}

We next analyze the root node of $\calT$. Since $L_1=1$, this node only carries out aggregation and no composition. $\root$ invokes the constant-accuracy guarantee from \Cref{prop:RL-half-constant} at each node in $\calT_{\frac{1}{2}}$, on a sequence of reweighted distributions over length-$T$ instances, and aggregates the resulting models by plurality vote. Because $L_1=1$, $\subsample$ produces \emph{exact} i.i.d.\ samples from the tilted distribution $\lambda_{\root_j}\propto\rho\cdot w_{\root_j}$, so no change-of-measure or good-event truncation is needed (in contrast with the half level). The analysis exactly mirrors the no-composition aggregation step of \Cref{lemma:easy-half}. For the purpose of this section, we will let $k_1$ be denoted as $k$.

\medskip
\noindent Throughout, we write $\root_j$ for the $j^\text{th}$ child of the root and set $\rho=\lambdahat_{\root}$. The models $\{\fhat^{(\root_j)}:0\le j<k\}$ produced by the half-level instances at the children of the root play the role of the weak learners at this level. For $X\sim\rho$, the rank $\rank(X\|\root_j)$ counts how many of the previous models $\fhat^{(\root_0)},\ldots,\fhat^{(\root_{j-1}})$ correctly predict $\fstar_T(X)$. On the other hand, the weight \smash{$w_{\root_j}(X)=\alpha^{j,k}_{\rank(X\|\root_j)}$} defines the distribution $\subsample$ samples from (via rejection sampling from $\rho$) in the $j^\text{th}$ child node of $\root$. In particular, since $L_1=1$, $\subsample$ samples i.i.d.\ from,
\begin{equation*}
    \lambdahat_{\root_j} (\cdot) \propto \rho(\cdot)\,w_{\root_j}(\cdot),
\end{equation*}
The marginal acceptance probability of $\subsample$ at round $j$ is, by specializing \cref{eq:pj-formula} to $L=1$,
\begin{equation} \label{eq:pj-RL-top}
    p_{\root_j} = \bbE_{X\sim\rho} \left[\frac{\alpha^{j,k}_{\rank(X\|\root_j)}}{\alpha^{j,k}_{\max}}\,\middle|\,\calH_{\root_{j-1}}\right],
\end{equation}
where $\calH_{\root_{j-1}}$ is all the randomness within the algorithm up until the dataset $D_{\root_j}$ is collected. 
$p_{\root_j}$ is the probability with which an instance \smash{$\bx \sim \lambdahat_{\root}$} is accepted by split-and-inverted-sampling (\Cref{alg:split,alg:semi-subsample}). For each child $\root_j\in\ch(\root)$, define the abort event,
\begin{equation} \label{eq:abort:rootj}
    \abort_1 [\root_j] = \big\{  p_{\root_j} \le p^\star_j \big\}, \text{ where, } p^\star_j = e^{-ck}
\end{equation}
where $c > 0$ is the constant in \cref{lemma:beta00}. Firstly, as a consequence of the test-error decomposition we proved earlier in \Cref{lemma:depthH-breakdown-half}, with $\nd \gets \root$ and $\lambdahat_{\nd} \gets \rho$, the model $\fhat^{(\root)}$ satisfies,
\begin{equation} \label{eq:breakdown-top}
    \Pr_{\bx\sim\rho} \big[ \fhat^{(\root)}(\bx)\ne \fstar_T(\bx)\big]
    \le \beta^{0,k}_0
    + \sum_{j=0}^{k-1}(\err_j-\err_\star) \cdot\bbE_{X\sim\rho} \big[ \alpha^{j,k}_{\rank(X\|\root_j)} \big],
\end{equation}
where $\err_\star=\frac{1}{4}$ and \smash{$\err_j=\Pr_{\Xbar\sim\lambdahat_{\root_j}} \big[ \fhat^{(\root_j)}(\Xbar)\ne \fstar_T(\Xbar)\big]$}. First we argue that in any iteration which doesn't abort, the dataset $D_{\root_j}$ must be sufficiently large.

% $|D_{\root_j}| < \nsample^\star(\tfrac{1}{2},\delta)$
% where $\nsample^\star(\tfrac{1}{2},\delta)=\widetilde\calO(d)$ is defined in \Cref{prop:RL-half-constant}. If $\abort_1[\root_j]$ occurs, the corresponding invocation of $\WLdepth{\frac{1}{2}}$ aborts.

\begin{lemma}[Truthful aborts at $\root$] \label{lemma:truthful-abort-RL-top}
Suppose \smash{$|D_\root| \ge \nsample^\star (1,2k\delta) \triangleq 2 e^{c k} \cdot \nsample^\star (\frac{1}{2},\delta)$}, where $c>0$ is the constant in the exponent of \smash{$\beta^{0,k}_0$} in \cref{lemma:beta00}. Then, for any iteration $j$ such that $p_{\root_j} \ge e^{-c k}$, $\Pr \big( |D_{\root_j}| < \nsample^\star (\frac{1}{2}, \delta) \mid \calH_{\root_{j-1}} \big) \le \delta$.
\end{lemma}
\begin{proof}
The proof of this result follows the same structure as \Cref{lemma:truthful-abort-semi-H-half}.
\end{proof}

\noindent As a consequence of this lemma, in the iterations that don't abort, the corresponding instantiations of $\WLdepth{\frac{1}{2}}$ are invoked on a dataset $D_{\root_j}$ containing at least $\nsample^\star(\frac{1}{2},\delta)$ samples.

\begin{lemma}[Weak learner guarantee for $h=\frac{1}{2}$] 
\label{lemma:weak-learner-RL-top}
Consider any $j$ such that $p_{\root_j} \ge p_j^\star$. With probability at least $1-2\delta$, the model \smash{$\fhat^{(\root_j)}$} returned by $\root_j$ satisfies,
\begin{align*}
    \Pr_{X \sim \lambda_{\root_j}} \big[ \fhat^{(\root_j)} (X) \ne \fstar_T (X) \big] = \err_j \le \err_\star \triangleq \frac{1}{4}
\end{align*}
\end{lemma}
\begin{proof}
Identical to \Cref{lemma:weak-learner-H-half}. If $\abort_1[\root_j]$ is false, then the instance of \smash{$\WLdepth{\frac{1}{2}}$} at $\root_j$ is invoked with at least $\nsample^\star(\tfrac{1}{2},\delta)$ accepted instances from \smash{$\lambdahat_{\root_j}$}. As a consequence of \Cref{prop:RL-half-constant}, we have that, \smash{$\Pr_{\Xbar\sim\lambda_{\root_j}} \big[ \fhat^{(\root_j)}(\Xbar)\ne \fstar_T(\Xbar)\big] \le \frac{1}{4}$} with probability at least $1-\delta$.
\end{proof}

\noindent Combining these lemmas, we establish a guarantee for the model $\pi^{(\root)}$.

\begin{lemma}[Guarantee at $\root$] \label{lemma:RL-top-boosting}
Fix $\varepsilon \in (0,1)$. Let $\fhat \gets \fhat^{(\root)}$ denote the model returned by the instantiation of \smash{$\WLdepth{1}$} corresponding to $\root$, i.e., the output of $\WLdepth{}.\RL$. Suppose $k = C\log(1/\varepsilon)$. Then, with probability at least $1 - 2 k \delta$, $\Pr_{X \sim \lambdahat_{\root}} \big[ \fhat (X) \ne \fstar_T (X) \big] \le \varepsilon$. Furthermore, this guarantee is achieved if $\root$ is invoked on a dataset of $\nsample^\star (1, 2k\delta)$ samples from \smash{$\lambdahat_{\root}$}.
\end{lemma}
\begin{proof}
Conditioned on the event that $\abort_1 [\root_j]$ is false for some $j$, with probability at least $1-2\delta$, we have that $\err_j \le \err_\star$ (\Cref{lemma:weak-learner-RL-top}). Noting that in any iteration $j$ where $\abort_1 [\root_j]$ is true, $p_{\root_j} \le p_j^\star$. Using the definition of $p_{\root_j}$ and by a union bound, w.p. at least $1 - 2k\delta$, 
\begin{align*}
    \Pr_{X \sim \lambdahat_{\root}} \big[ \fhat^{(\root)} (X) \ne \fstar_T (X) \big] &\le \beta_0^{0,k} + \sum_{j=0}^{k-1} e^{-ck} \cdot \alpha^{j,k}_{\max} \\
    &\overset{(a)}{\le} e^{-c k} + 2 \sqrt{k} \cdot e^{-c k} \\
    &\le 3 \sqrt{k} \cdot e^{-c k}
\end{align*}
In $(a)$, we use \Cref{lemma:CoT-alphamax-sum,lemma:beta00} (the bound on $\sum_{j=0}^{k-1} \alpha^{j,k}_{\max}$ is implied by the bound on $\sum_{j=0}^{k-1} \rjmax$). Finally, plugging in $k = C \log(1/\varepsilon)$ for sufficiently large $C>0$ completes the proof.
\end{proof}

\subsection{Proof of \Cref{theorem:main-RL}}
Recall that $\WLdepth{}.\RL$ is the invocation of $\WLdepth{1}$ at $\root$ with the base learner $\Vlearn$. The accuracy guarantee for $\WLdepth{}.\RL$ is immediate: \Cref{lemma:RL-top-boosting}, applied at $\root$ with \smash{$\lambdahat_{\root} = \rho$} shows that the model \smash{$\fhat^{(\root)}$} returned by $\WLdepth{}.\RL$ satisfies $\Pr_{\bx\sim\rho} \big[ \fhat^{(\root)}(\bx)\ne \fstar_T(\bx) \big] \le\varepsilon$ with probability at least $1-2k_1\delta$, and redefining $\delta \gets \delta/2k_1$ gives us the bound on accuracy. It remains to bound the sample, query, and computational complexity of $\WLdepth{}.\RL$.

\paragraph{Bound on sample complexity.} By \Cref{lemma:RL-top-boosting}, $\root$ requires a dataset of $\nsample^\star (1,\delta) = \calOtilde(d/\varepsilon)$ instances drawn from $\rho$, resulting in the upper bound on sample complexity of $\WLdepth{}.\RL$.

\paragraph{Bound on query complexity.} Verifier queries are incurred by $\WLdepth{}.\RL$ in two places. First, at $\root$, the aggregation step of $\WLdepth{1}$ uses $\calV$ to compute the rank of each of the $\nsample^\star(1,\delta)$ length-$T$ instances with respect to the $k_1$ child models, contributing $\nsample^\star(1,\delta)$ verifier queries in total. Second, at each child $\nd\in\calT_{\frac{1}{2}}$, by \Cref{prop:RL-half-constant}, each invocation of $\WLdepth{\frac{1}{2}}$ incurs a burn-in query cost of
\begin{equation*}
    \calOtilde\left( \frac{dT \Cout}{B} + d\Cblock\right)=\frac{\nburnin}{B},
\end{equation*}
Summing over the $k_1=\calOtilde(1)$ children in $\calT_{\frac{1}{2}}$ and combining with the contribution from $\root$ gives,
\begin{equation*}
    \nquery \le \calOtilde\left(\frac{d}{\varepsilon}\right)+\frac{\nburnin}{B}.
\end{equation*}
Note that these bounds are stated for when the target error probability is $2 k_1 \delta$. Replacing $\delta \gets \delta/2k_1$ results in a bound which is larger by a factor of $\log(2k_1) = \calOtilde(1)$.

\paragraph{Bound on computational cost.} The computational cost of $\WLdepth{}.\RL$ likewise decomposes into a contribution at $\root$ and contributions at the children $\nd\in\calT_{\frac{1}{2}}$. At $\root$, evaluating each of the $k_1$ length-$T$ child models on each of the \smash{$\calOtilde(d/\varepsilon)$} instances costs \smash{$\calOtilde(dT/\varepsilon)$} state computations. At each child $\nd\in\calT_{\frac{1}{2}}$, the invocation of $\WLdepth{\frac{1}{2}}$ incurs a burn-in computational cost of \smash{$\calOtilde(dT\Cout+dB\Cblock)=\nburnin$} by \Cref{prop:RL-half-constant}. Summing over the \smash{$k_1=\calOtilde(1)$} children gives
\begin{equation*}
    \ncomp \le \calOtilde\left(\frac{dT}{\varepsilon}\right)+\nburnin.
\end{equation*}
Note that, as in the analysis of the query complexity, this bound is stated for when the target error probability is $k_1 \delta$. Replacing $\delta \gets \delta/k_1$ results in a bound which is larger by a factor of $\log(k_1) = \calOtilde(1)$. Combining the accuracy, sample, query, and computational bounds completes the proof of \Cref{theorem:main-RL}.

\section{Proofs of Lemmas}
\label{sec:proof-lemmas}

\subsection[Prediction Error Decomposition via Potentials: Proof of Lemma~\ref{lemma:depth1-breakdown}]{Upper Bounding Prediction Error via Potentials: Proof of \Cref{lemma:depth1-breakdown}} \label{subsec:depth1-breakdown-proof}

In this section, we will produce an upper bound which decomposes the error of the model returned by the depth-$1$ construction into the regression accumulated across iterations which abort. In order to establish this lemma, we will first introduce a potential function which will serve as a proxy to bounding the accuracy of the plurality model. For $0 \le j \le k$, define the potential function,
\begin{equation*}
    \Phi_j = \sum_{i=0}^{L-1} \bbE \big[ \beta^{j,k}_{\rank_j (X_i)} \cdot \bbI ( (\underline{X},\bm{\eta}) \in \calK_{j-1} ) \mid \calH_{k-1}  \big]
\end{equation*}
Here, $\underline{X} \sim \sigma$ and $\bm{\eta} = (\bm{\eta}_j)_{j=0}^{k-1}$ where $\bm{\eta}_j = (\eta_{j,i})_{i=0}^{L-1}$. The proof of \Cref{lemma:depth1-breakdown} follows by first showing how the potential function $\Phi_k$ relates to the test loss of the final model returned by $\WLdepth{}.\warmup$, $\fhat$. Then, we show how $\Phi_k$ can be analyzed by writing it down as $\sum_{j=0}^{k-1} \Phi_{j+1} - \Phi_j$ and bounding the successive differences.

\begin{lemma}[Potential function decomposition] \label{lemma:loss-decomp-Phi}
Let $\fhat$ denote the model returned by $\WLdepth{}.\warmup$. Then,
\begin{equation*}
    \Pr_{\bx \sim \rho} \big[ \fhat (\bx) \ne \fstar_T (\bx) \big] \le \Phi_k + \Pr ( (\underline{X},\bm{\eta}) \not\in \calK_{k-1} \mid \calH_{k-1}).
\end{equation*}
\end{lemma}
\begin{proof}
The proof of this result is deferred to \Cref{sec:error-pot-relation-proof}.     
\end{proof}

\noindent Next we show how the potential difference $\Phi_{j+1} - \Phi_j$ evolves, by relating it to the $\err_j$ error terms. The following lemma is adapted from \cite[Lemma 3.7]{Freund}, with a subtle modification to account for the nested $\calK_j$ events.

\begin{lemma}[Recursive bound on potentials] \label{lemma:recursion}
Let $\err_j$ be as defined in the statement of \Cref{lemma:weak-learner}. The potential function $\Phi_j$ satisfies the following recurrence relation,
\begin{align*}
    \Phi_{j+1} - \Phi_j \le (\err_j - \err_\star) \cdot \bbE \big[ R_j \cdot \bbI ( (\underline{X},\bm{\eta}) \in \calK_j) \mid \calH_{k-1} \big]
\end{align*}
\end{lemma}
\begin{proof}
The proof of this result is deferred to \Cref{sec:pot-recursion-proof}.
\end{proof}

\subsubsection{Proof of \Cref{lemma:depth1-breakdown}}

By cascading \Cref{lemma:recursion} across $j$ and combining with \Cref{lemma:loss-decomp-Phi}, we get,
\begin{align*}
    \Pr_{\bx \sim \rho} \big[ \fhat (\bx) \ne \fstar_T (\bx) \big] &\le \Phi_k + \Pr ( (\underline{X},\bm{\eta}) \not\in \calK_{k-1} \mid \calH_{k-1}) \\
    &\le \Phi_0 + \sum_{j=0}^{k-1} (\err_j - \err_\star) \cdot \bbE \big[ R_j \cdot \bbI ( (\underline{X},\bm{\eta}) \in \calK_j) \mid \calH_{k-1} \big] + \Pr ( (\underline{X},\bm{\eta}) \not\in \calK_{k-1} \mid \calH_{k-1}).
\end{align*}
Noting that $\Phi_0 = \sum_{i=0}^{L-1} \bbE \big[ \beta^{0,k}_{\rank_0(X_i)} \mid \calH_{k-1} \big] = L \cdot \beta^{0,k}_{0}$ completes the proof.

\subsection[Sampling Approximately from $\lambda_j$: Proof of Lemma~\ref{lemma:density-ratio-1}]{Sampling Approximately from $\lambda_j$: Proof of \Cref{lemma:density-ratio-1}} \label{subsec:density-ratio-1-proof}

In this section, we will prove \Cref{lemma:density-ratio-1} by studying the underlying problem of sampling from a reweighted mixture of marginals (\Cref{def:srmm}) which captures its more general structure. 

\medskip
\noindent The problem of sampling from the reweighted mixture of marginals posits that the learner has sampling access to a distribution $\mu$ over some product space $\bX^L$, with the objective of approximately generating samples from the distribution $\nu (\cdot) \propto \frac{1}{L}\sum_{i=0}^{L-1} \mu(X_i = \cdot\,) w(\cdot)$ for some weight function $w : \bX \to [0,1]$. For appropriately chosen $w$, $\nu$ captures the idealized distribution to train weak learners under when boosting is used to train a model with small error under $\mubar$. When connected back to \Cref{lemma:density-ratio-1}, $\mu$ corresponds to the distribution $\sigma$ over tuples (of size $L$) of instances of length $\tau$, while $\nu (\cdot)$ corresponds to a reweighted version of the distribution $\rhobar (\cdot) = \frac{1}{L}\sum_{i=0}^{L-1} \rho (X_i = \cdot\,)$. We use $\underline{X} = (X_0,\cdots,X_{L-1})$ to denote samples from $\mu$. The first part of this section will focus on the case where there is a single weight function and the goal is to sample from its reweighted mixture of marginals, $\nu$.

\paragraph{What does it mean to sample from $\nu$ \textit{approximately}?} It is worth stopping to ask what kind notion of approximation we are willing to tolerate to sample from $\nu$. In the context of boosting, the learner aims to train a model which achieves constant (say, $\frac{1}{4}$) error under $\nu$. To achieve such a guarantee, it suffices for the learner to generate a sample from some other distribution $\nuhat$ which satisfies the following \textit{approximate coverage} guarantee:
\begin{equation} \label{eq:apx-cov}
    \Pr_{X \sim \nu} \left( \frac{\nu (X)}{\nuhat (X)} \ge \Cmax \right) \le \delta_{\text{cov}}
\end{equation}
In particular, under the above assumption, as long as $\delta_{\text{cov}} < \frac{1}{4}$, a model which achieves error $\varepsilon = \Cmax^{-1} (\frac{1}{4}-\delta_{\text{cov}})$ under the distribution over inputs $\nuhat$ also achieves error $\frac{1}{4}$ under the distribution $\nu$. The requirement in \cref{eq:apx-cov} is natural, and in fact (via an application of Markov's inequality) implied by closeness of $\nu$ and $\nuhat$ in any $f$-divergence as long as $f$ grows faster than a linear function. In particular, a bound on $\KL{\nu}{\nuhat}$ implies \cref{eq:apx-cov} with appropriately chosen $C$ and $\delta_C$. 

\medskip
\noindent Sampling from a distribution with approximate coverage to $\nu$ is only a sufficient condition for transferring error guarantees under $\nuhat$ to under $\nu$. Since we aim to train models via boosting, it may suffice to use a weaker notion of approximation, and we discuss a candidate in the next paragraph. The main reason for doing so is not merely convenience: in \Cref{lemma:worst-case-sampling} we establish a \textit{strong lower bound} to achieve the guarantee in \cref{eq:apx-cov} unless $\Cmax$ scales as $\poly(L)$ or the sampling algorithm uses $\poly(L)$ samples from $\mu$.

\subsubsection{A weakening of approximate coverage}

Rather than requiring the learner to satisfy \cref{eq:apx-cov}, we will consider the following weakened version. Let $\mutilde$ denote any distribution on $\bX^L$ such that $\TV{\mu}{\mutilde}$ is at most $c$. Define $\nutilde \propto \overline{\mutilde} (\cdot) w(\cdot)$ where $\overline{\mutilde}$ is the mixture of marginals of $\mutilde$. Define the following weakened version of approximate coverage,
\begin{equation} \label{eq:apx-apx-cov}
    \text{There exists } \mutilde \text{ such that } \TV{\mu}{\mutilde} \le c \text{ and } \Pr_{X \sim \nutilde} \left( \frac{\nutilde (X)}{\nuhat (X)} \ge \Cmax \right) \le \delta_{\text{cov}}.
\end{equation}
In particular, this relaxation only requires a sampling algorithm to achieve approximate coverage with respect to the reweighted mixture of marginals of \textit{any distribution $\mutilde$ which is close to $\mu$}. Note that the operation of taking the mixture of marginals and tilting is highly ``non-Lipschitz'', so even though $\mu$ and $\mutilde$ are close to one another, $\nu$ and $\nutilde$ may be far apart.

\medskip
\noindent In order to realize the above style of guarantee, we will construct $\mutilde$ by considering restrictions/conditionals of $\mu$ unto some high probability set. In particular consider some events $\calK \subseteq \bX^L \times [0,1]^{L}$ and define,
\begin{equation*}
    \mu_\calK (\cdot) = \Pr (\underline{X} = \cdot\, |\, (\underline{X},\bm{\eta}) \in \calK) \text{ and } \overline{\mu_\calK} (\cdot) = \frac{1}{L} \sum_{i=0}^{L-1} \Pr (X_i = \cdot\, |\, (\underline{X},\bm{\eta}) \in \calK)
\end{equation*}
where $\bm{\eta} \sim \Unif([0,1])^{\otimes L}$ is a sequence of i.i.d. uniform random variables. If $\calK = \bX^L \times [0,1]^L$ is total, $\mu_\calK = \mu$. The first result we establish in this section shows that when $\calK$ is chosen carefully, the guarantee in \cref{eq:apx-apx-cov} can be achieved with $\nutilde \gets \nu_\calK \propto \overline{\mu_\calK} (\cdot) w (\cdot)$, with $\delta_{\text{cov}} = 0$ and $\Cmax = \calO(c^{-1} \log (L))$, for any $c>0$ bounding the TV distance between $\mu$ and $\mu_\calK$. %We state this for a single weight function below, and generalize it to the case where there are $k$ weight functions in \Cref{lemma:apxsample-k}.

\begin{lemma}[Approximate sampling from $\nu_\calK$] \label{lemma:apxsample-specialcase}
Fix any distribution $\mu$ over $\bX^L$ and non-negative weight function $w : \bX \to [0,1]$. Let $\nuhat$ be the distribution over $\bX$ induced by the procedure:

\begin{procedure}[label=proc:nuhat]{}
Draw $\underline{X} \sim \mu$ until one is accepted by the following filter: Sample random variables $\bm{\eta} = (\eta_i)_{i=0}^{L-1} \sim \unif ([0,1])^{\otimes L}$, and define $\calI = \big\{ \eta_i \le \frac{w (X_i)}{\| w \|_\infty} : 0 \le i \le L-1 \big\}$; accept $\underline{X}$ and $|\calI|\ge 1$ and return $X_I$ for $I \sim \unif (\calI)$.
\end{procedure}
\smallskip
\noindent Fix any $c \in (0,1)$. There exists $\calK \subseteq \bX^L \times [0,1]^{L}$ such that, $\Pr ( (\underline{X},\bm{\eta}) \not\in \calK ) \le c$. Furthermore,
\begin{align} \label{eq:drb}
    \left\| \frac{\nu_\calK}{\nuhat} \right\|_\infty \le \Cmax, \text{ where } \Cmax = \calO \left( c^{-1} \cdot \log (L) \right)
\end{align}
Finally, the probability that $\underline{X} \sim \mu$ is accepted within \Cref{proc:nuhat} to generate a sample from $\nuhat$ is,
\begin{equation} \label{eq:pmuw}
    p_{\mu,w} = \bbE \left[ 1 - \prod_{i=0}^{L-1} \left( 1 - \frac{w (X_i)}{\| w \|_\infty} \right) \right]
\end{equation}
\end{lemma}
\begin{proof}
This result is a corollary of \Cref{lemma:apxsample} setting $\calU = \bX^L$.
\end{proof}

\begin{remark}
Note that applying $\subsample (D_{\bx} \| \Pi_j,k)$ where $D_\bx \gets \split (\bx|L)$ (\Cref{alg:split,alg:semi-subsample}) for $\bx = (s_0,\bw_{1:T}) \sim \rho$ is identical to applying \Cref{proc:nuhat} to $\underline{X}$ constructed from $\bx$ by splitting and labeling the intermediate states using $f^\star$: i.e., $\underline{X} = (X_i)_{i=0}^{L-1}$ where $X_i = (s_{i\tau},\bw_{i\tau+1:(i+1)\tau})$ where $s_{i\tau} = f^\star_{i \tau} (\bx)$.
\end{remark}

\medskip
\noindent Note that \Cref{proc:nuhat} achieves the guarantee in \Cref{lemma:apxsample-specialcase} using just a single draw from $\mu$ if \smash{$p_{\mu,w} = 1$}. Even in this special case, any sampling algorithm which is forced to achieve approximate coverage \cref{eq:apx-cov} with respect to $\mu$ (i.e., \cref{eq:apx-apx-cov} with $c=0$) achieves compete with the best guarantee achievable if $\mu$ is allowed to be slightly perturbed to $\mutilde$. We will discuss this further in \cref{sec:sampling-lb-coverage}.

\medskip
\noindent \Cref{lemma:apxsample-specialcase} itself is a special case of \Cref{lemma:apxsample} below, which extends this result to the setting where the learner aims to approximately sample from the reweighted mixture of marginals of $\mu$. 

\begin{lemma}[Approximate sampling from $\nu_\calK$] \label{lemma:apxsample}
Fix any distribution $\mu$ over $\bX^L$ and non-negative weight function $w : \bX \to [0,1]$. Let $\nuhat$ be the distribution over $\bX$ induced by \Cref{proc:nuhat}. Consider any event $\calU \subseteq \bX^L$ and fix $c \in (0,1)$. There exists $\calK \subseteq \bX^L \times [0,1]^{L}$ such that,
\begin{equation*}
    (\underline{X},\bm{\eta}) \in \calK \implies \underline{X} \in \calU \quad \text{ and } \quad \Pr ( (\underline{X},\bm{\eta}) \not\in \calK \mid \underline{X} \in \calU ) \le c.
\end{equation*}
Furthermore,
\begin{align} \label{eq:drb-relative}
    \left\| \frac{\nu_\calK}{\nuhat} \right\|_\infty \le \Cmax, \text{ where } \Cmax = \calO \left( \frac{c^{-1} \cdot \log (L)}{\Pr (\underline{X} \in \calU \mid |\calI| \ge 1 )} \right)
\end{align}
Finally, the probability that $\underline{X} \sim \mu$ is accepted within \Cref{proc:nuhat} to generate a sample from $\nuhat$ is $p_{\mu,w}$ (cf. \cref{eq:pmuw}).
\end{lemma}
\begin{proof}
This result is proved formally in \Cref{sec:apxsample-proof}, and a consequence of the two bounds we establish in \Cref{lemma:nu_calK} and \Cref{lemma:nuhat-lb}. The former is a pointwise upper bound on $\nu_{\calK}$, and the latter, a pointwise lower bound on $\nuhat$. The bound on the acceptance probability of $\underline{X} \sim \mu$ by \Cref{proc:nuhat} is a direct calculation using the fact that for $\underline{X}$ to be accepted, $|\calI| \ne 0$, equivalent to the existence of an index $0 \le i \le L-1$ such that \smash{$\eta_{i} \le \frac{w(X_i)}{\| w \|_\infty}$.}
\end{proof}

\noindent The extension of this result to the case where there are multiple weight functions, is relatively straightforward and discussed further in \Cref{sec:apxsample-k-proof} where we also prove \Cref{lemma:density-ratio-1}. In particular \Cref{lemma:density-ratio-1} considers the case we have a sequence of $k$ (potentially random) weight functions $(w_j)_{j=0}^{k-1}$ and events $(\calU_j)_{j=0}^{k-1}$, and constructs a sequence of events $(\calK_j)_{j=0}^{k-1}$ which enable density ratio bounds akin to \cref{eq:drb-relative}.

\subsubsection{A sampling lower bound for achieving approximate coverage} \label{sec:sampling-lb-coverage}

In order to further interpret the bound in \Cref{lemma:apxsample-specialcase}, it will be helpful to mention a worst case bound which pertains to the guarantee achievable in the absence of the conditioning event $\calK$. We show that sampling from a distribution satisfying approximate coverage (cf. \cref{eq:apx-cov}) to the target distribution $\nu \propto \overline{\mu} (\cdot) w(\cdot)$, even with a modest $\delta_{\text{cov}} = \frac{1}{4}$, either requires $\Cmax = \poly(L)$ or requires the sampling algorithm to draw $\poly(L)$ samples form $\mu$, even when considering the ideal scenario where $p_{\mu,w} = 1$.

\begin{lemma}[Worst-case bound when $\calK^c = \emptyset$] \label{lemma:worst-case-sampling}
Let $\bX = \{ 0,1,\cdots,L+1 \}$. There exists a class of distributions $\bm{\mu} = \{ \mu_{p,z} : (p,z) \in [0,1] \times [L] \}$, parameterized by $p \in (0,1)$ supported on $\bX^L$, as well as a fixed weight function $w : \bX \to [0,1]$ such that:
\begin{enumerate}
    \item $p_{\mu,w} = 1$ for all $\mu \in \bm{\mu}$.
    \item For a sufficiently small absolute constant $c > 0$, consider any sampling algorithm $\Alg$ which draws at most $N = c \sqrt{L}$ samples drawn from the distribution $\mu \in \bm{\mu}$ to generate a sample from some distribution $\nuhat \in \Delta_{\{0,1,\cdots,L+1\}}$, and such that $\nuhat$ is well defined for every base distribution $\mu \in \bm{\mu}$. Then the distribution $\nuhat$ realized by $\Alg$ must incur,
    \begin{equation*}
        \max_{\mu \in \bm{\mu}} \TV{\nuhat}{\nu} \ge \frac{1}{2}
    \end{equation*}
    where $\nu$ is the distribution $\propto \mubar (\cdot) w(\cdot)$.
    \item Furthermore, if $\Alg$ draws at most $N = c' L^{1/3}$ samples from $\mu \in \bm{\mu}$ for a sufficiently small constant $c'>0$, then the distribution $\nuhat$ realized by $\Alg$ must incur,
    \begin{equation*}
        \max_{\mu \in \bm{\mu}} \Pr_{X \sim \nu} \left( \frac{\nu(X)}{\nuhat(X)} \ge c' L^{1/3} \right) \ge \frac{1}{2}
    \end{equation*}
\end{enumerate}
\end{lemma}
\begin{proof}
The proof of this result is discussed in \Cref{subsec:sampling-lb}.
\end{proof}

\subsection{$\calK_j$ is a High Probability Event: Proof of \Cref{lemma:notUk}} \label{subsec:notUk-proof}

We first introduce an auxiliary lemma which will be helpful to prove \Cref{lemma:notUk}. The following result establishes that the $\calU_j$ events occur simultaneously with moderately high probability.

\begin{lemma} \label{lemma:UjHj} Recall that Across all iterations $0\le j \le k-1$, we have $\Pr (\calU_j \mid \calK_{j-1}, \calH_{j-1}) \ge 1-\frac{1}{32k^2}$.
\end{lemma}
\begin{proof}
The proof of this result is in \Cref{sec:UjHj-proof}.
\end{proof}

% \nived{remove}\begin{lemma}[Bounding probability mass outside $\calU$] \label{lemma:Ubound}
% With probability at least $1-\frac{\delta}{4}$, $\Pr (\underline{X} \not\in \calU \mid \calU ) \le \frac{1}{32}$.
% \end{lemma}
\noindent Having established this result, we next show how to prove \Cref{lemma:notUk}. As a consequence of the recursive definition of $\calK_j$ which is constructed from $\calU_j$ and $\calK_{j-1}$, and by \cref{eq:12311111},
\begin{align*}
    \Pr ( \calK_j \mid \calH_{j-1} ) &\ge (1-c) \cdot \Pr ( \calU_j \cap \calK_{j-1} \mid \calH_{j-1} ) \\
    &= (1-c) \cdot \Pr ( \calU_j \mid \calK_{j-1} , \calH_{j-1} ) \cdot \Pr (\calK_{j-1} \mid \calH_{j-1}) \\
    % &\ge (1-c) \left[ \Pr ( \{ \abort[j] \text{ and } R_j > \rjmax \}^c \mid \calH_{j-1}) - \Pr ( (\underline{X},\bm{\eta}) \not\in \calK_{j-1} \mid \calH_{j-1}) \right] \\
    &\overset{(a)}{\ge} (1-c) \cdot \left( 1 - \frac{1}{32k^2} \right) \cdot \Pr \big( \calK_{j-1} \mid \calH_{j-2} \big).
\end{align*}
where in $(a)$ we use the fact that $\calK_{j-1}$ is a measurable function of $(\underline{X},\bm{\eta})$ and $\calH_{j-2}$, and \Cref{lemma:UjHj}. Plugging in the value of $c=\frac{1}{32k}$, and solving the recursion over $k$ steps gives us: for all $0 \le j \le k-1$,
\begin{equation*}
    \Pr ( \calK_j \mid \calH_{j-1} ) \ge \frac{15}{16}.
\end{equation*}

\subsubsection{$\calU_{\le j}$ is a High Probability Event: Proof of \Cref{lemma:UjHj}} \label{sec:UjHj-proof}

By \Cref{lemma:abort=>lowreg},
\begin{align}
    \ptilde_j \le p_j^\star \implies &\Pr \big( R_j > \rjmax \mid \calK_{j-1}, \calH_{j-1} \big) \le \frac{1}{32k^2} \label{eq:abrheq}
\end{align}
Simplifying further, this implies for all $j$,
\begin{align}
    \Pr ( \calU_j^c \mid \calK_{j-1} , \calH_{j-1}) &= \Pr \big( \ptilde_j \le p_j^\star \text{ and } R_j > \rjmax \mid \calK_{j-1}, \calH_{j-1} \big) \nonumber\\
    &= \bbI ( \ptilde_j \le p_j^\star ) \cdot \Pr \big( R_j > \rjmax \mid \calK_{j-1}, \calH_{j-1} \big) \le \frac{1}{32k^2} \label{eq:901023010}
\end{align}
% This implies that
% \begin{align*}
%     \Pr (\calU_{\le k-1} \mid \calH_{k-2}) &= \Pr (\calU_{k-1} \mid \calU_{\le k-2}, \calH_{k-2}) \cdot \Pr ( \calU_{\le k-2} \mid \calH_{k-2}) \\
%     &= \Pr (\calU_{k-1} \mid \calH_{k-2}, \calU_{\le k-2}) \cdot \Pr ( \calU_{\le k-2} \mid \calH_{k-3})
% \end{align*}
% where in the last equation we use the fact that $\calU_{\le k-2}$ is a measurable function of $\calH_{k-3}$. Using \cref{eq:901023010} and multiplying out over $j=k-1,\cdots,0$ gives: for all $0 \le j \le k-1$,
% \begin{equation*}
%     \Pr (\calU_{\le j} \mid \calH_{j-1}) \ge \Big( 1-\frac{1}{32k^2} \Big)^k \ge 1 - \frac{1}{32k}
% \end{equation*}
This completes the proof of the lemma.
% We prove an upper bound on $\Pr ( (\underline{X},\bm{\eta}) \not\in\calK)$ by first upper bounding $\Pr ( \underline{X} \not\in \calU) \le \frac{1}{32}$, and connecting it to the former by using the relationship between $\calK$ and $\calU$ in \cref{eq:12311111}. Since we chose $c = \frac{1}{32k}$ in defining $\calK$ from $\calU$, this implies that $\Pr ( (\underline{X},\bm{\eta}) \not\in\calK) \le 2\Pr ( \underline{X} \not\in \calU)$.
% First by \Cref{lemma:truthful-abort-semi-1} and a union bound, we know that,
% \begin{align*}
%     \Pr \left( \calE \right) \le \frac{\delta}{4} \text{, where, } \calE = \left\{ \exists 0 \le j \le k-1 : \abort[j] \text{ and } p_j \ge \frac{1}{1024k^{3/2}} \right\}
% \end{align*}
% Throughout the remainder of the proof, we will assume that $\calE^c$ occurs.
% We upper bound $\Pr \left( \underline{X} \not\in \calU \right)$ by noticing that in any iteration $j$ where $\abort[j]$ occurs, under $\calE^c$, we know that $p_j \le \frac{1}{1024 k^{3/2}}$. In any such iteration, by \Cref{lemma:abort=>lowreg}, $\Pr ( R_j \ge \mathfrak{r}_j ) \le \frac{1}{32k}$. By taking a union bound over all $j$ corresponding to aborted iterations, we have that, $\Pr ( \forall j : \abort[j], \ R_j \ge \mathfrak{r}_j ) \le \frac{1}{32}$. Combining with the earlier discussion, we get that with probability at least $1-\frac{\delta}{4}$, $\Pr ( (\underline{X},\bm{\eta}) \not\in\calK \mid \calK) \le \frac{1}{16}$.

\subsection{Truthful Aborts: Proof of \Cref{lemma:truthful-abort-semi-1}} \label{subsec:truthful-abort-semi-proof}

By definition of $p_j$, and by the structure of the split-and-inverted-sampling pipeline (\Cref{alg:split,alg:semi-subsample}), the size of the dataset \smash{$|\Dout{j}|$} which the model \smash{$\fhat^j$} is trained on can be expressed as the sum of \smash{$n' = |\Dprompt{j}| = |\Dprompt{}|/k$} i.i.d. Bernoulli random variables, each with mean $p_j$. By an application of the multiplicative Chernoff bound,
\begin{equation*}
    \Pr \left( |\Dout{j}| < \frac{n' p_j}{2} \ \middle| \ \calH_{j-1} \right) \le \exp \left( -\frac{n' p_j}{8} \right)
\end{equation*}
In a non-aborted iteration, $p_j \ge \frac{15}{16} \cdot \ptilde_j > \frac{15}{16} \cdot p_j^\star$ by \cref{eq:pj-lb}.
By the sufficiently large choice of $|\Dprompt{}| = n' k$, we have that, $(a)$ $n' p_j \ge 8 \log(4k/\delta)$, and $(b)$ $n' p_j \gtrsim k^{7/2} \log(L) \cdot ( d\log(|S|) + \log(\delta/k))$. By its definition in \cref{eq:abort}, this implies,
\begin{equation*}
    \Pr \left( |\Dout{j}| < C_2 k^{7/2} \log(L) \cdot ( d\log(|S|) + \log(k/\delta)) \ \middle| \ \calH_{j-1} \right) \le \frac{\delta}{2k}
\end{equation*}
Plugging in $n'$ and $p_j^\star$ completes the proof.

\subsection[Training Weak Models: Proof of Lemma~\ref{lemma:weak-learner}]{Training Weak Models: Proof of \Cref{lemma:weak-learner}} \label{subsec:weak-learner-proof}

By \Cref{lemma:density-ratio-1}, in order to prove an upper bound on the prediction error under $\lambda_j$, it suffices to establish learning guarantees under $\lambdahat_j$ and use a change-of-measure argument. Noting that $\WLdepth{}.\warmup$ instantiates the base learner as $\TranscriptNTP$ (\Cref{alg:WLdepth0}), we next provide a learner guarantee for this next-token prediction based rule. As a corollary of \cite[Theorem B.5]{joshi2025theory}, we have the following result.

\begin{proposition} \label{theorem:lambdahatj-learn}
Given a dataset $D = \{ \bx^i \}_{i=1}^n$ of length $\tau$ instances drawn i.i.d. from $\lambdahat_j$, $\TranscriptNTP (\cdot)$ queries the $\etoe(\cdot)$ oracle $\tau n$ times in total to return a model $\fhat^{\NTP}$ such that, with probability $1-\frac{\delta}{2k}$ for any constant $C_0 > 0$,
\begin{equation*}
    \Pr_{\Xbar \sim \lambdahat_j} \big[ \fhat^{\NTP}_\tau (\Xbar) \ne \fstar_\tau (\Xbar) \big] \le \frac{1}{C_0 k^{7/2} \log(L)},
\end{equation*}
as long as for a sufficiently large absolute constant $C_1 > 0$ depending on $C_0$,
\begin{align*}
    n \ge C_1 k^{7/2} \log (L) \cdot (d \log(|S|) + \log (k/\delta)).
\end{align*}
\end{proposition}

\subsubsection{Proof of \Cref{lemma:weak-learner}}

Consider any non-aborted iteration $j$. Note that the model $\fhat^j$ is learned via $\TranscriptNTP$ on the dataset $\Dout{j}$. By the bound on the density ratio in \cref{eq:lambdalambdahat-dbound}, and a change-of-measure argument, with probability at least $1-\frac{\delta}{k}$,
\begin{equation*}
    \Pr_{\Xbar \sim \lambda_j} \big[ \fhat^j (\Xbar) \ne \fstar_\tau (\Xbar) \big] \le \left\| \frac{\lambda_j}{\lambdahat_j} \right\|_\infty \cdot \Pr_{\Xbar \sim \lambdahat_j} \big[ \fhat^j_\tau (\Xbar) \ne \fstar_\tau (\Xbar) \big] \le \frac{1}{4}.
\end{equation*}
The last inequality uses \Cref{theorem:lambdahatj-learn} with sufficiently large $C_0$, and \Cref{lemma:truthful-abort-semi-1}.

\subsection[Aborted Iterations Have Low Regression: Proof of Lemma~\ref{lemma:abort=>lowreg}]{Aborted Iterations Have Low Regression: Proof of \Cref{lemma:abort=>lowreg}} \label{lemma:abort=>lowreg-proof}

%By \Cref{lemma:truthful-abort-semi-1}, we know that if $p_j \ge \frac{1}{1024 k^{3/2}}$ in iteration $j$, then $\Pr (\abort[j] \mid \calH_{j-1} ) \le \frac{\delta}{2k} < \frac{1}{32k}$. Therefore, in order to prove the statement of \Cref{lemma:abort=>lowreg} it suffices to argue that when $p_j \le \frac{1}{1024 k^{3/2}}$, then $\Pr \Big( R_j \ge \frac{\alpha_{\max}^j}{16 \sqrt{k}} \Big) \le \frac{1}{32k}$.
Prior to proving \Cref{lemma:abort=>lowreg}, 
Next, we move onto establishing the moderate probability upper bound on $R_j$ via Markov's inequality. From its definition in \cref{eq:pjtilde},
\begin{align*}
    \ptilde_j &= \bbE \Bigg[ 1 - \prod_{i=0}^{L-1} \Bigg( 1 - \frac{\alpha^{j,k}_{\rank_j (X_i)}}{\alpha^{j,k}_{\max}} \Bigg) \ \Bigg| \ \calH_{j-1} , \calK_{j-1} \Bigg] \\
    &\ge \bbE \left[ 1 - \exp \left( - \frac{R_j}{\alpha^{j,k}_{\max}} \right) \ \middle| \ \calH_{j-1}, \calK_{j-1} \right] \\
    &\overset{(a)}{\ge} (1 - e^{-\theta}) \cdot \Pr \left( R_j \ge \theta \alpha^{j,k}_{\max} \ \middle| \ \calH_{j-1}, \calK_{j-1} \right) \\
    &\ge \frac{\theta}{2} \cdot \Pr \left( R_j \ge \theta \alpha^{j,k}_{\max} \ \middle| \ \calH_{j-1}, \calK_{j-1} \right)
\end{align*}
where $(a)$ is by an application of Markov's inequality for any $\theta > 0$. The proof concludes by choosing $\theta = \frac{1}{16\sqrt{k}}$ and using the upper bound on $\ptilde_j$ assumed in the statement of the lemma.

\subsubsection{Explicit formula for $p_j$}
Below we prove the explicit formula for $p_j$ we described in \cref{eq:pj-formula}. Let $\calI_j$ denote the (random) set of indices in $[L]$ accepted by $\subsample$, so that $\calI_j \ne \emptyset \iff \subsample (D_\bx \ \| \ \calF_j, k ) \ne \perp$. This implies that,
\begin{align}
    p_j &= \Pr \big( \subsample (D_\bx \ \| \ \calF_j, k ) \ne \perp \ \big| \ \calH_{j-1} \big) \label{eq:010291900}\\
    &= \Pr ( |\calI_j| \ge 1 \mid \calH_{j-1} ) \nonumber\\
    &= \bbE \left[ 1 - \Pr_{ ( \eta_i )_{i=0}^{L-1} \sim \unif ([0,1])^{\otimes L}} \left( \forall i \in \{0,\cdots,L-1\},\ \eta_i \ge \frac{w_j (X_i)}{\| w_j \|_\infty} \right) \ \middle| \ \calH_{j-1} \right] \nonumber\\
    &= \bbE \left[ 1 - \prod_{i=0}^{L-1} \left( 1 - \frac{w_j (X_i)}{\| w_j \|_\infty} \right) \ \middle| \ \calH_{j-1} \right]. \label{eq:010291901}
\end{align}

% \subsection{Non-aborted Iterations Have High Throughput: Proof of \Cref{lemma:noabort=>highthrough}} \label{sec:noabort=>highthrough-proof}

% If $p_j \ge \frac{1}{4096k^{3/2}}$, then, $\Pr \big[ R_j \le \mathfrak{r}_j \mid \calH_{k-1} \big] \le \frac{1}{32k}$.

% By definition of $p_j$,
% \begin{align*}
%     p_j &= \bbE \left[ \Pr_{ ( \eta_i )_{i=0}^{L-1} \sim \Unif ([0,1])^{\otimes L}} \left( \exists i \in \{0,\cdots,L-1\},\ \eta_i \le \frac{w_j (X_i)}{\| w_j \|_\infty} \right) \ \middle| \ \calH_{j-1} \right] \\
%     &= \bbE \left[ p_j (\underline{X}) \middle| \ \calH_{j-1} \right] \\
%     &\le \theta \Pr \left( \sum_{i=0}^{L-1} \frac{w_j (X_i)}{\| w_j \|_\infty} \le \theta \right) + \Pr \left( \sum_{i=0}^{L-1} \frac{w_j (X_i)}{\| w_j \|_\infty} > \theta \right) \\
%     &= \theta \Pr \left( R_j \le \theta \cdot \| w_j \|_\infty \right) + \left(1 - \Pr \left( R_j \le \theta \cdot \| w_j \|_\infty \right) \right)
% \end{align*}
% This implies that,
% \begin{align*}
%     \Pr \left( R_j \le \theta \cdot \| w_j \|_\infty \right) \le \frac{1-p_j}{1-\theta}
% \end{align*}
% Note that 
\newpage

\section{Proofs of Auxiliary Results}
\label{sec:proof-auxiliary}

\subsection{Lower Bounds Against Full-CoT Learners: Proof of \Cref{prop:lower-bound-CoT}} \label{proof:lower-bound-CoT}

Consider any family of models $\calG : S \to S$ with Natarajan dimension $d$. Construct the semiautomaton class $\calF$ on the state space $S$ and $\Sigma = \{ -1,+1 \}$. Each $\pi \in \calF$ corresponds to a $g \in \calG$ which we will indicate through the mapping $\pi_g$ and the transition of $\pi_g$ is defined as,
\begin{align*}
    \forall s \in S,\ &\pi_g ( (s,-1) ) = g(s) \\
    \forall s \in S,\ &\pi_g ( (s,+1) ) = s
\end{align*}
In other words, $\pi$ transitions according to $g(\cdot)$ when the input letter is $-1$, and behaves like the identity mapping when the input letter is $+1$. Suppose the distribution over instances, $\rho$, is constructed by taking the direct product of any distribution over states $\sigma^\star_0$, and the (deterministic) distribution over symbols $\delta_{-1} \times (\delta_{+1})^{T-1}$. Under this structure, note that almost surely for $s_0 \sim \sigma^\star_0$, $\pi_{g,T} (s_0) = g (s_0)$. Consider any PAC learning algorithm $\Alg (\cdot)$ which consumes a dataset of CoTs, $D = \{ ( (s_0,\bw) \mapsto \pi_{g^\star,1:T} ( (s_0,\bw)) ) \}$ for some unknown $\pi_{g^\star} \in \calF$ and $(s_0,\bw) \sim \rho$, to return a model $\fhat$ satisfying with probability at least $\frac{1}{2}$,
\begin{equation*}
    \Pr_{ (s_0, \bw) \sim \rho} \big( \fhat ( (s_0,\bw)) \ne \pi_{g^\star,T} ((s_0,\bw)) \big) \le \frac{1}{4}.
\end{equation*}
This can be rewritten as,
\begin{equation*}
    \Pr_{ s_0 \sim \sigma^\star_0} \big( \fhat ( (s_0,\{ -1,+1\cdots,+1\}) ) \ne g^\star (s_0) \big) \le \frac{1}{4}.
\end{equation*}
This implies that the induced model $\fhat \big( (\cdot, \{-1,+1,\cdots,+1\}) \big)$ is a PAC learner for $\calG$ under any distribution over states $\sigma^\star_0$. Furthermore, note that any fully labeled state sequence $(s_0,\bw) \mapsto \pi_{g^\star,1:T} ( (s_0,\bw) )$ can be constructed from the singular labeling $s_0 \mapsto g^\star (s_0)$ and vice versa.

\medskip
\noindent Now suppose $\sigma_0^\star$ is supported on a set of $\Ndim (\calG)$ points which are Natarajan-shattered by $\calG$. This implies that $\Alg$ must query at least $c \Ndim (\calG)$ instances to be able to achieve error $\frac{1}{4}$ with probability at least $\frac{1}{2}$ for some absolute constant $c > 0$. Finally, noting that $\Ndim (\calG) = \calO( \Ndim (\calF))$ (this follows immediately from the definition of Natarajan dimension), and the fact that $\nquery \ge T \cdot \nsample$ for any full CoT learner, we get the desired result.

\subsection[Bounding Prediction Error by Potential $\Phi_k$: Proof of Lemma~\ref{lemma:loss-decomp-Phi}]{Bounding Prediction Error by Potential $\Phi_k$: Proof of \Cref{lemma:loss-decomp-Phi}} \label{sec:error-pot-relation-proof}

Recall that the output of $\WLdepth{}.\warmup$ is the model $\fhat^{\circ L}$ where $\fhat = \maj \big( \big\{ \fhat^j_\tau : j \in \{ 0,\cdots,k-1\} \big\} \big)$. By definition of $\beta^{k,k}_r$, we have that,
\begin{align*}
    &\Pr_{\bx \sim \rho} \big[ \fhat^{\circ L} ( \bx ) \ne \fstar_T ( \bx ) \mid \calH_{k-1} \big] \\
    &\overset{(a)}{=} \Pr \big( \fhat^{\circ L} ( \texttt{concat} (\underline{X}) ) \ne \fstar_T ( \texttt{concat} (\underline{X})) \mid \calH_{k-1} \big)\\
    &\le \Pr \big( \exists i \in \{ 0,\cdots,L-1 \} : \fhat (X_i) \ne \fstar_\tau (X_i) \text{ and } (\underline{X},\bm{\eta}) \in \calK_{k-1} \mid \calH_{k-1} \big) + \Pr ((\underline{X},\bm{\eta}) \not\in \calK_{k-1} \mid \calH_{k-1} ) \\
    &\le \sum_{i=0}^{L-1} \Pr \big( \fhat (X_i) \ne \fstar_\tau (X_i) \text{ and } (\underline{X},\bm{\eta}) \in \calK_{k-1} \mid \calH_{k-1} \big) + \Pr ((\underline{X},\bm{\eta}) \not\in \calK_{k-1} \mid \calH_{k-1} ) \\
    &\overset{(b)}{\le} \sum_{i=0}^{L-1} \sum_{r=0}^{\lceil k/2 \rceil} \Pr \big( \rank_k (X_i) = r \text{ and } (\underline{X},\bm{\eta}) \in \calK_{k-1} \mid \calH_{k-1} \big) + \Pr ((\underline{X},\bm{\eta}) \not\in \calK_{k-1} \mid \calH_{k-1} ) \\
    &= \sum_{i=0}^{L-1} \bbE \big[ \beta^{k,k}_{\rank_k (X_i)} \cdot \bbI ( (\underline{X},\bm{\eta}) \in \calK_{k-1}) \mid \calH_{k-1} \big] + \Pr ((\underline{X},\bm{\eta}) \not\in \calK_{k-1} \mid \calH_{k-1} ) \\
    &= \Phi_k + \Pr \big( (\underline{X},\bm{\eta}) \not\in \calK_{k-1} \mid \calH_{k-1} \big)\\
\end{align*}
where in $(a)$, $\texttt{concat} (X)$ is defined as taking $X = (X_0,\cdots,X_{L-1}) \in (S \times \Sigma^\tau)^L$ where $X_i = (s^i, \bw^i)$ and mapping it to $(s^0,\bw^0,\cdots,\bw^{L-1})$, which is in $S \times \Sigma^T$, and $(b)$ uses the fact that if $\rank_k (X_i) > \lceil k/2 \rceil$, then the plurality vote satisfies $\fhat (X_i) = \fstar_\tau (X_i)$.

\subsection[Recursion for $\Phi_j$: Proof of Lemma~\ref{lemma:recursion}]{Recursion for $\Phi_j$: Proof of \Cref{lemma:recursion}} \label{sec:pot-recursion-proof}

Recall by definition,
\begin{align*}
    \Phi_{j+1} &= \sum_{i=0}^{L-1} \sum_{r=0}^{j+1} \beta^{j+1,k}_r \cdot \Pr \big( \rank_{j+1} (X_i) = r \text{ and } (\underline{X},\bm{\eta}) \in \calK_j \mid \calH_{k-1} \big)
\end{align*}
Note that $\rank_{j+1} (X_i) = r$ is only possible if $\rank_j (X_i) = r$ or $\rank_j (X_i) = r-1$. Likewise, under the condition $\rank_j (X_i) = r$, then $\rank_{j+1} (X_i) \ne r \implies \rank_{j+1} (X_i) = r+1$. With this, we decompose the above expression as,
\begin{align}
    \Phi_{j+1} &= \sum_{i=0}^{L-1} \sum_{r=0}^{j} \beta^{j+1,k}_r \cdot \Pr \big( (\underline{X},\bm{\eta}) \in \calK_j \text{ and } \rank_j (X_i) = r \mid \calH_{k-1} \big) \nonumber\\
    &\qquad - \sum_{i=0}^{L-1} \sum_{r=0}^{j} \beta^{j+1,k}_r \cdot \Pr \big( (\underline{X},\bm{\eta}) \in \calK_j \text{ and } \rank_{j+1} (X_i) = r+1 \text{ and } \rank_j (X_i) = r \mid \calH_{k-1} \big) \nonumber\\
    &\qquad + \sum_{i=0}^{L-1} \sum_{r=1}^{j+1} \beta^{j+1,k}_r \cdot \Pr\big( (\underline{X},\bm{\eta}) \in \calK_j \text{ and } \rank_{j+1} (X_i) = r \text{ and } \rank_j (X_i) = r-1 \mid \calH_{k-1} \big) \nonumber\\
    &= \sum_{i=0}^{L-1} \sum_{r=0}^{j} \beta^{j+1,k}_r \cdot \Pr \big( (\underline{X},\bm{\eta}) \in \calK_j \text{ and } \rank_j (X_i) = r \mid \calH_{k-1} \big) \nonumber\\
    &\qquad - \sum_{i=0}^{L-1} \sum_{r=0}^{j} \big( \underbrace{\beta^{j+1,k}_r - \beta^{j+1,k}_{r+1}}_{\alpha^{j,k}_r} \big) \cdot \Pr \big( (\underline{X},\bm{\eta}) \in \calK_j \text{ and } \underbrace{\rank_{j+1} (X_i) = r+1 \text{ and } \rank_j (X_i) = r}_{\equiv \big\{ \rank_j (X_i) \,=\, r \text{ and } \fhat^j_\tau (X_i) \,=\, \fstar_\tau (X_i) \big\}} \mid \calH_{k-1} \big)  \label{eq:1001010101}
\end{align}
Let $\sigma_{\calK_j}$ denote the distribution $\Pr (\underline{X} = \cdot\, |\, (\underline{X},\bm{\eta}) \in \calK_j, \calH_{j-1})$. Then, the last expression in \cref{eq:1001010101} can be further decomposed as,
\begin{align*}
    &\sum_{i=0}^{L-1} \sum_{r=0}^j \alpha^{j,k}_r \cdot \Pr \big( (\underline{X},\bm{\eta}) \in \calK_j \text{ and } \rank_j (X_i) = r \text{ and } \fhat^j_\tau (X_i) = \fstar_\tau (X_i) \mid \calH_{k-1} \big) \\
    &= \sum_{i=0}^{L-1} \bbE_{\underline{X} \sim \sigma_{\calK_j}} \big[ \alpha^{j,k}_{\rank_j (X_i)} \cdot \bbI \big( \fhat^j_\tau (X_i) = \fstar_\tau (X_i) \big) \mid \calH_{k-1} \big] \cdot \Pr \big( (\underline{X},\bm{\eta}) \in \calK_j \mid \calH_{k-1} \big)  \\
    &\overset{(a)}{=} L \cdot \bbE_{\Xbar \sim \overline{\sigma_{\calK_j}}} \big[ w_j (\Xbar) \cdot \bbI \big( \fhat^j_\tau (\Xbar) = \fstar_\tau (\Xbar) \big) \mid \calH_{j-1} \big] \cdot \Pr \big( (\underline{X},\bm{\eta}) \in \calK_j \mid \calH_{k-1} \big) \\
    &\overset{(b)}{=} L \cdot \bbE_{\Xbar \sim \overline{\sigma_{\calK_j}} } \big[ w_j (\Xbar) \big] \cdot \Pr_{\Xbar \sim \lambda_j} \big( \fhat^j_\tau (\Xbar) = \fstar_\tau (\Xbar) \big) \cdot \Pr \big( (\underline{X},\bm{\eta}) \in \calK_j \mid \calH_{k-1} \big) \\
    &\overset{(c)}{=} L \cdot \bbE_{\Xbar \sim \overline{\sigma_{\calK_j}}} \big[ \alpha^{j,k}_{\rank_j (\Xbar)} \big] \cdot (1-\err_j) \cdot \Pr \big( (\underline{X},\bm{\eta}) \in \calK_j \mid \calH_{k-1} \big) \\
    &= (1-\err_j) \cdot \sum_{i=0}^{L-1} \sum_{r=0}^j \bbE \big[ \alpha^{j,k}_r \cdot \bbI ( (\underline{X},\bm{\eta}) \in \calK_j \text{ and } \rank_j (X_i) = r) \mid \calH_{k-1} \big]
\end{align*}
where in $(a)$, recall that $w_j (\Xbar) = \alpha^{j,k}_{\rank_j(\Xbar)}$ which is measurable with respect to $\calH_{j-1}$. In $(b)$, $\lambda_j (\cdot)$ is the distribution proportional to $\overline{\sigma_{\calK_j}} (\cdot) w_j (\cdot)$. In equation $(c)$, we use the definition of $\err_j$ from \Cref{lemma:weak-learner}. We may combine this equation back with \cref{eq:1001010101}, noting that $\alpha^{j,k}_r = \beta^{j+1,k}_r - \beta^{j+1,k}_{r+1}$, $\err_\star = \frac{1}{4}$ and $\err_\star \cdot \beta^{j+1,k}_r + \left( 1 - \err_\star \right) \cdot \beta^{j+1,k}_{r+1} = \beta^{j,k}_r$,
\begin{align*}
    \Phi_{j+1} &= \sum_{i=0}^{L-1} \sum_{r=0}^{j} \Big( \err_\star \cdot \beta^{j+1,k}_r + \left( 1 - \err_\star \right) \cdot \beta^{j+1,k}_{r+1} \Big) \cdot \Pr \big( (\underline{X},\bm{\eta}) \in \calK_j \text{ and } \rank_j (X_i) = r \mid \calH_{k-1} \big) \\
    &\qquad + (\err_j - \err_\star) \sum_{i=0}^{L-1} \sum_{r=0}^j \alpha^{j,k}_r \cdot \Pr \big( (\underline{X},\bm{\eta}) \in \calK_j \text{ and } \rank_j (X_i) = r \mid \calH_{k-1} \big)\\
    &= \sum_{i=0}^{L-1} \sum_{r=0}^{j} \beta^{j,k}_r \cdot \Pr \big( (\underline{X},\bm{\eta}) \in \calK_j \text{ and } \rank_j (X_i) = r \mid \calH_{k-1} \big) + (\err_j - \err_\star) \sum_{i=0}^{L-1} \bbE \big[ \alpha^{j,k}_{\rank_j(X_i)} \cdot \bbI((\underline{X},\bm{\eta}) \in \calK_j) \mid \calH_{k-1} \big] \\
    &= \sum_{i=0}^{L-1} \sum_{r=0}^{j} \beta^{j,k}_r \cdot \Pr \big( (\underline{X},\bm{\eta}) \in \calK_j \text{ and } \rank_j (X_i) = r \mid \calH_{k-1} \big) + (\err_j - \err_\star) \sum_{i=0}^{L-1} \bbE \big[ \alpha^{j,k}_{\rank_j(X_i)} \cdot \bbI((\underline{X},\bm{\eta}) \in \calK_j ) \mid \calH_{k-1} \big] \\
    &\overset{(a)}{\le} \sum_{i=0}^{L-1} \sum_{r=0}^{j} \beta^{j,k}_r \cdot \Pr \big( (\underline{X},\bm{\eta}) \in \calK_{j-1} \text{ and } \rank_j (X_i) = r \mid \calH_{k-1} \big) + (\err_j - \err_\star) \sum_{i=0}^{L-1} \bbE \big[ \alpha^{j,k}_{\rank_j(X_i)} \cdot \bbI((\underline{X},\bm{\eta}) \in \calK_j ) \mid \calH_{k-1} \big] \\
    &= \Phi_j + (\err_j - \err_\star) \cdot \bbE \big[ R_j \cdot \bbI ((\underline{X},\bm{\eta}) \in \calK_j) \mid \calH_{k-1} \big].
\end{align*}
where $(a)$ uses the fact that $\calK_j \subseteq \calK_{j-1} \subseteq \cdots \subseteq \calK_0$.

\subsection[Density Ratio Bound between $\nu_{\calK}$ and $\nuhat$: Proof of Lemma~\ref{lemma:apxsample}]{Density Ratio Bound between $\nu_{\calK}$ and $\nuhat$: Proof of \Cref{lemma:apxsample}} \label{sec:apxsample-proof}

\noindent Prior to proving the main bounds on the densities of $\nuhat$ and $\nu_{\calK}$, we prove a lemma which will help understand the structure of the sampling from the mixture of marginal problem (\Cref{def:srmm}) for a single weight function $w$. Recall the notation: $\underline{X} = (X_i)_{i=0}^{L-1} \sim \mu$ is a distribution over $\bX^L$, and $w$ is a $[0,1]$-bounded weight function. We also let $m \triangleq \lceil \log_2(L+1) \rceil$.

\begin{lemma}[Sampling from $\nu$] \label{lemma:nu}
One sample from $\nu \propto \overline{\mu}(\cdot) w (\cdot)$ can be generated by the following procedure.

\begin{procedure}[label=proc:nu]{}
Keep drawing samples $\underline{X} \sim \mu$ until one is accepted by the following filter: Draw $L$ random variables $( \eta_i )_{i=0}^{L-1} \overset{\text{i.i.d.}}{\sim} \unif ([0,1])$ and define $\calI = \left\{ \eta_i \le \frac{w (X_i)}{\| w \|_\infty} : i \in \{ 0,\ldots,L-1 \} \right\}$; $\underline{X}$ is accepted with probability,
\begin{equation*}
    p_{\underline{X}} = \frac{\left| \calI \right|}{L}.
\end{equation*}
For the accepted $\underline{X}$, return $X_I \in \underline{X}$ for $I \sim \unif (\calI)$.
\end{procedure}
\end{lemma}
\begin{proof}
The probability that an $\underline{X} \sim \mu$ in \Cref{proc:nu} is accepted and $I=i$ is returned,
\begin{align*}
\Pr(\text{accept } \underline{X} \text{ and } I = i \mid \underline{X})
= \bbE \left[ \frac{|\calI|}{L} \cdot \frac{\mathbb{I} (i \in \calI)}{|\calI|} \right]
= \frac{1}{L} \, \Pr \left( \eta_i \le \frac{w(X_i)}{\| w \|_\infty} \right) = \frac{1}{L} \cdot \frac{w(X_i)}{\| w \|_\infty}.
\end{align*}
Therefore, for any measurable set $S \subseteq \bX$,
\begin{align*}
    \Pr(\text{output} \in S \text{ and accept } \underline{X} \mid \underline{X}) = \sum_{i=0}^{L-1} \frac{1}{L} \cdot \frac{w(X_i)}{\| w \|_\infty} \, \mathbb{I}( X_i \in S ).
\end{align*}
Taking expectation over $\underline{X} \sim \mu$ and using the definition of $\overline{\mu}$,
\begin{align*}
\Pr(\text{output} \in S \text{ and accept } \underline{X})
&= \sum_{i=0}^{L-1} \frac{1}{L} \cdot \bbE \left[ \frac{w(X_i)}{\| w \|_\infty} \, \mathbb{I}( X_i \in S ) \right] \\
&= \bbE \left[ \frac{w(\Xbar)}{\| w \|_\infty} \, \mathbb{I}( \Xbar \in S ) \right]
\end{align*}
With $S = \bX$, we see that the total acceptance probability is $\Pr(\text{accept}) = \bbE \left[ \frac{w(\Xbar)}{\| w \|_\infty} \right]$. Therefore,
\begin{align*}
\Pr(\text{output} \in S \mid \text{accept } \underline{X}) = \frac{\Pr(\text{output} \in S \text{ and accept } \underline{X})}{\Pr(\text{accept } \underline{X})} = \frac{\bbE \left[ w(\Xbar) \, \mathbb{I}( \Xbar \in S ) \right]}{\bbE \left[ w(\Xbar) \right]} = \nu(S).
\end{align*}
This completes the proof.
\end{proof}

\noindent In proving \Cref{lemma:apxsample}, we will first introduce some relevant notation and auxiliary results. For $\ell \in [m]$, define $\nu_\ell$ as the distribution over $\bX$ induced by the following procedure.

\smallskip
\begin{procedure}[label=proc:nuell]{}
\medskip
Keep drawing samples $\underline{X} \sim \mu$ until one is accepted by the following filter: Draw $L$ random variables $(\eta_i)_{i=0}^{L-1} \overset{\text{i.i.d.}}{\sim} \unif ([0,1])$ and define $\calI = \left\{ \eta_i \le \frac{w (X_i)}{\| w \|_\infty} : i \in \{ 0,\cdots,L-1\} \right\}$; accept $\underline{X}$ with probability,
\begin{equation*}
    p_{\underline{X}} = \frac{|\calI|}{L} \cdot \bbI (|\calI| \in [2^{\ell-1},2^\ell) \text{ and } \underline{X} \in \calU).
\end{equation*}
For the accepted $\underline{X}$, return $X_I$ for $I \sim \unif (\calI)$ if $\ell \ge 1$ and $I \sim \unif (\{ 0,\cdots,L-1\})$ otherwise.
\end{procedure}

\smallskip
\noindent Following the proof of \Cref{lemma:nu}, we arrive at the following lemma.

\begin{lemma}[Explicit formula for $\nu_\ell$] \label{lemma:nu_ell}
The distribution $\nu_\ell$ can be written as: for any $S \subseteq \bX$,
\begin{align*}
    \nu_\ell (S) = \frac{\sum_{i=0}^{L-1} \bbE \big[ \Pr \left(i \in \calI \text{ and } |\calI| \in [2^{\ell-1}, 2^\ell) \text{ and } \underline{X} \in \calU \ \middle| \ \underline{X} \right) \bbI ( X_i \in S) \big]}{\bbE \left[ |\calI| \cdot \bbI ( |\calI| \in [2^{\ell-1}, 2^\ell) \text{ and } \underline{X} \in \calU) \right]}.
\end{align*}
\end{lemma}
\begin{proof}
Following the proof of \Cref{lemma:nu},
\begin{align*}
    \Pr (\text{accept } \underline{X} \text{ and } I=i \mid \underline{X} ) &= \bbE \left[ \frac{|\calI|}{L} \cdot \frac{\mathbb{I} (i \in \calI)}{|\calI|} \cdot \mathbb{I} \big( |\calI| \in [2^{\ell-1}, 2^\ell) \text{ and } \underline{X} \in \calU \big) \ \middle| \ \underline{X} \right] \\
    &= \frac{1}{L} \Pr \left( i \in \calI \text{ and } |\calI| \in [2^{\ell-1}, 2^\ell) \text{ and } \underline{X} \in \calU \ \middle| \ \underline{X} \right)
\end{align*}
Therefore, for any measurable set $S \subseteq \bX$, taking expectation over $\underline{X} \sim \mu$,
\begin{align}
\Pr(\text{output} \in S \text{ and accept } \underline{X})
&= \sum_{i=0}^{L-1} \bbE \big[ \Pr (\text{accept } \underline{X} \text{ and } I=i \mid \underline{X} ) \cdot \bbI ( X_i \in S) \big]  \nonumber \\
&= \frac{1}{L} \sum_{i=0}^{L-1} \bbE \big[ \Pr \left(i \in \calI \text{ and } |\calI| \in [2^{\ell-1}, 2^\ell) \text{ and } \underline{X} \in \calU \ \middle| \ \underline{X} \right) \cdot \bbI ( X_i \in S) \big] \label{eq:9121012}
\end{align}
With $S = \bX$, we get, $\Pr(\text{accept } \underline{X}) = \frac{1}{L} \sum_{i=0}^{L-1} \Pr \left(i \in \calI \text{ and } |\calI| \in [2^{\ell-1}, 2^\ell) \right)$. Therefore,
\begin{align*}
\Pr(\text{output} \in S \mid \text{accept } \underline{X}) &= \frac{\Pr(\text{output} \in S \text{ and accept } \underline{X})}{\Pr(\text{accept } \underline{X})} \\
&= \frac{\sum_{i=0}^{L-1} \bbE \big[ \Pr \left(i \in \calI \text{ and } |\calI| \in [2^{\ell-1}, 2^\ell) \text{ and } \underline{X} \in \calU \ \middle| \ \underline{X} \right) \bbI ( X_i \in S) \big]}{\sum_{i=0}^{L-1} \Pr \left(i \in \calI \text{ and } |\calI| \in [2^{\ell-1}, 2^\ell) \text{ and } \underline{X} \in \calU \right)} \\
&= \frac{\sum_{i=0}^{L-1} \bbE \big[ \Pr \left(i \in \calI \text{ and } |\calI| \in [2^{\ell-1}, 2^\ell) \text{ and } \underline{X} \in \calU \ \middle| \ \underline{X} \right) \bbI ( X_i \in S) \big]}{\bbE \left[ |\calI| \cdot \bbI ( |\calI| \in [2^{\ell-1}, 2^\ell) \text{ and } \underline{X} \in \calU ) \right]}
\end{align*}
This completes the proof.
\end{proof}

\noindent Next, we will introduce some notation to help define the event $\calK$ in the statement of \Cref{lemma:apxsample}. Firstly, define,
\begin{equation} \label{eq:Lheavy}
    L_{\text{heavy}} = \left\{ \ell \in [m] : \Pr \big( |\calI| \in [2^{\ell-1},2^\ell) \ \big| \ |\calI| \ge 1,\ \underline{X} \in \calU \big) \ge \frac{c}{m} \right\}
\end{equation}
Note that $L_{\text{heavy}}$ is non-empty by the pigeonhole principle, and the fact that $c<1$. $L_{\text{heavy}}$ captures the set of dyadic intervals which witness a non-trivial mass of the random variable $|\calI|$, conditioned on $\underline{X} \in \calU$. Indeed, defining the event,
\begin{equation} \label{eq:calK}
    \calK = \big\{ |\calI| \in \cup_{\ell \in L_{\text{heavy}}} [2^{\ell-1},2^\ell) \text{ or } |\calI| = 0 \big\} \cap \big\{ \underline{X} \in \calU \big\}
\end{equation}
Note that $\{ (\underline{X},\bm{\eta}) \in \calK \} \implies \{ \underline{X} \in \calU \}$, and by a union bound, we have the inequality,
\begin{equation*}
    \Pr \big( (\underline{X},\bm{\eta}) \not\in \calK \mid \underline{X} \in \calU \big) \le c
\end{equation*}
With this definition of $\calK$, we will prove a pointwise upper bound on the density of the induced $\nu_\calK$.

\begin{lemma}[Sampling from $\nu_\calK$] \label{lemma:nu_calK}
The distribution $\nu_\calK$ satisfies: for any $S \subseteq \bX$,
\begin{align*}
    \nu_\calK (S) \le \calO \left( \frac{c^{-1} \log (L)}{\Pr (\underline{X} \in \calU \mid |\calI| \ge 1)} \right) \cdot \sum_{\ell=1}^{m} \Pr \left( |\calI| \in [2^{\ell-1}, 2^\ell) \text{ and } \underline{X} \in \calU \ \middle| \ |\calI| \ge 1 \right) \cdot \nu_\ell (S).
\end{align*}
\end{lemma}
\begin{proof}
Following the proof of \Cref{lemma:nu},
\begin{align*}
    \Pr (\text{accept } \underline{X} \text{ and } I=i \mid \underline{X} ) &= \bbE \left[ \frac{|\calI|}{L} \cdot \frac{\mathbb{I} (i \in \calI)}{|\calI|} \cdot \mathbb{I} \big( |\calI| = 0 \text{ or } \exists \ell \in L_{\text{heavy}} : |\calI| \in [2^{\ell-1}, 2^\ell) \big) \cdot \bbI \big( \underline{X} \in \calU \big) \ \middle| \ \underline{X} \right] \\
    &= \sum_{\ell \in L_{\text{heavy}}} \frac{1}{L} \Pr \left( i \in \calI \text{ and } |\calI| \in [2^{\ell-1}, 2^\ell) \text{ and } \underline{X} \in \calU \big) \ \middle| \ \underline{X} \right)
\end{align*}
Note that this is identical to the bound in \cref{eq:9121012}, with an additional summation over $\ell \in L_{\text{heavy}}$. Following through the same steps as in the proof of \Cref{lemma:nu_ell}, we arrive at the equation,
\begin{align*}
\nu_\calK (S) &= \Pr(\text{output} \in S \mid \text{accept } \underline{X}) \\
&= \frac{\sum_{\ell \in L_{\text{heavy}}} \sum_{i=0}^{L-1} \bbE \big[ \Pr \left(i \in \calI \text{ and } |\calI| \in [2^{\ell-1}, 2^\ell) \text{ and } \underline{X} \in \calU \ \middle| \ \underline{X} \right) \bbI ( X_i \in S) \big]}{\sum_{\ell \in L_{\text{heavy}}}\sum_{i=0}^{L-1} \Pr \left(i \in \calI \text{ and } |\calI| \in [2^{\ell-1}, 2^\ell) \text{ and } \underline{X} \in \calU \right)} \\
&\overset{(a)}{=} \frac{\sum_{\ell \in L_{\text{heavy}}} \bbE \left[ |\calI| \cdot \bbI ( |\calI| \in [2^{\ell-1}, 2^\ell) \text{ and } \underline{X} \in \calU ) \right] \cdot \nu_\ell (S)}{\sum_{\ell \in L_{\text{heavy}}} \bbE \left[ |\calI| \cdot \bbI ( |\calI| \in [2^{\ell-1}, 2^\ell) \text{ and } \underline{X} \in \calU ) \right]} \\
&\le \frac{\sum_{\ell \in L_{\text{heavy}}} 2^\ell \Pr \left( |\calI| \in [2^{\ell-1}, 2^\ell) \text{ and } \underline{X} \in \calU \ \middle| \ |\calI| \ge 1\right) \cdot \nu_\ell (S)}{\sum_{\ell \in L_{\text{heavy}}} 2^{\ell-1} \Pr \left( |\calI| \in [2^{\ell-1}, 2^\ell) \text{ and } \underline{X} \in \calU \ \middle| \ |\calI| \ge 1 \right)}.
\end{align*}
where $(a)$ relies on \Cref{lemma:nu_ell}. Defining $\ell_{\max} = \max \{ \ell \in L_{\text{heavy}} \}$, this can be bounded as,
\begin{align*}
    \nu_\calK (S) &\le \frac{\sum_{\ell \in L_{\text{heavy}}} 2^{\ell_{\max}} \Pr \left( |\calI| \in [2^{\ell-1}, 2^\ell) \text{ and } \underline{X} \in \calU \ \middle| \ |\calI| \ge 1 \right) \cdot \nu_\ell (S)}{2^{\ell_{\max}-1} \Pr \left( |\calI| \in [2^{\ell_{\max}-1}, 2^{\ell_{\max}}) \text{ and } \underline{X} \in \calU \ \middle| \ |\calI| \ge 1 \right)} \\
    &= 2 \frac{\sum_{\ell \in L_{\text{heavy}}} \Pr \left( |\calI| \in [2^{\ell-1}, 2^\ell) \mid \underline{X} \in \calU \text{ and } |\calI| \ge 1 \right) \cdot \nu_\ell (S)}{\Pr \left( |\calI| \in [2^{\ell_{\max}-1}, 2^{\ell_{\max}}) \mid \underline{X} \in \calU \text{ and } |\calI| \ge 1 \right)} \\
    &\overset{(a)}{\le} \frac{2 m}{c} \sum_{\ell \in L_{\text{heavy}}} \Pr \left( |\calI| \in [2^{\ell-1}, 2^\ell) \ \middle| \ \underline{X} \in \calU \text{ and } |\calI| \ge 1 \right) \cdot \nu_\ell (S) \\
    &= \frac{2 c^{-1} m}{\Pr ( \underline{X} \in \calU \mid |\calI| \ge 1)} \sum_{\ell \in L_{\text{heavy}}} \Pr \left( |\calI| \in [2^{\ell-1}, 2^\ell) \text{ and } \underline{X} \in \calU \mid |\calI| \ge 1 \right) \cdot \nu_\ell (S)
\end{align*}
where $(a)$ uses $\ell_{\max} \in L_{\text{heavy}}$ and the definition of this set. Since $L_{\text{heavy}} \subseteq [ m ]$, the proof concludes.
\end{proof}

\noindent Finally, we will prove a pointwise lower bound on the density $\nuhat$.

\begin{lemma}[Pointwise lower bound on $\nuhat$] \label{lemma:nuhat-lb}
The distribution $\nuhat$ satisfies: for any $S \subseteq \bX$,
\begin{align*}
    \nuhat (S) \ge \frac{1}{2} \sum_{\ell=1}^{m} \Pr \left( |\calI| \in [2^{\ell-1}, 2^\ell) \text{ and } \underline{X} \in \calU \ \middle| \ |\calI| \ge 1 \right) \cdot \nu_\ell (S)
\end{align*}    
\end{lemma}
\begin{proof}
Observe that for the distribution $\nuhat$, 
\begin{align*}
    \Pr (\text{accept } \underline{X} \text{ and } I=i \mid \underline{X} ) &= \bbE \left[ \frac{\mathbb{I} (i \in \calI)}{|\calI|} \cdot \mathbb{I} \big( |\calI| \ge 1 \big) \ \middle| \ \underline{X} \right]
\end{align*}
Following through the steps in the proof of \Cref{lemma:nu_ell}, we arrive at the equation,
\begin{align}
\nuhat (S) &= \Pr(\text{output} \in S \mid \text{accept } \underline{X}) \nonumber\\
&= \frac{\sum_{i=0}^{L-1} \bbE \left[ \bbE \left[ \frac{\bbI(i \in \calI \text{ and } |\calI| \ge 1)}{|\calI|} \ \middle| \ \underline{X} \right] \bbI ( X_i \in S) \right]}{\sum_{i=0}^{L-1} \bbE \left[ \frac{\bbI(i \in \calI \text{ and } |\calI| \ge 1)}{|\calI|} \right]} \label{eq:nuhatjformula}\\
&= \frac{\sum_{i=0}^{L-1} \sum_{\ell=1}^{m} \bbE \left[ \bbE \left[ \frac{\bbI (i \in \calI)}{|\calI|} \cdot \bbI(|\calI| \in [2^{\ell-1}, 2^\ell)) \ \middle| \ \underline{X} \right] \bbI ( X_i \in S) \right]}{\Pr (|\calI| \ge 1)} \nonumber\\
&\ge \frac{\sum_{\ell=1}^{m} \sum_{i=0}^{L-1} 2^{-\ell} \cdot \bbE \left[ \Pr \left( i \in \calI \text{ and } |\calI| \in [2^{\ell-1}, 2^\ell) \ \middle| \ \underline{X} \right) \bbI ( X_i \in S) \right]}{\Pr (|\calI| \ge 1)} \nonumber\\
&\ge \frac{\sum_{\ell=1}^{m} \sum_{i=0}^{L-1} 2^{-\ell} \cdot \bbE \left[ \Pr \left( i \in \calI \text{ and } |\calI| \in [2^{\ell-1}, 2^\ell) \text{ and } \underline{X} \in \calU \ \middle| \ \underline{X} \right) \bbI ( X_i \in S) \right]}{\Pr (|\calI| \ge 1)} \nonumber\\
&= \frac{\sum_{\ell=1}^{m} 2^{-\ell} \bbE \left[ |\calI| \cdot \bbI(|\calI| \in [2^{\ell-1}, 2^\ell) \text{ and } \underline{X} \in \calU) \right] \cdot \nu_\ell (S)}{\Pr (|\calI| \ge 1)} \nonumber\\
&\ge \frac{1}{2} \frac{\sum_{\ell=1}^{m} \Pr \left( |\calI| \in [2^{\ell-1}, 2^\ell) \text{ and } \underline{X} \in \calU \right) \cdot \nu_\ell (S)}{\Pr (|\calI| \ge 1)} \nonumber
\end{align}
\end{proof}

\subsection{Density Ratio Bound between $\nu_{\calK_j}$ and $\nuhat_j$: Proof of \Cref{lemma:density-ratio-1}} \label{sec:apxsample-k-proof}

Prior to discussing the proof of \Cref{lemma:density-ratio-1}, recall that the sequence of random events $(\calU_j)_{j=0}^{k-1}$ are such that each $\calU_j \subseteq \bX^L$ and is measurable with respect to $\calH_{j-1}$. We view $\bm{\eta} = (\bm{\eta}_j)_{j=0}^{k-1} \sim \Unif([0,1])^{\otimes kL}$ as a tuple of tuples of i.i.d. random variables; the $L$ coins in $\bm{\eta}_j$ are used to determine the accepted indices $\calI_j$ in \Cref{proc:nuhat} for $\nuhat_j$. We will define the heavy sets and the events $\calK_j$ in an iterative fashion. In particular,
\begin{equation} \label{eq:Lheavy-new}
    L_{\text{heavy}} [j] = \left\{ \ell \in [m] : \Pr \big( |\calI_j| \in [2^{\ell-1},2^\ell) \ \big| \ |\calI_j| \ge 1,\ (\underline{X},\bm{\eta}_{0:j-1}) \in \calU_j \cap \calK_{j-1}, \calH_{j-1} \big) \ge \frac{c}{m} \right\}
\end{equation}
where $\calK_{-1}^c = \emptyset$. And using this, we define the event: for $j \ge 0$,
\begin{equation} \label{eq:calK-new}
    \calK_j = \big\{ |\calI_j| \in \cup_{\ell \in L_{\text{heavy}}[j]} [2^{\ell-1},2^\ell) \text{ or } |\calI_j| = 0 \big\} \cap \big\{ (\underline{X},\bm{\eta}_{0:j-1}) \in \calU_j \cap \calK_{j-1} \big\}
\end{equation}
This statement precisely realizes the function $\calK_j \gets \construct (\calU_j \cap \calK_{j-1} \mid c, \sigma, \calH_{j-1})$ in the statement of \Cref{lemma:density-ratio-1}. Finally, similar to how the event $\calK$ was defined with respect to $\calU$ earlier in \cref{eq:calK}, here the event $\calK_j$ is defined with respect to the event $\calU_j \cap \calK_{j-1}$. Finally, note that the following formula can be derived for the distribution $\nuhat_j$ from \cref{eq:nuhatjformula},
\begin{equation*}
    \nuhat_j (S) = \frac{\sum_{i=0}^{L-1} \bbE \left[ \frac{\bbI(i \in \calI_j \text{ and } |\calI_j| \ge 1)}{|\calI_j|} \cdot \bbI (X_i \in S) \, \middle| \, \calH_{j-1} \right]}{\sum_{i=0}^{L-1} \bbE \left[ \frac{\bbI(i \in \calI_j \text{ and } |\calI_j| \ge 1)}{|\calI_j|} \, \middle| \, \calH_{j-1} \right]}
\end{equation*}
This definition can be used to derive a lower bound on $\nuhat_j$ in a similar manner as we did for $\nuhat$ earlier in \Cref{lemma:nuhat-lb}. We invoke this lemma with the event $\calU \gets \calU_j \cap \calK_{j-1}$ in an iterative fashion.

\begin{lemma} \label{lemma:nuhat-lb-new}
Consider any $S \subseteq \bX$. Then,
\begin{equation*}
    \nuhat_j (S) \ge \frac{1}{2} \sum_{\ell=1}^m \Pr \big( |\calI_j| \in [2^{\ell-1},2^\ell) \text{ and } (\underline{X},\bm{\eta}) \in \calU_j \cap \calK_{j-1} \mid |\calI_j| \ge 1, \calH_{j-1} \big) \cdot \nu_{j,\ell} (S)
\end{equation*}
Where,
\begin{equation*}
    \nu_{j,\ell} (S) = \frac{\sum_{i=0}^{L-1} \Pr \left(i \in \calI_j \text{ and } |\calI_j| \in [2^{\ell-1}, 2^\ell) \text{ and } (\underline{X},\bm{\eta}) \in \calU_j \cap \calK_{j-1} \text{ and } X_i \in S) \mid \calH_{j-1} \right)}{\bbE \left[ |\calI_j| \cdot \bbI ( |\calI_j| \in [2^{\ell-1}, 2^\ell) \text{ and } (\underline{X},\bm{\eta}) \in \calU_j \cap \calK_{j-1}) \mid \calH_{j-1} \right]}
\end{equation*}
\end{lemma}
\begin{proof}
The lemma is a direct consequence of the approach we use to prove \Cref{lemma:nuhat-lb}, only additionally conditioning on $\calH_{j-1}$ within the definitions of $\nuhat_j (S)$ and $\nu_{j,\ell} (S)$, which determines the (potentially random) weight function $w_j$, events $\calU_j$ and $\calK_{j-1}$.
\end{proof}

\noindent On the other hand, we can derive an upper bound on $\nu_{\calK_j}$ similar to how we did earlier for $\nu_\calK$ in \Cref{lemma:nu_calK}.

\begin{lemma} \label{lemma:nu_calK-new} Consider any $S \subseteq \bX$. Then,
\begin{equation*}
    \nu_{\calK_j} (S) \le \calO \left( \frac{c^{-1} \log (L)}{q_j} \right) \cdot \sum_{\ell=1}^{m} \Pr \left( |\calI_j| \in [2^{\ell-1}, 2^\ell) \text{ and } (\underline{X},\bm{\eta}) \in \calU_j \cap \calK_{j-1} \ \middle| \ |\calI_j| \ge 1, \calH_{j-1} \right) \cdot \nu_{j,\ell} (S).
\end{equation*}
where $q_j = \Pr \big( (\underline{X},\bm{\eta}) \in \calU_j \cap \calK_{j-1} \mid |\calI_j| \ge 1, \calH_{j-1} \big)$.
\end{lemma}
\begin{proof}
The lemma is a direct consequence of the approach we use to show \Cref{lemma:nuhat-lb}, with the additional conditioning on $\calH_{j-1}$ within the definitions of $\nuhat_j (S)$ and $\nu_{j,\ell} (S)$, as well as the definition of $q_j$, which depends on the event $\calU_j \cap \calK_{j-1}$, which is what we choose $\calU$ as in the $j^{\text{th}}$ iteration.
\end{proof}

\noindent Combining the bounds in \Cref{lemma:nuhat-lb-new,lemma:nu_calK-new}, we arrive at a bound on the density ratio between $\nu_{\calK_j}$ and $\nuhat_j$. In particular, we get the upper bound,
\begin{align}
    \left\| \frac{\nu_{\calK_j}}{\nuhat_j} \right\|_\infty \le \calO \left( \frac{c^{-1} \log (L)}{\Pr \big( (\underline{X},\bm{\eta}) \in \calU_j \cap \calK_{j-1} \mid |\calI_j| \ge 1, \calH_{j-1} \big)} \right). \label{eq:aa111}
\end{align}

\subsection[Sampling Lower Bound: Proof of Lemma~\ref{lemma:worst-case-sampling}]{Sampling Lower Bound: Proof of \Cref{lemma:worst-case-sampling}} \label{subsec:sampling-lb}

Let $\bm{1}$ denote the all $1$'s vector and $e_i$'s denote the standard basis vectors in $\bbR^L$. For parameters $p \in [0,1]$ and $z \in [L]$, and consider the base distribution,
\begin{align*}
    \mu_{p,z} ( z \bm{1} ) = p \text{ and } \forall i \in \{ 0,\cdots,L-1 \},\ \mu_{p,z} ( (L+1)e_i ) = \frac{1-p}{L}
\end{align*}
All other sequences have probability $0$ under $\mu_{p,z}$. Define $w(x) = \bbI ( x \in [L+1] )$. Note firstly that $p_{\mu_{p,z},w} = 1$, since almost surely, any sequence drawn from this distribution contains some $j \in [L+1]$. Next observe that, by Bayes rule, $\nu_{p,z} \propto \mubar_{p,z} (\cdot) w (\cdot)$ can be calculated to be the distribution,
\begin{align*}
    \nu_{p,z} (z) = 1 - \nu_p(L+1) = \frac{p}{p + \frac{1-p}{L}} = \frac{Lp}{(L-1)p + 1} \triangleq \theta_L (p)
\end{align*}
Thus, returning a sample from a distribution close to $\nu$ in TV distance, is equivalent to finding an unbiased estimator for the functional $\theta_L (p)$. Indeed, note that, consider a sampling algorithm $\Alg(\cdot)$ which draws a dataset $D$ of samples from $\mu$, and uses this dataset to generate a single sample distributed according to $\nuhat$. Then, overload notation to let $\nuhat(\cdot|D)$ denote the distribution over samples generated by $\Alg$ conditioned on $D$, we have that,
\begin{equation} \label{eq:TV}
    \TV{ \nuhat}{\nu} = \TV{ \bbE [ \nuhat (\cdot|D) ] }{\nu} \ge \left| \bbE[\nuhat(\cdot|D)] (z) - \theta_L (p) \right| 
\end{equation}
Next, observe that $\bN = (N_j)_{j \in [L]}$, where $N_j$ is frequency of sequences in $D$ containing a $j$, is a sufficient statistic for the dataset $D$. Indeed, given $\bN$, we can reconstruct a dataset $D'$ which is statistically indistinguishable from $D$, as follows. Pick a uniformly random partition of $[N]$ with $2L$ (potentially empty) parts, under the constraint that the number of elements in the $j^{\text{th}}$ part equals $N_j$ for each $j \in [L]$, and the number of elements in the remaining parts are uniform subject to containing $N-\sum_{j\in[L}N_j$ in total. With such a partition, define the dataset $D'$ as allocating the indices in $[N]$ belonging to the $j^{\text{th}}$ part as copies of the datapoint $j\bm{1}$ for $j \in [L]$, and of the remaining $L$ parts, allocating the $j^{\text{th}}$ one as copies of the datapoint $(L+1)e_j$. 

\medskip
\noindent With this observation, we argue that for each $j \in [L]$, $\bbE[ \nuhat (\cdot|D)] (j)$ for any estimator can be written down as a polynomial of degree $N$ in $(p_j)_{j\in[L]}$, for any $\Alg$ which processes a dataset of size $N$ drawn from $\mu$. This follows by Rao-blackwellization: for any sampling algorithm $\Alg$, we have that $\bbE [ \nuhat (\cdot|D) ] (j) = \bbE [ \bbE [ \nuhat (\cdot|D') | \bN ] ] (j) = \bbE [ \thetahat_j (\bN) ]$, where we let $\thetahat_j (\bN)$ denote $\bbE [ \nuhat (\cdot|D') | \bN ] (j)$. Finally, observe $\thetahat_j (\bN)$ can be written down as a polynomial:
\begin{equation} \label{eq:thetahat}
    \bbE [ \thetahat_j (\bN) ] = \sum_{\bN'} \thetahat_j ( \bN' ) \Pr (\bN = \bN')
\end{equation}
since for any fixed $\bN'$, $\Pr (\bN = \bN')$ itself is a polynomial of degree $N$ in $p$. Let $P_{N ,j} (p,z) = \bbE [ \thetahat_j (\bN) ]$ denote the degree-$N$ polynomial corresponding to the $j^{\text{th}}$ symbol, when the underlying base distribution is $\mu_{p,z}$. What remains to show is that degree $N$ polynomials cannot uniformly approximate the functional $\theta_L (p)$. In the subsequent \Cref{lemma:noapprox}, we show that uniform error of a constant is not achievable unless $N > c \sqrt{L}$ for some absolute constant $c > 0$.

\begin{lemma} \label{lemma:noapprox}
Fix $L \ge 2$ and define for $p \in [0,1]$, $\theta_L (p) = \frac{Lp}{(L-1)p + 1}$. Then there exist absolute constants $c,c' > 0$ such that for all polynomials $P_N$ of degree at most $N$ with $N \le c \sqrt{L}$, satisfying $\| P_N \|_\infty \le 1$,
\begin{equation*}
    \| P_{N,z} (z,\cdot) - \theta_L (\cdot) \|_\infty \ge c'.
\end{equation*}
\end{lemma}
\begin{proof}
This lemma is proved in \Cref{sec:noapprox-proof}.
\end{proof}
 
\begin{remark}
Note that the proofs of \Cref{lemma:noapprox} uses an elementary approach in polynomial approximation to argue an $\Omega (\sqrt{L})$ lower bound on the number of sampling oracle queries necessary to achieve constant TV distance. Using more involved techniques, via bounding the Ditzian-Totik moduli of smoothness \citep{ditzian2012moduli,kopotun2014new} of certain functionals involved in the analysis, we expect that to achieve constant TV error, this lower bound on the number of sampling oracle queries can be lifted to $\Omega (L/\omega(L))$ for any arbitrarily slow growing function $\omega (\cdot)$. Therefore in the worst case, in order to generate a sample from $\nu$, it is essentially impossible to beat the trivial sampling strategy of drawing $X \sim \mu$, sampling $I \sim \unif (\{ 0,\cdots,L-1 \})$ and returning $X_I$ if $I \in \calI$ (repeating until termination).
\end{remark}

\noindent The proof of the lower bound on the TV error in \Cref{lemma:worst-case-sampling} follows by combining the guarantee in \Cref{lemma:noapprox} with the bound on the TV error in \cref{eq:TV} and the polynomial structure of $\nuhat$ from \cref{eq:thetahat}. The lower bound on the density ratio in \Cref{lemma:worst-case-sampling} is shown next. Below, we show that there always exists an index $j^\star \in [L]$ such that $\nu (j^\star)$ must be polynomially small in $L$ as long as $N$ is not too large, when $\nu$ is induced by the base distribution $\mu_{p_0,j^\star}$ for $p_0 = \frac{1}{L}$.

\begin{lemma} \label{lemma:noapprox-strong}
There exists $c>0$ such that if $N \le c L^{1/3}$, then there exists a $j^\star \in [L]$ such that $P_{N,j^\star} (p_0,j^\star) \le L^{-1/3}$.
\end{lemma}
\begin{proof}
This lemma is proved in \Cref{sec:noapprox-strong-proof}.
\end{proof}

\noindent With this, we prove the lower bound on the density ratio in \Cref{lemma:worst-case-sampling}. For the index $j^\star \in [L]$ from \Cref{lemma:noapprox-strong}, let $\nu^\star = \nu_{p_0,j^\star}$ for $p_0 = \frac{1}{L}$. Then, note that $\nu^\star (j^\star) \ge \frac{1}{2}$ and furthermore, $\frac{\nu^\star (j^\star)}{\nuhat (j^\star)} = \frac{\nu^\star (j^\star)}{P_{N,j^\star} (p_0,j^\star)} \ge \frac{L^{1/4}}{2}$. This implies
\begin{align*}
    \max_{\mu \in \bm{\mu}} \Pr_{X \sim \nu} \left( \frac{\nu (X)}{\nuhat (X)} \ge \frac{L^{1/4}}{4} \right) =  \Pr_{X \sim \nu^\star} \left( \frac{\nu^\star (X)}{P_{N,X} (p_0,j^\star)} \ge \frac{L^{1/4}}{4} \right) \ge \nu^\star (j^\star) \ge \frac{1}{2}.
\end{align*}

\subsubsection{Proof of \Cref{lemma:noapprox}} \label{sec:noapprox-proof}
Let $P_N$ be any real polynomial of degree at most $N$. Our objective is to lower bound the uniform approximation error, $\inf_{P_N \in \calP_N} \| P_N - \theta_L \|_\infty$, where $\calP_N$ is the class of algebraic polynomials of degree at most $N$ which are bounded in the range $[0,1]$. Applying Markov brothers' inequality for the first derivative: there exists an absolute constant $C_{\mathrm{Markov}} > 0$ such that for every real polynomial $Q$ of degree at most $N$,
\begin{equation*}
    \| Q'(x) \|_\infty \le C_{\mathrm{Markov}} N^2 \| Q(x) \|_\infty.
\end{equation*}
Applying this to $P_N$ and using $\| P_N \|_\infty \le 1$ gives,
\begin{equation}
\label{eq:Markov-P}
\| P_N' \|_\infty \le C_{\mathrm{Markov}} N^2
\end{equation}
Next, we bound the curvature of the target function $\theta_L (p)$ around $p_0 = 1/L$. Let $p_0 \triangleq 1/L$ and define intervals $I_1 \triangleq \big[ p_0 - \frac{1}{10L},\ p_0 + \frac{1}{10L} \big]$ and $I_2 \triangleq \big[ p_0 - \frac{1}{20L},\ p_0 + \frac{1}{20L} \big]$. By explicit calculation, observe that, $\theta_L'(p) = \frac{L}{(1 + (L-1)p)^2}$. This implies that for all $p \in I_2$, there exists an absolute constant $c_1>0$ such that $|\theta_L'(p)| \ge c_1 L$. 

\medskip
\noindent Finally we integrate the magnitude of $\theta_L'' (\cdot)$ to argue that $\theta_L$ and $P_N$ cannot uniformly be close to each other. By the mean value theorem, for each fixed $x$ and step $\delta > 0$, there exists a $\xi \in [x,x+\delta]$ such that,
\begin{align*}
\Delta_\delta g(x) = g(x+\delta) - g(x) = g' (\xi) \delta
\end{align*}
for any function $g$ such that $g'$ is differentiable and continuous. With the choice of $\delta = \frac{1}{10L}$, observe that for any $x \in I_1$, the points $x$ and $x+\delta$ belong to $I_2$. Thus, applying the above presentation to $g = \theta_L$, we obtain for each $x \in I_1$,
\begin{align} \label{eq:theta-diff-lb}
    |\Delta_\delta \theta_L (x)| = |\theta_L' (\xi (x)) \delta| \ge c_1 L \delta = c_1 L \cdot \frac{1}{10L} \triangleq c_2 > 0,
\end{align}
Similarly, applying the mean value theorem to $g=P_N$ and using \cref{eq:Markov-P}, for each $x \in I_1$,
\begin{equation}
\label{eq:P-diff-ub}
|\Delta_\delta P_N (x)| \le \|P'_N \|_\infty \cdot \delta \le C_{\mathrm{Markov}} N^2 \cdot \frac{1}{10L} = \frac{C_{\mathrm{Markov}}}{10} \cdot \frac{N^2}{L} \le \frac{c_2}{2}.
\end{equation}
where the last inequality assumes that $N \le c \sqrt{L}$ for sufficiently small $c > 0$. \cref{eq:theta-diff-lb,eq:P-diff-ub} in conjunction establish a gap in the first-order differences between $P_N$ and $\theta_L$. Together these can be used to bound the uniform approximation error between $P_N$ and $\theta_L$. Indeed, observe that,
\begin{align*}
\frac{c_2}{2} &\le |\Delta_\delta P_N (x)- \Delta_\delta \theta_L (x)| \\
&= \bigl| (P_N (x+\delta)-\theta_L (x+\delta))
 - (P_N (x)-\theta_L (x)) \bigr| \\
&\le |P_N (x+\delta)-\theta_L(x+\delta)| + |P_N(x)-\theta_L (x)| \\
&\le 2 \| P_N - \theta_L \|_\infty,
\end{align*}
This implies that we can lower bound the uniform approximation error by $c' = \frac{c_2}{4}$.

\subsection{Proof of \Cref{lemma:noapprox-strong}} \label{sec:noapprox-strong-proof}

Recall that $P_{N,j} (p,z) = \bbE[\thetahat_j (\bN)]$ is the polynomial in $p$ capturing the probability that $\nuhat$ places on $j$ under the base distribution $\mu_{p,z}$. First, observe that when $p=0$, $\mu_{p,z}$ is the same across all values of $z \in [L]$, which implies that $P_{N,j} (p,z)$, which is a measurable function of a dataset drawn from $\mu_{p,z}$ must also satisfy the invariance of $P_{N,j} (0,z)$ across $z \in [L]$. This implies,
\begin{equation} \label{eq:min}
    \min_{j \in [L]} \max_{z \in [L]} p_{N,j} (0,z) \le \frac{1}{L} 
\end{equation}
Let $j^\star \in [L]$ denote the index of any minimizer of $\max_{z \in [L]} p_{N,j} (0,z)$. Next, we show that when $N \le c L^{1/4}$, with $p_0=\frac{1}{L}$, $\max \big\{ P_N(z,x) : x \in [0,p_0] \big\} \le \frac{c_3}{L}$ for some constant  $c_3>0$. Formally, by an application of Markov brothers' inequality,
\begin{equation*}
    \max_{z \in [L]} |P_{N,j^\star} (p_0,z)- P_{N,j^\star} (0,z)| \le p_0 \cdot \| P_{N,j^\star}' (z,\cdot) \|_\infty \le \frac{C_{\mathrm{Markov}}}{10} \cdot \frac{N^2}{L} \le \frac{1}{2L^{1/3}}
\end{equation*}
if $N \le c_3 L^{1/3}$ for some constant $c_3>0$. Combining with \cref{eq:min} gives the bound, $\max_{z \in [L]} P_{N,j^\star} (p_0,z) \le L^{-1/3}$. Replacing the maximum over $z \in [L]$ by $z \gets j^\star$ completes the proof.

\subsection{Improved $T$ Dependency when Online Learning is Possible: Proof of \Cref{theorem:online}} \label{sec:online}

For each instance $\bx = (s_0,\bw_{1:T})$ in the dataset, $\AlgOL$ predicts states $s_1,\cdots,s_T$. The binary search procedure described in \cref{alg:semiauto_online:start} to \ref{alg:semiauto_online:end} of \Cref{alg:semiauto_online} satisfies the property that if the while loop is broken out of, the learner must have found a time $t$ such that $s_t \ne s_t^\star$, but $s_{t-1} = s^\star_{t-1}$. If the while loop is never broken out of, then by \Cref{lemma:nobreak}, the prediction of $\AlgOL$ (which is a function of the prior instances and label observations) satisfies $s_T = s_T^\star$. On the other hand, each instance where the while loop is broken out of gives an $(s,w)$ tuple such that the prediction $s_+$ of $\AlgOL$ satisfies $\fstar (s,w) \ne s_+$ but where $s$ itself was predicted correctly. This is a ``true'' mistake, where the prediction was not incorrect because the prior state itself was predicted incorrectly. Since $\AlgOL$ makes at most $M^\star (\AlgOL)$ (true) mistakes, we have that,
\begin{align*}
    \sum_{m=1}^n \bbI ( y_T^m \ne \fstar_T (\bx^m) ) \le M^\star (\AlgOL).
\end{align*}
where $y_T^m \sim \fhat^m (\bx^m)$. Taking an expectation on both sides and noting the definition of \smash{$\fhat$} returned by \Cref{alg:semiauto_online}, by an online-to-batch argument we get,
\begin{equation*}
    \bbE_{\bx \sim \rho, y_T \sim \fhat(\cdot|\bx)} \big[ \bbI \big( y_T \ne \fstar_T (\bx) \big) \big] = \frac{1}{n} \sum_{m=1}^n \bbE \big[ \bbI( \fhat^m ( \bx^m ) \ne \pi^\star_T (\bx^m) \big] \le \frac{M^\star (\AlgOL)}{n}.
\end{equation*}
This shows that $n = \calO (M^\star (\AlgOL)/\varepsilon)$ samples suffices to get expected error $\varepsilon$. For each of the $n$ instances in the input dataset, $\etoe(\cdot)$ is queried at at most $2 \log_2 (T)$ timepoints. This implies the relation between the query complexity and sample complexity.

\begin{lemma} \label{lemma:nobreak}
If no $t \in [T]$ exists such that $s_t \ne s_t^\star$, but $s_{t-1} = s^\star_{t-1}$, then $s_T = s_T^\star$.
\end{lemma}
\begin{proof}
Start from $t=1$. Since $s_0 = s_0^\star$ by definition, the condition in the statement of the lemma implies that $s_1 = s_1^\star$. Repeating this argument inductively, this implies that for all $t \le T$, $s_t = s_t^\star$, which proves the claim.
\end{proof}

\begin{lemma} \label{lemma:Ldim-ub}
For a class of next-state/token predictors $\calF$ from $S \times \Sigma \to S$, $\Ldim (\calF) \le \Ndim (\calF) \log (|S| |\Sigma|)$.
\end{lemma}
\begin{proof}
By a counting argument $\Ldim(\calF) \le \log_2 |\calF|$, since a complete mistake tree of depth $m$ requires $2^m$ distinct hypotheses, one for each root-to-leaf path. On the other hand, viewing $\calF$ as a multiclass prediction class from $X = S \times \Sigma$ to $Y = S$, the multiclass Sauer--Shelah lemma for Natarajan dimension implies~\citep{natarajan1989learning,daniely2015multiclass}
\begin{equation*}
    |\calF| \le \sum_{i=0}^{d} \binom{|X|}{i} |Y|^{2i} \le \left( \frac{e |X| |Y|^2}{d} \right)^d.
\end{equation*}
Simplifying completes the proof.
\end{proof}

\begin{algorithm}[t]
\caption{$\WLdepthOL (\Dprompt{} \| T)$}\label{alg:semiauto_online}
\begin{algorithmic}[1]
\Statex {\color{gray} \# Learning semiautomata via curriculum: reduction to online learning}
\State \textbf{Input:} Class of semiautomaton transitions $\calF$ over state space $S$ and alphabet $\Sigma$,
\Statex \hspace{3em} Number of steps of semiautomaton simulation, $T$,
\Statex \hspace{3em} Dataset of instances $\Dprompt{} = \{ \bx^i = (s^i_0,\bw^i) \}_{i=1}^{n}$.
\Statex \hspace{3em} Online learning algorithm $\AlgOL$ for $\calF$ (\Cref{def:mistake-bound})
\State \textbf{Instantiate:} Let $\AlgOL$'s initial model be $\fhat^1 : S \times \Sigma \to S$ and let $m \gets 1$.
\For{instance $\bx^m = (s_0,\bw_{1:T}) \in \Dprompt{}$}
\For{$t = 1,\cdots,T$}
\State Let $s_{t} \gets \fhat^m (s_{t-1},w_t)$.
\EndFor
\State Instantiate $\Tmin \gets 1$, $\Tmax \gets T$.
\While{$\Tmax \ge \Tmin$} \LineLabel{alg:semiauto_online:start}
\State $t \gets \Tmin + \big\lfloor \frac{\Tmax - \Tmin}{2} \big\rfloor$
\State Query $s_t^\star \gets \fstar_t (\bx)$ and $s_{t-1}^\star \gets \fstar_{t-1} (\bx)$.
\If{$s_t \ne s_t^\star$ but $s_{t-1} = s_{t-1}^\star$}
\State Submit input $(s_{t-1},w_{t})$ to $\AlgOL$, observe its prediction ($s_t$) and  
\Statex \hspace{4.25em} reveal the correct label $s_t^\star$ \LineComment{{\color{blue} $\AlgOL$ updates itself on the mistake}}
\State \textbf{break}
\ElsIf{$s_t = s_t^\star$ and $s_{t-1} = s_{t-1}^\star$}
\State $\Tmin = t+1$
\ElsIf{$s_{t-1} \ne s_{t-1}^\star$}
\State $\Tmax = t-1$ \LineLabel{alg:semiauto_online:end}
\EndIf
\EndWhile
\State Denote the (potentially updated) model of $\AlgOL$ as $\fhat^{m+1} : S \times \Sigma \to S$
\State Update $m \gets m + 1$.
\EndFor
\State \textbf{Return:} $\Unif \big( \{ \fhat^m \}_{m = 1}^n \big)$.
\end{algorithmic}
\end{algorithm}

\subsection{Composition for Stochastic Models: Proof of \Cref{prop:chain-rule}}
Let $S^\star_t \sim \fstar_t (\cdot|\bx)$ for $\bx \sim \rho$. For $i = 0,\cdots,L-1$, let $\widehat{S}_{\tau (i+1)} \sim \fhat (\cdot|\bx)$ for $\bx = (S_{\tau i}^\star, \bw_{\tau i+1:\tau (i+1)}) \sim \rho_{\tau i +1: \tau (i+1)}$. Let $S^\star_0 = \widehat{S}_0 = S_0$ (distributed according to the initial state distribution within $\rho$). Then,
\begin{align*}
    &\bbE_{\bx \sim \rhobar} \big[ \KL{\fstar_\tau (\cdot|\bx)}{\fhat (\cdot|\bx)} \big] \\
    &= \frac{1}{L} \sum_{i=0}^{L-1} \bbE_{\bx \sim \rho_{\tau i +1: \tau (i+1)}} \big[ \KL{S^\star_{\tau (i+1)} \,\big|\, \bx}{\widehat{S}_{\tau (i+1)} \,\big|\, \bx} \big] \\
    &= \frac{1}{L} \sum_{i=0}^{L-1} \bbE_{\bx \sim \rho} \big[ \KL{S^\star_{\tau (i+1)} \,\big|\, \bx, S^\star_{0}, \cdots, S^\star_{\tau (i-1)},S^\star_{\tau i}}{\widehat{S}_{\tau (i+1)} \,\big|\, \bx, S^\star_0,\cdots,S^\star_{\tau (i-1)},S^\star_{\tau i}}  \big] \\
    &= \frac{1}{L} \bbE_{\bx \sim \rho} \big[ \KL{S^\star_0,S^\star_\tau,\cdots,S^\star_{T} \,\big|\, \bx}{\widehat{S}_0,\widehat{S}_{\tau},\cdots,\widehat{S}_T \,\big|\, \bx}  \big] \\
    &\ge \frac{1}{L} \bbE_{\bx \sim \rho} \big[ \KL{S^\star_{T} \,\big|\, \bx}{\widehat{S}_T \,\big|\, \bx}  \big]
\end{align*}
The proof concludes by noting that $\widehat{S}_T | \bx$ is distributed according to $\fhat^{\circ L} (\cdot|\bx)$.

\end{document}